%% file: thesis.tex
\documentclass[openright,twoside]{iitbthesis}

\input{chapter0/packages}

\input{chapter0/notations}

\begin{document}

	\input{chapter0/prelude}

	% %%% Start arabic page numbers %%%
	\pagebreak
	\pagenumbering{arabic}

	\chapter{Introduction}
	\label{cha:introduction}
	\input{chapter1/introduction}

	\chapter{Unsupervised Sequential Selection}
	\label{cha:uss}
	\input{chapter2/main}

	\chapter{Contextual Unsupervised Sequential Selection}
	\label{cha:contextual_uss}
	\input{chapter3/main}

	\chapter{Censored Semi-Bandits}
	\label{cha:csb}
	\input{chapter4/main}

	\chapter{Distributed Learning in Multi-Player Multi-Armed Bandits}
	\label{cha:mp-mab}
	\input{chapter5/main}

	\chapter{Summary and Future Directions}
	\label{cha:conclusion}
	\input{chapter6/conclusion}

	\bibliographystyle{plainnat} 			% Other options: unsrtnat, abbrvnat, plainnat
	\bibliography{ref}

\end{document}

%% file: chapter0/packages.tex
\usepackage{amsthm}
\usepackage{times}
\usepackage{array}
\usepackage{verbatim}
\usepackage{rotating}
\usepackage{multirow}
\usepackage{pifont} 			% \checkmark
\usepackage[utf8]{inputenc}	    % allow utf-8 input
\usepackage[T1]{fontenc}    	% use 8-bit T1 fonts
\usepackage{url}				% simple URL typesetting
\usepackage{booktabs}       	% professional-quality tables
\usepackage{amsfonts}       	% blackboard math symbols
\usepackage{nicefrac}       	% compact symbols for 1/2, etc.
\usepackage{microtype}      	% microtypography
\usepackage{dsfont}
\usepackage{amssymb,amsbsy,epsfig,float}
\usepackage{graphicx,wrapfig,lipsum}
\usepackage{algorithm,algpseudocode}
\usepackage[makeroom]{cancel}
\usepackage{xspace}
\usepackage{mathtools}
\usepackage{bbm}
\usepackage{color}
\usepackage{xcolor}
\usepackage{enumerate, cases}
\usepackage{thmtools,thm-restate}
\usepackage{caption,subcaption}
\usepackage{tabularx}
\usepackage{multicol,diagbox}
\usepackage[round]{natbib}
\usepackage[hypertexnames=false]{hyperref}			% hyperlinks
\usepackage[capitalize]{cleveref}
\usepackage{breakcites}
\usepackage[none]{hyphenat}
\usepackage{titlesec}
\usepackage[titletoc]{appendix}
\usepackage[nottoc]{tocbibind}
\usepackage{makecell}
\usepackage{pdfpages}

\newcolumntype{M}[1]{>{\centering\arraybackslash}m{#1}}
\newcolumntype{C}[1]{>{\centering}m{#1}}
\hypersetup{
	colorlinks   = true,		% Colours links instead of ugly boxes
	urlcolor     = blue, 		% Colour for external hyperlinks
	linkcolor    = purple, 		% Colour of internal links
	citecolor    = violet 		% Colour of citations
}
\hypersetup{pageanchor=false}

\allowdisplaybreaks

\usepackage{etoolbox}

%% file: chapter0/notations.tex
\newcommand{\istar}{{i^\star}}
\newcommand{\USS}{\cP_{\text{USS}}}
\newcommand{\bc}{\boldsymbol{c}}
\newcommand{\bd}{\boldsymbol{Q}}

\newcommand{\SD}{\mathrm{SD}}
\newcommand{\WD}{\mathrm{WD}}
\newcommand{\PSD}{\cP_{\SD}}
\newcommand{\PWD}{\cP_{\WD}}
\newcommand{\Bt}{\cB_t}

\newcommand{\Yti}{Y_t^i}
\newcommand{\Ytj}{Y_t^j}

\newcommand{\Yt}{Y_t}
\newcommand{\Yi}{Y^i}
\newcommand{\Yj}{Y^j}
\newcommand{\Yis}{Y^{i^\star}}

\newcommand{\Ht}{\mathcal{H}_t}

\newcommand{\pij}{p_{ij}}
\newcommand{\pijt}{p_{ij}^{(t)}}
\newcommand{\pjis}{p_{ji^\star}}
\newcommand{\pijs}{p_{i^\star j}}
\newcommand{\pijts}{p_{i^\star j}}
\newcommand{\opij}{\tilde{p}_{ij}^{(t)}}

\newcommand{\opji}{\tilde{p}_{ji}^{(t)}}

\newcommand{\opijtst}{\tilde{p}_{ij}^{(t)}}
\newcommand{\opijsts}{\tilde{p}_{i^\star j}^{(t)}}

\newcommand{\hpij}{\hat{p}_{ij}(t)}
\newcommand{\hpji}{\hat{p}_{ji}(t)}
\newcommand{\hpijs}{\hat{p}_{i^\star j}}
\newcommand{\hpijst}{\hat{p}_{i^\star j}^{(t)}}
\newcommand{\hpjis}{\hat{p}_{ji^\star}}

\newcommand{\hpijx}{\hat{p}_{i^\star j,s}}

\newcommand{\Dij}{\mathcal{D}_{ij}(t)}
\newcommand{\Nij}{\mathcal{N}_{ij}(t)}
\newcommand{\Dijm}{\mathcal{D}_{ij}(t-1)}
\newcommand{\Nijm}{\mathcal{N}_{ij}(t-1)}

\newcommand{\Nijms}{\mathcal{N}_{i^\star j}(t-1)}

\newcommand\numberthis{\addtocounter{equation}{1}\tag{\theequation}}

\newcommand{\btheta}{\boldsymbol{\theta}}

\newcommand{\PCWD}{\cP_{\text{CWD}}}
\newcommand{\SA}{\mathrm{SA}}

\newcommand{\CSD}{\mathrm{CSD}}
\newcommand{\CWD}{\mathrm{CWD}}
\newcommand{\TSA}{\Theta_{\SA}}

\newcommand{\TCSD}{\Theta_{\CSD}}

\newcommand{\bq}{\boldsymbol{Q}}
\newcommand{\ist}{{i^\star_t}}

\newcommand{\Xt}{X_t}
\newcommand{\xt}{x_t}

\newcommand{\Yts}{Y_t^\ist}

\newcommand{\F}{\cF}

\newcommand{\pist}{p_{i^\star_t I_t}^{(t)}}
\newcommand{\pijst}{p_{i^\star_t j}^{(t)}}

\newcommand{\pjist}{p_{ji^\star_t}^{(t)}}

\newcommand{\tpijt}{\tilde{p}_{ij}^{(t)}}
\newcommand{\tpjit}{\tilde{p}_{ji}^{(t)}}
\newcommand{\tplit}{\tilde{p}_{li}^{(t)}}

\newcommand{\tpiht}{\tilde{p}_{ih}^{(t)}}
\newcommand{\pihts}{p_{i^\star_t h}^{(t)}}

\newcommand{\Ts}{\theta_{ij}^\star}
\newcommand{\Te}{\hat\theta_{ij}^t}
\newcommand{\Tsl}{\theta_{li}^\star}
\newcommand{\Tel}{\hat\theta_{li}^t}

\newcommand{\Tsh}{\theta_{ih}^\star}
\newcommand{\Teh}{\hat\theta_{ih}^t}

\newcommand{\Vt}{V_{ij}^t}
\newcommand{\oVt}{\overline{V}_{ij}^t}
\newcommand{\oVijt}{\overline{V}_{ij}^t}

\newcommand{\Vinv}{(V_{ij}^t)^{-1}}
\newcommand{\Vtinv}{(V_{ij}^t)^{-1}}
\newcommand{\oVtinv}{(\overline{V}_{ij}^t)^{-1}}
\newcommand{\Vtl}{V_{li}^t}

\newcommand{\Vinvl}{(V_{li}^t)^{-1}}

\newcommand{\Vth}{V_{ih}^t}
\newcommand{\Vths}{V_{i^\star_t h}^t}

\newcommand{\Vinvh}{(V_{ih}^t)^{-1}}

\newcommand{\A}{\cA_{Q}}
\newcommand{\PCSB}{\mathcal{P}}

\newcommand{\PMP}{\cP_{\text{MP}}}
\newcommand{\PCoSB}{\mathcal{P}_{\text{CoSB}}}
\newcommand{\bmu}{\boldsymbol{\mu}}
\newcommand{\ba}{\boldsymbol{a}}

\newcommand{\C}{\mathcal{C}_T}

\newcommand{\Regret}{\mathfrak{R}}

\newcommand{\EE}[1]{\bE\left[#1\right]}

\newcommand{\Prob}[1]{\bP\left\{#1\right\}}

\newcommand{\R}{\bR}
\newcommand{\N}{\bN}

\newcommand{\one}[1]{\mathds{1}_{\left\{#1\right\}}}

\renewcommand{\phi}{\varphi}
\renewcommand{\epsilon}{\varepsilon}
\newcommand{\norm}[1]{\left\|#1\right\|}

\newcommand{\al}[1]{ \begin{align} #1  \end{align}}
\newcommand{\eq}[1]{ \begin{equation} #1  \end{equation}}
\newcommand{\als}[1]{ \begin{align*} #1  \end{align*}}
\newcommand{\eqs}[1]{ \begin{equation*} #1  \end{equation*}}

\newcommand{\el}{\end{flushleft}}
\newcommand{\bl}{\begin{flushleft}}
\newcommand{\argmin}{\arg\!\min}
\newcommand{\argmax}{\arg\!\max}
\newcommand{\floor}[1]{  \left\lfloor #1 \right\rfloor}
\newcommand{\ceil}[1]{  \left\lceil #1 \right\rceil}

\newcommand{\bE}{\mathbb{E}}

\newcommand{\bN}{\mathbb{N}}

\newcommand{\bP}{\mathbb{P}}

\newcommand{\bR}{\mathbb{R}}

\newcommand{\cA}{\mathcal{A}}
\newcommand{\cB}{\mathcal{B}}

\newcommand{\cF}{\mathcal{F}}

\newcommand{\cM}{\mathcal{M}}

\newcommand{\cP}{\mathcal{P}}

\newcommand{\cX}{\mathcal{X}}

\theoremstyle{plain}
\newtheorem{thm}{Theorem}
\newtheorem{lem}{Lemma}
\newtheorem{prop}{Proposition}
\newtheorem{cor}{Corollary}

\newtheorem{rem}{Remark}
\newtheorem{defi}{Definition}
\newtheorem{assu}{Assumption}
\newtheorem{fact}{Fact}

\newtheorem{theorem}{Theorem}

\newtheorem{corollary}{Corollary}
\newtheorem{definition}{Definition}

%% file: chapter0/prelude.tex
%!TEX root =  ../thesis.tex

% prelude.tex file contains:
%   - titlepage
%   - dedication (optional)
%   - approval sheet
%   - course certificate
%   - table of contents, list of tables and list of figures
%   - nomenclature
%   - abstract
%============================================================================

%\clearpage\pagenumbering{roman}  % This makes the page numbers Roman (i, ii, etc)
\pagenumbering{gobble}

% %%% TITLE PAGE %%%
\title{Sequential Decision Problems with Weak Feedback}
\what{Thesis}
\author{Arun Verma}
\date{April 2021}

\rollnum{154190002} 
\iitbdegree{Doctor of Philosophy}
\thesis
\department{Industrial Engineering and Operations Research}
\setguide{ Prof. Manjesh K. Hanawal
	\\ \vspace{0pt plus 0.1fil}
	{\it and} \\
	\vspace{0pt plus 0.1fil}	
	Prof. N. Hemachandra
}
\maketitle
%\newpage\null\thispagestyle{empty}\newpage % % For blank pages

% %%% DEDICATION %%%
\begin{dedication}
	\it{\Large Dedicated to my beloved parents.}
\end{dedication}

\newpage\null\thispagestyle{empty}\newpage % % For blank pages
\setboolean{@twoside}{false}
\includepdf[pages=-, offset=75 -75]{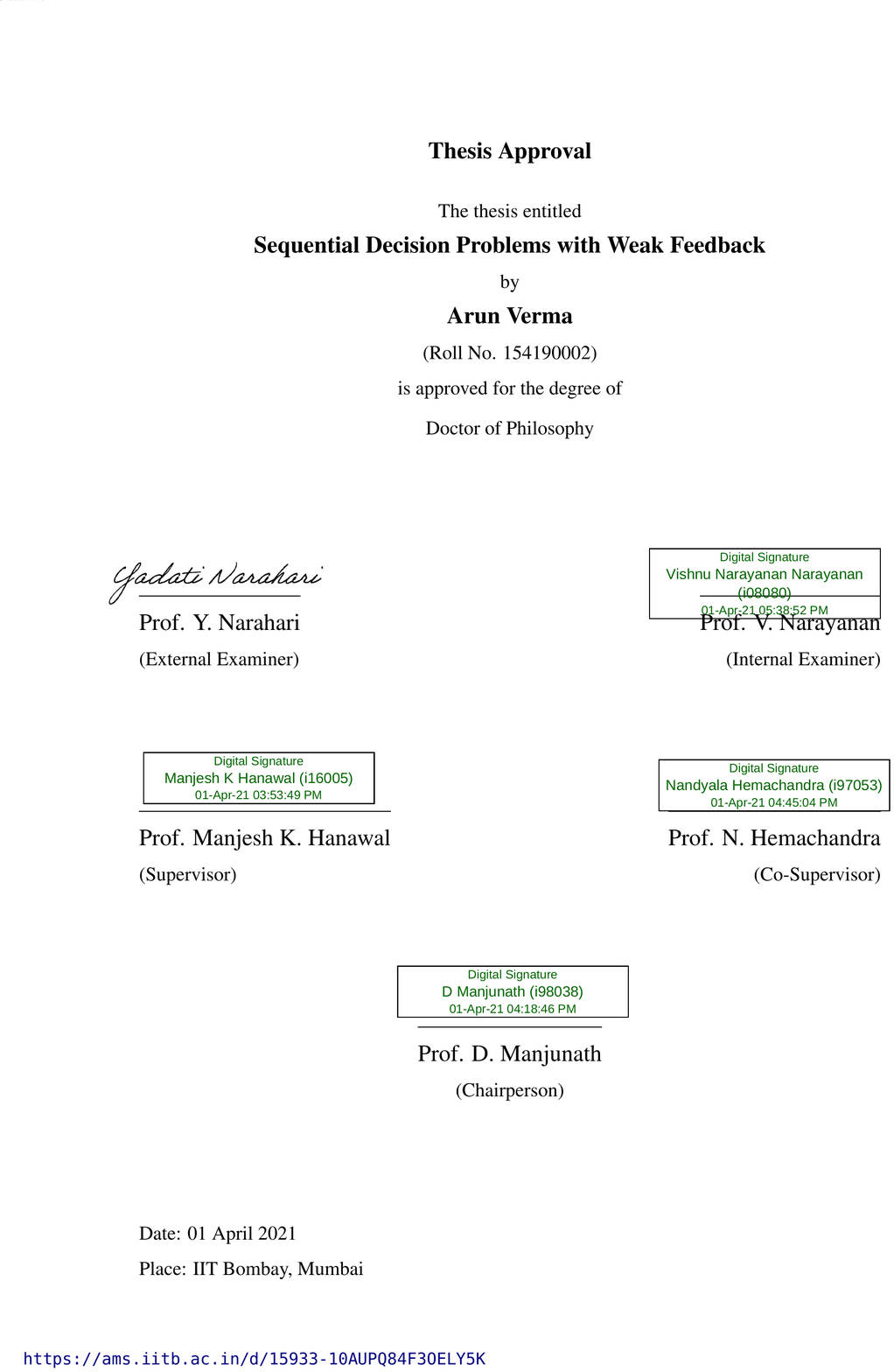}

\iffalse
 %%% APPROVAL SHEET %%%
 %   - for final thesis, you need Approval Sheet. So, uncomment the
%     \makeapproval command.
\makeapproval
\input{chapter0/approval.tex}  % Separate file
%\newpage\null\thispagestyle{empty}\newpage % % For blank pages
\fi

% %%% Declaration %%%
\newpage\null\thispagestyle{empty}\newpage % % For blank pages
\input{chapter0/declaration}
\newpage\null\thispagestyle{empty}\newpage % % For blank pages

% %%% CERTIFICATE OF COURSE WORK %%%
%\joiningdate{December 30, 2015}
%\begin{coursecertificate}
%	\addcourse{IE 613}{Online Machine Learning}{6}
%	\addppcourse{HS 699}{Communication and Presentation Skills}{PP}
%\end{coursecertificate}
%\newpage\null\thispagestyle{empty}\newpage % % For blank pages

%for quotations
%\quotationpage

% %%% COPYRIGHT PAGE %%%
%\copyrightpage
%\newpage\null\thispagestyle{empty}\newpage % % For blank pages

% %%% Acknowledgment %%%
%\input{chapter0/acknowledgment}
%\newpage\null\thispagestyle{empty}\newpage

\chapter*{Acknowledgment}
\label{cha:acknowledgment}
\input{chapter0/acknowledgment}

\newpage\null\thispagestyle{empty}\newpage  % For blank pages

% %%% ABSTRACT %%%
\pagenumbering{roman}
\begin{abstract}
  \input{chapter0/abstract}
\end{abstract}

\chapter*{Publications}
\label{cha:publications}
\input{chapter0/publications}

% %%% CONTENTS, TABLES, FIGURES %%%
\tableofcontents
\listoftables
\listoffigures
%\newpage\null\thispagestyle{empty}\newpage % % For blank pages

% Here is the file for Abbrivations
%\input{abbreviations}

% To automate abbriavtions using Nomencluture  package.
% Comment the  \input{abbreviations}
% Then include NOMEN...... package 
% Refer for this q31_nom_tex 

%\addcontentsline{toc}{chapter}{Abbreviations and Nomenclature}
%--------------------------------------------------------------------%
% %%% NOMENCLATURE %%%
%\begin{nomenclature}
%\begin{description}
%\item{\makebox[0.75in][l]{$C_1$}} Constant 1
%
%\item{\makebox[0.75in][l]{$V$}}    Voltage 
%
%\item{\makebox[0.75in][l]{\$}}     US Dollars
%\end{description}
%\end{nomenclature}
%\cleardoublepage\pagenumbering{arabic} % Make the page numbers Arabic (1, 2, etc)

% %%% ABBREVIATIONS %%%
%\input{abbreviations}

% %%% SYMBOLS %%%
%\input{symbols}

%% file: chapter0/declaration.tex
%!TEX root =  ../thesis.tex

\clearpage
\thispagestyle{empty}
\begin{center}
	{\Large \bf Declaration}
\end{center}
\vspace{0.5in}

I declare that this written submission represents my ideas in my own words, and where others' ideas or words have been included, I have adequately cited and referenced the original sources. I also declare that I have adhered to all principles of academic honesty and integrity and have not misrepresented or fabricated, or falsified any idea/data/fact/source in my submission. I understand that any violation of the above will be cause for disciplinary action by the Institute and can also evoke penal action from the sources which have thus not been properly cited or from whom proper permission has not been taken when needed.

\vspace{1in}

~\hfill Arun Verma \\
Date: April 27, 2021	\hfill 154190002 

\vfill
%\begin{table}[h]
%\begin{flushleft}
%	
%	\vspace{-3.2in} 
%	\begin{tabular}{ccccc}
%		% %\hline 	\rule[5ex]{0pt}{-10ex} &&(Signature) (Date) && \\ 
%		\rule[5ex]{0pt}{-10ex}&& Date: November 15, 2020&& \\ 
%	\end{tabular}
%\end{flushleft}
%
%\vspace{-0.5in} 
%\begin{flushright}
%	\begin{tabular}{ccccc}
%%		 \hline 	\rule[5ex]{0pt}{-10ex} &&(Signature) (Date) && \\ 
%		
%%		\hline 	
%		\rule[5ex]{0pt}{-10ex}&& Arun Verma&& \\ 
%		\rule[5ex]{0pt}{-10ex}&& 154190002 && \\ \\
%	\end{tabular}
%\end{flushright}
%\end{table}
\pagebreak

% \noindent 
%\vspace{0.2in}
%\begin{flushright}
%Signature~~~~~~~~~~~~~~~~~~~~~~~~ \\[5mm]
%--------------------------------------\\[5mm]
%%--------------------------------------\\[5mm
%Name of the Student~~~~~~~~~~\\[5mm]
%--------------------------------------\\[7mm]
%Roll No.~~~~~~~~~~~~~~~~~~~~~~~~~~ \\[5mm]
%--------------------------------------\\[7mm]
%\end{flushright}
%Date :---------------------------\\
\newpage

%% file: chapter0/acknowledgment.tex
%!TEX root =  ../thesis.tex

\vspace{-5.45mm}

A lot of people have contributed in getting this thesis to its present form. First and foremost, I would like to express my gratitude to my supervisor Prof. Manjesh K. Hanawal and co-supervisor Prof. N. Hemachandra, for their continuous guidance and support throughout my thesis work.

I would also like to thank my Research Program Committee (RPC) members, Prof. Shivaram Kalyanakrishnan and Prof. Vishnu Narayanan, for their critical assessment and constructive suggestions during the Annual Progress Seminars (APS) that helped me to significantly improve the quality of this thesis. I also thank the faculty members at the IEOR department for their direct and indirect support during my stay at IIT Bombay. I am grateful for the continuous and prompt help provided by the staff at the IEOR office -- Abasaheb Molavane, Siddhartha Salve, Pramod Pawar, and Amlesh Kumar.

I would like to acknowledge Csaba Sepesv\'ari, Odalric-Ambrym Maillard, Arun Rajkumar, Raman Sankaran, Venkatesh Saligrama, and Rahul Vaze for collaborating on various projects that I worked on during my Ph.D. study. I want to extend my special thanks to my labmates Sandhya Tripathi, Puja Sahu, Indrajit Saha, and Prashant Trivedi, who helped me in proofreading my research papers, cross-checking some of my proofs, and understanding many complex topics.

I would like to show my appreciation to Google, Microsoft, LRN Foundation, COMSNETS Association, NeurIPS, Asian Universities Alliance,  ACM-India/IARCS, LinkedIn, and CoDS-COMAD for their generous travel grants, which allowed me to attend many conferences like NeurIPS, AIStats, and INFOCOM. I want to thank Conduent India for providing me the opportunity to do a three months internship with them. I am also thankful to Indo French Centre (IFCPAR/CEFIPRA) for awarding me Raman–Charpak Fellowship. It enabled my six months visit to INRIA’s Scool (previously SequeL) team which is a group of eminent researchers actively working in my research area.

Finally, I would like to express my deepest gratitude to my beloved parents for their unconditional love and support.

\hfill {\bf Arun Verma}

%% file: chapter0/abstract.tex
%!TEX root =  ../thesis.tex

Many variants of sequential decision problems that are considered in the literature depend upon the type of feedback and the amount of information they reveal about the associated rewards. Most of the prior work studied the cases where feedback from actions reveals rewards associated with the actions. However, in many areas like crowd-sourcing, medical diagnosis, and adaptive resource allocation,  feedback from actions may be weak, i.e., may not reveal any information about rewards at all. Without any information about rewards, it is not possible to learn which action is optimal. Clearly, learning an optimal action is only feasible in such cases if the problem structure is such that an optimal action can be identified without explicitly knowing the rewards. Our goal in this thesis is to study the class of problems where optimal action can be inferred without explicitly knowing the rewards. Specifically, we study Unsupervised Sequential Selection (USS), where rewards/losses for selected actions are never revealed, but the problem structure is amenable to identify the optimal actions. We also introduce a novel setup named Censored Semi-Bandits (CSB), where the reward observed from an action depends on the amount of resources allocated to it.

The major part of this thesis focuses on the USS problem. In the USS problem, the loss associated with action cannot be inferred from observed feedback. Such cases arise in many real-life applications. For example, in a medical diagnosis, the patients' true state may not be known; hence, the test's effectiveness cannot be known. In crowd-sourcing systems, crowd-sourced workers' expertise level is unknown; therefore, the quality of their work cannot be known. In such problems, the prediction from test/worker is observed, but one cannot ascertain their reliability due to the absence of ground truth. We show that by comparing the feedback obtained from different actions, one can find the optimal action for a class of USS problems when the `weak dominance' property is satisfied by the problem. For this problem, we develop upper confidence bound and Thompson Sampling based algorithms with optimal performance guarantees.

In this thesis, we introduce a new setup called the censored semi-bandits (CSB), where feedback observed from an action depends on the amount of resource allocated. The feedback gets `censored' if enough resource is not allocated. In the CSB setup, a learner allocates resources among different activities (actions) in every round and receives censored loss from each action as feedback. The goal is to learn a policy for resource allocation that minimizes the cumulative loss. The loss at each time step depends on two unknown parameters, one specific to the action but independent of the allocated resources, and the other depends on the amount of resource allocated. More specifically, the loss equals zero if the action's resource allocation exceeds a constant (but unknown) threshold that can be dependent on the action. The CSB setup finds applications in many resource allocation problems like police patrolling, traffic regulations and enforcement, poaching control, advertisement budget allocation, stochastic network utility maximization, among many others.

The last part of this thesis focuses on distributed learning in multi-player multi-armed bandits for identifying the best subset of actions. The setup is such that the reward is only observed for those actions which are played by only one player. These problems find applications in wireless ad hoc networks and cognitive radios to find the best communication channels.

Our contribution in this thesis is to address the above described sequential decision problems by exploiting specific structures these problems exhibit. For each of these setups with weak feedback, we develop provably optimal algorithms. Finally, we validate their empirical performance on different problem instances derived from synthetic and real datasets.

~\\ \noindent
\textbf{Keywords:} Stochastic Multi-Armed Bandits, Unsupervised Online Learning, Censored Feedback, Upper Confidence Bound, Thompson Sampling, Contextual Bandits, Generalized Linear Models, Multi-Play Multi-Armed Bandits, Combinatorial Semi-Bandits, Stochastic Network Utility Maximization, Pure Exploration, Multi-Player Multi-Armed Bandits

\vfill

%% file: chapter0/publications.tex
%!TEX root =  ../thesis.tex

\noindent
\textbf{Peer-reviewed Publications} 
\begin{enumerate}
	\item {Arun Verma}, Manjesh K Hanawal, Csaba Szepesv\'ari, and Venkatesh Saligrama.  Online Algorithm for Unsupervised Sequential Selection with Contextual Information. In Advances in Neural Information Processing Systems (pp. 778–788). 2020.
	
	\item {Arun Verma}, Manjesh K Hanawal, and N. Hemachandra. Thompson sampling for unsupervised sequential selection. In Asian Conference on Machine Learning (pp. 545-560). 2020.
	
	\item {Arun Verma} and Manjesh K Hanawal. Stochastic Network Utility Maximization with Unknown Utilities: Multi-Armed Bandits Approach. In IEEE INFOCOM 2020-IEEE Conference on Computer Communications (pp. 189-198). 2020.
	
	\item {Arun Verma}, Manjesh K Hanawal, and N. Hemachandra. Unsupervised Online Feature Selection for Cost-Sensitive Medical Diagnosis. In 2020 International Conference on COMmunication Systems \& NETworkS (COMSNETS) (pp. 1-6). 2020.
	
	\item {Arun Verma}, Manjesh K Hanawal, Arun Rajkumar, and Raman Sankaran. Censored Semi-Bandits: A Framework for Resource Allocation with Censored Feedback. In Advances in Neural Information Processing Systems (pp. 14526-14536). 2019.
	
	\item {Arun Verma}, Manjesh K Hanawal, and Rahul Vaze. Distributed algorithms for efficient learning and coordination in ad hoc networks. In 2019 International Symposium on Modeling and Optimization in Mobile, Ad Hoc, and Wireless Networks (WiOPT) (pp. 1-8). 2019.
	
	\item {Arun Verma}, Manjesh K Hanawal, Csaba Szepesv\'ari, and Venkatesh Saligrama. Online Algorithm for Unsupervised Sensor Selection. In The 22nd International Conference on Artificial Intelligence and Statistics (pp. 3168-3176). 2019.
	
	\item {Arun Verma} and Manjesh K Hanawal. Unsupervised cost sensitive predictions with side information. In Proceedings of the ACM India Joint International Conference on Data Science and Management of Data (pp. 364-367). 2018.
\end{enumerate}

\vspace{4.5mm}
\noindent
\textbf{Submitted and Working Papers} 
\begin{enumerate}	
	\item {Arun Verma}, Manjesh K Hanawal, Arun Rajkumar, and Raman Sankaran. Censored Semi-Bandits for Resource Allocation. Under review at Mathematics of Operations Research.
	
	\item {Arun Verma}, Manjesh K Hanawal, and Rahul Vaze. Channels Selection in Cognitive Radio using Collaborative Multi-Player Bandits. To be communicated.
\end{enumerate}

%% file: chapter1/introduction.tex
%!TEX root =  ../thesis.tex

Many real-life systems involve actions being applied sequentially, with each action resulting in a certain reward. The goal is to apply an action from a set of available actions that results in maximum reward (or minimum cost) in each round. However, which action is the best in a given round, may not be known. In such systems, the `goodness' of action has to be inferred from the results of past actions. For example, in the medical diagnosis of an ailment, a doctor prescribes a treatment from a set of available treatments, and by observing their effectiveness, narrows down on the most appropriate treatment \citep{BIOMETRIKA33_thompson1933likelihood}. Such a decision-making process is commonly known as online learning in literature, where the goal is to learn the best action as quickly as possible.

Online machine learning is a branch of machine learning where actions are made sequentially using past feedback, and associated rewards of actions are observed. As opposed to batch learning where all the data are available beforehand, in online machine learning, an instance gets revealed in each round, and after selecting an action, feedback is observed. The feedback could reveal how good the selected action is or may not provide any information at all. The goal is to identify the best action to play in each round. In the batch learning methods, the entire data are available in one go, and the best predictor is obtained using all the data and used for prediction on any new instance. In online learning methods, learning is incremental: as more and more feedback is available, the action to apply in the next instance is dynamically updated using all the past feedback.

Several variants of batch learning methods have been studied in the past: supervised, unsupervised, and active learning. In supervised learning methods, true labels for all data points are known, whereas, in unsupervised methods, they are not known \citep{friedman2001elements, shalev2014understanding}. In the active learning (a special case of semi-supervised learning where the algorithm has a small labeled set of training instances with a larger unlabeled set), the labels are not available by default but can be acquired if necessary \citep{settles1648active}. For example, obtaining labels could be expensive, in which case labels for the most 'important' data points can be acquired on payment or with lots of human effort. All three variants in the batch setting are well-studied in literature.

Many variants of online machine learning are considered in the literature, depending on the type of feedback and the amount of information they reveal about the rewards. The multi-armed bandits and the expert setting \citep{ML02_auer2002finite, bubeck2012regret} are well-studied problems where feedback provides direct information about the rewards. In the multi-armed bandit setting, feedback observed from an action reveals only the reward associated with that action. However, in the expert setting, the feedback observed from an action reveals reward associated with the action played as well as all the other actions. The settings that span in between these two extreme cases are also studied, namely, bandits with side-information \citep{mannor2011bandits,alon2013bandits,alon2015online,NIPS15_wu2015online}.

The online learning methods studied in the literature work under the assumption that feedback from an action provides information about the reward of that action, though the amount of information revealed could vary from one action to another. It is also possible that the feedback from actions is indirectly tied to the rewards. This setting is referred as partial monitoring setting \citep{MOR06_cesa2006regret,ALT12_bartok2012partial,MOR14_bartok2014partial}. It includes all the previously described online learning setups as special cases. A summary of the various online learning setups discussed is shown in below Figure \ref{fig:oml}.
\begin{figure}[!ht]
	\centering
	\includegraphics[width=\linewidth]{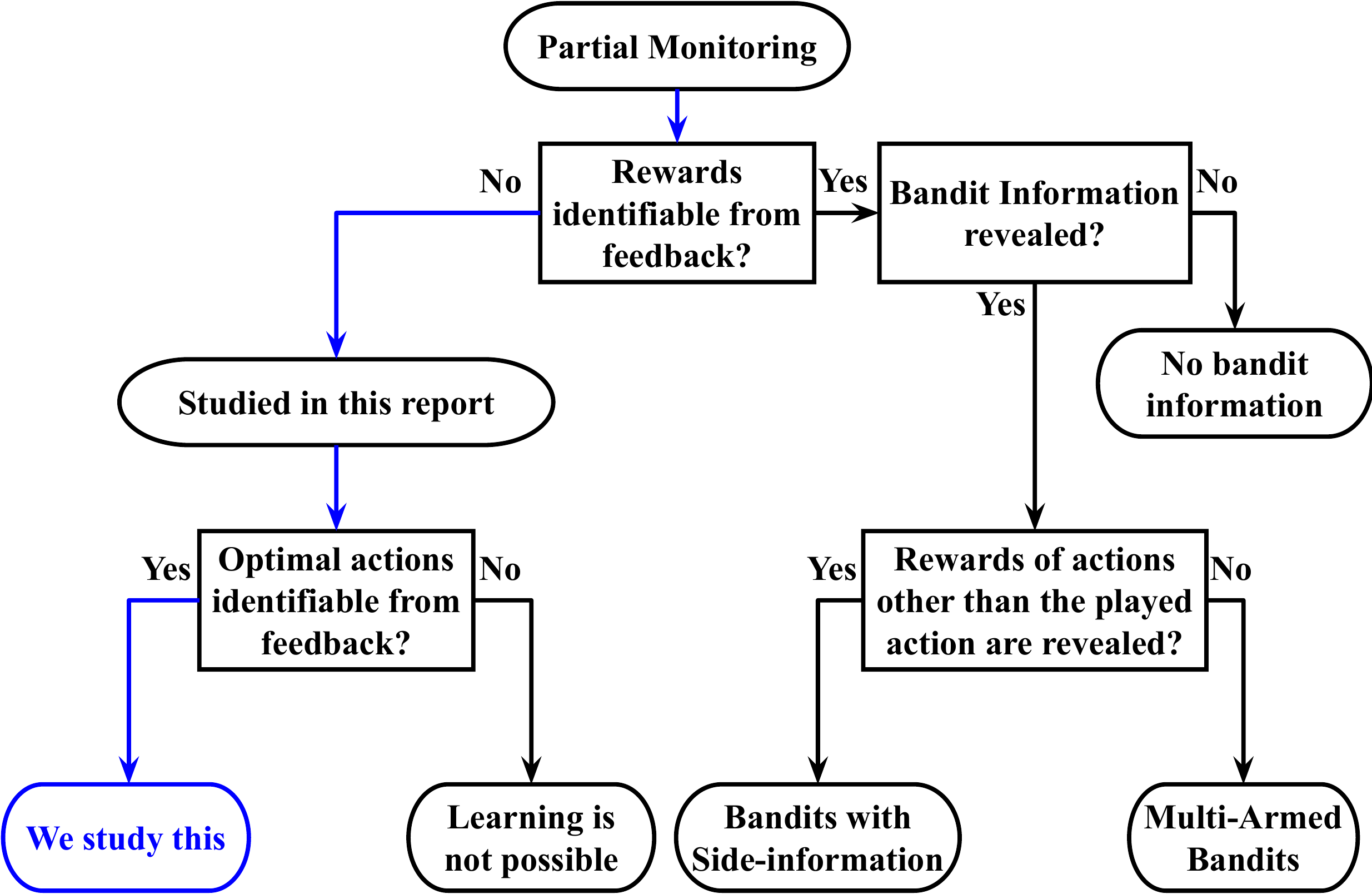}
	\caption{Classification of various online learning settings.}
	\label{fig:oml}
\end{figure}

The finite partial monitoring setting \citep{MOR14_bartok2014partial} is characterized by a pair of $K \times O$ matrices, the reward matrix $R$, and the feedback matrix $F$. In each round $t$, the learner selects an action $I_t \in \{1,2,\ldots, K\}$ and simultaneously, the environment selects an outcome $J_t \in \{1,2,\ldots, O\}$. Then, the learner observes entry $F_{I_t, J_t}$ as feedback and receives instantaneous reward $R_{I_t, J_t}$, which is not revealed to him. The feedback can be thought of as masked information about the outcome of $J_t$. Sometimes, $F_{I_t, J_t}$  may uniquely determine the outcome; otherwise, the feedback may give only partial or no information about the outcome. The learner's performance is measured in terms of regret, i.e., the total reward  $\sum_{t=1 }^T R_{I_t, J_t}$ of the learner is compared with the total reward of best-fixed action in hindsight. Generally, the regret grows with the number of rounds $T$. The learner must have sub-linear regret, which means that the learner's average per-round reward approaches the average per-round reward of the best action in hindsight.

The partial monitoring setting is the most general online machine learning setting, but most of the earlier work on partial monitoring is restricted to cases where feedback from the actions allows the learner to fully identify the rewards of the actions. However, in many areas like crowd-sourcing, medical diagnosis, security management, resource allocation, and many others, feedback from actions may not even be sufficient to fully identify their rewards. In this thesis, the algorithms that identify optimal actions without requiring explicit reward values are developed. Broadly, the aim is to identify the scenarios where the reward values cannot be identified and then develop algorithms to learn optimal action by exploiting the problem structure. We categorized such problem as \underline{s}equential \underline{d}ecision problems with \underline{w}eak \underline{f}eedback (SDWF).

Our main objective in this thesis is to identify scenarios when it is possible to act optimally in SDWF. Clearly, without any information about rewards, one cannot learn which action is good in general. So, the specific problem structure has to be exploited to learn good actions. We propose to identify a set of necessary and sufficient conditions that lead to the learnability of optimal actions. We then identify classes of problems that satisfy these conditions and prepare a common framework for such classes to study partial monitoring problems with weak feedback. For the class of SDWF problems that can be learned, we develop learning algorithms. The goal is to develop provably optimal algorithms that are also easy to implement in practice. The performance of the developed algorithms is evaluated on synthetic and real datasets drawn from various application fields and compared with natural benchmarks.

In this thesis, the following classes of SDWF problems are studied: Uncensored Sequential Selection, Censored Semi-Bandits, and Collaborative Learning in Multi-Player Multi-Armed Bandits. In subsequent, these classes of SDWF problems are briefly presented.

\subsection*{Unsupervised Sequential Selection}
The problems involving the SDWF arise naturally in areas like security, medical diagnosis, and crowd-sourcing, where we may not benefit from knowing the ground-truth associated with the task. For example, in medical diagnosis, patients may not reveal the outcome of treatment due to privacy reasons. Hence, the treatment's effectiveness may not be known. In crowd-sourcing systems, crowd-sourced workers may not reveal their expertise levels, and hence their effectiveness is not known. We named this class of problems as `Unsupervised Sequential Selection' (USS).

The main difficulty in the study of the USS problem is the lack of reward information. Moreover, the techniques developed for the partial monitoring problem, where rewards are identifiable from feedback, are not applicable to the USS setup. However, one would still like to act optimally.  We develop algorithms that learn to select the optimal action when rewards are not fully identifiable from the feedback. Specifically, we exploit the problem structure to learn the optimal action rather than rely on the availability of the reward information.

\subsection*{Censored Semi-Bandits}
Sequential allocation problems with a censored feedback structure have received significant interest in recent times \citep{NIPS16_abernethy2016threshold}. Censoring occurs naturally in several applications such as police patrolling of opportunistic crimes, supplier selection, budget allocation, and several others. In all these applications, an allocation to a set of options (allocation of price to items, allocation of the patrol to locations, etc.) is made by a learner, and then the response to the allocation (purchase, crime, etc.) is observed. If a positive response is observed (crime occurs, a purchase is made), the learner receives information about the goodness of allocation. However, if a negative response (censored) is observed (no crime, no purchase, etc.), the learner cannot decide if it is because of a good allocation or because of an inherent low propensity of a response. The goal then is to make repeated allocations and learn the best possible allocation strategy. We model these problems as stochastic semi-bandits and refer to this class of problems as `Censored Semi-Bandits' (CSB). 

In the CSB problem, the amount of reward information revealed to the learner via feedback depends on the resource allocation. Hence the reward is only partially observable from feedback. We develop algorithms that first estimate the threshold so that we fully observe the reward from feedback and then find the best resource allocation.

\subsection*{Distributed Learning in Multi-Player Multi-Armed Bandits}
We consider a communication network consisting of multiple users and multiple channels where each user wishes to use one of the available channels. A user can transmit on any of the channels. To keep the model more general, users cannot communicate directly between them, and any information about the other users can be obtained only by collisions. In particular, if the user transmits on a channel, it gets a sample of the gain of the channel only if no other user transmits on that channel. Otherwise, it only gets to know that at least one other user is transmitted on that channel. We assume that the expected gain of each of the channels is distinct. Even though users cannot communicate and have to make decisions in a distributed fashion, we consider a non-strategic setting where users have a common goal to maximize the total gain in the network, i.e., to achieve an optimal allocation, which is realized when all the users occupy non-overlapping channels in the top channels. Here the top channels refer to the set of channels with the highest expected gain. As the goal is to achieve optimal network allocation,  we assume that each user is satisfied if she gets one of the top channels when no other better channel is free. Our aim in this work is to find a distributed strategy that minimizes the time by which the users reach an optimal allocation while keeping the regret and number of collisions as small as possible. In this thesis, this setup is modeled as `Multi-Player Multi-Armed Bandits' (MP-MAB).

The reward information is the same as feedback in the MP-MAB problem unless two players select the same arm. We propose distributed strategies that assign arms to players and then learn in a distributed fashion by using all players for collecting information about the arms with a minimal number of collisions.

\subsection*{Organization of Thesis}
In \cref{cha:uss}, we introduce Unsupervised Sensor Selection problem. We provide an optimal algorithm for solving the USS problem under certain conditions. We study the contextual version of the USS problem in \cref{cha:contextual_uss}, where the optimal action could be task-dependent. In \cref{cha:csb}, we proposed a novel framework for resource allocation problems using a variant of semi-bandits named censored semi-bandits. We use censored semi-bandits for solving the stochastic network utility maximization problem. We showcase the application of the CSB setup for network utility maximization. In \cref{cha:mp-mab}, we set up the channels selection problem in the communication network as a stochastic multi-player multi-armed bandits. We developed distributed strategies without centralized communication that achieve constant regret with a high probability. Since we studied the different sequential decision problems in this thesis, the related literature is given in the respective chapters.

%% file: chapter2/main.tex
%!TEX root =  ../thesis.tex

In this chapter, we study {\em \underline{U}nsupervised \underline{S}equential \underline{S}election} (USS) problem. The USS problem is a variant of the stochastic multi-armed bandits problem, where the loss of an arm can not be inferred from the observed feedback. In the USS setup, arms are associated with fixed costs and are ordered, forming a cascade. The learner selects an arm in each round and observes the feedback from arms up to the selected arm. The learner's goal is to find the arm that minimizes the expected total loss. The total loss is the sum of the cost incurred for selecting the arm and the stochastic loss associated with the selected arm. The problem is challenging because, without knowing the mean loss, one cannot compute the total loss for the selected arm. Clearly, learning is feasible only if the optimal arm can be inferred from the problem structure. It is observed that this happens if the problem satisfies the `Weak Dominance' (WD) property.  We set up the USS problem as a stochastic partial monitoring problem and develop algorithms with sub-linear regret under the WD property. We argue that our algorithms are optimal and evaluate its performance on problem instances generated from synthetic and real-world datasets.

\section{Unsupervised Sequential Selection (USS)}
\label{sec:uss_introduction}
\input{chapter2/introduction}

\section{Problem Setting}
\label{sec:uss_setup}
\input{chapter2/problem_setting}

\section{Conditions for Learning Optimal Arm}
\label{sec:uss_learning}
\input{chapter2/learning}

\section{Upper Confidence Bound based Algorithm for USS problem}
\label{sec:uss_ucb}
\input{chapter2/uss_ucb}
\nocite{AISTATS19_verma2019online}

\section{Thompson Sampling based Algorithm for USS problem}
\label{sec:uss_ts}
\input{chapter2/uss_ts}
\nocite{ACML20_verma2020thompson,COMSNETS20_verma2020unsupervised}

\section{Experiments}
\label{sec:uss_experiments}
\input{chapter2/experiment}

\section{Appendix}
\label{sec:uss_appendix}
\input{chapter2/appendix}

%% file: chapter2/introduction.tex
%!TEX root =  ../thesis.tex

In many applications, one has to trade-off between accuracy and cost. For example, for detecting some event, it is not only the accuracy of a test/sensor that matters, but the associated cost is important as well. Also, one may have to predict labels of instances for which ground-truth cannot be obtained. In such scenarios, feedback about the correctness of predictions remains unknown. Such problems arise naturally in medical diagnosis, crowd-sourcing, security system \citep{AISTATS17_hanawal2017unsupervised}, and online features selection in unsupervised setting. In the medical diagnosis problem, either the true state of the patients may not be known or the patients may not reveal the outcome of treatment due to privacy concerns; hence, the test's effectiveness cannot be known. Whereas in the crowd-sourcing systems, the expertise level of self-listed-agents (workers) is unknown; therefore, the quality of their work cannot be known. In these prediction problems, we can observe prediction from test/worker, but we cannot ascertain their reliability due to the absence of ground truth.

In many of the real-world situations like those found in medical diagnosis, airport security, and manufacturing, a set of tests or classifiers is used to monitor patients, people, and products. Tests have cost with the more informative ones resulting in higher monetary costs and higher latency. Thus, they are often organized as a cascade \citep{AISTATS12_chen2012classifier, AISTATS13_trapeznikov2013supervised}, so that a new input is first probed by an inexpensive test then more expensive one. We refer to such cascaded systems as {\it Unsupervised Sequential Selection} (USS) problem\footnote{Note that the unsupervised sequential selection problem is referred to as the unsupervised sensor selection problem in the prior work \citep{AISTATS17_hanawal2017unsupervised}.}, where an arm represents a test/ worker. A learner's goal in the USS problem is to select the most cost-effective arm so that the overall system maintains high accuracy at low average costs.

Clearly, without the knowledge of the ground-truth, one cannot find the optimal arm as their losses cannot be computed. In the USS setup, the structure of the problem needs to be exploited, and it is shown that under certain conditions, namely strong dominance ($\SD$) and weak dominance ($\WD$), learning is possible. The $\SD$ property requires the loss of an arm to stochastically dominate the losses of other arms with lower costs in the cascade. Specifically, it assumes that if an arm's prediction is correct, then all the arms that follow this arm in the cascade also have correct predictions.

Under the $\SD$ property, \citet{AISTATS17_hanawal2017unsupervised} established that USS problem is equivalent to a multi-armed bandit with side observations and exploit the equivalence to give an algorithm with sub-linear regret. The $\SD$ property is quite strong and posits that disagreement probability of the predictions of two arms is equal to the difference in error rates. This property implies that we can measure losses by measuring disagreement probabilities leading to a direct multi-armed bandit (MAB) reduction and analysis.

The $\WD$ property relaxes strict stochastic ordering on predictions and allows errors on some instances from better arms. It is argued that the set of instances satisfying the $\WD$ property is maximally learnable, and any further relaxation of this property renders the problems unlearnable. The reduction techniques used under $\SD$ property does not apply/extend to $\WD$ property. For this case, a heuristic algorithm without any performance guarantee is given in \cite{AISTATS17_hanawal2017unsupervised}.  Our work bridges this gap. Specifically, we propose algorithms that use the notion of weak dominance \citep{AISTATS17_hanawal2017unsupervised} that helps to find optimal arm using observed disagreements between arms. We then validate the performance of these algorithms on several problem instances derived from synthetic and real datasets. Our contributions of this chapter is summarized as follows:
\begin{itemize}
	\item \cite{AISTATS17_hanawal2017unsupervised} assume that arms are ordered, i.e., their accuracy improve with their index, and used this fact in their algorithms. We relax this assumption in \cref{ssec:uss_scd} where the arms can have an arbitrary order. For this setup, we show that the same $\WD$ property determines the learnability. Hence, we demonstrate that the requirement of arms being ordered by error rate in the USS problem is unnecessary and as long as the $\WD$ property is satisfied the USS problem is learnable.
		
	\item We develop an UCB based algorithm named \ref{alg:USS_WD} that has sub-linear regret under $\WD$ property. We characterize regret in terms of how `well' the problem instances satisfy the $\WD$ property and then provide a bound that holds uniformly for all $\WD$ instances.

	\item We give problem independent bounds on the regret of \ref{alg:USS_WD}. We show that it is of order $T^{2/3}$ under $\WD$ property and improves to $T^{1/2}$ under $\SD$ property. We establish that the bounds are optimal using results from partial monitoring in \cref{sec:uss_ucb}.

	\item We develop a Thompson Sampling based algorithm named \ref{alg:TS_USS} for the USS problem. \ref{alg:TS_USS} uses a one-sided test to find the optimal arm, whereas the \ref{alg:USS_WD} uses a two-sided test to identify the optimal arm. The one-sided test leads to a simpler algorithm. 
	
	\item In \cref{sec:uss_ts}, we characterize the regret of \ref{alg:TS_USS} in terms of how well the problem instance satisfies the $\WD$ property and show that it has sub-linear regret under $\WD$ property. We also give problem independent regret bound and establish that the regret bounds are near-optimal.
	
	\item We demonstrate empirical performance of our algorithms on both synthetic and real datasets in \cref{sec:uss_experiments}. The experimental results show that regret of \ref{alg:USS_WD} and \ref{alg:TS_USS} is always lower than the heuristic algorithm in \cite{AISTATS17_hanawal2017unsupervised}. Further, our experimental results show that regret of \ref{alg:TS_USS} is always lower than \ref{alg:USS_WD}.
\end{itemize}

\subsection{Related Work} 
Several works consider the problem similar to the USS setup in either batch, or online settings \citep{AISTATS13_trapeznikov2013supervised,ICML14_seldin2014prediction}. 
However, they all require that the label of each data point is available or the reward is obtained for each action. 
\cite{NIPS13_zolghadr2013online} considers that the labels are available on payment. \cite{AI02_greiner2002learning} and \cite{ICML09_poczos2009learning} consider costs associated with tests. However, they assume that loss/reward associated with the players' action is revealed. In contrast, in our setting, the labels are not revealed at any point and are thus completely unsupervised, and the cost in our setup is related to sensing cost and not that of acquiring a label. 

\citet{UAI14_platanios2014estimating, ICML16_platanios2016estimating, NIPS17_platanios2017estimating} consider the problem of estimating accuracies of the multiple binary classifiers with unlabeled data. Most of these works make strong assumptions such as independence given the labels, knowledge of the true distribution of the labels.  
\citet{UAI14_platanios2014estimating} proposed logistic regression based methods using the classifiers' agreement rates over unlabeled data, \cite{ICML16_platanios2016estimating} extend this work to use graphical models, and \citet{NIPS17_platanios2017estimating} proposes method using probabilistic logic. Further, \citet{NIPS17_platanios2017estimating} also use weighted majority vote for label prediction. All this is in the batch setting and differs from our online setup.

In the crowd-sourcing problems, various methods have been proposed to estimate unknown skill-level of crowd-workers from the noisy labels they provide (\citet{NIPS17_bonald2017minimax, ICLM18_kleindessner2018crowdsourcing}). These methods assume that all workers are having the same cost and aggregate the predictions on a  given dataset for estimating the accuracy of each worker. Unlike our problem setup, these methods are not online.

Our work is closely related to the stochastic partial monitoring setting  
\citep{MOR06_cesa2006regret,ALT12_bartok2012partial,MOR14_bartok2014partial,NIPS15_wu2015online}, where the feedback from actions is indirectly tied to the rewards.  In our setting, we exploit the problem structure to learn an optimal arm without explicitly knowing the loss associated with each action.

%% file: chapter2/problem_setting.tex
%!TEX root =  ../thesis.tex

We consider a stochastic $K$-armed bandits problem. The set of arms is denoted by $[K]$ where $[K] \doteq \{1,2,\ldots, K\}$. In round $t$, the environment generates a task $x_t$ independently and corresponding $(K+1)$-dimensional binary vector $\left(\Yt, \{\Yti\}_{i \in [K]}\right)$. The variable $\Yt$ denotes the best binary feedback for round $t$, which is hidden from the learner. The vector $\left(\{\Yti\}_{i \in [K]}\right) \in \{0, 1\}^{K}$ represents observed feedback at time $t$, where $\Yti$ denote the feedback\footnote{In the USS setup, an arm $i$ could represent a classifier. After using the first $i$ classifiers, the final label can be a function of labels predicted by the first $i$ classifiers, $i \in[K]$.} observed after playing arm $i$. We denote the cost for using arm $i \in [K]$ as $c_i\geq 0$ that is known to learner and the same for all rounds. \cref{fig:USS_Cascade} depicts the USS setup.
\begin{figure}[!ht]
	\centering
	\includegraphics[width=0.8\linewidth]{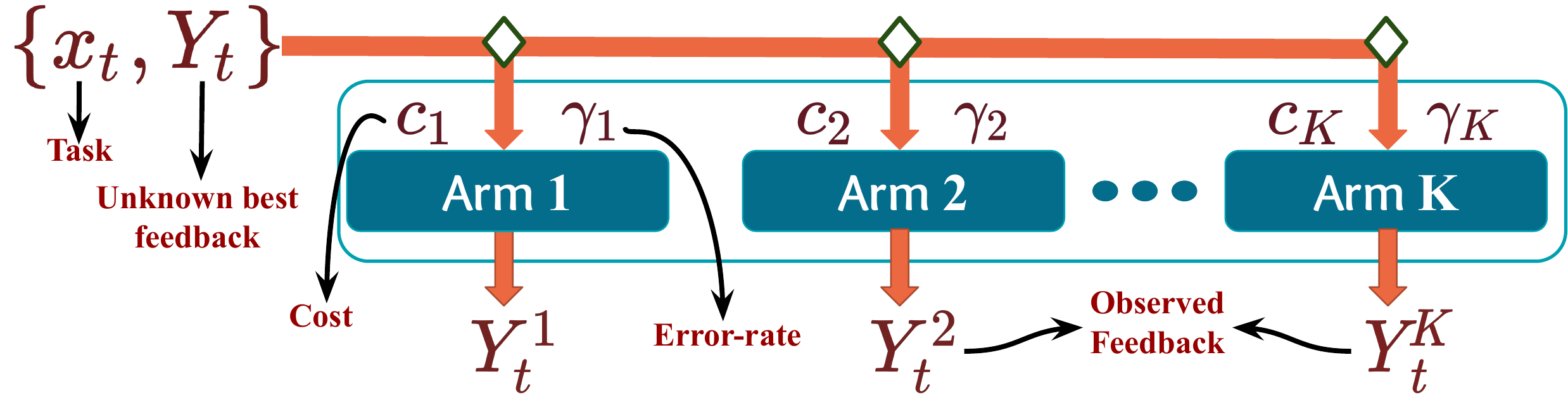}
	\caption{Cascaded Unsupervised Sequential Selection (USS) setup. In round $t$, $x_t$ represents the task, $Y_t$ is the hidden state of the instance, and $Y_t^1, Y_t^2 \ldots Y_t^K$ are feedback observed from the arms. $\gamma_i$ denotes error rate of the arm $i$ and $c_i$ denotes the cost of using the arm $i$.}
	\label{fig:USS_Cascade}
\end{figure}

In the USS setup, the arms are assumed to be ordered and form a cascade. When the learner selects an arm $i \in [K]$, the feedback from all arms till arm $i$ in the cascade is observed. The expected loss of playing the arm $i$ is denoted as $\gamma_i \doteq \EE{\one{\Yi \neq Y}} = \Prob{\Yi \neq Y}$, where $\one{A}$ denotes indicator of event $A$. The {\it expected total cost} incurred by playing arm $i$ is defined as $\gamma_i +\lambda_iC_i$, where $C_i \doteq c_1 + \ldots + c_i$ and $\lambda_i$ is a trade-off parameter that normalizes the loss and the incurred cost of playing arm $i$. The interaction between the environment and a learner is given in Algorithm \ref{alg:USS}.
\begin{algorithm}[!ht]
	\caption{Learning with USS instance $(\bd, \bc)$}
	\label{alg:USS}
	For each round $t$: 
	\begin{enumerate}
		\item \textbf{Environment} chooses a vector $(\Yt, \{\Yti\}_{i \in [K]})\sim \bd$.
		\item \textbf{Learner} selects an arm $I_t \in [K]$ to stop in cascade.
		\item \textbf{Feedback and Loss:} The learner observes feedback $(\Yt^1, \Yt^2, \ldots, \Yt^{I_t})$ and incurs a total loss $\one{Y^{I_t} \neq \Yt} + \lambda_{I_t}C_{I_t}$. 
	\end{enumerate}
\end{algorithm}

Since the best binary feedback are hidden from the learner, the expected loss of an arm cannot be inferred from the observed feedback. We thus have a version of the stochastic partial monitoring problem, and we refer to it as \underline{u}nsupervised \underline{s}equential \underline{s}election (USS) problem. Let $\bd$ be the unknown joint distribution of $(Y, Y^1, Y^2 \ldots, Y^K)$. Henceforth we identify an USS instance as $P \doteq (\bd,\bc)$ where $\bc \doteq (c_1, c_2, \ldots, c_K)$ is the known cost vector of arms. We denote the collection of all USS instances as $\USS$. For instance $P \in \USS$, the optimal arm is given by
\eq{
	\label{equ:optimalArm}
	\istar \in \max\left\{\argmin _{i \in [K]} \left( \gamma_i+\lambda_iC_i \right)\right\}
}
where the `max' operator selects the arm with the largest index among the minimizers. The choice of $\istar$ in \cref{equ:optimalArm} is risk-averse as we prefer the arm with the lowest error among the optimal arms. 

The learner's goal is to learn a policy that finds an arm such that the cumulative expected loss is minimized. Specifically, for $T$ rounds, we measure the performance of a policy that selects an arm $I_t$ in round $t$ in terms of regret given by
\eq{
	\label{eq:cum_regret}
	{\Regret_T} = \sum_{t=1}^T\left( \gamma_{I_t} +  \lambda_{I_t}C_{I_t} - \left(\gamma_\istar +  \lambda_\istar C_{\istar} \right) \right).
}

A good policy should have sub-linear regret, i.e.,
$\lim\limits_{T \rightarrow \infty}{\Regret_T}/T = 0$. The sub-linear regret implies that the learner collects almost as much reward in expectation in the long run as an oracle that knew the optimal arm from the first round. 
We say that a problem instance $P \in \USS$ is learnable if there exists a policy with sub-linear regret.

%% file: chapter2/learning.tex
%!TEX root =  ../thesis.tex

The purpose of this section is to introduce the notions of strong and weak dominance from the work of \citet{AISTATS17_hanawal2017unsupervised}. While \citeauthor{AISTATS17_hanawal2017unsupervised} studied learning under strong dominance, we will focus on weak dominance in this chapter. We also modify the definition of weak dominance of \citeauthor{AISTATS17_hanawal2017unsupervised} to correct an oversight of them. Next, we define the strong and weak dominance property of the USS problem instance that makes the learning of the optimal arm possible.
\begin{defi}[Strong  Dominance $(\SD)$ \citep{AISTATS17_hanawal2017unsupervised}] 
	\label{def:CSD} 
	A problem instance is said to satisfy $\SD$ property if
	\eqs{
		Y^i = Y \mbox{ for some } i \in[K] \implies  Y^j = Y, ~~ \forall j > i.
	}
	We represent the set of all instances in $\USS$ that satisfy $\SD$ property by $\PSD$.
\end{defi}

The $\SD$ property implies that if the feedback of an arm is same as the true reward, then the feedback of all the arms in the subsequent stages of the cascade is also same as the true reward. \cite{AISTATS17_hanawal2017unsupervised} show that the set of all instances satisfying $\SD$ property is learnable by mapping such instances to stochastic multi-armed bandits problem with side information \citep{NIPS15_wu2015online}. A weaker version of the $\SD$ property is defined as follows:
\begin{defi}[Weak Dominance $(\WD)$ \citep{AISTATS17_hanawal2017unsupervised}] 
	\label{def:WD} 
	Let $i^\star$ denote the optimal arm. Then an instance $P \in \USS$ is said to satisfy {weak dominance property} if
	\eq{
		\label{equ:WDProp}
		\rho \doteq \min_{j > i^\star} \frac{C_j-C_{i^\star}}{\Prob{Y^j \neq Y^{i^\star}}}>1.	
	}
	We denote the set of all instances in $\USS$ that satisfy $\WD$ property by $\PWD$.
\end{defi}
The $\WD$ property implies that for all sub-optimal arms that come after the optimal arm in the cascade, the cost difference is more than their disagreement probability.
Let $\PWD = \left\{P \in \USS:P \mbox{ satisfies $\WD$ condition}\right\}$ denote the set of instances satisfying the $\WD$ property. \citet{AISTATS17_hanawal2017unsupervised} claimed that $\PWD$ is learnable. However, their definition allowed $\rho \ge 1$. As it turns out, permitting $\rho=1$ can prevent $\PWD$ from being learnable:

\begin{prop}\label{prop:twdbug}
	The set $\PWD'=\{ P \in \USS\,:\, \rho \ge 1 \}$ is not learnable.
\end{prop}

\begin{proof}
	Let $C_2-C_1=1/4$. Theorem~19 of  \citet{AISTATS17_hanawal2017unsupervised} constructs instances $P,P'\in \PWD'$ such that the optimal decision for $P$ is arm $1$, for $P'$ is arm $2$. The sub-optimality gap on instance $P$ is $1/4$, while on instance $P'$ is $\epsilon$, where $\epsilon\in [0,1]$ is a tunable parameter. At the same time $\Prob{Y^1\ne Y^2}=1/4$ in $P$ and $\Prob{Y^1\ne Y^2}=1/4+\epsilon$  in $P'$. Theorem~17 of  \citet{AISTATS17_hanawal2017unsupervised} implies that a sound algorithm must check $1/4=C_2-C_1 \ge \Prob{Y^1\ne Y^2}$. However, no finite amount of data is sufficient to decide this: In particular, one can show that if an algorithm on $P$ achieves sub-linear regret, then it must suffer linear regret on $P'$ for $\epsilon>0$ small enough. Hence, all algorithms will suffer linear regret on some instance in $\PWD'$.
\end{proof}

The following theorem is obtained directly from Theorem $14$ and Theorem $19$ in \cite{AISTATS17_hanawal2017unsupervised} after excluding the case $\rho=1$ in their proofs.  
\begin{thm}
	The set $\PWD$ is a maximal learnable set.
\end{thm}
Since the set of problems satisfying the $\WD$ property is maximally learnable, any relaxation of $\WD$ property makes the problem unlearnable. In the following equation, we use an alternative characterization of the $\WD$ property, given as
\eq{
	\label{def:Xi}
	\xi \doteq \min_{j>\istar}\left\{C_j - C_\istar - \Prob{Y^\istar \ne \Yj} \right\} > 0.
}

After multiplying LHS of \cref{equ:WDProp} by $\Prob{Y^j \neq Y^{i^\star}}$ and then subtracting both sides by $\Prob{Y^j \neq Y^{i^\star}}$, we get LHS of \cref{def:Xi}. Observe that $\rho>1$ if and only if $\xi>0$. The larger the value of $\xi$, `stronger' is the $\WD$ property, and easier to identify an optimal arm. We later characterize the regret upper bound of our algorithms in terms of $\xi$.

\subsection{Optimal Arm Selection}
Without loss of generality, we set $\lambda_i=1$ for all $i\in [K]$ as their value can be absorbed into the costs. Since $\istar = \max\big\{\arg\min\limits_{i \in [K]}\left(\gamma_i+ C_i \right)\big\}$, it must satisfy following equation:
\begin{subequations}
	\label{eq:cost_exp_err}
	\al{
		&\forall j<\istar \,:\, C_\istar - C_j \leq \gamma_j-\gamma_\istar \,, \label{eq:wd1}\\ 
		&\forall j>\istar \,:\, C_j - C_\istar > \gamma_\istar - \gamma_j \,. \label{eq:wd2}
	}
\end{subequations}

As the loss of an arm is not observed, the above equations can not lead to a sound arm selection criteria. We thus have to relate the unobservable quantities in terms of the quantities that can be observed. In our setup, we can compare the feedback of two arms, which can be used to estimate their disagreement probability. For notation convenience, we define $\pij \doteq \Prob{\Yi \ne \Yj}$. The value of $\pij$ can be estimated as it is observable. We use the following result from \cite{AISTATS17_hanawal2017unsupervised} that relates the differences in the unobserved error rates in terms of their observable disagreement probability.
\begin{prop}[Proposition 3 in \cite{AISTATS17_hanawal2017unsupervised}]
	\label{lem:err_prob_contx}
	For any two arms $i$ and $j$, $\gamma_{i} - \gamma_{j} = \pij - 2\Prob{\Yi = Y, \Yj \ne Y}$.
\end{prop}

\noindent
Now, using \cref{lem:err_prob_contx}, we can replace  \cref{eq:wd1} by
\eq{
	\label{eq:selectDisProbLow}
	\forall j<\istar \,:\, C_{\istar} - C_j \leq  \pjis,
}
which only has observable quantities. For $j>\istar$, we can replace \cref{eq:wd2} by using the $\WD$ property as follows:
\eq{
	\label{eq:selectDisProbHigh}
	\forall j>\istar \,:\, C_j - C_{\istar} >  \pijs.
}

\noindent
Using \cref{eq:selectDisProbLow} and \cref{eq:selectDisProbHigh}, we define the selection criteria based on the following sets:
\begin{align}
&\mathcal{B}^l= \Big\{i: \forall j<i, C_i - C_j \leq \Prob{\Yi \ne \Yj}\Big\} \cup \{1\} \label{set:Bl}, \\
&\mathcal{B}^h = \Big\{i: \forall j>i, C_j - C_i > \Prob{\Yi \ne \Yj} \Big\} \cup \{K\} \label{set:Bh}.
\end{align}

\noindent
Our next result gives the optimal arm for a problem instance.
\begin{restatable}{lem}{SetB}
	\label{lem:B}
	Let $P \in \PWD$. Let $\mathcal{B} \doteq \mathcal{B}^l \cap \mathcal{B}^h$. Then $\mathcal{B}$ contains the optimal arm.  
\end{restatable}
\noindent
The detailed proof of \cref{lem:B} and all other missing proofs of this chapter appear in \cref{sec:uss_appendix}.

\begin{rem}
	The $\WD$ property holds trivially for the problem instances that satisfy $\SD$ property as the difference of mean losses is the same as the disagreement probability between two arms due to $\Prob{\Yti = \Yt, \Ytj \ne \Yt} = 0$ for $j>i$. Also, by definition, the $\WD$ property holds for all problem instances where the last arm of the cascade is an optimal arm.
\end{rem}

\subsection{Unknown Ordering of Arms}
\label{ssec:uss_scd}
The error rates of arms are unknown in the USS setup and cannot be estimated due to unavailability of ground-truth. Thus, it may happen that we do not know whether error rate of arms in the cascade is decreasing or not. In this section, we remove the requirement that arms are arranged in the decreasing order of their error rates and allow them to be arranged in an arbitrary order that is unknown. We denote the set of USS instances with unknown ordering of arms by their error-rates as $\USS^\prime$. The rest of the setup remains same. We show that even with this relaxation, $\WD$ property defined earlier, continues to characterize the learnability of the USS problem. We begin with the following observation.
\begin{lem}
	\label{lem:ErrorRateOrder}
	Let $i^\star$ be an optimal arm. Then, error rate of any arm $j<i^\star$ is higher than that of $i^\star$.
\end{lem}
\begin{proof}	
	We have $\gamma_j -\gamma_{i^\star} \geq C_{i^\star}- C_i$ for all $j \in [K]$. For $j < i^\star$, $C_{i^\star}- C_j \geq 0$ as costs are increasing with arms. Hence  $\gamma_j\geq\gamma_{i^\star}$.
\end{proof}

\noindent
The following corollary directly follows from \cref{lem:err_prob_contx}.

\begin{cor}
	For any $i,j \in [K]$, $\max\{0, \gamma_{j}-\gamma_{i}\}\leq \Prob{Y^i\neq Y^j}$.
\end{cor}

The following two propositions provide the conditions on arm costs that allows comparison of their total costs based on disagreement probabilities.

\begin{restatable}{prop}{PropCostRangeHigh}
	\label{prop:CostRange1}
	Let $i<j$. Assume 
	\begin{equation}
	\label{eqn:CostRange1}
	C_j -C_i \notin \left (\max\{0, \gamma_i-\gamma_j\}, \Prob{Y^i \neq Y^j}\right].
	\end{equation}
	Then, $C_j-C_i >  \max \{0, \gamma_i-\gamma_j \}$ iff $C_j-C_i > \Prob{Y^j \neq Y^i}$.
\end{restatable}

\begin{restatable}{prop}{PropCostRangeLow}
	\label{prop:CostRange2}
	Let $i>j$. Assume 
	\begin{equation}
	\label{eqn:CostRange2}
	C_i -C_j \notin \left (\max\{0, \gamma_j-\gamma_i\}, \Prob{Y^i \neq Y^j}\right ].
	\end{equation}
	Then, $C_i-C_j \leq  \max\{0, \gamma_j-\gamma_i\} $ iff $C_j-C_i \leq \Prob{Y^i \neq Y^j}$.
\end{restatable}

From \cref{lem:ErrorRateOrder}, for any $j < i^*$ we have $\max\{0, \gamma_j -\gamma_{i^\star}\}=\gamma_j -\gamma_{i^\star}$. \cref{prop:CostRange1} and \cref{prop:CostRange2} then suggests that the value of $\Prob{Y^i \neq Y^j}$ are sufficient to select the optimal arm if the arms costs satisfy (\cref{eqn:CostRange1}) for all $j > i^\star$ and (\cref{eqn:CostRange2}) for all $j < i^\star$. Since the values of $\Prob{Y^i \neq Y^j}$ can be estimated for all $i,j \in [K],$ we can establish the following result.

\begin{restatable}{prop}{PropWDUnorder}
	\label{prop:WD2} 
	Let  $i^\star$ be an optimal arm. Any problem instance $P \in \USS^\prime$  is learnable if
	\begin{equation*}
	\forall\; j > i^\star \;\; C_j -C_{i^\star} \notin \left (\max\{0, \gamma_{i^\star}-\gamma_j\}, \Prob{Y^i \neq Y^j}\right ].
	\end{equation*}
\end{restatable}

Notice that  for $j > i^\star$, $C_j- C_i^\star\geq 0$ and $C_j-C_{i^\star}\geq \gamma_{i^\star}-\gamma_j $. Hence, the learnability condition reduces to $\forall\; j >i^\star, C_j-C_{i^\star} > \Pr\{Y^{i^\star}\neq Y^j\}$, i.e., same as the WD condition. Hence, we have the following result.

\begin{thm}
	The set $\{P \in \USS^\prime: \rho(\theta)>1\} $ is learnable.
\end{thm}

\begin{proof}
	Using equation \eqref{eq:wd2} and the fact that the costs of using an arm is increasing in the cascade, we have $C_j - C_{i^\star} > \max\{0, \gamma_{i^\star}- \gamma_{j}\}$. The max function is used because arbitrary ordering of arms by error-rates can lead to $\gamma_{i^\star}- \gamma_{j}$ being negative whereas $C_j - C_{i^\star}$ is a non-negative quantity.  Therefore, 
	\begin{align}
		\forall j > i^\star,\; C_j - C_{i^\star}  > \Prob{\Yis \ne \Yj}
	\end{align}
	It implies that $\WD$ property holds for problem instance even with the arbitrary ordering of arms by their error-rates. This relationship ensures learnability of such USS problems.
\end{proof}

%% file: chapter2/uss_ucb.tex
%!TEX root =  ../thesis.tex

In bandit problems, the upper confidence bound (UCB) \citep{ML02_auer2002finite, COLT11_garivier2011kl} is highly effective for dealing with the trade-off between exploration and exploitation. Using UCB idea, we develop an algorithm, named \ref{alg:USS_WD}, that utilizes the sets \eqref{set:Bl} and \eqref{set:Bh} and looks for an index that belongs to both.  Since the value of disagreement probabilities, $\left(\pij \doteq \Prob{\Yi \ne \Yj}\right)$ are unknown (but fixed and observable), they are replaced by their optimistic empirical estimates at round $t$, denoted by $\hpij + \Psi_{ij}(t)$, where $\hpij$ is empirical estimate of $\pij$ and $\Psi_{ij}(t)$ is the confidence term associated with $\hpij$ as in UCB algorithm. The new sets for selection criteria are defined as follows:
\begin{subequations}
	\label{set:eoptimal}
	\begin{align}
		&\mathcal{\hat{B}}_t^l= \{i: \forall j<i, C_{i}-C_{j} \leq \hpji + \Psi_{ji}(t)\} \cup \{1\}, \label{set:eBl} \\
		&\mathcal{\hat{B}}_t^h=\{i: \forall j>i;C_{j} -C_{i} > \hpij + \Psi_{ij}(t)\} \cup \{K\}. \label{set:eBh}
	\end{align}
\end{subequations}
From the definition, it is easy to verify that $\hpij = \hpji$ and $\Psi_{ij}(t) = \Psi_{ji}(t)$ for any $(i,j)$ pair. Therefore, it is enough for algorithm to only keep track of $\hpij$ and $\Psi_{ji}(t)$ for $i<j$. 
\begin{rem}
	\label{rem:DecisionSet}
	It might be tempting to use lower confidence, i.e., $\hpij - \Psi_{ij}(t)$ term instead of the upper confidence term in \eqref{set:eBh}.  However, such a change can make the algorithm converge to a sub-optimal arm. A detailed discussion is given in \cref{sec:uss_appendix}.
\end{rem}

\subsection{Algorithm: \ref{alg:USS_WD}}
The pseudo code of Algorithm is given in \ref{alg:USS_WD} and it works as follows. It takes $\alpha$ as an input that trades-off between exploration and exploitation. In the first round, it selects arm $K$ and  initializes the value of number of comparisons and counter of disagreements for each pair $(i,j), i<j$, denoted $\mathcal{N}_{ij}(1)$ and $\mathcal{D}_{ij}(1)$, respectively. In each subsequent round, the algorithm computes estimate for the disagreement probability $(\hpij)$  and the associated confidence $(\Psi_{ij}(t))$. Then the values of $\hpij$ and $\Psi_{ij}(t)$ are used for computing sets $\mathcal{B}_t^l$ and $\mathcal{B}_t^h$ which are then used to select the arm. Specifically, the algorithm selects a arm $I_t$ that satisfies the conditions given in \eqref{set:eBl} and \eqref{set:eBh}. 
\begin{center}
	\begin{algorithm}[!ht]
		\renewcommand{\thealgorithm}{USS-UCB}
		\floatname{algorithm}{}
		\caption{UCB based Algorithm for \textbf{USS} under \textbf{WD} property}
		\label{alg:USS_WD}
		\begin{algorithmic}[1]
			\Statex \hspace{-0.44cm}\textbf{Input:} $\alpha>0.5$
			\State Select arm $I_1 = K$ and observe $Y^1_1,\dots,Y^{I_1}_1$
			\State Set $\mathcal{D}_{ij}(1) \leftarrow \one{\Yi_1 \ne \Yj_1}, \;\mathcal{N}_{ij}(1) \leftarrow 1\;\; \forall i< j \le I_1$
			\For{$t=2,3,...$}
				\State $\hpij \leftarrow \frac{\Dijm}{\Nijm}\;\; \forall i< j \le K$
				\State $\Psi_{ij}(t) \leftarrow \sqrt{\frac{\alpha\log f(t)}{\Nijm}}\;\; \forall i< j \le K$
				\State Compute $\mathcal{\hat{B}}_t^l$ and $\mathcal{\hat{B}}_t^h$ as given in \eqref{set:eBl} and \eqref{set:eBh}
				\State $\mathcal{\hat{B}}_t := \mathcal{\hat{B}}_t^l \cap \mathcal{\hat{B}}_t^h$
				\State $I_t\leftarrow \min \big\{\mathcal{\hat{B}}_t\cup \{K\}\big\}$
				\State Select arm $I_t$ and observe $Y^1_t,\dots,Y^{I_t}_t$
				\State $\Dij \leftarrow \Dijm + \one{\Yti \ne \Ytj}\;\; \forall i< j \le I_t$\label{alg:USS_WD_D} 
				\State $\Nij \leftarrow \Nijm + 1\;\; \forall i< j \le I_t$\label{alg:USS_WD_N}
			\EndFor
		\end{algorithmic}
	\end{algorithm}
\end{center}

Since initial estimates for $\pij$ are not good enough, $\hat{ \mathcal{B}}_t$ can be empty. In such a case, the algorithm selects the arm $K$. After selection of arm $I_t$, $Y^j_t, j\in [I_t] $ are observed which are then used to update the $\Dij$ and $\Nij$ in the algorithm.

\subsection{Regret Analysis}
The following notations and definition are useful in subsequent proofs. For the optimal arm $i^\star$ and each $j \in [K]$, let 
\begin{align}
	\Delta_j := C_j + \gamma_j - (C_{i^\star} + \gamma_{i^\star}) \label{def_delta},
\end{align}
\begin{subnumcases}
	{\kappa_j := }
	\pijs - (\gamma_j - \gamma_{i^\star}),\;\; \text{ if } j<i^\star  \label{def_kappa_l}\\
	\pijs - (\gamma_{i^\star} - \gamma_j), \;\; \text{ if } j>i^\star \label{def_kappa_h} 
\end{subnumcases}
\begin{subnumcases}
	{\xi_j := }
	\Delta_j + \kappa_j, \;\; \text{ if } j<i^\star  \label{def_xi_l}\\
	\Delta_j - \kappa_j, \;\; \text{ if } j>i^\star \label{def_xi_h} 
\end{subnumcases}

Notice that the values of $\kappa_j$ and $\xi_j$ for all $j\in [K]$ are positive under the $\WD$ property. Their relations are depicted in Figure \ref{fig:preference}.
\begin{figure}[!ht]
	\centering
	\includegraphics[width=0.7\linewidth]{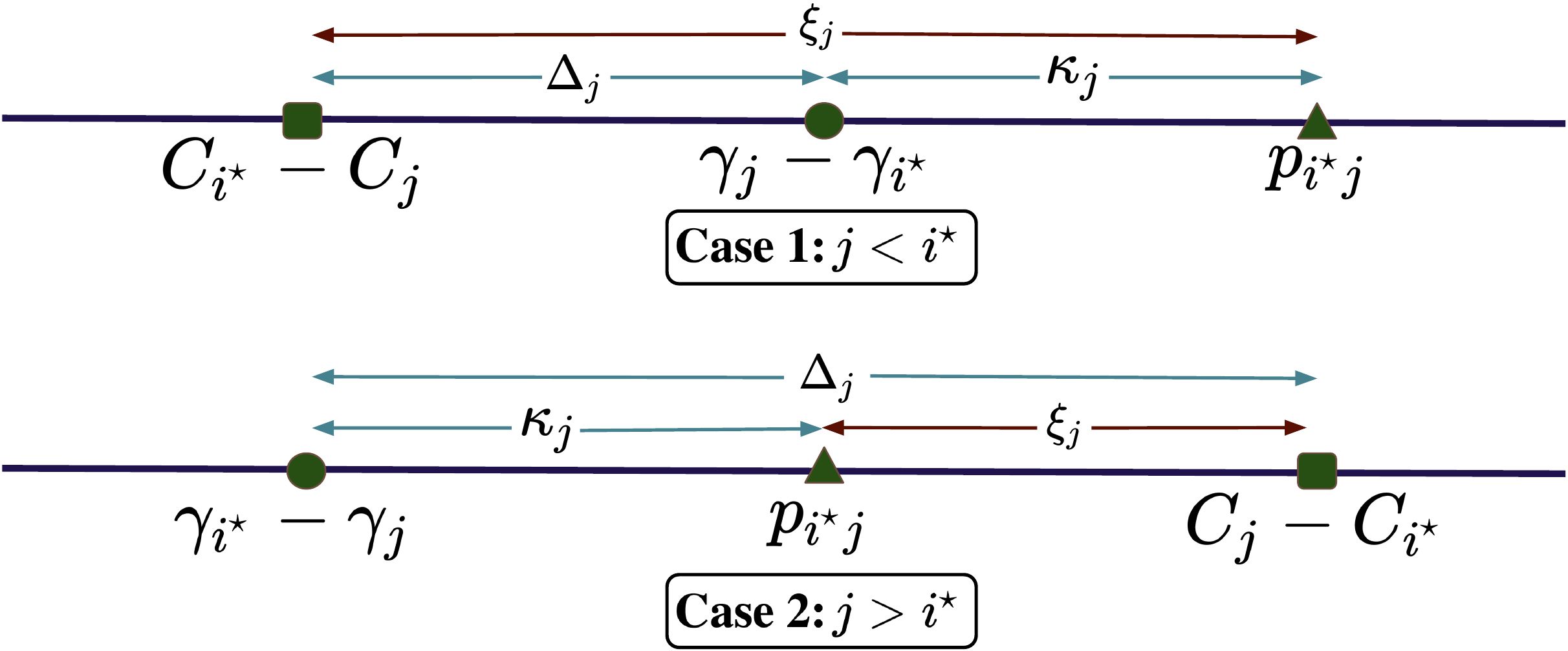}
	\caption{Relationship among $\Delta_j,\; \kappa_j$ and $\xi_j$ for a problem instance that satisfies WD property.}
	\label{fig:preference}
\end{figure}

Let $N_j(T)$ denote the number of times arm $j$ is selected until round $T$. The following proposition gives the mean number of times a sub-optimal arm is selected. 

\begin{restatable}{prop}{MeanPulls}
	\label{prop:meanpulls}
	Let  $f(t)$ be a  positive valued increasing function such that $C=\lim\limits_{T\rightarrow\infty}\sum\limits_{t=1}^T\dfrac{1}{f(t)^{2\alpha}}< \infty$ in \textnormal{\ref{alg:USS_WD}}. For any $P \in \PWD$, the mean number of times a arm $j \neq  i^\star$ is selected, is bounded as follows: 
	\begin{itemize}
		\item for any $j < i^\star$
		\begin{equation*}
			\EE{N_j(T)}  \le \dfrac{C}{2\xi_j^2},
		\end{equation*} 
		\item and for any $j>i^\star$
		\vspace{-1mm}
		\begin{equation*}
			\EE{N_j(T)}\le 1 + \frac{1}{\xi_j^2}\left(\alpha\log f(T) + \sqrt{\frac{\pi\alpha\log f(T)}{2}} + \frac{1}{2}\right).
		\end{equation*} 
	\end{itemize}
\end{restatable}
Notice that the mean number of times a arm $j < i^\star$ is selected, is finite. The regret bounds follows by noting that $\EE{\Regret_T}=\sum_{j < i^\star}\EE{N_j(T)}\Delta_j + \sum_{j > i^\star}\EE{N_j(T)}$ $\Delta_j$. Formally, we have the following regret bound. 
\begin{thm}
	\label{thm:regret}
	Let $f(t)$ be set as in \cref{prop:meanpulls}. Then, for any $P \in \PWD$, the expected regret of \textnormal{\ref{alg:USS_WD}} in $T$ rounds is bounded as below:
	\begin{align*}
		\EE{\Regret_T} \le \sum\limits_{j < i^\star} \dfrac{\Delta_j C}{2\xi_j^2} + \sum\limits_{j > i^\star} \Delta_j \Bigg[1 \;+ \frac{1}{\xi_j^2}\Bigg(\alpha\log f(T)+ \sqrt{\frac{\pi\alpha\log f(T)}{2}} + \frac{1}{2}\Bigg) \Bigg].
	\end{align*}
\end{thm}

\begin{proof}
	Let $I_t$ is the arm selected by algorithm at round $t$. Then expected regret $\EE{\Regret_T}$ of \ref{alg:USS_WD} for $T$ rounds is:
	\begin{align*} 
		\EE{\Regret_T} &= \EE{\sum\limits_{t=1}^T r_t\hspace{-0.4mm}} = \EE{\sum\limits_{t=1}^T\left(C_{I_j} + \gamma_{I_t} - \left(C_{i^\star} + \gamma_{i^\star}\right)\right)}	\\
	 	&= \EE{\sum\limits_{j\ne i^\star} \sum\limits_{t=1}^T \one{I_t=j} \Delta_j}=  \sum\limits_{j\ne i^\star} \sum\limits_{t=1}^T \Delta_j\Prob{I_t = j} \\
		&=\sum\limits_{j < i^\star}\sum\limits_{t=1}^T \Prob{I_t = j}  \Delta_j + \sum\limits_{j> i^\star} \sum\limits_{t=1}^T \Prob{I_t = j}  \Delta_j.
	\end{align*}
	Since $\sum_{t=1}^T \Prob{I_t = j} = \sum_{t=1}^T \EE{\one{I_t = j}} = \EE{N_j{T}}$, using \cref{prop:meanpulls},
	\begin{equation*}
		\EE{\Regret_T} \le \sum\limits_{j < i^\star} \dfrac{\Delta_j C}{2\xi_j^2} + \sum\limits_{j > i^\star} \Delta_j \Bigg(1 \;+ \frac{1}{\xi_j^2} \Bigg(\alpha\log f(T) + \sqrt{\frac{\pi\alpha\log f(T)}{2}} + \frac{1}{2}\Bigg)\Bigg). \tag*{\qedhere}
	\end{equation*}
\end{proof}

\begin{cor}
	\label{cor:fix_para}
	Let $\alpha=1$ and $f(t) = t$ in \cref{thm:regret}. Then, expected regret of \textnormal{\ref{alg:USS_WD}} for any $P \in \PWD$ in $T$ rounds is of $O\left(\sum\limits_{j>i^\star}\frac{\Delta_j\log T}{\xi^2}\right)$.
\end{cor}
\begin{proof}
	The regret contribution of arms whose index is smaller than the optimal arm is constant. With this fact and $\xi = \min\limits_{j>i^\star} \xi_j$, the proof follows from \cref{thm:regret} with $\alpha=1$ and $f(t)=t$.
\end{proof}

\begin{rem}
	The function $f(t)$ and $\alpha$ control the exploration in \ref{alg:USS_WD}. The larger the value, the more will be the exploration. The constraint given in \cref{prop:meanpulls} ensures that there is no linear term in the regret given in \cref{thm:regret} due to inefficient exploration. 
	% However, the choice of $f(t) = \exp(t)$ leads to more exploration due to exponential growth in exploration with time. It leads to less exploitation, and hence overall, it causes linear regret. We are not claiming the regret is sub-linear for any arbitrary function $f(t)$. We only give guarantee of sub-linear regret in some cases, e.g., Corollary 2, where $f(t)=t$ and $\alpha=1$.
\end{rem}

\begin{cor}
	\label{cor:sd}
	Let technical conditions stated in \cref{cor:fix_para} hold. Then expected regret of \textnormal{\ref{alg:USS_WD}} for any $P \in \PSD$  in $T$ rounds is of $O\left(\sum\limits_{j>i^\star}\frac{\log T}{\xi}\right)$.
\end{cor}

\begin{proof}
	Since $|\gamma_j - \gamma_{i^\star}| = \pijs$ for $P \in \PSD$, $\kappa_j = 0, \forall j \in [K]\Rightarrow \xi_j = \Delta_j$. Rest follows from Corollary \ref{cor:fix_para}. 
\end{proof}

\noindent
We next present problem independent bounds on the expected regret of \ref{alg:USS_WD}.
\begin{restatable}{thm}{probIndBound}
	\label{thm:prob_independent_bound}
	Let $f(t)$ be set as in \cref{prop:meanpulls}. The expected regret of \textnormal{\ref{alg:USS_WD}} in $T$ rounds 
	\begin{itemize}
		\item for any instance in $\PWD$ is bounded as
		\begin{align*}
			\EE{\Regret_T} \le 3\left(3\alpha K\log f(T)\right)^{1/3}T^{2/3}.
		\end{align*}
		\item for any instance in $\PSD$ is bounded as
		\begin{align*}
			\EE{\Regret_T} \le 4\left(\alpha KT\log f(T)\right)^{1/2}.
		\end{align*}
	\end{itemize}
\end{restatable}

\begin{cor}
	Let technical conditions stated in \cref{cor:fix_para} hold. The expected regret of \textnormal{\ref{alg:USS_WD}} on $\PSD$ is $\tilde{O}(T^{1/2})$ and on $\PWD$ it is $\tilde{O}(T^{2/3})$, where $\tilde{O}$ hides logarithmic terms.
\end{cor}
The proof of Theorem \ref{thm:prob_independent_bound} can be found in \cref{sec:uss_appendix}. We note that the above uniform bounds do not contradict Theorem $19$ in \cite{AISTATS17_hanawal2017unsupervised} which claimed non-existence of uniform bounds. The $\PWD$ condition considered in \cite{AISTATS17_hanawal2017unsupervised} incorrectly includes the class of instances satisfying $\rho=1$ which renders $\PWD$ not learnable, whereas in our definition of  $\PWD$ these instances are excluded and $\PWD$ is learnable.

\paragraph{Discussion on optimality of \ref{alg:USS_WD}:} Any partial monitoring problem can be classified as an `easy', `hard' or `hopeless' problem if it has  expected regret bounds of the order $\Theta(T^{1/2}), \Theta(T^{2/3})$ or $\Theta(T)$, respectively, and there exists no other class in between \cite{MOR14_bartok2014partial}. The class $\PSD$ is {regret equivalent} to a stochastic multi-armed bandit with side observations \cite{AISTATS17_hanawal2017unsupervised},  for which regret scales as $\Theta(T^{1/2})$, hence $\PSD$ resides in the easy class and our bound on it is optimal.  Since $\PWD \supsetneq \PSD$, $\PWD$ is not easy, and also $\PWD$ is learnable, it cannot be hopeless. Therefore, the class $\PWD$ is hard. We thus conclude that the regret bound of \ref{alg:USS_WD} is near-optimal in $T$. However, optimality concerning other leading constants (in terms of $K$) is to be explored further.

%% file: chapter2/uss_ts.tex
%!TEX root =  ../thesis.tex

Upper Confidence Bound (UCB) based methods are useful for dealing with the trade-off between exploration and exploitation in bandit problems \citep{ML02_auer2002finite, COLT11_garivier2011kl}. UCB has been widely used for solving various sequential decision-making problems. On the other hand, Thompson Sampling (TS) \citep{COLT12_agrawal2012analysis, ALT12_kaufmann2012thompson, AISTATS13_agrawal2013further} is an online algorithm based on Bayesian updates. TS selects an arm to play according to its probability of being the best arm, and it is shown that TS is empirically superior then UCB based algorithms for various MAB problems \citep{NIPS11_chapelle2011empirical}. TS also achieves lower bound for MAB when rewards of arms have Bernoulli distribution, as shown by \cite{ALT12_kaufmann2012thompson}.

In this section, we propose a TS based algorithm for the USS problem and show that it is a near-optimal algorithm. Using \cref{eq:selectDisProbLow} and \cref{eq:selectDisProbHigh}, our next result gives the optimal arm in  the USS problem that only uses one-sided test (\cref{eq:selectDisProbHigh}) for the arm's selection.
\begin{restatable}{lem}{SetBx}
	\label{lem:Bx}
	Let $P \in \PWD$ and $\cB = \left\{i: \forall j>i, C_j - C_i > \pij \right\}\cup \{K\}$. Then the arm $I_t =\min(\cB)$ is the optimal arm for the problem instance $P$.
\end{restatable}
\begin{proof}
	Let $\istar$ be an optimal arm for the problem instance $P$. Since $\pijs \doteq \Prob{\Yis \ne\Yj}$, we have $\forall j<\istar:\, C_{\istar} - C_j \le \Prob{\Yis \ne\Yj} \implies C_{\istar} - C_j \ngtr\Prob{\Yis \ne\Yj} \implies j \notin \cB, \forall j < \istar$.  If any sub-optimal arm $h \in \cB$ then the index of arm $h$ must be larger than the index of optimal arm $\istar$ in the cascade. Hence the element of the set $\cB$ in round $t$ is given as follows:
	\eqs{
		\cB = \{\istar, h_1, \ldots, h_t, K\},
	} 	
	where $\istar < h_1 < \cdots < h_t < K$.	By construction of set $\cB$, the minimum indexed arm in set $\cB$ is the optimal arm.
\end{proof}

\subsection{Algorithm: \ref{alg:TS_USS}}
We develop a Thompson Sampling based algorithm, named \ref{alg:TS_USS}, that uses \cref{lem:Bx} to select optimal arm. The algorithm works as follows: It sets the prior distribution of disagreement probability for each pair of arms as the Beta distribution, Beta$(1, 1)$, which is the same as Uniform distribution on $[0,1]$. The variable $S_{ij}$ represents the number of rounds when a disagreement is observed between arm $i$ and $j$. Whereas, the variable $F_{ij}$ represents the number of rounds when an agreement is observed. The variables $S_{ij}^{(t)}$ and $F_{ij}^{(t)}$ denote the values of $S_{ij}$ and $F_{ij}$ at the beginning of round $t$.
\begin{algorithm}[!ht]
	\renewcommand{\thealgorithm}{USS-TS}
	\floatname{algorithm}{}
	\caption{Thompson Sampling based Algorithm for Unsupervised Sequential Selection} 
	\label{alg:TS_USS}
	\begin{algorithmic}[1]
		\State Set $ \forall 1 \le i< j \le K: \mathcal{S}_{ij}^{(1)} \leftarrow 1, \mathcal{F}_{ij}^{(1)} \leftarrow 1$
		\For{$t=1,2,\ldots$}
			\State Set $i=1$ and $I_t=0$
			\While{$I_t = 0$}
				\State Play arm $i$
				\State $\forall j \in [i+1, K]:$ compute $\opijtst \leftarrow \mbox{Beta}(\mathcal{S}_{ij}^{(t)}, \mathcal{F}_{ij}^{(t)})$
				\State If $\forall j \in [i+1, K]: C_j - C_i > \opijtst$ or $i=K$ then set $I_t = i$ else  set $i=i+1$
			\EndWhile
		\State Select arm $I_t$ and observe $Y_t^1, Y_t^2, \dots, Y_t^{I_t}$
		\State $\forall 1\le i< j \le I_t:$ update $\mathcal{S}_{ij}^{(t+1)} \leftarrow\mathcal{S}_{ij}^{(t)}+ \one{\Yti \ne \Ytj}, \mathcal{F}_{ij}^{(t+1)} \leftarrow\mathcal{F}_{ij}^{(t)} + \one{\Yti = \Ytj}$
		\EndFor
	\end{algorithmic}
\end{algorithm}

In round $t$, the learner plays the arm $i=1$ and then observe its feedback. For each $(i,j)$ pair, a sample $\opijtst$ is independently drawn from Beta$(S_{ij}^{(t)}, F_{ij}^{(t)})$. Then algorithm checks whether the arm $i$ is the best arm using \cref{eq:selectDisProbHigh} with $\opijtst$ in place of $\pijt$. If the arm $i$ is not the best, then the algorithm plays the next arm, and the same process is repeated. If the arm $i$ is the best arm for the round $t$, then the algorithm stops at arm $I_t=i$ in the round $t$.

After selecting arm $I_t$, the feedback from arms $1, \ldots, I_t$ are observed, which is used to update the values of $S_{ij}^{(t+1)}$  and $F_{ij}^{(t+1)}$. The same process is repeated in the subsequent rounds.

\begin{rem}
	\ref{alg:TS_USS} is adapted for the USS problem from the Thompson Sampling algorithm for stochastic multi-armed bandits. However, the feedback structure and the way arms are selected in the USS setup differ from that in the stochastic multi-armed bandits.
\end{rem}

\subsection{Regret Analysis}
The following definitions and results are useful in subsequent proof arguments. 
\begin{defi}[Action Preference ($\succ_t$)]
	\label{def:arm_pref}
	\ref{alg:TS_USS} prefers the arm $i$ over arm $j$ in round $t$ if:
	\begin{subnumcases}	
	{i \succ_t j \doteq }
	\opji \ge C_i - C_j &\text{if } j<i \label{arm_pref_l} \\
	\opij < C_j - C_i&\text{if } j>i \label{arm_pref_h}
	\end{subnumcases}
\end{defi}

\begin{defi}[Transitivity Property]
	\label{def:trans_prop}
	If $i \succ_t j$ and $j \succ_t k$ then $i \succ_t k$.
\end{defi}

\begin{defi}
	Let $\Ht$ denote the $\sigma$-algebra generated by the history of selected arms and observations at the beginning of the time $t$ and given as follows:
	\begin{equation*}
		\Ht \doteq \left\{I_s, \left\{ Y_s^i \right\}_{i \le I_s}, s = 1, \ldots, t-1\right\},
	\end{equation*}
	where $I_s$ denotes the arm selected and set $\left\{ Y_s^i \right\}_{i \le I_s}$ denotes the observations from arm $1$ to $I_s$ in the round $s$. Define $\mathcal{H}_1 \doteq \{\}.$
\end{defi}

\begin{fact}[Beta-Binomial equality, Fact 1 in \cite{COLT12_agrawal2012analysis}]
	\label{fact:beta_binomial}
	Let $F_{\alpha,\beta}^{beta}(y)$ be the cumulative distribution function (cdf) of the beta distribution with integer parameters $\alpha$ and $\beta$. Let $F_{n,p}^B(\cdot)$ be the cdf of the binomial distribution with parameters $n$ and $p$. Then,
	\begin{equation*}
		F_{\alpha,\beta}^{beta}(y) = 1 - F_{\alpha+\beta-1,y}^B(\alpha-1).
	\end{equation*}
\end{fact}

\begin{lem}[Lemma 2 in \cite{AISTATS13_agrawal2013further}]
	\label{lem:betaExpBound}
	Let $n \ge 0$ and $\hat{\mu}_n$ be the empirical average of $n$ samples from Bernoulli($\mu$). Let $x< \mu$ and $q_n(x) \doteq 1 -  F_{n\hat{\mu}_n + 1, n(1-\hat{\mu}_n) + 1}^{beta}(x)$ be the probability that the posterior sample from the Beta distribution with its parameter $n\hat{\mu}_n + 1, n(1-\hat{\mu}_n) + 1$ exceeds $x$. Then,
	\begin{equation*}
		\EE{\frac{1}{q_n(x)}-1} \le
		\begin{cases}
			\frac{3}{\Delta(x)} & \mbox{if } n < 8/\Delta(x) \\
			\Theta\left(\exp^{-\frac{n\Delta(x)^2}{2}} + \frac{\exp^{-{nd(x,\mu)}}}{(n+1)\Delta(x)^2} + \frac{1}{\exp^{\frac{n\Delta(x)^2}{4}} - 1} \right) & \mbox{if } n \ge 8/\Delta(x),
		\end{cases}
	\end{equation*}
	where $\Delta(x) \doteq \mu - x$ and $d(x, \mu) \doteq x\log\left(\frac{x}{\mu}\right) + (1-x)\log\left(\frac{1-x}{1-\mu}\right)$.
\end{lem}

Recall that $\pijs$ is the disagreement probability between arm $i^\star$ and $j$ and $\opijsts$ is the sample of $\pijs$ using Beta distribution with the $t$ samples. 
Next, we bound the probability by which \ref{alg:TS_USS} selects the sub-optimal arm whose index is smaller than the optimal arm.
\begin{defi}
	\label{def:q}
	For any $j < i^\star$, define $q_{j,t}$ as the probability
	\begin{equation*}
		q_{j,t} \doteq \Prob{\opijsts \ge \pijs - \xi_j|\Ht}.
	\end{equation*}
\end{defi}

\begin{restatable}{lem}{relationLowerOptimal}
	\label{lem:relationLowerOptimal}
	Let $P \in \PWD$ and satisfies the transitivity property. If $j < i^\star$ then the probability by which \ref{alg:TS_USS} selects any sub-optimal arm $j$ over the optimal arm is given by
	\begin{equation*}
		\Prob{I_t =j, j< \istar|\Ht} \le \frac{(1-q_{j,t})}{q_{j,t}}\Prob{I_t \ge i^\star|\Ht}.
	\end{equation*}
\end{restatable}

\begin{restatable}{lem}{probLow}
	\label{lem:probLow}
	Let $P \in \PWD$ and satisfies the transitivity property. If $s$ be the number of times the sub-optimal arm $j$ is selected by \ref{alg:TS_USS} then, for any $j < i^\star$,
	\begin{equation*}
		\sum_{t=1}^T\Prob{I_t=j, j< \istar} \le \frac{24}{\xi_j^2} + \sum_{s \ge 8/\xi_j} \Theta\left(\exp^{-{s\xi_j^2}/{2}} + \frac{\exp^{-{sd( \pijs - \xi_j,\pijs)}}}{(s+1)\xi_j^2} + \frac{1}{\exp^{{s\xi_j^2}/{4}} - 1} \right).
	\end{equation*}
\end{restatable}
\begin{proof}\textbf{(sketch)}
	Using \cref{lem:relationLowerOptimal} and property of conditional expectations, we can have $\sum_{t=1}^T \Prob{I_t=j, j< \istar} = \sum_{t=1}^T \EE{ \Prob{jI_t=j, j< \istar|\Ht}}$. By using some simple algebraic manipulations on quantity $\sum_{t=1}^T  \EE{ \Prob{I_t=j, j< \istar|\Ht}}$ with \cref{lem:betaExpBound}, we can get the above stated upper bound.
\end{proof}

Our next result is useful to bound the probability by which \ref{alg:TS_USS} prefers the sub-optimal arms whose index is larger than the optimal arm.

\begin{restatable}{lem}{probHighPartTwo}
	\label{lem:probHighPart2}
	Let $\hpijst$ be the empirical estimate of $\pijs$ and $j>\istar$. Then, for any $x_j > \pijs$ and $y_j > x_j$,
	\begin{equation*}
		\sum_{t=1}^T \Prob{\hpijst \le x_j, \opijsts > y_j} \le \frac{\ln T}{d(x_j, y_j)} + 1.
	\end{equation*}
\end{restatable}

\begin{restatable}{lem}{probHighPartOne}
	\label{lem:probHighPart1}
	For any $x_j > \pijs$,
	\begin{equation*}
		\sum_{t=1}^T \Prob{\hpijst > x_j} \le \frac{1}{d(x_j, \pijs)}.
	\end{equation*}
\end{restatable}
\begin{proof}\textbf{(sketch)}
	This result is easily proved by using Chernoff-Hoeffding bound. See details in \cref{sec:uss_appendix}.
\end{proof}

\begin{restatable}{lem}{probHigh}
	\label{lem:probHigh}
	Let $P \in \PWD$. For any $\epsilon > 0$ and $j > i^\star$,
	\begin{equation*}
		\sum_{t=1}^T\Prob{j \succ_t i^\star, j> \istar} \le (1+\epsilon)\frac{\ln T}{d(\pijs, \pijs + \xi_j)} + O\left(\frac{1}{\epsilon^2}\right).
	\end{equation*}
\end{restatable}
\begin{proof}\textbf{(sketch)}
	Let $\pijs < x_j < y_j < \pijs + \xi_j$ where $j>\istar$. Then, it can be easily shown that $\sum_{t=1}^T\Prob{j \succ_t i^\star, j > \istar} \le  \sum_{t=1}^T \Prob{\hpijst \le x_j, \opijsts > y_j} + \sum_{t=1}^T \Prob{\hpijst > x_j}.$ The upper bound on first term of right hand side quantity is given by \cref{lem:probHighPart2} and the upper bound of the second term of right hand side quantity is given by \cref{lem:probHighPart1}.  Then, for $\epsilon \in (0,1)$ with suitable values of $x_j$ and $y_j$, we can get the above stated upper bound.
\end{proof}

Let $\Delta_j = C_j + \gamma_j - (C_{i^\star} + \gamma_{i^\star})$ be the sub-optimality gap for arm $j$. Now we state the problem dependent regret upper bound of \ref{alg:TS_USS}.
\begin{restatable}[Problem Dependent Bound]{thm}{depRegretBound}
	\label{thm:depRegretBound}
	Let $P \in \PWD$ and satisfies the transitivity property. If $\epsilon > 0$ then, the expected regret of \ref{alg:TS_USS} in $T$ rounds is bounded by
	\begin{equation*}
		{\Regret_T} \le \sum_{j > i^\star} \frac{(1+\epsilon)\ln T}{d(\pijs, \pijs + \xi_j)}\Delta_j + O\left(\frac{K-i^\star}{\epsilon^2}\right),
	\end{equation*}
\end{restatable}
\begin{proof}\textbf{(sketch)}
	Let $M_j(T)$ is the number of times arm $j$ is selected by \ref{alg:TS_USS}. Then, the regret of \ref{alg:TS_USS} is given by ${\Regret_T} = \sum_{j \in [K]}\EE{M_j(T)}\Delta_j = \sum_{j \in [K]}\sum_{t=1}^{T}\EE{\one{I_t = j}}\Delta_j = \sum_{j \in [K]}\sum_{t=1}^{T}\Prob{I_t = j}\Delta_j$. We divide the regret into two parts and it can be re-written as ${\Regret_T} \le \sum_{j < i^\star} \sum_{t=1}^{T} \Prob{I_t=j, j < i^\star} \Delta_j + \sum_{j > i^\star}\sum_{t=1}^{T}\Prob{I_t=j, j > i^\star} \Delta_j$. The first part of the regret is upper bounded by using  \cref{lem:probLow}. For the second part, when arm $I_t > \istar$  is selected, then there exists at least one arm $k > \istar$, which must be preferred over $\istar$. Using transitivity property and a recursive argument, we can show that the selected arm is preferred over the optimal arm.  Hence, $\sum_{j > i^\star}\sum_{t=1}^{T}\Prob{I_t=j, j > i^\star} \Delta_j$ can be upper bounded by $\sum_{j > i^\star} \sum_{t=1}^{T}$ $\Prob{j \succ_t \istar, j > i^\star} \Delta_j$. We can upper bound $\sum_{j > i^\star} \sum_{t=1}^{T}\Prob{j \succ_t \istar, j > i^\star} \Delta_j$ by using \cref{lem:probHigh} to get the above stated regret upper bound for \ref{alg:TS_USS}.
\end{proof}

\noindent
Next we present problem independent bounds on the regret of \ref{alg:TS_USS}.
\begin{restatable}[Problem Independent Bound]{thm}{indepRegretBound}
	\label{thm:indepRegretBound}
	Let $P \in \PWD$ and satisfies the transitivity property. Then the expected regret of \ref{alg:TS_USS} in $T$ rounds 
	\begin{itemize}
		\item for any instance in $\PSD$ is bounded as
		\begin{align*}
			{\Regret_T} \le O\left( \sqrt{KT\ln T}\right).
		\end{align*}
		\item 	for any instance in $\PWD$ is bounded as
		\begin{align*}
			{\Regret_T} \le O\left( \left(K\ln T\right)^{1/3}T^{2/3}\right).
		\end{align*}
	\end{itemize}
\end{restatable}
\begin{proof}\textbf{(sketch)}
	To get the above problem independent regret upper bound, we maximize the problem-dependent regret of \ref{alg:TS_USS} with respect to the value of $\xi_j$. 
\end{proof}

\begin{cor}
	Let $P \in \PWD$ and satisfies the transitivity property. Then the expected regret of \ref{alg:TS_USS} on $\PSD$ is $\tilde{O}(T^{1/2})$ and on $\PWD$ it is $\tilde{O}(T^{2/3})$, where $\tilde{O}$ hides $K$ and the logarithmic terms that are having $T$ in them.
\end{cor}

\noindent
\paragraph{Discussion on optimality of \ref{alg:TS_USS}:} As discussed in \cref{sec:uss_ucb}, the stochastic partial monitoring problems can be classified as an `easy,' `hard,' or `hopeless' problem with expected regret bounds of the order $\Theta(T^{1/2}), \Theta(T^{2/3})$, or $\Theta(T)$, respectively. And there exists no other class of problems in between \citep{MOR14_bartok2014partial}. After matching the regret bounds, we conclude that the regret bound of \ref{alg:TS_USS} is also near-optimal in $T$ up to a logarithmic term.

%% file: chapter2/experiment.tex
%!TEX root =  ../thesis.tex

In this section, we evaluate the performance of proposed algorithms on different problem instances derived from synthetic and two `real' datasets: PIMA Indians Diabetes \citep{UCI16_pima2016kaggale} and Heart Disease (Cleveland) \citep{HEART98_robert1988va, UCI17_Dua2017}. In our experiments, each arm is represented by a classifier that is arranged in order of their decreasing misclassification error, i.e., error-rate for each dataset. The cost of using a classifier is assigned based on its error-rate -- smaller the error-rate higher the cost. The case where arms' error-rate need not to decrease in the cascade is also considered.

\subsection{Datasets}
\paragraph{Synthetic Dataset:} 
We generate synthetic Bernoulli Symmetric Channel (BSC) dataset \citep{AISTATS17_hanawal2017unsupervised} as follows: The true binary feedback $Y_t$ is generated from i.i.d. Bernoulli random variable with mean $0.7$. The problem instance used in the experiment has three arms. We fix feedback as true binary feedback for the first arm with probability $0.6$, second arm with probability $0.7$, and third arm with probability $0.8$. To ensure strong dominance, we impose the condition during data generation. When the feedback of arm $1$ matches the true binary feedback, we introduce error up to 10\%  to the feedback of arm $2$ and $3$. We use five problem instances of the BSC dataset by varying the cumulative cost of playing the arms as given in \cref{table:bsc}.  

\begin{table}[!ht]
	\centering
	\setlength\tabcolsep{10pt}
	\setlength\extrarowheight{2pt}
	\begin{tabular}{ |c|c|c|c|c|} 
		\hline
		\multirow{1}{*}{\bf Values/ \newline Arms}&Arm $1$ &Arm $2$&Arm $3$& \multirow{1}{*}{\parbox{1.5cm}{~\\ $\WD$ Property}}\\ 
		\cline{1-4}
		Error-rate $(\gamma_i)$ & 0.3937 & 0.2899 & 0.1358 &  \multicolumn{1}{c|}{~}  \\
		\hline
		Instance 1 Costs& \textcolor{red}{\textbf{0.05}} & 0.285          & 0.45        &  \multicolumn{1}{c|}{\checkmark}  \\ 
		\hline
		Instance 2 Costs& 0.05         & \textcolor{red}{\textbf{0.1}} & 0.53      & \multicolumn{1}{c|}{\checkmark} \\ 
		\hline
		Instance 3 Costs& \textcolor{red}{\textbf{0.05  }} & 0.3           & 0.45      &    \multicolumn{1}{c|}{\checkmark}  \\ 
		\hline
		Instance 4 Costs& 0.05         & 0.25         & \textcolor{red}{\textbf{0.29}} &  \multicolumn{1}{c|}{\checkmark}\\ 
		\hline
		Instance 5 Costs& 0.1        & \textcolor{red}{\textbf{0.2}} & 0.41       &  \multicolumn{1}{c|}{\ding{53}} \\
		\hline
	\end{tabular}
	\caption{$\WD$ propoerty doesn't hold for Instance 5. Optimal arm's cost is in \textcolor{red}{\bf red bold} font.}
	\label{table:bsc}
\end{table}

\paragraph{Real Datasets:} An arm $i$ represents a classifier whose prediction is treated as the feedback of the arm $i$. The disagreement label for $(i,j)$ pair is computed using the labels of classifier (Clf.) $i$ and $j$. In Heart Disease dataset, each sample has $12$ features. We split the features into three subsets and train a logistic classifier on each subset. We associate 1st classifier with the first $6$ features as input, including cholesterol readings, blood sugar, and rest-ECG. The 2nd classifier, in addition to the $6$ features, utilizes the thalach, exang and oldpeak features, and the 3rd classifier uses all the features. In PIMA Indians Diabetes dataset, each sample has $8$ features related to the conditions of the patient. We split the features into three subsets and train a logistic classifier on each subset. We associate 1st classifier with the first $6$ features as input. These features include patient profile. The 2nd classifier, in addition to the $6$ features, utilizes the feature on the glucose tolerance test, and the 3rd classifier uses all the previous features and the feature that gives values of insulin test. The PIMA Indians Diabetes dataset has $768$ samples, whereas the Heart Disease dataset has only $297$ samples. As $10000$ rounds are used in our experiments, we select a sample from the original dataset in a round-robin fashion and give it as input to the algorithm. The details about the different costs used in five problem instances of the real datasets are given in \cref{table:real_datasets}.

\begin{table}[!ht]
	\centering
	\setlength\tabcolsep{8pt}
	\setlength\extrarowheight{2pt}
	\begin{tabular}{|c|c|c|c|c|c|c|c|}
		\hline
		\multirow{2}{*}{\parbox{3.22cm}{\bf Values/ Classifiers (Arms)}} &\multicolumn{3}{|c|}{\bf PIMA Indians Diabetes}&\multicolumn{3}{|c|}{\bf Heart Disease} &\multirow{3}{*}{\parbox{1.25cm}{$\WD$ Property}}\\ 
		\cline{2-7} 
		&Clf. 1 &Clf. 2&Clf. 3&Clf. 1 &Clf. 2&Clf. 3&\\ \cline{1-7}
		Error-rate ($\gamma_{i}$) & 0.3098 &0.233&0.2278&0.2929 &0.2025& 0.1483 &\\ 
		\hline
		Instance  1 Costs& \textcolor{red}{\textbf{0.05}}& 0.28& 0.45&\textcolor{red}{\textbf{0.02}}& 0.32& 0.45&\multicolumn{1}{|c|}{\checkmark}\\ 
		\hline
		Instance  2 Costs& 0.2& \textcolor{red}{\textbf{0.25}}& 0.269&0.2& \textcolor{red}{\textbf{0.25}}& 0.395&\multicolumn{1}{|c|}{\checkmark}\\ 
		\hline
		Instance  3 Costs& \textcolor{red}{\textbf{0.05}}& 0.309& 0.45&\textcolor{red}{\textbf{0.02}}& 0.34& 0.45&\multicolumn{1}{|c|}{\checkmark}\\ 
		\hline
		Instance  4 Costs& 0.2& 0.25& \textcolor{red}{\textbf{0.255}}&0.2& 0.25& \textcolor{red}{\textbf{0.3}}&\multicolumn{1}{|c|}{\checkmark}\\ 
		\hline
		Instance  5 Costs& \textcolor{red}{\textbf{0.05}}& 0.146& 0.3&0.2& \textcolor{red}{\textbf{0.25}}& 0.40&\multicolumn{1}{|c|}{\ding{53}}\\ 
		\hline
	\end{tabular}
	\caption{Costs of different problem instances which are derived from real datasets. $\WD$ property doesn't hold for Instance 5 and cost of optimal arm is in  \textcolor{red}{\bf red bold} font.}
	\label{table:real_datasets}
\end{table}

\paragraph{Verifying $\WD$ property:} The error-rate associated with each arm is known to us as given in  \cref{table:bsc} and \cref{table:real_datasets} (but note that the error-rates are unknown to the algorithm); hence we can find an optimal arm for a given problem instance. After knowing optimal arm, $\WD$ property is verified by using the disagreement probability estimates after $10000$ rounds.

\subsection{Experiment Results}
We fix the time horizon to $10000$ in all experiments and repeat each experiment $100$ times. The average regret is presented with a $95$\% confidence interval. The vertical line on each plot shows the confidence interval.

\noindent
\textbf{Comparison between different Algorithms:} 
In our first experiment, we compare the performance of \ref{alg:USS_WD} (with value of $\alpha=0.5$) and \ref{alg:TS_USS} with Algorithm 2 of \cite{AISTATS17_hanawal2017unsupervised} with value of $\alpha=1.5$ (as used in the paper) on Heart Disease and PIMA Indians Diabetes datasets. As expected, \ref{alg:TS_USS} outperforms other algorithms with large margins as shown in \cref{fig:dCompare} (PIMA Indians Diabetes dataset) and \cref{fig:hCompare} (Heart Disease dataset). %Due to better performance of \ref{alg:TS_USS} over \ref{alg:USS_WD}, the remaining experiments that are run for \ref{alg:TS_USS}.

\begin{figure}[!ht]
	\centering
	\begin{subfigure}[b]{0.40\linewidth}	\includegraphics[width=\linewidth]{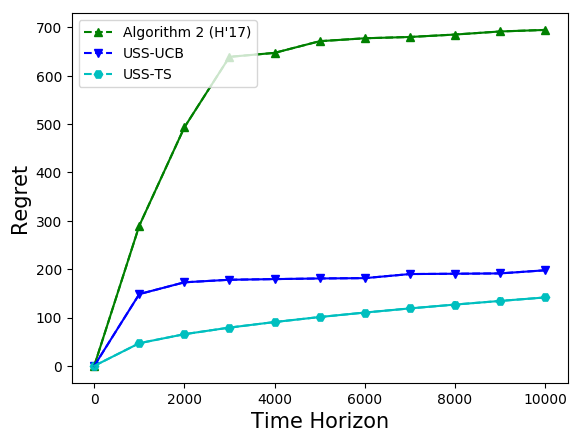}
		\caption{PIMA Indians Diabetes.}
		\label{fig:dCompare}
	\end{subfigure}\qquad
	\begin{subfigure}[b]{0.40\textwidth}
		\includegraphics[width=\linewidth]{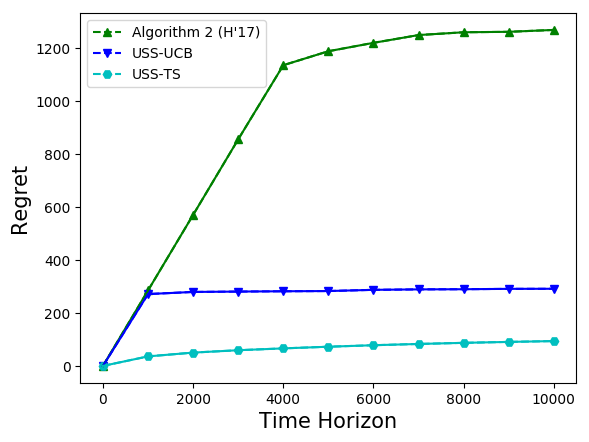}
		\caption{Heart Disease Dataset.}
		\label{fig:hCompare}
	\end{subfigure}
	\caption{Comparing regret of \ref{alg:USS_WD} and \ref{alg:TS_USS} with Algorithm 2 \citep{AISTATS17_hanawal2017unsupervised} for real datasets (\cref{fig:dCompare} and \cref{fig:hCompare}).}
	\label{fig:allExp}
\end{figure}

\noindent
\textbf{Regret versus Time Horizon:} The {\it regret} of \ref{alg:TS_USS} versus {\it Time Horizon} plots for the different problem instances derived from BSC Dataset and two real datasets are shown in \cref{fig:ts_bsc} and \cref{fig:regret_USS_TS} respectively. These plots verify that any instance that satisfies $\WD$ property has sub-linear regret. Note that the regret of \ref{alg:TS_USS} is linear for the Instance $5$ as it does not satisfy $\WD$ property.

\begin{figure}[!ht]
	\centering
	\begin{subfigure}[b]{0.40\textwidth}
		\includegraphics[width=\linewidth]{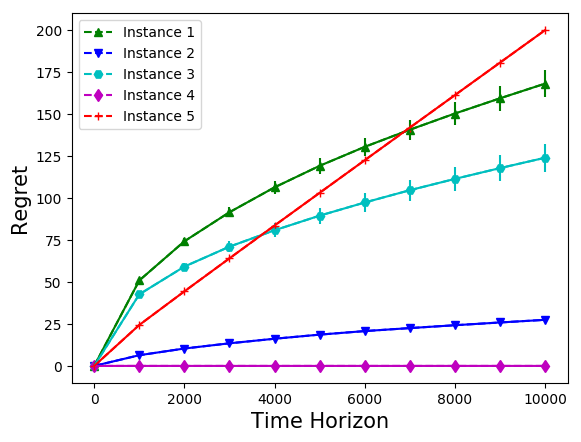}
		\caption{PIMA Indians Diabetes.}
		\label{fig:ts_pima}
	\end{subfigure}\qquad
	\begin{subfigure}[b]{0.40\textwidth}
		\includegraphics[width=\linewidth]{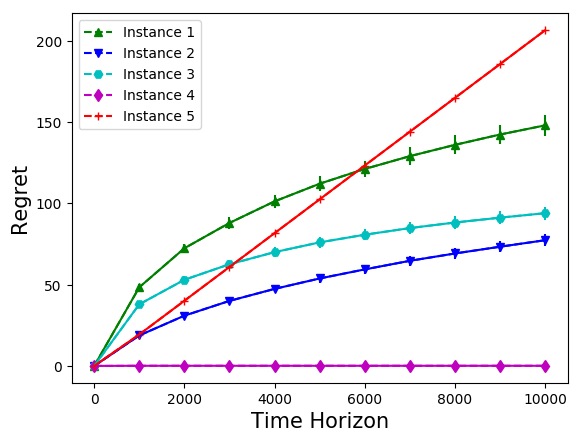}
		\caption{Heart Disease Dataset.}
		\label{fig:ts_hear}
	\end{subfigure}
	\caption{Regret of \ref{alg:TS_USS} for different problem instances derived from real datasets.}
	\label{fig:regret_USS_TS}
\end{figure}

\noindent
\textbf{Learnability v/s $\WD$ Property:} We experiment with different problem instances of the BSC dataset to know the relationship between regret of \ref{alg:TS_USS} and $\WD$ property. We fixed an optimal arm and vary the cumulative cost of using arms in such a way that we pass from the case where $\WD$ property does not hold ($\xi \le 0$ or $C_j - C_{i^\star} \in (\gamma_{i^\star} - \gamma_j, p_{i^\star j}]$  for any $j>i^\star$ where $\xi := \min_{j>i^\star} \xi_j$) to the situation where $\WD$ property holds $(\xi > 0$). 
When $\WD$ property does not hold for any problem instance, \ref{alg:TS_USS} treats a sub-optimal arm as the optimal arm. In such problem instances, as $C_j - C_{i^\star}$ increases, the regret will also increase due to selection of sub-optimal arm by \ref{alg:TS_USS} until $\WD$ property does not satisfy for that problem instance. When $\WD$ property does not satisfy for a problem instance then $C_j - C_{i^\star} \in (\gamma_{i^\star} - \gamma_j, p_{i^\star j}]$ holds in such cases, hence, it is easy to verify that $\xi$ can not be smaller than $-\max(p_{i^\star j}- (\gamma_{i^\star} - \gamma_j))$. 

\begin{figure}[!ht]
	\centering
	\begin{subfigure}[b]{0.40\linewidth}	\includegraphics[width=\linewidth]{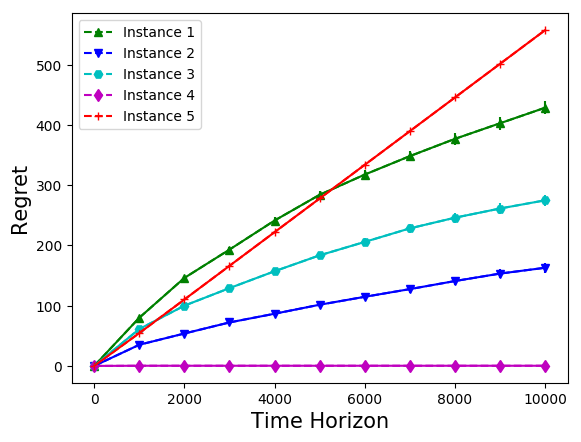}
		\caption{Synthetic BSC Dataset.}
		\label{fig:ts_bsc}
	\end{subfigure}\qquad
	\begin{subfigure}[b]{0.40\textwidth}
		\includegraphics[width=\linewidth]{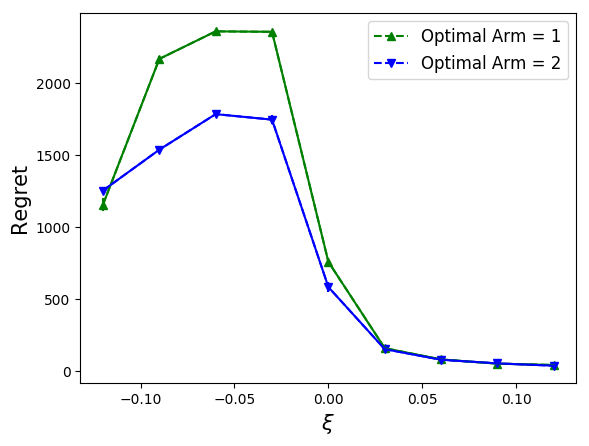}
		\caption{Synthetic BSC Dataset.}
		\label{fig:xi}
	\end{subfigure}
	\caption{Regret of \ref{alg:TS_USS} for different problem instances derived from synthetic BSC datasets. Regret behavior of \ref{alg:TS_USS} versus WD property for BSC Dataset is shown in \cref{fig:xi}.}
	\label{fig:BSC}
\end{figure}

We consider the problem instances with the minimum possible value of $\xi$ for which problem instance satisfies $\WD$ property. Then we increase the value of $\xi$ by increasing the cumulative cost of the arm. The regret versus $\xi$ plots for BSC Dataset is shown in \cref{fig:xi}. It can be observed that there is a transition at $\xi = 0$. Through our experiments, we show that the stronger the $\WD$ property (large value of $\xi$) for the problem instance, easier it is to identify the optimal arm and, hence the less regret is incurred by \ref{alg:TS_USS}.

%% file: chapter2/appendix.tex
%!TEX root =  ../thesis.tex

\subsection*{Missing proofs from \cref{sec:uss_learning}}

\SetB*
\begin{proof}
	Let $i^\star$ be an optimal sensor.
	Define 
	\begin{align}
	 	&\mathcal{B}^l :=\big\{i: i \in [2, K];\, \forall j<i \,\,\ni\,\, C_i - C_j \le \mathbb{P}\{ Y^i \ne Y^j \} \big\}  \cup \; \{1\} \label{equ:Bl} \\
	 	&\mathcal{B}^l_{-1} := \mathcal{B}^l\setminus\{1\} \\
		&\mathcal{B}^h := \big\{i: i \in [1, K-1];\, \forall j>i \,\,\ni\,\, C_j - C_i > \mathbb{P}\{ Y^i \ne Y^j \} \big\} \cup \; \{K\}\label{equ:Bh} \\
		&\mathcal{B}^h_{-K} := \mathcal{B}^h\setminus\{K\} \\
		&\mathcal{B} := \mathcal{B}^l \cap \mathcal{B}^h \label{equ:setBlh}
	\end{align}
	Consider following three cases:
	\begin{enumerate}[I.]
		\item $1 < i^\star < K$
		\item $i^\star = 1$
		\item $i^\star = K$
%		\item $i^\star$ is a cost-averse sensor
	\end{enumerate}
	\textbf{Case I:} $1 < i^\star < K$ \\As $i^\star$ is an optimal sensor therefore $\forall j>i^\star:\, C_j - C_{i^\star} > \mathbb{P}\{ Y^{i^\star} \ne Y^j \} \Rightarrow C_j - C_{i^\star} \nleq \mathbb{P}\{ Y^{i^\star}\ne Y^j \}\Rightarrow \forall j>i^\star \notin \mathcal{B}^l_{-1}$. If  any sensor $l \in \mathcal{B}^l_{-1}$ then $l\le i^\star$ i.e.,
	\begin{align}
		&\mathcal{B}^l_{-1} = \{l_1, l_2, \ldots, l_m, i^\star\} ~~~~\mbox{ where }1 < l_1 < \cdots < l_m < i^\star	\label{equ:setBl} \\
		&\mathcal{B}^l = \mathcal{B}^l_{-1} \cup \{1\} = \{1, l_1, l_2, \ldots,l_m, i^\star\}\label{equ:setBl1}
	\end{align} 
	Similarly, $\forall j<i^\star:\, C_{i^\star} - C_j \le \mathbb{P}\{ Y^{i^\star} \ne Y^j \} \Rightarrow C_{i^\star} - C_j \ngtr \mathbb{P}\{ Y^{i^\star}\ne Y^j \}\Rightarrow \forall j<i^\star \notin \mathcal{B}^h_{-K}$.  If any sensor $h \in \mathcal{B}^h_{-K}$ then $h \ge i^\star$ i.e.,
	\begin{equation}
		\label{equ:setBh}
		\hspace{-1mm}\mathcal{B}^h_{-K} = \{i^\star, h_1, \ldots, h_n\} \mbox{ where } i^\star < h_1 < \cdots < h_n < K	
	\end{equation} 
	\begin{equation}
		\label{equ:setBhk}
		\mathcal{B}^h = \mathcal{B}^h_{-K} \cup \{K\} = \{i^\star, h_1, h_2, \ldots, h_n, K\}	
	\end{equation} 

	From \eqref{equ:setBlh}, \eqref{equ:setBl1} and \eqref{equ:setBhk}, we get
	\begin{align}
		\mathcal{B} &= \mathcal{B}^l \cap \mathcal{B}^h \nonumber \\
		&= \{1, l_1, l_2, \ldots, i^\star\} \cap \{i^\star, h_1, h_2, \ldots, h_k, K\} \nonumber \\
		\Rightarrow \mathcal{B} &=\{i^\star\} \label{equ:setBi}
	\end{align}
	\textbf{Case II:} $i^\star = 1$ \\ 
	Using \eqref{equ:setBl}, we get $\mathcal{B}^l_{-1} = \phi$, hence $\mathcal{B}^l =\{1\}$. Similarly, using \eqref{equ:setBhk}, we have $\mathcal{B}^h = \{1, h_1, h_2, \ldots, h_n, K\}$ that implies
	\begin{align}
		\mathcal{B} = \{1\} 
		\Rightarrow \mathcal{B} = \{i^\star\} \label{equ:setBii} 
	\end{align}
	\textbf{Case III:} $i^\star = K$\\ 
	Using \eqref{equ:setBh}, we get $\mathcal{B}^h_{-K} = \phi$, hence $\mathcal{B}^h =\{K\}$. Similarly, using \eqref{equ:setBl1}, we have $\mathcal{B}^l = \{1, l_1, l_2, \ldots,l_m, K\}$ that implies
	\begin{align}
		\mathcal{B} = \{K\}  \Rightarrow \mathcal{B} = \{i^\star\} \label{equ:setBiii}
	\end{align}
%	~\\
%	\textbf{Case iv:} $i^\star$ is a cost-averse sensor 
%	
%	Assume that $i_1^\star$ and $ i_2^\star$ are optimal sensors i.e., $C_{i^\star_1} + \gamma_{i^\star_1}=C_{i^\star_2} + \gamma_{i^\star_2}$ and $i_1^\star \neq i_2^\star$. If $i^\star_1 \in \mathcal{B}$, we have   $C_{j_1} - C_{i^\star_1} \ge \mathbb{P}\{ Y^{i^\star_1} \ne Y^{j_1 }\}$ for all $i_1^\star < j_1$. Also, if $i^\star_2 \in \mathcal{B}$, we have  $C_{i^\star_2} - C_{j_2 }< \mathbb{P}\{ Y^{i^\star_2} \ne Y^{j_2} \}$ for all $j_2 < i^\star_2$. 
%	\\ \\
%	Now, setting $j_1=i_2^\star$ and $j_2=i_1^\star$ above, we have    $C_{i_2^\star} - C_{i^\star_1} \ge \mathbb{P}\{ Y^{i^\star_1} \ne Y^{i_2^\star}\}$ and $C_{i^\star_2} - C_{i_1^\star }< \mathbb{P}\{ Y^{i^\star_2} \ne Y^{i_1^\star} \}$.
%	\\ \\
%	WLOG, assume that $i_1^\star < i_2^*$,
%	\begin{align}
%		&C_{i_2^\star} - C_{i^\star_1} \ge \mathbb{P}\{ Y^{i^\star_1} \ne Y^{i_2^\star}\} \nonumber \\
%		\Rightarrow \;&C_{i^\star_2} - C_{i_1^\star } \nless \mathbb{P}\{ Y^{i^\star_2} \ne Y^{i_1^\star} \} \nonumber \\
%		\Rightarrow \;& i^\star_2 \notin \mathcal{B} \nonumber\\
%		\Rightarrow \;& \mathcal{B} =\{i^\star_1\} \nonumber\\
%		\mbox{Du}& \mbox{e to cost-averse nature of selection criteria, }i^\star= \min\{i^\star_1, i^\star_2\}   \nonumber \\
%		\Rightarrow \;& \mathcal{B} =\{i^\star\} \label{equ:setBiv}
%	\end{align}
	\eqref{equ:setBi},\eqref{equ:setBii} and \eqref{equ:setBiii} $\Rightarrow \mathcal{B}$ is a singleton set and contains the optimal sensor.
\end{proof}

%\begin{proof}
%	
%	 We prove the second part by contradiction.  Assume that $i_1^\star, i_2^\star \in \mathcal{B}$ and $i_1^\star \neq i_2^\star$. WLOG, let $i_1^\star < i_2^*$. Since $i^\star_1 \in \mathcal{B}$, we have   $C_{j_1} - C_{i^\star_1} > \Prob{ Y^{i^\star_1} \ne Y^{j_1 }}$ for all $i_1^\star < j_1$. Also, since $i^\star_2 \in \mathcal{B}$, we have  $C_{i^\star_2} - C_{j_2 }< \Prob{ Y^{i^\star_2} \ne Y^{j_2} }$ for all $j_2 < i^\star_2$. 
%	\noindent
%	Now, setting $j_1=i_2^\star$ and $j_2=i_1^\star$ above, we get    $C_{i_2^\star} - C_{i^\star_1} > \Prob{ Y^{i^\star_1} \ne Y^{i_2^\star}}$ and $C_{i^\star_2} - C_{i_1^\star }< \Prob{ Y^{i^\star_2} \ne Y^{i_1^\star} }$. Hence a contradiction.
%\end{proof}

The following definition is convenient for the proof arguments.
\begin{defi}[Action Preference ($\succ_t$)]
	\label{def:preference}
	The sensor $i$ is optimistically preferred over sensor $j$ in round $t$ if:
	\begin{subnumcases}	
	{\hspace{-1cm}i \succ_t j := }
	C_i - C_j \le \hpji +  \Psi_{ji}(t)\; \text{if } j<i \label{def_prefer_l} \\
	C_j - C_i > \hpij +  \Psi_{ij}(t)\; \text{if } j>i \label{def_prefer_h} 
	\end{subnumcases}
\end{defi}

\subsection*{Discussion of Remark \ref{rem:DecisionSet}}
The algorithm can converge to a sub-optimal sensor when we replace the term   $\hpijs + \Psi_{i^\star j}(t)$ in \eqref{set:eBh} by $\hpijs - \Psi_{i^\star j}(t)$. To verify this claim, assume algorithm selects sub-optimal sensor $j$ in around $t$ and $j<i^\star$ then, 
\begin{align*}
C_{i^\star} - C_j & > \hpijs - \Psi_{i^\star j}(t)
\intertext{Since sensor $i^\star$ is not used then there is no update in $\hat{p}_{i^\star j}(t+1)$ but by definition $\Psi_{i^\star j}(t+1) > \Psi_{i^\star j}(t)$ therefore,}
C_{i^\star} - C_j& > \hat{p}_{i^\star j}(t+1) - \Psi_{i^\star j}(t+1).
\end{align*}
Hence sub-optimal sensor will always be preferred over the optimal sensor in the subsequent rounds. This can be avoided by using UC term in \eqref{set:eBh} because,
\begin{align}
&\hat{p}_{i^\star j}(t+1) + \Psi_{i^\star j}(t+1) > \hpijs + \Psi_{i^\star j}(t) \nonumber\\
\intertext{The sub-optimal sensor $j$ will not be preferred after sufficient $n$ rounds,}
\Rightarrow\; & C_{i^\star} - C_j < \hat{p}_{i^\star j}(t+n) + \Psi_{i^\star j}(t+n).
\end{align}
As using LC term can make the decisions stuck to sub-optimal sensor, UC term is used in \eqref{set:eBh}.

\subsection*{Missing proofs from \cref{sec:uss_ucb}}

We first recall the standard Hoeffding's inequality \citep[Theorem 2]{JASA63_hoeffding1963probability} that we use in the proof.
\begin{thm}
	\label{def:che_hoeffiding}
	Let $X_1, \ldots,X_n$  be independent random variables with common range $[0,1]$,  $\mu = \EE{X_i}$, and $\hat{\mu}_n = \frac{1}{n}\sum_{t=1}^n X_t$. Then for all $\epsilon \ge 0$,
	\begin{subequations}
		\begin{align}
		&\Prob{\hat{\mu}_n - \mu \le -\epsilon} \le e^{-2n\epsilon^2} \label{def:ch_hf1} \\
		&\Prob{\hat{\mu}_n - \mu \ge \epsilon} \le e^{-2n\epsilon^2} \label{def:ch_hf2} 
		\end{align}
	\end{subequations}
\end{thm}

We need the following lemmas to prove the Proposition \ref{prop:meanpulls}.
\begin{lem}
	\label{lem:erf_integral}
	\begin{align}
		\operatorname{erf}(x) = \int_0^x \mathrm{e}^{-t^2}dt = \int \mathrm{e}^{-x^2}dx \label{equ:erf_integral}
	\end{align}
\end{lem}

\begin{proof}
	Leibniz's rule for $0< g(x)\le h(x)<\infty$,
	\begin{align*}
		\frac{d}{dx}\int_{g(x)}^{h(x)} &f(x,t) dt = f(x,h(x))\frac{dh(x)}{dx} - f(x,g(x))\frac{dg(x)}{dx} + \int_{g(x)}^{h(x)} \frac{\partial f(x,t)}{\partial x} dt
	\end{align*}
	Leibniz's integral rule without any common variable,
	\begin{align*}
		\frac{d}{dx}\int_{g(x)}^{h(x)} f(t) dt = f(h(x))\frac{dh(x)}{dx} - f(g(x))\frac{dg(x)}{dx}
	\end{align*}
	Using Leibniz's rule in \eqref{equ:erf_integral}, we get
	\begin{align*}
		&\frac{d}{dx}\int_0^x \mathrm{e}^{-t^2}dt = \mathrm{e}^{-x^2}\frac{dx}{dx} - 1\frac{d0}{dx}  = \mathrm{e}^{-x^2} \\
		\Rightarrow\; &{d}\int_0^x \mathrm{e}^{-t^2}dt  = \mathrm{e}^{-x^2} dx
	\end{align*}
	Integrating both side,
	\begin{align*}
		&\int{d}\int_0^x \mathrm{e}^{-t^2}dt = \int \mathrm{e}^{-x^2}dx\\
		\Rightarrow & \int_0^x \mathrm{e}^{-t^2}dt = \int \mathrm{e}^{-x^2}dx \tag*{\qedhere}
	\end{align*}
\end{proof}

\begin{lem}
	\label{lem:integral}
	Let $a,b,c,d \in \mathbb{R}^+$ and $t_c = b\sqrt{at} \pm \sqrt{acd}$. Then
	\begin{align*}
		\int \mathrm{e}^{-t_c^2}dt &= \mp \dfrac{\sqrt{{\pi}cd}\operatorname{erf}(t_c)}{\sqrt{a}b^2}  -\dfrac{\mathrm{e}^{-t_c^2}}{ab^2} + C \\
		&\mbox{where } \operatorname{erf}(x) = \dfrac{2}{\sqrt{\pi}}\int\mathrm{e}^{-x^2}dx ~~~~~ \operatorname{(using \; Lemma\; \ref{lem:erf_integral})}
	\end{align*}
\end{lem}

\begin{proof}
	Let $x = t_c \Rightarrow x = b\sqrt{at} \pm \sqrt{acd}$. Then,
	\begin{align*}
	t = \frac{(x \mp \sqrt{acd})^2}{ab^2}
	\end{align*}
	Now differentiate $x$ w.r.t. $t$,
	\begin{align*}
		\frac{dx}{dt} = \frac{b\sqrt{a}}{2\sqrt{t}} \Rightarrow dt = \frac{2\sqrt{t}}{b\sqrt{a}}dx  = \frac{2(x \mp \sqrt{acd})}{ab^2}dx
	\end{align*}
	By changing the variable from $t$ to $x$ in given integral,
	\begin{align*}
		\int \mathrm{e}^{-t_c^2}dt &= \int \mathrm{e}^{-x^2}\frac{2(x \mp \sqrt{acd})}{ab^2}dx \\
		\Rightarrow \int \mathrm{e}^{-t_c^2}dt&= \frac{2}{ab^2}\int x\mathrm{e}^{-x^2}dx \mp \frac{2\sqrt{cd}}{\sqrt{a}b^2}\int \mathrm{e}^{-x^2}dx \numberthis \label{equ:split_integral}
		\intertext{As $\int \mathrm{e}^{-cx^2}dx = \sqrt{\frac{\pi}{4c}} \operatorname{erf}(\sqrt{c}x) + C$ and  $\int x\mathrm{e}^{-cx^2}dx = -\frac{\mathrm{e}^{-cx^2}}{2c} + C$, then \eqref{equ:split_integral}  with $c=1$ is,}
		&= -\frac{\mathrm{e}^{-cx^2}}{ab^2} \mp \frac{\sqrt{\pi cd}\operatorname{erf}(x)}{\sqrt{a}b^2} + C\\
		\Rightarrow \int \mathrm{e}^{-t_c^2}dt &=  \mp \frac{\sqrt{\pi cd}\operatorname{erf}(t_c)}{\sqrt{a}b^2} -\frac{\mathrm{e}^{-t_c^2}}{ab^2} + C \tag*{\qedhere}
	\end{align*}
\end{proof}

\begin{lem}
	\label{lem:integral_m}
	Let $a,b,c,d \in \mathbb{R}^+$ and $t_0 = cdb^{-2}$. Then
	\begin{align*}
		\int^\infty_{t_0} \mathrm{e}^{-(b\sqrt{at} - \sqrt{acd})^2} dt = \dfrac{\sqrt{{\pi}cd}}{\sqrt{a}b^2} + \dfrac{1}{ab^2} 
	\end{align*}
\end{lem}

\begin{proof}
	Using Lemma \ref{lem:integral} with $t_c = b\sqrt{at} - \sqrt{acd}$. 
	\begin{align*}
	\int_{t_0}^{\infty} \mathrm{e}^{-t_c^2}dt &= \left.\left(\dfrac{\sqrt{{\pi}cd}\operatorname{erf}(t_c)}{\sqrt{a}b^2} -\dfrac{\mathrm{e}^{-t_c^2}}{ab^2}\right)\right\vert_{t_0}^\infty \\ 
	\intertext{Since $t_0 = cdb^{-2},\;t_c = 0$ for $t = t_0,\; \operatorname{erf}(0) = 0$ and $\operatorname{erf}(\infty) = 1$, we get }
	\Rightarrow\int_{t_0}^{\infty} \mathrm{e}^{-t_c^2}dt &= \dfrac{\sqrt{{\pi}cd}}{\sqrt{a}b^2} + \dfrac{1}{ab^2} = \frac{1}{b^2}\left(\sqrt{\dfrac{{\pi}cd}{a}} + \dfrac{1}{a}\right) \tag*{\qedhere}
	\end{align*}
\end{proof}

\begin{lem}
	\label{lem:sum_prob_bound_m}
	Let $b, c, d \in \R^+$, $\{X_t\}_{t\ge1}$ be a sequence of independent random variables, $\hat \mu_t = \frac{1}{t}\sum_{s=1}^t X_s$, and $\mu = \EE{X_t}$ where $X_t \in [0,1] ,\;\forall t$. Then
	\begin{align*}
		\sum_{t=1}^n \Prob{\hat{\mu}_t - \mu \ge b  -  \sqrt{\frac{cd}{t}}} \le   1  + \dfrac{1}{b^2} \left( cd + \sqrt{\frac{{\pi}cd}{2}} + \dfrac{1}{2}  \right)
	\end{align*}
\end{lem}

\begin{proof}
		Assume $t_0=\ceil{cdb^{-2}}$  and $t_0 << n$. We divide sum of the interest into two parts as:
		\begin{align*} 
			\sum_{t=1}^n \Prob{\hat{\mu}_t - \mu \ge b  -  \sqrt{\frac{cd}{t}}} = \sum_{t=1}^{t_0} \Prob{\hat{\mu}_t - \mu \ge b - \sqrt{\frac{cd}{t}}\hspace{-0.035cm}}  +  \sum_{t={t_0}}^n \Prob{\hat{\mu}_t - \mu \ge b - \sqrt{\frac{cd}{t}}\hspace{-0.035cm}} 
		\end{align*}
		As $\Prob{\text{any event}} \le 1$ and for $t>t_0$, $b - \sqrt{\frac{cd}{t}} > 0$. Using Hoeffding's inequality \eqref{def:ch_hf2}, we get
		\begin{align*} 
			\sum_{t=1}^n \Prob{\hat{\mu}_t - \mu \ge b  -  \sqrt{\frac{cd}{t}}} &\le \ceil{t_0} + \sum_{t=\ceil{t_0}}^n \mathrm{e}^{-2\left(b\sqrt{t} -\sqrt{cd}\right)^2} \\ 
			&\le 1 + \dfrac{cd}{b^2} + \int^\infty_{t_0}  \mathrm{e}^{-\left(b\sqrt{2t} -\sqrt{2cd}\right)^2}dt
			\intertext{Now using Lemma \ref{lem:integral_m} with $a= 2$,}
			\sum_{t=1}^n \Prob{\hat{\mu}_t - \mu \ge b  -  \sqrt{\frac{cd}{t}}}&\le 1 + \dfrac{1}{b^2} \left( cd + \sqrt{\frac{{\pi}cd}{2}} + \dfrac{1}{2} \right) 
			\tag*{\qedhere}
	\end{align*}
\end{proof}

\MeanPulls*
\begin{proof}
	Assume $N_T(j)$ be the number of times sensor $j$ is selected till $T$ rounds and $I_t$ be the sensor selected by algorithm at round $t$. Then mean number of pulls for any arm $j$ is: 
	\begin{equation*}
		\EE{N_T(j)} = \EE{\sum\limits_{t=1}^T \one{I_t=j}} = \sum\limits_{t=1}^T\Prob{I_t = j}
	\end{equation*}
	
We prove the proposition by considering the case $j<i^\star$ and $j>i^\star$ separately.

\begin{itemize}
	\item {\bf Case I:} $j < i^\star$\\
	If sensor $j$ is preferred over $i^\star$ at round $t$ then,
	\begin{align*}
	&C_{i^\star} - C_ j > \hpjis  +  \Psi_{ji^\star}(t)&& \left(\mbox{from \ref{def_prefer_h}}\right)
	\intertext{It is easy to verify that $C_{i^\star} - C_ j = \pijs - \xi_j$. %(Figure \ref{fig:preference}). 
	By definition, $\hpjis = \hpijs$ and $\Psi_{ji^\star}(t) = \Psi_{i^\star j}(t)$,}
	&\hpijs +  \Psi_{i^\star j}(t) < \pijs - \xi_j && \left(\mbox{from \ref{def_delta}, \ref{def_kappa_l} and \ref{def_xi_l}}\right)
	\end{align*}
	If algorithm selects sensor $j$ in round $t$ then it is preferred over an optimal sensor in that round, i.e.,
	\begin{align*}
	\Prob{I_t = j} &= \Prob{I_t = j, j \succ_t i^\star} \le \Prob{j \succ_t i^\star} = \Prob{\hpijs + \sqrt{\frac{\alpha\log{f(t)}}{\Nijms}} < \pijs - \xi_j} 
	\end{align*}
	As $\Nijms$ is a random variable,  Hoeffding's inequality (\ref{def:ch_hf1}) cannot be directly used here. Let	$\hpijx$ denote the value of $\hpijs$ when $\Nijms =s$. Then, we get  \\
	\begin{align*}
	\Prob{I_t = j} &\le \sum_{s=1}^t \Prob{\hpijx+ \sqrt{\frac{\alpha \log f(t)}{s}} \leq \pijs - \xi_j} \\ 
	&= \sum_{s=1}^t \Prob{\hpijx - \pijs \leq - \left(\xi_j + \sqrt{\frac{\alpha \log f(t)}{s}}\right) } 
	\end{align*}

	Now using Hoeffding's inequality (\ref{def:ch_hf1}),
	\begin{align*}
	\Rightarrow\Prob{I_t = j} &\le \sum_{s=1}^t \mathrm{e}^{-2s\left(\xi_j + \sqrt{\frac{\alpha \log f(t)}{s}}\right)^2} \le \sum_{s=1}^t \left(\mathrm{e}^{-2\xi_j^2s}\;\mathrm{e}^{- 2\alpha \log f(t)}\right) \\ 
	& \le\int_{0}^t \dfrac{\mathrm{e}^{-2\xi_j^2s}}{f(t)^{2\alpha}}ds \le\int_{0}^\infty \dfrac{\mathrm{e}^{-2\xi_j^2s}}{f(t)^{2\alpha}}ds \\
	&\le \left.\left(\frac{\mathrm{e}^{-2\xi_j^2s}}{-2f(t)^{2\alpha}\xi_j^2}\right)\right\vert_{0}^\infty = \dfrac{1}{2\xi_j^2f(t)^{2\alpha}} 
	\end{align*}
	The mean number of time a sub-optimal sensor selected in $T$ rounds is:
	\begin{align*}
		\EE{N_T(j)} = \sum\limits_{t=1}^T\Prob{I_t=j}& \leq \sum\limits_{t=1}^T\dfrac{1}{2\xi_j^2f(t)^{2\alpha}} \leq \dfrac{1}{2\xi_j^2} \sum\limits_{t=1}^\infty\dfrac{1}{f(t)^{2\alpha}} = \dfrac{C}{2\xi_j^2}.
	\end{align*}

	\item {\bf Case II:} $j> i^\star$ \\
	If sensor $j$ is preferred over $i^\star$ at round $t$ then,
	\begin{align*}
	&C_j - C_{i^\star} \le \hpijs  +  \Psi_{i^\star j}(t) && \left(\mbox{\hfill from \ref{def_prefer_l}}\right) \\
	\Rightarrow\;\;&\hpijs +  \Psi_{i^\star j}(t) \ge \pijs + \xi_j && \left(\mbox{from \ref{def_delta}, \ref{def_kappa_h} and \ref{def_xi_h}}\right)
	\end{align*}
	 The mean number of time a sub-optimal sensor selected in $T$ rounds is given by:
	\begin{align*}
		\EE{N_T(j)} &= \sum\limits_{t=1}^T \Prob{I_t =j} = \sum\limits_{t=1}^T \Prob{j \succ_t i^\star, I_t = j} \le \sum\limits_{t=1}^T \Prob{j \succ_t i^\star}\\
		&=  \sum\limits_{t=1}^T\Prob{\hpijs + \sqrt{\frac{\alpha\log{f(t)}}{\Nijms}} \ge \pijs + \xi_j,\; I_t=j } \\
		&\le \sum\limits_{t=1}^T\Prob{\hpijs +  \sqrt{\frac{\alpha\log{f(T)}}{\Nijms}} \ge \pijs + \xi_j,\; I_t=j} \\
		\intertext{As $\Prob{A\cap B} \le \min\left\{\Prob{A}, \Prob{B}\right\} \Rightarrow \Prob{A\cap B} \le \Prob{A}$ or $\Prob{A\cap B} \le \Prob{B}$, we get}
		&\le \sum\limits_{s=1}^T\Prob{\hpijx +  \sqrt{\frac{\alpha\log{f(T)}}{s}} \ge \pijs + \xi_j}\\
		&= \sum\limits_{s=1}^T\Prob{\hpijx  - \pijs \ge \xi_j - \sqrt{\frac{\alpha\log{f(T)}}{s}}}.
		\intertext{Using Lemma \ref{lem:sum_prob_bound_m}  with $b = \xi_j$, $c=\alpha$, $d=\log{f(T)}$, we get}
		\EE{N_T(j)} &\le 1 + \frac{1}{\xi_j^2}\left(\alpha\log f(T) + \sqrt{\frac{\pi\alpha\log f(T)}{2}} + \frac{1}{2}\right).	\tag*{\qedhere}
	\end{align*}
\end{itemize}
\end{proof}

%\begin{rem}
%	If $f(t) = t$ then $\lim\limits_{T\rightarrow\infty}\sum\limits_{t=1}^T\dfrac{1}{t^{2\alpha}}$ is a Riemann zeta function ($\zeta(2\alpha)$) and converges for any $\alpha>0.5$. The value of $C = \pi^2/6$ for $\alpha = 1$. As $alpha$ increases, the value of $C$ decreases and for sufficient large $\alpha$, $C$ becomes so small that \textnormal{\ref{alg:USS_WD}} will not select any sensor $j<i^\star$.
%\end{rem}

\probIndBound*
\begin{proof}
	Let $N_j(T)$ is the number of times sensor $j$ selected in $T$ rounds. Then expected cumulative regret of \ref{alg:USS_WD} for any instance $\theta \in \TSA$ is:
	\begin{align*}
		\EE{\Regret_T} &= \sum\limits_{j\neq i^\star} \EE{N_j(T)}\Delta_j \\
		&= \sum\limits_{j<i^\star}\EE{N_j(T)}\Delta_j + \sum\limits_{j>i^\star}\EE{N_j(T)}\Delta_j 
	\end{align*}
	Now, using the fact that for $j<i^\star$, $\Delta_j = \xi_j - \kappa_j $ and for $j>i^\star$, $\Delta_j = \xi_j + \kappa_j $, we get
	\begin{align*}
		\EE{\Regret_T} &= \sum\limits_{j<i^\star}\EE{N_j(T)}(\xi_j - \kappa_j) + \sum\limits_{j>i^\star}\EE{N_j(T)}(\xi_j + \kappa_j) \\
		& \leq \sum\limits_{j<i^\star}\EE{N_j(T)}\xi_j + \sum\limits_{j>i^\star}\EE{N_j(T)}(\xi_j + \kappa_j) \\
		\Rightarrow \EE{\Regret_T} & \leq \underbrace{\sum\limits_{j<i^\star}\EE{N_j(T)}\xi_j}_{\EE{\Regret^1_T}} + \underbrace{\sum\limits_{j>i^\star}\EE{N_j(T)}(\xi_j + \beta)}_{\EE{\Regret^{2}_T}} \numberthis \label{equ:regret_def},
	\end{align*} 
	where 
	\begin{equation}
		\forall j,\; \kappa_j \le \beta \;\;\begin{cases}
			=0  \quad \mbox{if }   P \in  \PSD \\
			\leq 1  \quad \mbox{if }   P \in \PWD \mbox{ and sensors are ordered by their error-rate} \\
			\leq 2 \quad \mbox{if } P \in \PWD \mbox{ and sensors are arbitrarily ordered by their error-rates} 
	  	\end{cases}
	\end{equation}
	
	In the following, we set $\xi = \min\limits_{j>i^\star} \xi_j$.  Now we consider the case  $\xi \ge 1$ and $\xi < 1$ separately.
	
	\begin{itemize}
		\item {\bf Case I: }$\xi \ge 1 \Leftrightarrow \forall j > i^\star,\; \xi_j \ge 1$ \\		
		Using Proposition \ref{prop:meanpulls} to upper bound $\EE{\Regret^2_T}$, we get
		\begin{align*}
		\EE{\Regret^2_T}  &\le \sum\limits_{j > i^\star}\Bigg(1 \;+ \frac{1}{\xi_j^2}\Bigg(\alpha\log f(T) \;+ \sqrt{\frac{\pi\alpha\log f(T)}{2}} + \frac{1}{2} \Bigg)\Bigg)(\xi_j + \beta) \\
		&\le K\Bigg((\xi_j + 2)+ \Bigg(\alpha\log f(T) + \sqrt{\frac{\pi\alpha\log f(T)}{2}} + \frac{1}{2} \Bigg)\left(\frac{1}{\xi_j} + \frac{\beta}{\xi_j^2}\right)\Bigg) \\
		\Rightarrow \EE{\Regret^2_T}  &\le K(\xi_j + 2)+ K\Bigg(\alpha\log f(T) + \sqrt{\frac{\pi\alpha\log f(T)}{2}} + \frac{1}{2} \Bigg)\left(\frac{1}{\xi} + \frac{\beta}{\xi^2}\right) \numberthis \label{equ:upper_regret}
		\end{align*} 
		For $\xi \ge 1$, $\left(\frac{1}{\xi} + \frac{\beta}{\xi^2}\right) \le \beta + 1$. Further, by definition $\forall j>i^\star,\;\; \xi_j = C_j - C_{i^\star} - \Prob{\Yis \ne \Yj}$, one can easily verify that $\max\limits_{j > i^\star} \xi_j \le C_K - C_1$ as $\Prob{\Yis \ne \Yj} \ge 0$. Assume $C_K - C_1 \le C_1^K$, then \eqref{equ:upper_regret} can be written as:
		\begin{align*}
			\EE{\Regret^2_T} &\le K(C_1^K + 2)+ (\beta + 1)K\Bigg(\alpha\log f(T) + \sqrt{\frac{\pi\alpha\log f(T)}{2}} + \frac{1}{2} \Bigg)	\\			
			&\le \frac{(2C_1^K + \beta + 5)K}{2} + (\beta + 1)K\Bigg(\alpha\log f(T) + \sqrt{\frac{\pi\alpha\log f(T)}{2}}\Bigg) \numberthis \label{equ:upperRegret1}
		\end{align*}
		
		For any $\xi^\prime$, $\EE{\Regret_T^1}$ can written as:
		\begin{align*}
			\EE{\Regret_T^1}  &= \sum\limits_{\substack{\xi^\prime > \xi_j\\ j < i^\star}} \EE{N_j(T)}\xi_j + \sum\limits_{\substack{\xi^\prime < \xi_j\\ j < i^\star}} \EE{N_j(T)}\xi_j  \\
			& \le \sum\limits_{j < i^\star} \EE{N_j(T)}\xi^\prime + \sum\limits_{\substack{\xi^\prime < \xi_j\\ j < i^\star}} \frac{C}{2\xi_j^2}\xi_j ~~~~~\mbox{(using Proposition \ref{prop:meanpulls})} \\
			\Rightarrow \EE{\Regret_T^1} &\le T\xi^\prime + \frac{CK}{2\xi^\prime} ~~~~~\left(\mbox{since } \sum\limits_{j < i^\star} \EE{N_j(T)} \le T\right) \numberthis \label{equ:lowerRegret}
		\end{align*}
		
		From definition, $\forall j < i^\star,\; \xi_j = \Prob{\Yis \ne \Yj} - (C_{i^\star} - C_j)$. As sensors are ordered by increasing cost, one can verify that $\xi_j < 1$ for $P \in \PWD$. With this fact, by combining \eqref{equ:upperRegret1} and \eqref{equ:lowerRegret}, \eqref{equ:regret_def} can be written as:
		\begin{align*}
			\EE{\Regret_T} \le T\xi^\prime + \frac{CK}{2\xi^\prime} + \frac{(2C_1^K + \beta + 5)K}{2} + 3K\Bigg(\alpha\log f(T) + \sqrt{\frac{\pi\alpha\log f(T)}{2}}\Bigg) 
			\intertext{Choose $\xi^\prime = \sqrt{\frac{CK}{T}}$ which maximize the upper bound and we get,}
			\EE{\Regret_T} \le \sqrt{2CKT} + \frac{(2C_1^K + \beta + 5)K}{2} + 3K\Bigg(\alpha\log f(T) + \sqrt{\frac{\pi\alpha\log f(T)}{2}}\Bigg) \numberthis \label{equ:regret_xiM1}
		\end{align*}

		\item {\bf Case II: }$\xi < 1$  \\
		Assume $T \ge T_0$ for $j>i^\star$ such that
		\begin{align}
			1+\frac{1}{\xi_j^2} \left( \alpha\log f(T) + \sqrt{\frac{\alpha \pi \log f(T)}{2}} +\frac{1}{2} \right) \leq \frac{2 \alpha \log f(T)}{\xi_j^2} \label{equ:min_T}
		\end{align}
		For $\alpha =1$ and $T_0  = 56$ \eqref{equ:min_T} holds for all $T \ge T_0$. Let $0<\xi^\prime<\xi$. Then  $\EE{\Regret^2_T}$ can be written as:
		\begin{align*}
			\EE{\Regret_T} & \leq \sum\limits_{\substack{\xi^\prime > \xi_j\\ j < i^\star}} \EE{N_j(T)}\xi_j + \sum\limits_{\substack{\xi^\prime < \xi_j\\ j < i^\star}} \EE{N_j(T)}\xi_j + \sum\limits_{\substack{\xi^\prime > \xi_j\\ j > i^\star}} \EE{N_j(T)}(\xi_j+\beta) \\
			&\qquad \qquad + \sum\limits_{\substack{\xi^\prime < \xi_j\\ j > i^\star}} \EE{N_j(T)}(\xi_j+\beta) 
		\end{align*}
		Since $\sum\limits_{\substack{\xi^\prime > \xi_j\\ j < i^\star}} \EE{N_j(T)} \le T$ and for every $j>i^\star, \; \xi_j > \xi^\prime$. Using Proposition \ref{prop:meanpulls} and \eqref{equ:min_T}, we get
		\begin{align*}
			\EE{\Regret_T} &  \le T\xi^\prime + \sum\limits_{\xi^\prime < \xi_j} \frac{C}{2\xi_j^2}\xi_j + \sum\limits_{\xi^\prime < \xi_j} \frac{2 \alpha \log f(T)}{\xi_j^2}(\xi_j+\beta)\\
			& \le  T\xi^\prime + \frac{CK}{2\xi^\prime} + 2\alpha K\log f(T)\left(\frac{1}{\xi^\prime} + \frac{\beta}{{\xi^\prime}^2}\right)
			\intertext{As $C = \lim\limits_{T \rightarrow \infty}\sum\limits_{t=1}^T\frac{1}{t^{2\alpha}}$, one can verify that for $T_0 = 56$ and $\alpha=1$, $C < 2\alpha\log f(T)$ holds. Using this fact,}
			&\le  T\xi^\prime +  4\alpha K\log f(T)\left(\frac{1}{\xi^\prime} + \frac{\beta}{{\xi^\prime}^2}\right) \\
			\EE{\Regret_T} &\le T\xi^\prime +  4\alpha K\log f(T)\left(\frac{1}{\xi^\prime} + \frac{\beta}{{\xi^\prime}^2}\right)	\numberthis \label{equ:regret_xiL1} \\
		\end{align*}		
	\end{itemize}

	We first consider $\PWD$ class of problems. For $\xi^\prime < 1$ and $\beta \le 2$, we have $\left(\frac{1}{\xi^\prime} + \frac{\beta}{{\xi^\prime}^2}\right) \le \frac{\beta + 1}{{\xi^\prime}^2}\le \frac{3}{{\xi^\prime}^2}$. Then
	\begin{align*}
		\EE{\Regret_T} &\le T\xi^\prime + \frac{12\alpha K\log f(T)}{{\xi^\prime}^2}
		\intertext{Choose $\xi^\prime = \left( \frac{24\alpha K\log f(T)}{T}\right)^{1/3}$ which maximize above upper bound and we get,}
		\Rightarrow \EE{\Regret_T} &\le \left(24\alpha K\log f(T)\right)^{1/3}T^{2/3} + \frac{\left(24\alpha K\log f(T)\right)^{1/3}}{2}T^{2/3} \\
		&\le 2\left(3\alpha K\log f(T)\right)^{1/3}T^{2/3} + \left(3\alpha K\log f(T)\right)^{1/3}T^{2/3} \\
		\Rightarrow \EE{\Regret_T} &\le 3\left(3\alpha K\log f(T)\right)^{1/3}T^{2/3} \numberthis \label{equ:regret_xi_wd}
	\end{align*}		
	As $C<2\alpha\log f(T)$ and $K<<T$, it is clear that upper bound in \eqref{equ:regret_xi_wd} is worse than \eqref{equ:regret_xiL1}. Hence it completes our proof for the case when any problem instance belongs to $\PWD$.
	
	Now we consider any problem instance $P \in \PSD$. For any $P \in \PSD \Rightarrow \forall j \in [K],\; \kappa_j = 0 \Rightarrow \beta = 0$. Hence \eqref{equ:regret_xiL1} can be written as
	\begin{align*}
		\EE{\Regret_T} &\le T\xi^\prime +  \frac{4\alpha K\log f(T)}{\xi^\prime}
		\intertext{Choose $\xi^\prime = \left( \frac{4\alpha K\log f(T)}{T}\right)^{1/2}$ which maximize above upper bound and we get,}
		\Rightarrow \EE{\Regret_T} &\le 2\left(4\alpha KT\log f(T)\right)^{1/2} = 4\left(\alpha KT\log f(T)\right)^{1/2} \numberthis \label{equ:regret_xi_sd}
	\end{align*}
	As $C<2\alpha\log f(T)$ and $K<<T$ then upper bound of expected regret in \eqref{equ:regret_xiL1} is $3\left(\alpha KT\log f(T)\right)^{1/2}$ which is better than \eqref{equ:regret_xi_sd}. It complete proof for second part of Theorem \ref{thm:prob_independent_bound}.
\end{proof}

\PropCostRangeHigh*
\begin{proof}
	Assume that $C_j-C_i >  \max \{0, \gamma_i-\gamma_j \}$. Since $C_j -C_i \notin \left(\max\{0, \gamma_i-\gamma_j\}, \Prob{Y^i \neq Y^j}\right]$, we get $C_j-C_i > \Prob{Y^j \neq Y^i}$.
	
	The other direction follows by noting that  $ \Prob{Y^j \neq Y^i}\geq  \max \{0, \gamma_i-\gamma_j \}$.
\end{proof}

\PropCostRangeLow*
\begin{proof}
	Assume that $C_i-C_j \leq  \max\{0, \gamma_j-\gamma_i\} $. Since $\max\{0, \gamma_j-\gamma_i\} \leq \Prob{Y^i \neq Y^j}$, we get $C_j-C_i \leq \Prob{Y^i \neq Y^j}$.
	
	The condition $C_j-C_i \leq \Prob{Y^i \neq Y^j}$ along with $	C_i -C_j \notin \left (\max\{0, \gamma_j-\gamma_i\}, \Prob{Y^i \neq Y^j}\right ]$ implies the other direction, i.e., $C_i-C_j \leq  \max\{0, \gamma_j-\gamma_i\} $. 
\end{proof}

\PropWDUnorder*
\begin{proof}
	From Proposition \ref{prop:CostRange1} and \ref{prop:CostRange2}, if the optimal sensor satisfies
	for $j> i^\star$
	\[C_j -C_{i^\star} \notin \left (\max\{0, \gamma_i-\gamma_j\}, \Prob{Y^{i^\star} \neq Y^j}\right] \]
	and for $j< i^\star$
	\[C_{i^*} -C_j \notin \left (\max\{0, \gamma_j-\gamma_i\}, \Prob{Y^{i^\star}\neq Y^j}\right ],\]
	
	Then, for $j > i^\star, C_j -C_{i^\star}> \gamma_{i^\star} - \gamma_j$	iff 	$C_j -C_{i^\star} >\Prob{Y^{i^\star} \neq Y^j}$ \\
	and for 
	$j < i^\star, C_{i^\star}-C_j\leq  \gamma_{j} - \gamma_{i^\star}$	iff 	$C_j -C_{i^\star} \leq \Prob{Y^{i^\star} \neq Y^j}$. Hence we can use $\Prob{Y^i \neq Y^j}$ as a proxy for $\gamma_i -\gamma_j$ to make decision about the optimal arm.
	
	Now notice that for $j < i^\star$, $C_{i^\star}-C_j \leq \gamma_{j} -\gamma_{i^\star} \le \max\{0, \gamma_{j} -\gamma_{i^\star}\}$ (from Lemma \ref{lem:ErrorRateOrder}). Hence for $j < i^\star$ the condition 	\[C_{i^*} -C_j \notin \left (\max\{0, \gamma_j-\gamma_i\}, \Prob{Y^{i^\star}\neq Y^j}\right ]\] is satisfied. Then, the condition 	\[C_j -C_{i^\star} \notin \left (\max\{0, \gamma_i-\gamma_j\}, \Prob{Y^{i^\star} \neq Y^j}\right] \] for $j > i^\star$ is sufficient for learnability.
\end{proof}

\subsection*{Missing proofs from \cref{sec:uss_ts}}
We use the following results in our proofs.

\begin{fact}[Chernoff bound for Bernoulli distributed random variables]
	\label{fact:chernoff}
	Let $X_1, \ldots, X_n$ be i.i.d. Bernoulli distributed random variables. Let $\hat{\mu}_n=\frac{1}{n}\sum_{i=1}^{n}X_i$ and $\mu = \EE{X_i}$. Then, for any $\epsilon \in (0, 1 - \mu)$,
	\begin{equation*}
	\Prob{\hat{\mu}_n \ge \mu + \epsilon} \le \exp\left( -d(\mu + \epsilon, \mu)n\right),
	\end{equation*}
	and, for any $\epsilon \in (0, \mu)$,
	\begin{equation*}
	\Prob{\hat{\mu}_n \le \mu - \epsilon} \le \exp\left( -d(\mu - \epsilon, \mu)n\right),
	\end{equation*}
	where $d(x, \mu)= x\log\left(\frac{x}{\mu}\right) + (1-x)\log\left(\frac{1-x}{1-\mu}\right)$.
\end{fact}
\noindent
See Section 10.1 of Chapter 10 of book `Bandit Algorithms' \citep{BOOK_lattimorebandit} for proof.

\begin{fact}[Pinsker's Inequality for Bernoulli distributed random variables]
	\label{fact:pinsker}
	For $p,q \in (0,1)$, the KL divergence between two Bernoulli distributions is bounded as:
	\begin{equation*}
	d(p,q) \ge 2(p-q)^2.
	\end{equation*}
\end{fact}

\begin{fact}
	\label{fact:expUB}
	Let $x > 0$ and $D > 0$. Then, for any $a \in (0,1)$,
	\begin{equation*}
	\frac{1}{\exp^{Dx} - 1} \le 
	\begin{cases}
	\frac{\exp^{-Dx}}{1-a} & \left(x \ge \ln\left(1/a\right)/D \right)\\
	\frac{1}{Dx} & \left(x < \ln\left(1/a\right)/D \right). 
	\end{cases}
	\end{equation*}
	Further, we have,
	\begin{equation*}
	\sum_{x=1}^{n} \frac{1}{\exp^{Dx} - 1} \le \Theta\left( \frac{1}{D^2} + \frac{1}{D}\right).
	\end{equation*}
\end{fact}

\begin{proof}
	Using $\exp^y \ge y + 1$ (by Taylor Series expansion), we have $\frac{1}{\exp^{Dx} - 1} \le \frac{1}{Dx}$ as $\exp^{Dx} - 1 \ge Dx$. We can re-write, $\frac{1}{\exp^{Dx} - 1} = \frac{\exp^{-Dx}}{1-\exp^{-Dx}}$. Since $\exp^{-Dx}$ is strictly decreasing function for all $Dx>0$, it is easy to check that $\exp^{-Dx} \le a$ holds for any $x \ge \ln\left(1/a\right)/D$ and $a \in (0,1)$. Hence, $\frac{\exp^{-Dx}}{1-\exp^{-Dx}} \le \frac{\exp^{-Dx}}{1-a}$ for all $x \ge \ln\left(1/a\right)/D$.
	
	~\\
	Now we will prove the second part,
	\begin{align*}
	\sum_{x=1}^{n} \frac{1}{\exp^{Dx} - 1} &\le \frac{\ln(1/a)}{D^2} + \sum_{x\ge \ln\left(1/a\right)/D}^{n} \frac{\exp^{-Dx}}{1-a}  \\
	&\le \frac{\ln(1/a)}{D^2} + \frac{1}{(1-a)} \int_{x = 0}^{\infty} \exp^{-Dx} dx \\
	&= \frac{\ln(1/a)}{D^2} + \frac{1}{(1-a)}  \left.\left(\frac{\exp^{-Dx}}{-D}\right)\right|_{x = 0}^{\infty} \\
	&= \frac{\ln(1/a)}{D^2} + \frac{1}{(1-a)}  \left( 0 - \frac{\exp^{0}}{-D}\right) \\
	&= \frac{\ln(1/a)}{D^2} + \frac{1}{(1-a)D}\\
	\implies \sum_{x=1}^{n} \frac{1}{\exp^{Dx} - 1} &\le \Theta\left( \frac{1}{D^2} + \frac{1}{D}\right). \tag*{\qedhere}
	\end{align*}
\end{proof}

\begin{fact}
	\label{fact:diffKL}
	Let $\epsilon \in (0,1)$ and $0< x< y < z < 1$. If $d(y,z) = d(x,z)/(1+\epsilon)$ then
	\begin{equation*}
	y-x \ge \frac{\epsilon}{1 + \epsilon} \cdot \frac{d(x,z)}{\ln\left( \frac{z(1-x)}{x(1-z)}\right)}.
	\end{equation*}
\end{fact}

\begin{proof}
	By definition
	\begin{align*}
	d(p,q) &= p\ln\frac{p}{q} + (1-p)\ln\left(\frac{1-p}{1-q}\right)\\
	&=\ln\left(\left(\frac{p}{q}\right)^p \left(\frac{1-p}{1-q}\right)^{1-p} \right) \\
	&= \ln\left(\left(\frac{q(1-p)}{p(1-q)}\right)^{-p} \right) + \ln\left(\frac{1-p}{1-q}\right)\\
	\implies d(p,q) &= -p\ln\left(\frac{q(1-p)}{p(1-q)}\right) + \ln\left(\frac{1-p}{1-q}\right).
	\end{align*}
	Set $l(p,q) = \ln\left(\frac{q(1-p)}{p(1-q)}\right)$. Note that $l(p,\cdot)$ is a strictly decreasing function of $p$ and positive for all $p < q$. We can re-arrange above equation as
	\begin{equation*}
	p\cdot l(p,q) = -d(p,q) + \ln\left(\frac{1-p}{1-q}\right).
	\end{equation*}
	Using above equation, we have
	\begin{align*}
	y\cdot l(y,z) - x\cdot l(x,z) &= -d(y,z) +  \ln\left(\frac{1-y}{1-z}\right) + d(x,z) -  \ln\left(\frac{1-x}{1-z}\right).
	\intertext{Using $d(y,z) = d(x,z)/(1+\epsilon)$,}
	y\cdot l(y,z) - x\cdot l(x,z) &= \frac{\epsilon}{1+\epsilon}d(x,z) +  \ln\left(\frac{1-y}{1-x}\right).
	\intertext{After adding $y(l(x,z) - l(y,z))$ both side, we have }
	(y - x) l(x,z) &= \frac{\epsilon}{1+\epsilon}d(x,z) +  \ln\left(\frac{1-y}{1-x}\right) + y(l(x,z) - l(y,z)).
	\intertext{Using $l(x,z) = \ln\left(\frac{z(1-x)}{x(1-z)}\right)$ and $l(y,z)= \ln\left(\frac{z(1-y)}{y(1-z)}\right)$}
	&= \frac{\epsilon}{1+\epsilon}d(x,z) +  \ln\left(\frac{1-y}{1-x}\right) + y \ln\left(\frac{y(1-x)}{x(1-y)}\right) \\
	&= \frac{\epsilon}{1+\epsilon}d(x,z) +  \ln\left(\left( \frac{y(1-x)}{x(1-y)} \right)^y\cdot \frac{1-y}{1-x}\right) \\
	&=\frac{\epsilon}{1+\epsilon}d(x,z) +  \ln\left(\left(\frac{y}{x}\right)^y\left( \frac{1-y}{1-x}\right)^{1-y}\right)\\
	&=\frac{\epsilon}{1+\epsilon}d(x,z) + d(y,x)
	\intertext{As $d(p,q) \ge 0$ and dividing both side by $l(x,z)$,}
	\implies y- x  &\ge \frac{\epsilon}{1+\epsilon}\cdot \frac{d(x,z)}{l(x,z)}.
	\end{align*}
	Substituting value of $l(x,z)$ in the above equation, we get
	\begin{equation*}
	y -x \ge \frac{\epsilon}{1 + \epsilon} \cdot \frac{d(x,z)}{\ln\left( \frac{z(1-x)}{x(1-z)}\right)}. \tag*{\qedhere}
	\end{equation*}	
\end{proof}

\relationLowerOptimal*
\begin{proof}
	If the sub-optimal arm $j$ is selected then arm $j$ is preferred over the arms whose indexed is larger than $j$ (\cref{lem:Bx}). Hence we have
	\begin{align*}
	\hspace{-10mm}\Prob{I_t =j, j< \istar|\Ht} = &~ \Prob{j \succ_t k, \forall k>j, j < i^\star|\Ht} \le \Prob{j \succ_t k, \forall k \ge \istar, j < i^\star|\Ht}.
	\intertext{Since the feedback from an arm is independent of the feedback of other arms,}
	=&~ \Prob{j \succ_t i^\star , j < i^\star|\Ht} \Prob{j \succ_t k, \forall k > \istar, j < i^\star|\Ht}.
	\intertext{If arm $j$ is preferred over the arm $i^\star$ then $\opijsts < C_{i^\star} - C_j$. As $C_{i^\star} - C_j = \pijts - \xi_j$ for $j<i^\star$,}
	=&~ \Prob{ \opijsts < \pijts - \xi_j|\Ht}  \Prob{j \succ_t k, \forall k > \istar, j < i^\star|\Ht} \\
	=&~ \left(1 - \Prob{ \opijsts \ge \pijts - \xi_j|\Ht}\right)  \Prob{j \succ_t k, \forall k > \istar, j < i^\star|\Ht} \\
	\implies \Prob{I_t=j, j< \istar|\Ht}  \le &~  (1 - q_{j,t})  \Prob{j \succ_t k, \forall k > \istar, j < i^\star|\Ht}. \mbox{\hspace{4mm} (\cref{def:q})} \numberthis \label{eq:lowerOptimal}
	\end{align*}
	Similarly, the probability of selecting an arm whose index is larger than the optimal arm can be lower bounded as follows:
	\begin{align*}
	\Prob{I_t \ge i^\star|\Ht} &\ge \Prob{I_t = i^\star|\Ht} \ge \Prob{I_t = i^\star, \istar \succ_t j, j < \istar|\Ht}\\
	&= \Prob{i^\star \succ_t k, \forall k > \istar,  \istar \succ_t j, j < \istar|\Ht} \mbox{\hspace{4mm} (\cref{lem:Bx})} \\
	&\ge \Prob{\istar \succ_t j, j \succ_t k, \forall k > \istar, j < \istar|\Ht} \mbox{\hspace{5mm} (\cref{def:trans_prop})} \\
	&=  \Prob{\istar \succ_t j, j < \istar|\Ht}  \Prob{j \succ_t k, \forall k > \istar, j < \istar|\Ht}.
	\intertext{If arm $\istar$ is preferred over the arm $j$ then $\opijsts \ge C_{i^\star} - C_j$. As $C_{i^\star} - C_j = \pijts - \xi_j$ for $j<i^\star$,}
	&= \Prob{ \opijsts \ge \pijts - \xi_j|\Ht} \Prob{j \succ_t k, \forall k > \istar, j < i^\star}\\
	\implies \Prob{I_t \ge \istar|\Ht} & \ge q_{j,t} \Prob{j \succ_t k, \forall k > \istar, j < i^\star}. \mbox{\hspace{17mm} (\cref{def:q})} \numberthis \label{eq:optimalLower}
	\end{align*}
	Combining the \cref{eq:lowerOptimal} and \cref{eq:optimalLower}, we get
	\begin{equation*}
	\Prob{I_t=j, j< \istar|\Ht} \le \frac{(1-q_{j,t})}{q_{j,t}} \Prob{I_t \ge \istar|\Ht}. \tag*{\qedhere}
	\end{equation*}
\end{proof}

\probLow*
\begin{proof}
	Applying \cref{lem:relationLowerOptimal} and properties of conditional expectations, we have
	\begin{align*}
	\sum_{t=1}^T \Prob{I_t=j, j < \istar} &= \sum_{t=1}^T  \EE{ \Prob{I_t=j, j < \istar|\Ht}}.
	\intertext{As $q_{j,t}$ is fixed given $\Ht$,}
	\implies \sum_{t=1}^T  \Prob{I_t=j, j < \istar}  &\le \sum_{t=1}^T  \EE{\frac{(1-q_{j,t})}{q_{j,t}}\Prob{I_t \ge i^\star|\Ht}}\\
	&\le \sum_{t=1}^T  \EE{\EE{\frac{(1-q_{j,t})}{q_{j,t}}\one{I_t \ge i^\star}|\Ht}}.
	\intertext{Using law of iterated expectations,}
	\implies \sum_{t=1}^T  \Prob{I_t=j, j < \istar}  &\le  \sum_{t=1}^T  \EE{\frac{(1-q_{j,t})}{q_{j,t}}\one{I_t \ge i^\star}}.  \numberthis \label{eq:lowProbSum}
	\end{align*}
	Let $s_m$ denote the time step at which the output of arm $i^\star$ is observed for the $m^{th}$ time for $m\ge1$, and let $s_0 = 0$. For $j<i^\star$, whenever the output from arm $i^\star$ is observed then the output from arm $j$ is also observed due to the cascade structure. Note that $q_{j,t} = \Prob{\opijsts > \pijs - \xi_j|\Ht}$ changes only when the distribution of $\opijsts$ changes, that is, only on the time step when the feedback from arms $i^\star$ and $j$ are observed. It only happens when selected arm $I_t\ge \istar$. Hence, $q_{j,t}$ is the same at all time steps $t \in \{s_m +1, \ldots, s_{m+1}\}$ for every $m$. Using this fact, we can decompose the right hand side term in \cref{eq:lowProbSum} as follows,
	\begin{align*}
	\sum_{t=1}^T  \EE{\frac{(1-q_{j,t})}{q_{j,t}}\one{I_t \ge i^\star}} &= \sum_{m=0}^{T-1}  \EE{\frac{(1-q_{j,s_m+1})}{q_{j,s_m+1}} \sum_{t=s_m + 1}^{s_{m+1}}\one{I_t \ge i^\star}} \\
	&\le  \sum_{m=0}^{T-1}  \EE{\frac{(1-q_{j,s_m+1})}{q_{j,s_m+1}}} \\
	&=  \sum_{k=0}^{T-1} \EE{\frac{1}{q_{j,s_m+1}} - 1}.
	\end{align*}
	Using above bound in \cref{eq:lowProbSum}, we get
	\begin{equation*}
	\sum_{t=1}^T  \Prob{I_t=j, j < \istar}  \le  \sum_{m=0}^{T-1} \EE{\frac{1}{q_{j,s_m+1}}-1}.
	\end{equation*}
	Substituting the bound from \cref{lem:betaExpBound} with $\mu = \pijs, x = \pijs - \xi_j, \Delta(x) = \xi_j,$ and $q_n(x)= q_{j,s_m}$, we obtain the following bound,
	\begin{equation*}
	\sum_{t=1}^T  \Prob{I_t=j, j < \istar}  \le \frac{24}{\xi_j^2} + \sum_{s \ge 8/\xi_j} \Theta\left(\exp^{-{s\xi_j^2}/{2}} + \frac{\exp^{-{sd( \pijs - \xi_j,\pijs)}}}{(s+1)\xi_j^2} + \frac{1}{\exp^{{s\xi_j^2}/{4}} - 1} \right). \tag*{\qedhere}
	\end{equation*}
	
\end{proof}

\probHighPartTwo*
\begin{proof}
	Define $L_j(T) = \frac{\ln T}{d(x_j, y_j)}$. Let $N_j(t)$ be the number of times the output from arm $j$ is observed in $t$ rounds. Then, the given probability term can be decomposed into two parts:
	\begin{align*}
	\sum_{t=1}^T \Prob{\hpijst \le x_j, \opijsts > y_j} &= \sum_{t=1}^T \Prob{\hpijst \le x_j, \opijsts > y_j, N_j(t) \le L_j(T)} + \\
	&\qquad\qquad  \sum_{t=1}^T \Prob{\hpijst \le x_j, \opijsts > y_j, N_j(t) > L_j(T)} \\
	\le  L_j(T)& + \sum_{t=1}^T \Prob{\hpijst \le x_j, \opijsts > y_j, N_j(t) > L_j(T)}. \label{eq:decomUB} \numberthis 
	\end{align*}
	The first term of the above decomposition is bounded trivially by $L_j(T)$. To bound the second term, we demonstrate that if $N_j(t)$ is large enough and event $\hpijst \le x_j$ is satisfied, then the probability that the event $\opijtst > y_j$ happens, is small. Then,
	\begin{align*}
	\sum_{t=1}^T &\Prob{\hpijst \le x_j, \opijsts > y_j, N_j(t) > L_j(T)} \\
	&\qquad= \sum_{t=1}^T \EE{\one{\hpijst \le x_j, \opijsts > y_j, N_j(t) > L_j(T)}} \\
	&\qquad = \EE{ \sum_{t=1}^T \EE{\one{\hpijst \le x_j, \opijsts > y_j, N_j(t) > L_j(T)}| \Ht}}.
	\intertext{Since $N_j(t)$ and $\hpijst$ are determined by the history $\Ht$,}
	&\qquad = \EE{ \sum_{t=1}^T \one{\hpijst \le x_j, N_j(t) > L_j(T)} \Prob{\opijsts > y_j| \Ht} }. \numberthis \label{eq:secondTerm}
	\end{align*}
	Now, by definition, $\mathcal{S}_{i^\star j}(t) = \hpijst N_j(t)$, and therefore, $\opijsts$ is a Beta$(\hpijst N_j(t) + 1, (1-\hpijst)N_j(t) + 1)$ distributed random variable. A Beta$(\alpha, \beta)$ random variable is stochastically dominated by Beta$(\alpha^\prime, \beta^\prime)$ if $\alpha^\prime \ge \alpha, \beta^\prime \le \beta$. Therefore, if $\hpijst \le x_j$, the distribution of $\opijsts$ is stochastically dominated by Beta$(x_j N_j(t) + 1, (1-x_j)N_j(t))$. Therefore, given a history $\Ht$ such that $\hpijst \le x_j$ and $N_j(t) > L_j(T)$, we have
	\begin{equation*}
	\Prob{\opijsts > y_j| \Ht} = 1 - F_{x_j N_j(t) + 1, (1-x_j)N_j(t)}^{beta}(y_j).
	\end{equation*}
	Now, using Beta-Binomial equality (\cref{fact:beta_binomial}), we obtain that for any fixed $N_j(t) > L_j(T)$,
	\begin{align*}
	1 - F_{x_j N_j(t) + 1, (1-x_j)N_j(t)}^{beta}(y_j) &= F_{N_j(t), y_j}^{B}(x_j N_j(t)) &\mbox{(using \cref{fact:beta_binomial})}
	\end{align*}
	Here $F_{N_j(t), y_j}^{B}(x_j N_j(t))$ is the cdf of Binomial distribution with parameter $y_j$ and $N_j(T)$ observations. Let $\mathcal{S}_t^\prime$ be the number of successes observed in $N_j(T)$ observations. Then,
	\begin{align*}
	1 - F_{x_j N_j(t) + 1, (1-x_j)N_j(t)}^{beta}(y_j) &= \Prob{\mathcal{S}_t^\prime \le x_j N_j(t)} \\
	&= \Prob{\frac{\mathcal{S}_t^\prime}{N_j(t)} \le x_j}\\
	&= \Prob{\hat{y}_j \le x_j} \hspace{17mm} \mbox{(using $\hat{y}_j = \mathcal{S}_t^\prime/N_j(t)$)}\\
	&\le \exp^{-N_j(t)d(x_j, y_j)} \hspace{10mm} \mbox{(using Chernoff-Hoeffding bound)}\\
	&\le \exp^{-L_j(t)d(x_j, y_j)}, \hspace{10mm} \mbox{(as $N_j(t) > L_j(T)$)}
	\end{align*}
	which is smaller than $1/T$ because $L_j(T) = \frac{\log(T)}{d(x_j, y_j)}$. Substituting, we get that for a history $\Ht$ such that $\hpijst \le x_j$ and $N_j(t) > L_j(T)$,
	\begin{align*}
	\Prob{\opijsts > y_j| \Ht} \le \frac{1}{T}.
	\end{align*}
	For other history $\Ht$, the indicator term $\one{\hpijst \le x_j, N_j(t) > L_j(T)}$ in \cref{eq:secondTerm} will be 0 as either event $\hpijst \le x_j$ or event $N_j(t) > L_j(T)$ is violated. Summing over $t$, this bounds the right hand side term in \cref{eq:secondTerm} as follows:
	\begin{align*}
	\sum_{t=1}^T \Prob{\hpijst \le x_j, \opijsts > y_j, N_j(t) > L_j(T)} &\le \EE{ \sum_{t=1}^T \frac{\one{\hpijst \le x_j, N_j(t) > L_j(T)}}{T} } \\
	&\le \EE{ \sum_{t=1}^T \frac{1}{T} } \\
	&= 1.
	\end{align*}
	Replacing the second term in \cref{eq:decomUB} by its upper bound and $L_j(T)$ with its value,
	\begin{equation*}
	\sum_{t=1}^T \Prob{\hpijst \le x_j, \opijsts > y_j} \le \frac{\ln T}{d(x_j, y_j)} + 1. \tag*{\qedhere}
	\end{equation*}
\end{proof}

\probHighPartOne*
\begin{proof}
	Let $s_m$ denote the time step at which the outputs of arm $i^\star$ and $j$ is observed for the $m^{th}$ time for $m\ge1$, and let $s_0 = 0$. Note that probability $\Prob{\hpijst > x_j}$ changes when the outputs from both arm $i^\star$ and $j$ are observed. Hence, we have
	\begin{align*}
	\sum_{t=1}^T \Prob{\hpijst >x_j} &\le \sum_{m=0}^{T-1} \Prob{\hpijs(s_{m+1}) >x_j}\\
	&= \sum_{m=0}^{T-1} \Prob{\hpijs(s_{m+1}) - \pijs > x_j-\pijs}\\
	&\le \sum_{m=0}^{T-1} \exp^{-kd(\pijs + x_j-\pijs, \pijs)}  \hspace{10mm} \text{(using \cref{fact:chernoff}})\\
	&=\sum_{m=0}^{T-1} \exp^{-kd(x_j, \pijs)}.
	\end{align*}
	Using $\sum_{s\ge 0}\exp^{-sa} \le 1/a$, we get
	\begin{equation*}
	\sum_{t=1}^T \Prob{\hpijst >x_j} \le \frac{1}{d(x_j, \pijs)}. \tag*{\qedhere}
	\end{equation*}
\end{proof}

\probHigh*
\begin{proof}
	Let $\pijs < x_j < y_j < \pijs + \xi_j$ for any $j > i^\star$. Than, 
	\begin{align*}
	\sum_{t=1}^T\Prob{j \succ_t i^\star, j > \istar} &= \sum_{t=1}^T \Prob{\opijsts > \pijs + \xi_j} \\
	&\le \sum_{t=1}^T \Prob{\opijsts > y_j}\\
	& \le  \sum_{t=1}^T \Prob{\hpijst \le x_j, \opijsts > y_j} + \sum_{t=1}^T \Prob{\hpijst > x_j}.
	\end{align*}
	Using \cref{lem:probHighPart1} and \cref{lem:probHighPart2}, we have
	\begin{equation*}
	\sum_{t=1}^T\Prob{j \succ_t i^\star, j > \istar} \le  \frac{\ln T}{d(x_j, y_j)} + 1 + \frac{1}{d(x_j, \pijs)}.
	\end{equation*}
	For $\epsilon \in (0,1)$, we set $x_j \in (\pijs, \pijs + \xi_j)$ such that $d(x_j, \pijs + \xi_j) = d(\pijs, \pijs + \xi_j)/(1+\epsilon)$, and set $y_j \in (x_j, \pijs + \xi_j)$ such that $d(x_j, y_j) = d(x_j, \pijs + \xi_j)/(1+\epsilon) = d(\pijs, \pijs + \xi_j)/(1+\epsilon)^2$. Then this gives
	\begin{equation*}
	\frac{\ln(T)}{d(x_j,y_j)} = (1+\epsilon)^2\frac{\ln(T)}{d(\pijs, \pijs + \xi_j)}.
	\end{equation*}
	Using \cref{fact:diffKL}, if $\epsilon \in (0,1)$, $x_j \in (\pijs, \pijs + \xi_j)$, and $d(x_j, \pijs + \xi_j) = d(\pijs, \pijs + \xi_j)/(1+\epsilon)$ then
	\begin{equation*}
	x_j - \pijs \ge \frac{\epsilon}{1 + \epsilon}. \frac{d(\pijs, \pijs + \xi_j)}{\ln\left( \frac{(\pijs + \xi_j)(1-\pijs)}{\pijs(1-\pijs - \xi_j)}\right)}.
	\end{equation*}
	Using Pinsker's Inequality (\cref{fact:pinsker}), $1/d(x_j, \pijs) \le 1/2(x_j - \pijs)^2 = O({1}/{\epsilon^2})$ where big-Oh is hiding functions of the $\pijs$ and $\xi_j$,
	\begin{align*}
	\sum_{t=1}^T \Prob{j \succ_t \istar, j > i^\star} &\le (1+\epsilon)^2\frac{\ln(T)}{d(\pijs, \pijs + \xi_j)} + O\left(\frac{1}{\epsilon^2}\right)\\
	&\le (1+3\epsilon)\frac{\ln(T)}{d(\pijs, \pijs + \xi_j)} + O\left(\frac{1}{\epsilon^2}\right)\\
	&\le (1+\epsilon^\prime)\frac{\ln(T)}{d(\pijs, \pijs + \xi_j)} + O\left(\frac{1}{{\epsilon^\prime}^2}\right),
	\end{align*}
	where $\epsilon^\prime = 3\epsilon$ and the big-Oh above hides $\pijs$ and $\xi_j$ in addition to the absolute constants. Replacing $\epsilon$ by $\epsilon^\prime$ completes the proof.
\end{proof}

\depRegretBound*
\begin{proof}
	Let $M_j(T)$ is the number of times arm $j$ is selected by \ref{alg:TS_USS}. Than, the regret is
	\begin{align*}
	{\Regret_T} &= \sum_{j \in [K]}\EE{M_j(T)}\Delta_j = \sum_{j \in [K]}\EE{\sum_{t=1}^{T}\one{I_t = j}}\Delta_j \\ 
	&= \sum_{j \in [K]}\sum_{t=1}^{T}\EE{\one{I_t = j}}\Delta_j = \sum_{j \in [K]}\sum_{t=1}^{T}\Prob{I_t = j}\Delta_j \\
	&=  \sum_{j \in [K]}\sum_{t=1}^{T}\Prob{I_t=j, j \ne \istar}\Delta_j \\
	\implies {\Regret_T}&=  \sum_{j < i^\star} \sum_{t=1}^{T}\Prob{I_t=j, j < i^\star} \Delta_j + \sum_{j > i^\star}\sum_{t=1}^{T}\Prob{I_t=j, j > i^\star} \Delta_j \numberthis \label{eq:regretSum}
	\end{align*}
	First, we bound the first of term of summation. From \cref{lem:probLow}, we have
	\begin{equation*}
	\sum_{t=1}^{T}\Prob{I_t =j, j < i^\star} \le \frac{24}{\xi_j^2} + \sum_{s \ge 8/\xi_j} \Theta\left(\exp^{-{s\xi_j^2}/{2}} + \frac{\exp^{-{sd( \pijs - \xi_j,\pijs)}}}{(s+1)\xi_j^2} + \frac{1}{\exp^{{s\xi_j^2}/{4}} - 1} \right).
	\end{equation*}
	Using $\sum_{s\ge 0}\exp^{-sa} \le 1/a$, $d( \pijs - \xi_j,\pijs) \le 2\xi_j^2$ (\cref{fact:pinsker}), and \cref{fact:expUB}, we have
	\begin{align*}
	\sum_{t=1}^T \Prob{I_t =j,  j < i^\star} &\le \frac{24}{\xi_j^2} + \Theta\left(\frac{1}{\xi_j^2} + \frac{1}{\xi_j^4} + \left(\frac{1}{\xi_j^4} + \frac{1}{\xi_j^2} \right) \right) 
	\le O(1). \numberthis \label{equ:lowerRegretUSS_TS}
	\end{align*}
	If arm $I_t > \istar$ is selected then there exists at least one arm $k_1 > \istar$ which must be preferred over $\istar$. If the index of arm $k_1$ is smaller than the selected arm, then there must be an arm $k_2>k_1$, which must be preferred over $k_1$. By transitivity property, arm $k_2$ is also preferred over $\istar$. If the index of arm $k_2$ is still smaller of the selected arm, we can repeat the same argument. Eventually, we can find an arm $k^\prime$ whose index is larger than the selected arm, and it is preferred over arm $k_i, \ldots, k_1, \istar$. Note that the selected arm must be preferred over $k^\prime$; hence the selected arm is also preferred over $\istar$. We can write it as follows:
	\begin{align*}
	\sum_{t=1}^{T}\Prob{I_t=j, j > i^\star} \Delta_j =&~\sum_{t=1}^{T}\Prob{I_t=j, j > i^\star, k^\prime \succ_t k, k \succ_t \istar, k^\prime > j, k > \istar} \Delta_j\\
	=&~\sum_{t=1}^{T}\Prob{I_t=j, j > i^\star, k^\prime \succ_t \istar, k^\prime > j} \Delta_j \mbox{\hspace{2mm} (\cref{def:trans_prop})}\\
	=&~\sum_{t=1}^{T}\Prob{j \succ_t k, \forall k >j, j > i^\star, k^\prime \succ_t \istar, k^\prime > j} \Delta_j \mbox{\hspace{1mm}(\cref{lem:Bx})}\\
	=&~\sum_{t=1}^{T}\Prob{j \succ_t k, \forall k >j, j > i^\star, j \succ_t \istar} \Delta_j \mbox{\hspace{2mm} (\cref{def:trans_prop})}\\
	\implies \sum_{t=1}^{T}\Prob{I_t=j, j > i^\star} \Delta_j \le&~\sum_{t=1}^{T}\Prob{j \succ_t \istar, j > \istar} \Delta_j. \numberthis \label{equ:selectToPrefer}
	\end{align*}
	Using \cref{lem:probHigh} to upper bound $\sum_{t=1}^{T}$ $ \Prob{j \succ_t \istar, j > i^\star} \Delta_j$ and with \cref{equ:lowerRegretUSS_TS}, we get
	\begin{align*}
	{\Regret_T} &\le O (1)+ \sum_{j > i^\star}\left( (1+\epsilon)\frac{\ln(T)}{d(\pijs, \pijs + \xi_j)} + O\left(\frac{1}{\epsilon^2}\right) \right) \Delta_j\\
	\implies {\Regret_T} &\le  \sum_{j > i^\star}\ \frac{(1+\epsilon)\ln(T)}{d(\pijs, \pijs + \xi_j)}\Delta_j + O\left(\frac{K-i^\star}{\epsilon^2}\right). \tag*{\qedhere}
	\end{align*}
\end{proof}

\indepRegretBound*
\begin{proof}
	Let $M_j(T)$ is the number of times arm $j$ preferred over the optimal arm in $T$ rounds. From \cref{lem:probLow}, for any $j<i^\star$, we have
	\begin{align*}
	\EE{M_j(T)} &= \sum_{t=1}^{T}\Prob{I_t = j, j < i^\star} \\
	&\le \frac{24}{\xi_j^2} + \sum_{s \ge 8/\xi_j} \Theta\left(\exp^{-{s\xi_j^2}/{2}} + \frac{\exp^{-{sd( \pijs - \xi_j,\pijs)}}}{(s+1)\xi_j^2} + \frac{1}{\exp^{{s\xi_j^2}/{4}} - 1} \right).
	\end{align*}
	It is east to show that $\frac{\exp^{-{sd( \pijs - \xi_j,\pijs)}}}{(s+1)\xi_j^2} \le \frac{1}{(s+1)\xi_j^2}$ and ${\exp^{{s\xi_j^2}/{4}} - 1}  \ge {s\xi_j^2}/4 $ (as $\exp^y \ge y+1$), 
	\begin{align*}
	\EE{M_j(T)} \le \frac{24}{\xi_j^2} + \sum_{s \ge 8/\xi_j} \Theta\left(\frac{1}{\xi_j^2} + \frac{1}{(s+1)\xi_j^2} + \frac{4}{s\xi_j^2} \right).
	\end{align*}
	By using $\sum_{s\ge 0}\exp^{-sa} \le 1/a$ and $\sum_{s=1}^T (1/s) = \log T$, 
	\begin{align*}
	\EE{M_j(T)} &\le \frac{24}{\xi_j^2} + \Theta\left(\frac{1}{\xi_j^2} + \frac{\ln T}{\xi_j^2} \right) 
	\implies \EE{M_j(T)} \le O\left(\frac{\ln T}{\xi_j^2}\right). \numberthis \label{eq:lowIndRegret}
	\end{align*}
	
	\noindent
	For any $j>i^\star$, using \cref{lem:probHighPart2} and \cref{lem:probHighPart1} with \cref{equ:selectToPrefer}, we have
	\begin{equation*}
	\EE{M_j(T)} = \sum_{t=1}^T \Prob{I_t = j, j > i^\star} \le \sum_{t=1}^T \Prob{j \succ_t \istar, j > i^\star} \le  \frac{\ln T}{d(x_j, y_j)} + 1 + \frac{1}{d(x_j, \pijs)}.
	\end{equation*}
	By setting $x_j = \pijs + \frac{\xi_j}{3}$ and $y_j = \pijs + \frac{2\xi_j}{3}$, we have $d(x_j, y_j) \ge \frac{2\xi_j^2}{9}$ and $d(x_j, \pijs) \ge \frac{2\xi_j^2}{9}$ (using \cref{fact:pinsker}).
	\begin{align*}	
	\EE{M_j(T)} &\le \frac{9\ln T}{2\xi_j^2} + 1 + \frac{9}{2\xi_j^2}\\
	\implies \EE{M_j(T)} &\le O\left(\frac{\ln T}{\xi_j^2}\right). \numberthis	\label{eq:highIndRegret}
	\end{align*}
	
	\noindent
	The regret of \ref{alg:TS_USS} is given by
	\begin{align*}
	{\Regret_T} = \sum\limits_{j\neq i^\star} \EE{M_j(T)}\Delta_j = \sum\limits_{j<i^\star}\EE{M_j(T)}\Delta_j + \sum\limits_{j>i^\star}\EE{M_j(T)}\Delta_j 
	\end{align*}
	Recall $\Delta_j = C_j + \gamma_j - (C_{i^\star} + \gamma_{i^\star})$ and for any two arms $i$ and $j$, $0 \le \pij - (\gamma_j - \gamma_{i^\star} ) \le \beta$. By using Eq. \eqref{def_xi_l} for $j<i^\star$, we have $\Delta_j = \xi_j - (\pijs - (\gamma_{i^\star} - \gamma_j)) \implies \Delta_j \le \xi_j$, and using Eq. \eqref{def_xi_h} for $j>i^\star$, we have $\Delta_j = \xi_j + (\pijs - (\gamma_{i^\star} - \gamma_j)) \implies \Delta_j \le \xi_j + \beta$. Replacing $\Delta_j$, 
	\begin{equation*}
	\Rightarrow {\Regret_T} \leq \sum\limits_{j<i^\star}\EE{M_j(T)}\xi_j + \sum\limits_{j>i^\star}\EE{M_j(T)}(\xi_j + \beta). 
	\end{equation*} 
	Let $0<\xi^\prime<1$. Then  ${\Regret_T}$ can be written as:
	\begin{align*}
	{\Regret_T} & \leq \sum\limits_{\substack{\xi^\prime > \xi_j\\ j < i^\star}} \EE{M_j(T)}\xi_j + \sum\limits_{\substack{\xi^\prime < \xi_j\\ j < i^\star}} \EE{M_j(T)}\xi_j \\
	&\qquad + \sum\limits_{\substack{\xi^\prime > \xi_j\\ j > i^\star}} \EE{M_j(T)}(\xi_j + \beta) + \sum\limits_{\substack{\xi^\prime < \xi_j\\ j > i^\star}} \EE{M_j(T)}(\xi_j + \beta).
	\end{align*}
	Using $\sum\limits_{\xi^\prime > \xi_j} \EE{M_j(T)} \le T$ for any $j$ such that $\xi^\prime > \xi_j$, 
	\begin{align*}
	{\Regret_T} &  \le T\xi^\prime + \sum\limits_{\substack{\xi^\prime < \xi_j\\ j < i^\star}} \EE{M_j(T)}\xi_j + \sum\limits_{\substack{\xi^\prime < \xi_j\\ j > i^\star}} \EE{M_j(T)}(\xi_j + \beta).
	\end{align*}
	Substituting the value of ${\Regret_T}$ from \cref{eq:lowIndRegret} and \cref{eq:highIndRegret},
	\begin{align*}
	{\Regret_T} &  \le T\xi^\prime + \sum\limits_{\substack{\xi^\prime < \xi_j\\ j < i^\star}} O\left(\frac{\xi_j\ln T }{\xi_j^2}\right) + \sum\limits_{\substack{\xi^\prime < \xi_j\\ j > i^\star}} O\left(\frac{(\xi_j + \beta)\ln T }{\xi_j^2}\right)\\
	&\le T\xi^\prime + \sum\limits_{\substack{\xi^\prime < \xi_j\\ j < i^\star}} O\left(\frac{\ln T }{\xi_j}\right) + \sum\limits_{\substack{\xi^\prime < \xi_j\\ j > i^\star}} O\left(\frac{\ln T}{\xi_j} + \frac{\beta\ln T}{\xi_j^2}\right)\\
	&  \le T\xi^\prime + O\left(\frac{K\ln T }{\xi^\prime}\right) + O\left(\frac{K\ln T }{\xi^\prime} + \frac{\beta K\ln T }{{\xi^\prime}^2}\right)\\
	&  = T\xi^\prime + O\left(K\ln T\left(\frac{1}{\xi^\prime} + \frac{\beta}{{\xi^\prime}^2}\right)\right)
	\intertext{Let there exist a variable $\alpha$ such that $O\left(K\ln T\left(\frac{1}{\xi^\prime} + \frac{\beta}{{\xi^\prime}^2}\right)\right) \le \alpha K\ln T\left(\frac{1}{\xi^\prime} + \frac{\beta}{{\xi^\prime}^2}\right)$,}
	\implies {\Regret_T} & \le = T\xi^\prime + \alpha K\ln T\left(\frac{1}{\xi^\prime} + \frac{\beta}{{\xi^\prime}^2}\right). \numberthis \label{eq:indRegret}
	\end{align*}		
	Consider $\PWD$ class of problems. As $\xi^\prime < 1$ and $\beta \le 2$ (as arms in the cascade may not be ordered by their error-rates, it is possible that $\gamma_{i} < \gamma_{j}$), we have $\left(\frac{1}{\xi^\prime} + \frac{\beta}{{\xi^\prime}^2}\right) \le \frac{\beta + 1}{{\xi^\prime}^2}\le \frac{3}{{\xi^\prime}^2}$,
	\begin{align*}
	{\Regret_T} & \le = T\xi^\prime +\frac{3\alpha K\ln T}{{\xi^\prime}^2}.
	\intertext{Choose $\xi^\prime = \left( \frac{6\alpha K\ln T}{T}\right)^{1/3}$ which maximize above upper bound and we get,}
	{\Regret_T} &\le \left(6\alpha K\ln T\right)^{1/3}T^{2/3} + \frac{\left(6\alpha K\ln T\right)^{1/3}}{2}T^{2/3} \\
	\implies {\Regret_T} &\le 2\left(6\alpha K\ln T\right)^{1/3}T^{2/3} = O\left( \left(K\ln T\right)^{1/3}T^{2/3}\right)
	\end{align*}		
	It completes our proof for the case when any problem instance belongs to $\PWD$.
	
	~\\
	\noindent
	Now we consider any problem instance $\theta \in \PSD$. For any $\theta \in \PSD \Rightarrow \forall j \in [K],\; \pij = \gamma_{i} - \gamma_j \implies \beta = 0$ (Setting $\Prob{\Yi = Y, \Yj \ne Y} = 0$ for $j>i$ in Proposition 3 of \cite{AISTATS17_hanawal2017unsupervised}). We can rewrite \cref{eq:indRegret} as
	\begin{align*}
	{\Regret_T} &\le T\xi^\prime +\frac{\alpha K\ln T}{{\xi^\prime}}.
	\intertext{Choose $\xi^\prime = \left( \frac{\alpha K\ln T}{T}\right)^{1/2}$ which maximize above upper bound and we get,}
	\implies {\Regret_T} &\le 2\left(\alpha KT\ln T\right)^{1/2} = O\left( \sqrt{KT\ln T}\right) 
	\end{align*}
	This complete proof for second part of Theorem \ref{thm:indepRegretBound}.
\end{proof}

%% file: chapter3/main.tex
%!TEX root =  ../thesis.tex

In this chapter, we study {\em Contextual \underline{U}nsupervised \underline{S}equential \underline{S}election} (USS), a variant of the stochastic contextual bandits problem where the loss incurred by the arm cannot be inferred from the observed feedback. In our setup, arms are associated with fixed costs and are ordered, forming a cascade. In each round, a context is presented, and the learner selects the arms sequentially till some depth.  The total cost incurred by stopping at an arm is the sum of fixed costs of arms selected and the stochastic loss associated with the arm.  The learner's goal is to learn a decision rule that maps contexts to arms with the goal of minimizing the total expected loss. The problem is challenging as we are faced with an unsupervised setting as the total loss cannot be estimated. Clearly, learning is feasible only if the optimal arm can be inferred (explicitly or implicitly) from the problem structure. We observe that learning is still possible when the problem instance satisfies the so-called `Contextual Weak Dominance' $(\CWD)$ property. Under $\CWD$, we propose an algorithm for the contextual USS problem and demonstrate that it has sub-linear regret. Experiments on synthetic and real datasets validate our algorithm.

\section{Contextual USS}
\label{sec:contxuss_introduction}
\input{chapter3/introduction}

\section{Problem Setting}
\label{sec:contxuss_problem_setting}
\input{chapter3/problem_setting}

\subsection{Contextual Weak Dominance}			
\label{ssec:contxuss_learning}
\input{chapter3/learning}

\section{Parameterization of Pairwise Disagreement  Probability}
\label{sec:contxuss_glmModel}
\input{chapter3/glmModel}

\section{Algorithm for Contextual USS: \ref{alg:CUSS_WD}}
\label{sec:contxuss_algorithm}
\input{chapter3/algorithm}    
\nocite{CODS18_verma2018unsupervised,NeurIPS20_verma2020online}

\section{Experiment}
\label{sec:contxuss_experiments}
\input{chapter3/experiment}

\section{Appendix}
\label{sec:contxuss_appendix}
\input{chapter3/appendix}

	\subsection*{Algorithm with $\lambda\mathrm{I}_	{d^\prime}$ Initialization}
	\label{asec:contxuss_lambda_algorithm}
	\input{chapter3/lambda_algorithm}

	\subsection*{More Experiments}
	\label{asec:contxuss_moreExperiments}
	\input{chapter3/supp_experiments}

%% file: chapter3/introduction.tex
%!TEX root =  ../thesis.tex

Industrial systems, such as those found in medical, airport security, and manufacturing, utilize a suite of tests or classifiers for monitoring patients, people, and products. Tests have costs with the more intrusive and informative ones resulting in higher monetary costs and higher latency. For this reason, they are often organized as a classifier cascade \citep{chen:2012, AISTATS13_trapeznikov2013supervised, NIPS15_wang2015efficient}, so that new input is first probed by an inexpensive test and then a more expensive one. The goal of a cascaded system is to resolve easy to handle examples early so that the overall system maintains high accuracy at low average costs.

Over time, due to environmental changes or test calibrations, sequential testing protocols (STP) may no longer be accurate, resulting in higher costs. While one can leverage off-line methods such as supervised training of cascades ~\citep{NIPS15_wang2015efficient}, they require new annotated data collection. In many scenarios, new data cannot be collected in-situ, and system shutdown is not an option. In the absence of annotated data, we face a dilemma. While we can observe test outcomes, we cannot ascertain their reliability due to the absence of ground truth, necessitating {\it unsupervised sequential selection (USS)} methods, where an arm represents a test/classifier. Recent works~\citep{AISTATS17_hanawal2017unsupervised,AISTATS19_verma2019online, ACML20_verma2020thompson} propose methods for solving the USS problem; however, they focus exclusively on the non-contextual setting, which in essence requires inputs (people, objects, or products) to be homogeneous, and as such, these methods are unrealistic since contexts (high vs. low risk) can guide the arm selection.

In this context, we propose the contextual USS. In our setup, inputs arrive sequentially, and the learner observes a continuous-valued context as input. While the learner knows the cost of each arm, he does not know the associated stochastic loss. Furthermore, the learner does not benefit from feedback obtained by his arm selection, in contrast to the conventional contextual bandit works \citep{AISTATS11_ContextualBandit_BeygelzimerSchapire}. Thus, while being agnostic to the true loss, the learner must sequentially choose the arm that leads to the smallest total loss, where the total loss is the sum of the cost of using an arm and the mean loss associated with the arm. As such, our proposed problem is a special case of the stochastic partial monitoring problem with contextual inputs \citep[Chapter 37]{LaSze19:book}. 
Most of the prior work on partial monitoring problem is restricted to cases where observed feedback can identify the losses for selected actions. However, in many areas like crowd-sourcing \citep{NIPS17_bonald2017minimax, ICLM18_kleindessner2018crowdsourcing}, resource allocation \citep{NeurIPS19_verma2019censored},  medical diagnosis \citep{COMSNETS20_verma2020unsupervised}, and many others, feedback from actions may not even be sufficient to identify the losses.

While we draw upon several concepts introduced in earlier work \citep{AISTATS17_hanawal2017unsupervised}, there are additional challenges in the contextual case due to the unsupervised nature of the problem. First, unlike vanilla-USS, the loss here is context-dependent. We propose notions of contextual weak dominance as a means to relate observed disagreements to differences in losses between any two arms. We then propose a parameterized Generalized Linear Model (GLM) to model the context-conditional disagreement probability between any two arms and validate the model empirically.

A fundamental technical challenge is in the estimation of disagreement probabilities uniformly across all contexts in the finite time while ensuring sufficient exploration between different arm selection protocols, required for honing in on the optimal selection strategy. In particular, since contexts are continuous-valued, and because we have no control over inputs, the contextual observations, in the finite time, may not persistently span the whole space, and estimates are often unreliable. To this end, we adapt techniques from parameterized contextual bandits \citep{AISTATS11_chu2011contextual,ICML17_li2017provably} for our unsupervised setting. We propose an algorithm based on the principle of optimism, namely, the larger indexed arm in cascade is chosen when uncertain. We show that our algorithm navigates the exploration-exploitation tradeoffs in different ways and lead to sub-linear cumulative regret. We then validate it on several problem instances derived from synthetic and real datasets.

\subsection{Related Work}
{\em Stochastic Contextual multi-armed Bandits (SCB):} In each round, the learner observes the context and decides which arm, among a finite number of arms, to apply \citep{AISTATS11_ContextualBandit_BeygelzimerSchapire}. By playing an arm, the learner observes a stochastic reward that depends on the context and the arm selected. The most commonly studied model assumes that each arm is parameterized, and the mean reward of an arm is the inner product of the context and an unknown parameter associated with the arm. Contextual bandits have been applied to problems ranging from online advertising \citep{WWW10_li2010contextual,AISTATS11_chu2011contextual} and recommendations \citep{NIPS08_langford2008epoch} to clinical trials \citep{JASA79_woodroofe1979one} and mobile health \citep{MB17_tewari2017ads}.
{\em Generalized linear models (GLM)} assume that the mean reward is a non-linear link function of the inner product between the context vector and the unknown parameter vector \citep{NIPS10_filippi2010parametric,ICML17_li2017provably}. GLMs are also useful models for the classification problems where rewards, in the context of online learning problems, could be binary \citep{ICML16_zhang2016online,NIPS17_jun2017scalable}. A more challenging non-parameterized version of the stochastic contextual bandits is studied in \citep{ICML14_TamingTheMonster_AgarwalSchapire}. 

Another framework that is closely related to SCB is {\em stochastic linear bandits (SLB)} \citep{JMLR02_auer2002using,COLT08_dani2008stochastic,MOR10_rusmevichientong2010linearly,NIPS11_abbasi2011improved}. In this setup, the environment is parameterized, and there could be uncountably many arms (within some bounded radius), also referred to as decision set. The arms are characterized into their feature vectors, and the mean reward for playing an arm is given as the inner product of the parameter (unknown) and the feature vector associated with the arm. In situations where the decision set is allowed to vary in each round and are finite, SLBs are equivalent to SCBs, where feature vectors correspond to context-arm pairs \citep{WWW10_li2010contextual,ICML17_li2017provably}. 
For our work, we leverage GLMs as models for disagreement probability between any two arms. While it is tempting to reduce contextual USS to SCBs, note that, unlike prior works, we do not observe loss for our action choices, and so conventional algorithms such as LinUCB and UCB-GLM \citep{WWW10_li2010contextual,ICML14_TamingTheMonster_AgarwalSchapire,ICML17_li2017provably} cannot be applied.

Most of the prior work \citep{AISTATS17_hanawal2017unsupervised, AISTATS19_verma2019online, ACML20_verma2020thompson} considered the problem of learning an optimal action but ignored the contextual information. In this work, we incorporated contextual information, which is readily available in many applications.  Exploiting the {\it real-valued} contextual information (features) for improving the arm selection strategy is non-trivial due to the unsupervised nature of the problem where the standard analysis of contextual bandits does not apply. We made necessary modeling assumptions to leverage GLMs to parameterize the disagreement probability between two arms and extended the existing definitions to address the new setup's learnability issues. However, the problem still requires new ideas and analysis methods to derive an efficient algorithm, which poses new technical challenges for analysis.

%% file: chapter3/problem_setting.tex
%!TEX root =  ../thesis.tex

We consider a stochastic contextual bandits problem with $K$ arms. The set of arms is denoted as $[K]$ where $[K] \doteq \{1,2,\ldots, K\}$. In each round $t$, the environment generates a vector $\left(\Xt, \Yt, \{\Yti\}_{i \in [K]}\right)$. The vector $X_t$ denotes the context in round $t$ and forms an independent and identically distributed (IID) sequence drawn from a bounded set $\cX \subset \R^d$ according to an unknown but fixed distribution $\nu$. The binary reward for context $X_t$ is denoted by $Y_t \in \{0,1\}$, which is hidden from the learner. The vector $\left(\{\Yti\}_{i \in [K]}\right) \in \{0, 1\}^{K}$ represents observed feedback at time $t$, where $\Yti$ denotes the feedback observed after playing arm $i$ with $X_t$ as input\footnote{In our setup, an arm $i$ could be a classifier that outputs label $Y^i$. The classifier's input could be a context and any combinations of feedback observed from classifiers coming before the arm $i$ in the cascade. For example, consider a case where each arm represents a crowd-sourced worker. After using the first $i$ crowd-sourced workers, the final label can be a function of predicted labels of the first $i$ crowd-sourced workers.}. We denote the cost for using arm $i$ as $c_i\geq 0$ that is known and the same for all contexts.

In contextual USS, the arms are assumed to be ordered and form a cascade. When the learner selects an arm $i \in [K]$, the feedback from all arms till arm $i$ in the cascade are observed. The expected loss of playing the arm $i$ for a given context $\xt$ is denoted as $\gamma_i(\xt) \doteq \EE{\one{\Yti \neq \Yt|X=\xt}} = \Prob{\Yti \neq \Yt|X=\xt}$, where $\one{A}$ denotes indicator of event $A$. For soundness, we assume that the probability density function of context distribution is strictly positive on $\mathcal{X}$ such that the conditional probabilities are well defined. The total expected loss incurred by playing arm $i$ for context $\xt$ is defined as $\gamma_i(\xt)+\lambda_iC_i$, where $C_i \doteq c_1 + \ldots + c_i$ and $\lambda_i$ is a trade-off parameter that normalizes the incurred cost and the loss of playing arm $i$. 
\cref{fig:ContxUSS_Cascade} depicts the USS setup.
\begin{figure}[!ht]
	\centering
	\includegraphics[width=0.8\linewidth]{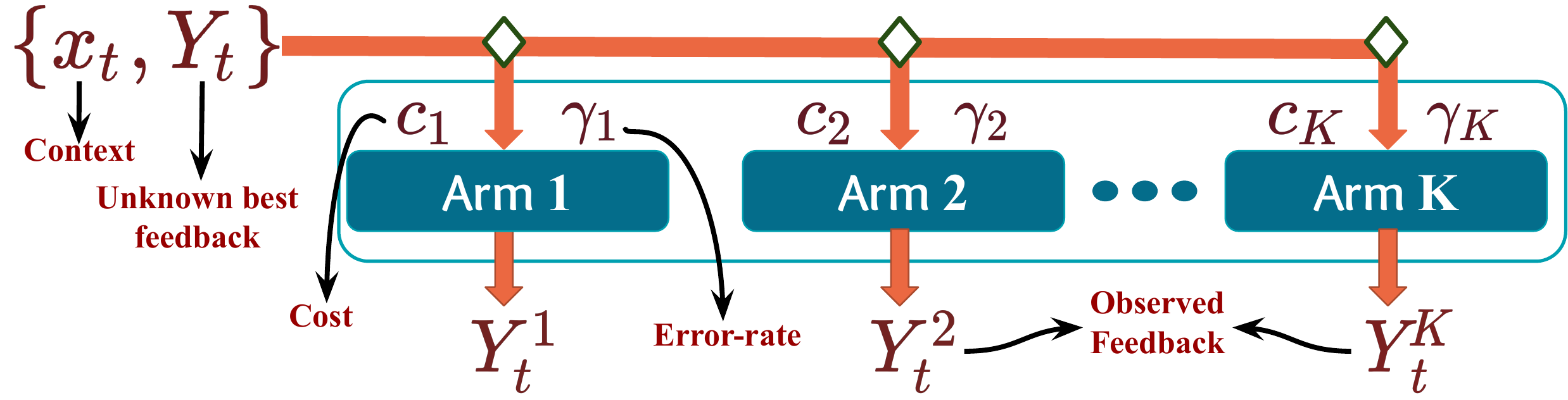}
	\caption{Cascaded Contextual Unsupervised Sequential Selection (USS) setup. In round $t$, $x_t$ is the the context generated by environment, $Y_t$ is the hidden state associated with the context, and $Y_t^1, Y_t^2 \ldots Y_t^K$ are feedback observed from the arms. $\gamma_i$ denotes error rate of the arm $i$ and $c_i$ denotes the cost of using the arm $i$.}
	\label{fig:ContxUSS_Cascade}
\end{figure}

Since the true rewards are hidden from the learner, the expected loss of an arm cannot be inferred from the observed feedback. We thus have a version of the stochastic partial monitoring problem \citep{MOR06_cesa2006regret,ALT12_bartok2012partial,MOR14_bartok2014partial}, and we refer to it as contextual \underline{u}nsupervised \underline{s}equential \underline{s}election (USS). Let $\bq$ be the unknown joint distribution of $(X, Y, Y^1,Y^2 \ldots, Y^K)$. Henceforth we identify a contextual USS instance as $P \doteq (\bq,\bc)$ where $\bc \doteq (c_1, c_2, \ldots, c_K)$ is the known cost vector of arms. We denote the collection of contextual USS instances as $\USS$. For instance $P \in \USS$, the optimal arm for a context $x_t$  is given as follows: 
\eq{
	\label{equ:optimalArmContx}
	\ist \in \max\left\{\argmin _{i \in [K]} \left( \gamma_i(\xt)+\lambda_iC_i \right)\right\},
}
where the choice of $\ist$ is risk-averse as we prefer the arm with the lowest error among the optimal arms. 
The interaction between the environment and a learner is given in Algorithm \ref{alg:ContxUSS}.
\begin{algorithm}[H]
	\caption{Learning on contextual USS instance $(\bq, \bc)$}
	\label{alg:ContxUSS}
	For each round $t$: 
	\begin{enumerate}
		\item \textbf{Environment} chooses a vector $(X_t, \Yt, \{\Yti\}_{i \in [K]})\sim \bq$.
		\item \textbf{Learner} observes a context $X_t=\xt$ and selects an arm $I_t \in [K]$ to stop in cascade.
		\item \textbf{Feedback and Loss:} The learner observes feedback $(\Yt^1, \Yt^2, \ldots, \Yt^{I_t})$ and incurs a total loss $\one{\Yti \neq \Yt|X=\xt} + \lambda_{I_t}C_{I_t}$. 
	\end{enumerate}
\end{algorithm}

The learner's goal is to find an arm for each context such that the cumulative expected loss is minimized. Specifically, for $T$ contexts, we measure the performance of a policy that selects an arm $I_t$ for a context $x_t$ in terms of regret given by
\eq{
	\label{equ:cum_regret}
	\Regret_T = \sum_{t=1}^T\left( \gamma_{I_t}(\xt) + \lambda_{I_t}C_{I_t} - \left(\gamma_\ist(\xt) + \lambda_\ist C_{i^\star_t} \right) \right).
}

We seek policies that yield sub-linear regret, i.e., $\Regret_T/T \rightarrow 0$ as $T \rightarrow \infty$. It implies that the learner collects almost as much reward in the long run as an oracle collects that knew the optimal arm for every context. We say that a problem instance $P \in \USS$ is learnable if there exists a policy such that $\lim\limits_{T \rightarrow \infty}\Regret_T/T = 0$.

In the sequel, we discuss the selection criteria for optimal arm for a given context and the conditions under which instances of $\USS$ are learnable.

%% file: chapter3/learning.tex
%!TEX root =  ../thesis.tex

Next, we introduce the contextual weak dominance property of a problem instance. 
\begin{defi}[Contextual Weak Dominance $(\CWD)$] 
	\label{def:CWD} 
	Let $\ist$ denote optimal arm for context $\xt$. Then the context $\xt$ is said to satisfy weak dominance $(\WD)$ property if
	\eq{
		\label{equ:ContxWDProp}
		\forall j>\ist: C_j - C_\ist > \Prob{\Yts \ne \Ytj |X = \xt}.	
	}
	A problem instance $P \in \USS$  is said to satisfy the $\CWD$ property if all contexts of $P$ satisfy $\WD$ property. We denote the set of all instances in $\USS$ that satisfies $\CWD$ property by $\PCWD$.
\end{defi}

In the following, we use an alternative characterization of the $\CWD$ property, given as
\eq{
    \label{def:ContxXi}
    \xi(x_t) \doteq \min_{j>\ist}\left\{C_j - C_\ist - \Prob{\Yts \ne \Ytj|X = \xt} \right\} > 0.
}
We define $\xi \doteq \inf_{x \in \cX} \xi(x)$ and assume that $\xi>0$. The larger the value of $\xi$, `stronger' is the $\CWD$ property, and easier it is to identify an optimal arm for given contexts. We later characterize the regret upper bounds of proposed algorithms in terms of $\xi$. We also discuss the case when a fraction of contexts satisfies $\WD$ property in \cref{sec:contxuss_appendix}.

\subsection{Selection Criteria for Optimal Arm}
Without loss of generality, we set $\lambda_i=1$ for all $i\in [K]$ as their value can be absorbed into the costs. Since $\ist = \max\big\{\arg\min\limits_{i \in [K]}\left(\gamma_i(x_t)+ C_i \right)\big\}$, it must satisfy following equation:
\begin{subequations}
	\label{eq:ContxCost_exp_err}
	\al{
		&\forall j<\ist \,:\, C_\ist - C_j \leq \gamma_j(\xt)-\gamma_\ist(\xt) \,, \label{eq:cwd1}\\ 
		&\forall j>\ist \,:\, C_j - C_\ist > \gamma_\ist(\xt) - \gamma_j(\xt) \,. \label{eq:cwd2}
	}
\end{subequations}

As the loss of an arm is not observed, the above equations can not lead to a sound arm selection criteria. We thus have to relate the unobservable quantities in terms of the quantities that can be observed. In our setup, we can compare the feedback of two arms, which can be used to estimate the disagreement probabilities between them. For notation convenience, we define $\pijt \doteq \Prob{\Yti \ne \Ytj|X=\xt}$ for $i<j$. The value of $\pijt$ can be estimated as it is observable.
Our next result bounds unobserved error rates differences in terms of their observable disagreement probabilities for a given context.
\begin{restatable}{lem}{ContxErrProbContx}
	\label{lem:ContxErr_prob_contx}
	For any  $i$, $j$, and $x_t\in \mathcal{X}$, $\gamma_{i}(\xt) - \gamma_{j}(\xt) = \pijt - 2\Prob{\Yti = \Yt, \Ytj \ne \Yt| X = \xt}$. 
\end{restatable}
\noindent
The detailed proof of \cref{lem:ContxErr_prob_contx} and all other missing proofs appear in \cref{sec:contxuss_appendix}.

Now, using \cref{lem:ContxErr_prob_contx}, we can replace  \cref{eq:cwd1} by
\eq{
	\label{eq:selectDisProbLowContx}
	\forall j<\ist \,:\, C_{\ist} - C_j \leq  \pjist,
}
which only has observable quantities. 
For $j>\ist$, using the $\CWD$ property, we replace \cref{eq:cwd2} by
\eq{
	\label{eq:selectDisProbHighContx}
	\forall j>\ist \,:\, C_j - C_{\ist} >  \pijst.
}

Using \cref{eq:selectDisProbLowContx} and \cref{eq:selectDisProbHighContx}, our next result gives the optimal arm for a given context $x_t$.
\begin{restatable}{lem}{SetBContx}
	\label{lem:BContx}
	Let $P \in \PCWD$ and $\Bt = \left\{i: \forall j>i, C_j - C_i > \pijt \right\}\cup \{K\}$. Then the arm $I_t =\min(\Bt)$ is the optimal arm for a context $x_t$.
\end{restatable}
By construction, the optimal arm lies in set $\Bt$. Because of \cref{eq:selectDisProbHighContx}, any sub-optimal arm having smaller index than optimal arm do not satisfy \cref{eq:selectDisProbHighContx}, hence it can not be in set $\Bt$. Therefore, the smallest arm of set $\Bt$ is the optimal arm. 
\begin{restatable}{thm}{learnCWD}
	\label{thm:learnCWD}
	The set $\PCWD$ is maximal learnable. 
\end{restatable}

\noindent
The proof establishes that under the $\CWD$ property, there exists a `sound' arm selection policy that identifies the optimal arm for each context. The sound policy only uses conditional disagreement probabilities between pairs of arms that can be estimated from the feedback of arms.

%% file: chapter3/glmModel.tex
%!TEX root =  ../thesis.tex

Since the number of contexts could be much larger (can be infinite) than the learning horizon, in stochastic contextual bandits, a correlation structure is assumed between the reward (loss) and the contexts \citep{JMLR02_auer2002using, WWW10_li2010contextual, ICML17_li2017provably}. It is often realized via parameterization of the arms such that expected rewards (or losses) observed from an arm depend on the unknown parameter. In our setting, we cannot observe a loss for any arm. Hence parameterization of an expected loss of the arms is not useful.  However, we can obtain feedback of two arms for a given context and can compare them. For example, we can check whether two arms' feedback agrees or disagrees for a given context. Thus, we assume a correlation structure on the disagreement probability for a pair of arms across the contexts and parameterize it using generalized linear models. For $i<j$ and context $x_t$, the disagreement probability for $(i,j)$ pair of arms is given via a function $\mu$ as follows:
\eq{
	\label{equ:disProb}
	\Prob{\Yti \ne \Ytj|X=\xt}=\mu(\Phi_{ij}(\xt)^\top \Ts), 
} 
where $\xt \in \R^{d}$, $\Phi_{ij}: \R^{d} \rightarrow \R^{d^\prime}$ is a feature map for some $d^\prime \ge d$,\footnote{Let $\R^{d_{ij}}$ be the space where \cref{equ:disProb} holds for $(i,j)$ pair of arms, and $\Phi_{ij}$ is the feature map that lift $\xt$ from $\R^{d}$ space to $\R^{d_{ij}}$ space. For simplicity, we take $d^\prime = \max_{\forall i < j \le K} d_{ij}$.} and $\Ts \in\R^{d^\prime}$ is the unknown parameter for $(i,j)$ pair.

We assume the following assumptions on context distribution $\nu$ and function $\mu$, which is standard in the GLM bandit literature \citep{NIPS10_filippi2010parametric,ICML17_li2017provably}:
\begin{assu}[GLM]
	\begin{itemize}
		\setlength{\itemsep}{0pt}
		\setlength{\parskip}{3pt}
		\setlength{\itemindent}{-1.75em}
		\item For all $x \in \mathcal{X}$ and $(i,j)$ pairs, $\norm{\Phi_{ij}(x)}_{2} \le 1$.
		\item $\kappa \doteq \inf_{\norm{ x }_2 \le 1, \norm{ \theta - \Ts }_2 \le 1}$ $\dot{\mu}(\Phi_{ij}(x)^\top \theta) > 0$ for all $(i,j)$ pairs.
		\item There exists a constant $\lambda_\Sigma > 0$ such that $\lambda_{min}\left( \EE{\Phi_{ij}(X)\Phi_{ij}(X)^\top} \right) \ge \lambda_\Sigma$ for all $(i,j)$ pairs.
		\item The function $\mu : \R \rightarrow [0,1]$ is continuously differentiable and Lipschitz with constant $k_\mu$. 
	\end{itemize}	
\end{assu}

For our setting, the function $\mu$ is defined as $\mu(z) ={1}/{(1+\mathrm{e}^{-z})}$, which is the logistic function. The logistic function is widely used function for binary classification model and has $k_\mu \le 1/4$.

In contextual USS setup, we can compare the arms' feedback and check whether they agree or not for a given context. These binary observations (agree or disagree) can be treated as noisy samples of the disagreement probability. The noise in the binary observation obtained by comparing the feedback of $(i, j)$ pair of arms in round $t$, is given by 
\als{
	\epsilon_{ij}^{(t)}=
	\begin{cases}
		1 - \mu(\Phi_{ij}(\xt)^\top \Ts ), & \text{with probability } \mu(\Phi_{ij}(\xt)^\top \Ts ) \\
		 - \mu(\Phi_{ij}(\xt)^\top \Ts ), & \text{with probability } \left(1-\mu(\Phi_{ij}(\xt)^\top \Ts) \right)
	\end{cases}	
}
where $\epsilon_{ij}^{(t)}$ is $\F_{t}$-measurable with $\EE{\epsilon_{ij}^{(t)}|\F_t}=0$. Here $\F_t$ denotes sigma algebra generated by history $\left\{\left(X_s, I_s, \left\{Y_s^i\right\}_{i \in [I_s]}\right) \right\}_{s \in [t]}$ till time $t$. Since $\epsilon_{ij}^{(t)}$ is a zero-mean shifted Bernoulli random variable, $\epsilon_{ij}^{(t)}$ satisfies the following sub-Gaussian condition with parameter $\sigma \in (0,1)$:
\eqs{
	\EE{\exp(\lambda\epsilon_{ij}^{(t)})|\F_{t}} \le \exp\left(\frac{\lambda^2\sigma^2}{2}\right), \hspace{2mm} \forall \lambda \in \R.
}

Let $d_{ij}(t) \doteq \one{\Yti \neq \Ytj|X=\xt}$ be the disagreement indicator for a context $\xt$ and $S_{ij}^t$ be the set of indices of contexts for which disagreements are observed for $(i,j)$ pair of arms till round $t$. In round $t$, we estimate $\Ts$, denoted by $\Te$, using the following equation adapted from the maximum likelihood estimator (MLE) used for GLM bandits \citep{ICML17_li2017provably}:
\al{
	\sum_{s\in S_{ij}^t}&\left(d_{ij}(s) - \mu(\Phi_{ij}(x_s)^\top\theta)\right) \Phi_{ij}(x_s) = 0. \label{equ:glm_est}
}

In the next section, we develop an algorithm that exploits \cref{lem:BContx} for selecting the optimal arm to each context. It replaces the terms $\pijt$ in \cref{lem:BContx} by their optimistic estimates.

%% file: chapter3/algorithm.tex
%!TEX root =  ../thesis.tex

Our algorithm, named \ref{alg:CUSS_WD}, is based on the {\it optimism-in-the-face-of-uncertainty} (OFU) principle. \ref{alg:CUSS_WD} works as follows: it takes $\delta$ and $m$ as inputs, where $\delta$ is the confidence in the estimated parameters and used for computing confidence bound for $\Ts$ as given by \cref{lem:theta_est_glm}. The choice of $m$ ensures that with probability at least $(1-\delta)$, the sample correlation matrix $V_{ij}^t=\sum_{s \in S_{ij}^t}\Phi_{ij}(x_s)\Phi_{ij}(x_s)^\top$ for each $(i,j)$ pair where $i<j$, is invertible. A high probability upper bound on $m$ is computed using \cref{lem:min_eigen_lb}. The algorithm collects feedback from all arms by selecting the arm $K$ irrespective of the context received for first $m$ rounds. After $m$ rounds, the sample correlation matrix and the estimate of $\Ts$ are computed for each $(i,j)$ pair.

For $t>m$, the learner receives a context $x_t$ and plays the arm $i=1$ and then observe its feedback. For each $(i,j)$ pair and context $\xt$, the upper bound on disagreement probability $\tpijt$ is computed using $\Te$ and confidence bonus $\alpha_{ij}^t\norm{ \Phi_{ij}(x_t)}_{\Vinv}$. Here the notation $\norm{x}^2_A \doteq x^\top Ax$ denotes the weighted $l_2$-norm of vector $x \in \R^d$ with respect to a positive definite matrix $A \in \R^{d\times d}$. The confidence bonus has two terms. The first term $\alpha_{ij}^t$ is a slowly increasing function in $t$ whose value is specified in Lemma \ref{lem:theta_est_glm}, and the second term $\norm{ \Phi_{ij}(x_t)}_{\Vinv}$ decreases to zero as $t$ increases. 

\begin{algorithm}[H]
    \renewcommand{\thealgorithm}{USS-PD}
    \floatname{algorithm}{}
    \caption{Algorithm for Contextual USS using Pairwise Disagreement}
	\label{alg:CUSS_WD}
	\begin{algorithmic}[1]
		\State \textbf{Input:} Tuning parameters: $\delta \in (0,1)$ and $m>0$
		\State Select arm $K$ for first $m$ contexts
		\State $\forall i< j \le K:$ set ${V}_{ij}^{m} \leftarrow \sum_{t=1}^{m}\Phi_{ij}(x_{t}) {\Phi_{ij}(x_{t})}^\top$ and update $\hat\theta_{ij}^m$ by solving \cref{equ:glm_est} 
		\For{$t= m+1, m+2, \ldots$}
			\State Receive context $x_t$. Set $i=1$ and $I_t=0$
			\While {$I_t=0$}
				\State Play arm $i$ 
				\State $\forall j \in [i+1, K]:$ compute $\tpijt \leftarrow \mu\left(\Phi_{ij}(\xt)^\top \hat\theta_{ij}^{t-1} + \alpha_{ij}^{t-1}\norm{ \Phi_{ij}(\xt)}_{\left({V}_{ij}^{t-1}\right)^{-1}} \right)$
				\State If $\forall j \in [i+1, K]: C_j - C_i > \tpijt$ or $i=K$ then set $I_t = i$ else  set $i=i+1$
			\EndWhile
			\State Select arm $I_t$ and observe $Y_t^1, Y_t^2, \dots, Y_t^{I_t}$
			\State $\forall i< j \le I_t:$  update $\Vt \leftarrow {V}_{ij}^{t-1} + \Phi_{ij}(x_{t}) {\Phi_{ij}(x_{t})}^\top$ and $\Te$ by solving \cref{equ:glm_est} 
		\EndFor
	\end{algorithmic}
\end{algorithm}

\iffalse
\begin{algorithm}[H]
	\renewcommand{\thealgorithm}{USS-PD}
	\floatname{algorithm}{}
	\caption{Algorithm for Contextual USS using Pairwise Disagreement}
	\label{alg:CUSS_WD}
	\begin{algorithmic}[1]
		\State \textbf{Input:} Tuning parameters: $\delta \in (0,1)$ and $m>0$
		\State Select arm $K$ for first $m$ contexts
		\State $\forall i< j \le K:$ set ${V}_{ij}^{m} \leftarrow \sum_{t=1}^{m}\Phi_{ij}(x_{t}) {\Phi_{ij}(x_{t})}^\top$ and update $\hat\theta_{ij}^m$
		\For{$t= m+1, m+2, \ldots$}
		\State Receive context $x_t$. Set $i=1$ and $I_t=0$
		\Do
		\State Play arm $i$ 
		\State $\forall j >i: \tpijt \leftarrow \mu\left(\Phi_{ij}(\xt)^\top \hat\theta_{ij}^{t-1} + \alpha_{ij}^{t-1}\norm{ \Phi_{ij}(\xt)}_{\left({V}_{ij}^{t-1}\right)^{-1}} \right)$
		\If {$\forall j >i: C_j - C_i > \tpijt$ or $i=K$}
			\State set $I_t = i$ 
		\Else 
			\State set $i=i+1$
		\EndIf
		\doWhile{$I_t=0$}
		\State Select arm $I_t$ and observe $Y_t^1, Y_t^2, \dots, Y_t^{I_t}$
		\State $\forall i< j \le I_t:$  update $\Vt \leftarrow {V}_{ij}^{t-1} + \Phi_{ij}(x_{t}) {\Phi_{ij}(x_{t})}^\top$ and $\Te$
		\EndFor
	\end{algorithmic}
\end{algorithm}
\fi
After computing $\tpijt$, the algorithm checks whether the arm $i$ is the best arm using \cref{eq:selectDisProbHighContx} with $\tpijt$ in place of $\pijt$. If the arm $i$ is not the best, then the algorithm plays the next arm, and then the same process is repeated. If the arm $i$ is the best arm for context, then the algorithm stops at that arm with $I_t =i$ for that context. After selecting arm $I_t$, the feedback from arms $1, \ldots, I_t$ are observed. After that, the values of $V^t_{ij}$ are updated, and $\Te$ are re-estimated. The same process is repeated for subsequent contexts.

\begin{rem}
    GLM bandits are well studied but require reward or loss information. In the USS setup, loss of selected arm can not be observed; hence finding the optimal arm is challenging. Due to binary disagreement, \ref{alg:CUSS_WD} uses the MLE estimator for $\Ts$ as used in GLM bandits \citep{NIPS10_filippi2010parametric,ICML17_li2017provably}. However, the feedback structure and the way arms are selected in the USS setup differ from that in the GLM bandits. Further, our analysis needs carefully connecting the regret with the bad events that make \ref{alg:CUSS_WD} selects non-optimal arms.
\end{rem}

\begin{rem}
    We force the algorithm to explore until the correlation matrix $V_{ij}^t$ is invertible for all $(i,j)$ pairs. The invertibility can also be ensured by adding a regularization term \citep{NIPS11_abbasi2011improved,ICML16_zhang2016online,NIPS17_jun2017scalable} to avoid forced exploration. However, the analysis of \ref{alg:CUSS_WD} with regularization term still required to the non-regularized part of the sample correlation matrix becomes invertible. See \cref{asec:contxuss_lambda_algorithm} for the algorithm and its analysis. 
\end{rem}

\subsection{Regret Analysis of \ref{alg:CUSS_WD}}

The following definition is useful in our regret analysis. 
\begin{defi}[Arm Preference ($\succ_t$)]
	\label{def:contx_preference} \ref{alg:CUSS_WD} prefers an arm $i$ over $j$ for context $x_t$ if 
	\begin{subnumcases}	
		{i \succ_t j \doteq }
		C_i - C_j < \tpjit,  &\textnormal{if $j<i$} \label{def_contx_prefer_l} \\
		C_j - C_i > \tpijt,	&\textnormal{if $j>i$} \label{def_contx_prefer_h}.  
	\end{subnumcases}
\end{defi}

Our next result bounds the number of disagreement observations required from a pair of arms say $(i,j)$, such that the smallest eigenvalue of its sample correlation matrix $V_{ij}$ matrices is larger than a fixed value. This result uses the standard results from random matrix theory \citep{Book_vershynin_2012}.
\begin{restatable}{lem}{minEigenLB}
	\label{lem:min_eigen_lb}
	Let ${V}_{ij}^{t} = \sum_{s \in S_{ij}^t}\Phi_{ij}(x_{s}) {\Phi_{ij}(x_{s})}^\top$, $\Sigma_{ij} = \EE{\Phi_{ij}(X)\Phi_{ij}(X)^\top}$, $\Psi$ and $\delta \in (0,1)$ be two positive constants. Then, there exist positive universal constants $C_1$ and  $C_2$ such that the minimum eigenvalue of $\lambda_{min}({V}_{ij}^{t}) \ge \Psi$ with probability at least $1-2\delta/K^2$, iff	
	\eqs{
		|S_{ij}^t| \ge \left(\frac{C_1\sqrt{d^\prime}+C_2\sqrt{\log (K^2/2\delta)}}{\lambda_{min}(\Sigma_{ij})}\right)^2+ \frac{2\Psi}{\lambda_{min}(\Sigma_{ij})}.
	}
\end{restatable}

The next result is adapted to our setting from the confidence bounds for  maximum likelihood estimator used in GLM bandits \citep{ICML17_li2017provably}.
\begin{restatable}[Confidence Ellipsoid]{lem}{thetaEstGLM}
	\label{lem:theta_est_glm} 
	Let $m$ be such that $\lambda_{min}(V_{ij}^{m+1}) \ge 1$ for any pair $(i,j)$. Then the following event holds with probability at least $1-2\delta/K^2$ for \ref{alg:CUSS_WD}:
	\begin{align*}
	&\norm{ \Te - \Ts }_{\Vt} \le \alpha_{ij}^t, \;\forall t > m 
	\end{align*}
	where $\alpha_{ij}^t = \frac{2\sigma}{\kappa}\sqrt{\frac{d^\prime}{2}\log\left(1 + \frac{2t}{d^\prime} \right) + \log\left(\frac{K^2}{2\delta}\right)}$.
\end{restatable}

The regret analysis of GLM bandits hinges on bounding the instantaneous regret in each round, which is tied to the estimation error of the GLM parameters. Due to the unsupervised setting and cascade structure, this way of bounding regret does not work in our setup. Our analysis goes by bounding the number of pulls of the sub-optimal arms. However, unlike standard bandits, we have to distinguish whether the sub-optimal arm pulled by \ref{alg:CUSS_WD} is on the `left' or `right' of the optimal arm in the cascade. It requires our analysis to handle both the cases carefully. Since \ref{alg:CUSS_WD} uses a similar MLE estimator for parameter estimation as in GLM bandits, we only adapt their asymptotic normality results. Our next results give conditions when \ref{alg:CUSS_WD} prefers a sub-optimal arm for a context. 
\begin{restatable}{lem}{subOptimalLowerSelection}
	\label{lem:subOptimalLowerSelection}
	Let $P \in \PCWD$. Then \ref{alg:CUSS_WD} prefers any sub-optimal arm $l < i^\star_t$ for context $x_t$ with probability at most $\delta/2$.
\end{restatable}

\begin{restatable}{lem}{subOptimalUpperSelection}
	\label{lem:subOptimalUpperSelection}
	Let $P \in \PCWD$. If \ref{alg:CUSS_WD} prefers a sub-optimal arm $h > i^\star_t$ for context $x_t$ then
	\begin{equation*}
		2k_\mu \alpha_{i^\star_t h}^t > \xi_{i^\star_t h}(x_t)\sqrt{\lambda_{min}(\Vths)}.
	\end{equation*} 
	where $\xi_{\ist h} = C_h - C_\ist - \pihts$ and $\alpha_{ij}^t$ is given by \cref{lem:theta_est_glm}.
\end{restatable}

Let $m \doteq C\lambda_\Sigma^{-2}\left(d^\prime+\log(k^2/2\delta)\right) + 2\lambda_\Sigma^{-1}$, where $C>0$ is the universal constant and $R_{max} \doteq \max_{i\in [K], x\in \mathcal{X}}$ $ \left[C_i +  \gamma_{i}(x) - \left( C_{i^\star} + \gamma_{i^\star}(x)\right)\right]$, where $i^\star$ is the optimal arm for context $x$.
Now we state the regret upper bound of \ref{alg:CUSS_WD}.
\begin{restatable}[Regret Upper Bound]{thm}{cumRegGLM}
	\label{thm:cum_reg_glm}
	Let $P \in \PCWD$, $\delta \in (0,1)$, Assumption 1 holds, and $\xi_h = \min\limits_{t \ge 1} \xi_{i^\star_t h}(x_t)$. Then with probability at least $1-2\delta$, the regret of \ref{alg:CUSS_WD} for $T > m$ contexts is
	\begin{align*}
		\Regret_T &\le R_{max}\Bigg[m + \sum_{h=2}^{K} \Bigg(\hspace{-1mm} \Bigg(\frac{C_1\sqrt{d^\prime} + C_2\sqrt{\log \left(\frac{K^2}{2\delta}\right)}}{\lambda_{\Sigma}}\Bigg)^2  \hspace{-2mm} + \frac{16}{\lambda_{\Sigma}}\\
		&\qquad\left(\frac{k_\mu\sigma}{\xi_{h}\kappa}\right)^2 \left(\frac{d^\prime}{2}\log\left(1 + \frac{2T}{d^\prime} \right) + \log\left(\frac{K^2}{2\delta}\right) \right)\Bigg)\Bigg].
	\end{align*} 
\end{restatable}

\begin{restatable}{cor}{orderCumRegGLM}
	\label{cor:cum_reg_glm}
	Let technical conditions stated in \cref{thm:cum_reg_glm} hold. Then with probability at least $1-2\delta$
	\begin{equation*}
        \Regret_T \le O\left({Kd^\prime\log(T)}/{\xi^2}\right).	    
	\end{equation*}
\end{restatable}

The regret of \ref{alg:CUSS_WD} for instance $P \in \PCWD$ is logarithmic in $T$ and grows linearly with $d^\prime$ and $K$. The regret is inversely dependent on the value of $\xi \doteq \min\limits_{h\ge 2} \xi_h$ (measure how well $\CWD$ holds), which implies the problem instance with smaller $\xi$ has more regret and vice-versa. The value of $\xi$ is analogous to the minimum sub-optimality gap in the standard Multi-Armed Bandits setting. With a large context set, $\xi$ can be small, and its inverse relation in the regret captures the difficulty of the USS problem.

%% file: chapter3/experiment.tex
%!TEX root =  ../thesis.tex

We evaluate the performance of \ref{alg:CUSS_WD} on different problem instances derived from synthetic and real datasets. In our experiments, the data samples are treated as contexts. The labels of contexts are known but are never revealed to the algorithm. We use the labels to train classifiers offline that act as arms. Arm $i$ represents a logistic classifier with trained parameter $\theta_i$. A context (data sample) $x$ is assigned label $1$ from the $i$-th classifier with probability $\mu(x^\top\theta_i)$ and label $0$ with probability $1-\mu(x^\top\theta_i)$. The disagreement labels for $(i,j)$ pair is computed using the labels of classifier $i$ and $j$. To satisfy \cref{equ:disProb}, we use the polynomial kernel of degree two for mapping context into higher-dimensional space. Unlike other kernels, the polynomial kernel uses a well-defined feature map to lift the contexts into fixed, higher-dimensional space. 
The details of the used problem instances are as follows.

\noindent
\textbf{Synthetic Dataset:} We consider $3$-dimensional synthetic dataset with $5000$ data samples. Each sample is represented by $x=(x_1, x_2, x_3)$, where the value of $x_j$ is drawn uniformly at random from $(-1, 1)$. A sample $x$ is labeled $0$ if the value of $(x_1 + x_1x_2 + x_3^2)$ is negative otherwise it is labeled $1$. We train five logistic classifiers on this synthetic dataset by varying the regularization parameter.  We then assign a positive cost to each classifier and order them by their increasing cost. We vary the cost of using classifiers to get different problem instances (see details in \cref{asec:contxuss_moreExperiments}). 

\noindent
\textbf{Real Datasets:} We applied our algorithm on PIMA Indian Diabetes \citep{UCI16_pima2016kaggale} dataset. Each sample has $8$ features related to the conditions of the patient. We split the features into three subsets and train a logistic classifier on each subset. We associate 1st classifier with the first $6$ features as input. These features include patient history/profile. The 2nd classifier, in addition to the $6$ features, utilizes the feature on the glucose tolerance test, and the 3rd classifier uses all the previous features and the feature that gives values of insulin test. Due to space constraints, the experiment results on Heart Disease dataset \citep{HEART98_robert1988va, UCI17_Dua2017} are given in \cref{asec:contxuss_moreExperiments}.

\subsection{Experiments Results} 
We compare the performance of \ref{alg:CUSS_WD} on four problem instances derived from the synthetic dataset. The instances vary based on the cost of arms. All contexts in Instance $1$ do not satisfy $\WD$ property; hence it suffers linear regret as shown in \cref{fig:syn}. For the remaining instances, we set costs such that the value of $\xi$ increasing from Instance $2$ to $4$. As expected, the regret decreases from Instance $2$ to $4$, as seen in \cref{fig:syn}. We also compare \ref{alg:CUSS_WD} against an algorithm where the learner receives true labels as feedback. In particular, the learner knows whether the classifier's output is correct or not and can estimate their error rates. We implement this `supervised' setting by replacing disagreement probability in \cref{eq:selectDisProbHighContx} with estimated error rates. As expected, the regret with supervision has lower than the \ref{alg:CUSS_WD} regret (unsupervised) in \cref{fig:supComp}. It is qualitatively interesting because these plots demonstrate that, in typical cases, our unsupervised algorithm can eventually learn to perform as good as an algorithm with knowledge of true labels. 

\begin{figure}[!ht]
	\captionsetup[subfigure]{justification=centering}
	\begin{subfigure}[b]{0.30\linewidth}
		\centering
		\includegraphics[scale=0.30]{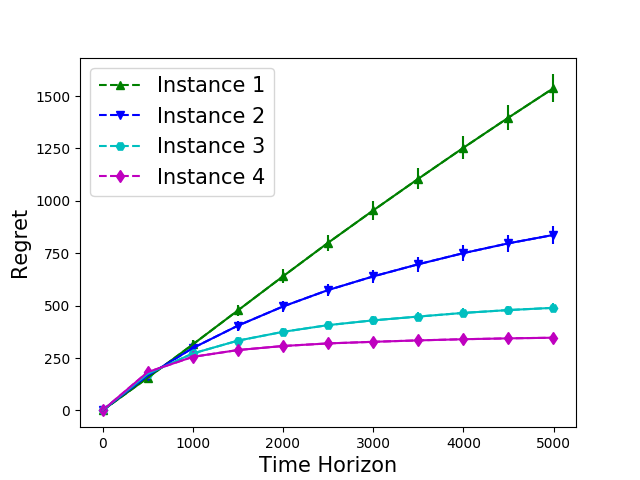}
		\caption{\small Synthetic dataset: Regret for different instance}
		\label{fig:syn}
	\end{subfigure}
	\quad
	\begin{subfigure}[b]{0.30\linewidth}
		\centering
		\includegraphics[scale=0.30]{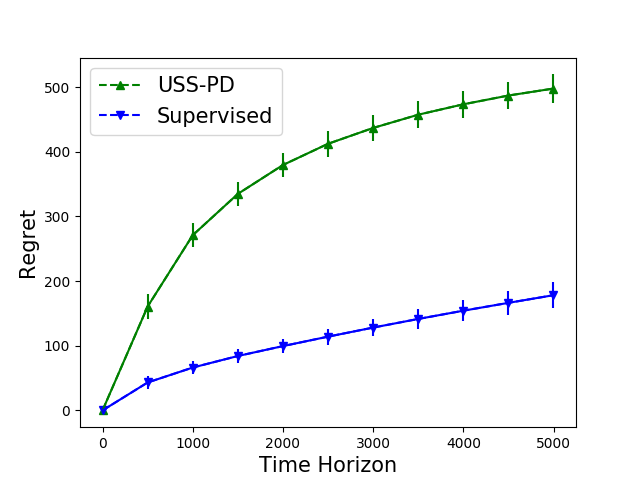}
		\caption{\small Supervised Setting (PI 3 of Synthetic Dataset)}
		\label{fig:supComp}
	\end{subfigure}
	\quad
	\begin{subfigure}[b]{0.275\linewidth}
		\centering
		\includegraphics[scale=0.275]{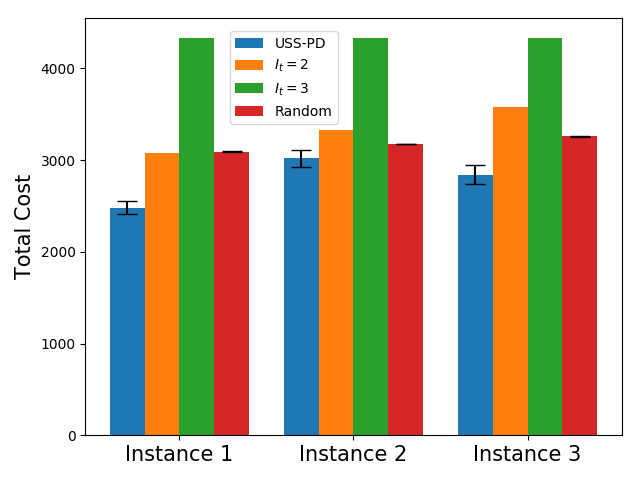}
		\caption{\small Total cost for PIMA Indian Diabetes dataset}
		\label{fig:diabetes}
	\end{subfigure}	
	\caption{\small Performance of \ref{alg:CUSS_WD} on different problem instances (PIs) derived from synthetic and real datasets.}
	\label{fig:ussExp}
\end{figure}

We derive three problem instances from PIMA Indian Diabetes dataset by varying the costs of using classifies. Since all contexts of these problem instances do not satisfy $\WD$ property (see details in \cref{asec:contxuss_moreExperiments}), we used cumulative total expected cost as a performance measure, where the cumulative total expected cost is given by $\sum_{t=1}^T(\gamma_{I_t}(x_t) + C_{I_t})$. We compare the performance of \ref{alg:CUSS_WD} with three baseline policies -- the first baseline policy uses the third classifier irrespective of contexts, and it is denoted as policy `$I_t=3$' (plays arm $3$ in each round). The second baseline policy uses the second classifier for all contexts, and it is denoted as policy `$I_t=2$'. The third baseline policy is `Random,' which selects an arm uniformly at random in each round. In all three problem instances, we observe that \ref{alg:CUSS_WD} performs better than the baselines, as shown in \cref{fig:diabetes}.

We repeat each of the above experiments $100$ times, and then the average regret is presented with a $95$\% confidence interval. The vertical line on each plot shows the confidence interval.

%% file: chapter3/appendix.tex
%!TEX root =  ../thesis.tex

\subsection*{Missing proofs from \cref{sec:contxuss_problem_setting}}

\ContxErrProbContx*
\begin{proof} 
	Using definition of $\gamma_i(\xt) \doteq \Prob{\Yti \neq \Yt|X=\xt}$, we get	
	\begin{align*}
		\gamma_i(\xt) - \gamma_j(\xt) &= \Prob{\Yti \neq \Yt|X=\xt} - \Prob{\Ytj \neq \Yt|X=\xt}. 
	\end{align*}
	As the observed feedback is binary, if $\Yti = \Ytj$ and $\Yti \ne \Yt$ then $\Ytj \ne \Yt$,
	\begin{align*}
		\gamma_i(\xt) - \gamma_j(\xt)&= \cancel{\Prob{\Yti \neq\Yt, \Yti = \Ytj|X=\xt}} + \Prob{\Yti \neq\Yt, \Yti \ne \Ytj|X=\xt} \\
		&- \cancel{\Prob{\Ytj \neq\Yt, \Yti = \Ytj|X=\xt}} - \Prob{\Ytj \neq\Yt, \Yti \ne \Ytj|X=\xt}.
	\end{align*}
	Adding and subtracting $\Prob{\Yti =\Yt, \Yti \ne \Ytj|X=\xt}$,
	\begin{align*}
		\gamma_i(\xt) - \gamma_j(\xt)&= \Prob{\Yti \neq\Yt, \Yti \ne \Ytj|X=\xt} + \Prob{\Yti =\Yt, \Yti \ne \Ytj|X=\xt} \\
		& - \Prob{\Ytj \neq\Yt, \Yti \ne \Ytj|X=\xt} - \Prob{\Yti =\Yt, \Yti \ne \Ytj|X=\xt}. 
	\end{align*}
	If $\Yti \ne \Ytj$ and $\Ytj \ne\Yt$ then $\Yti =\Yt$,
	\begin{align*}
		\gamma_i(\xt) - \gamma_j(\xt)&= \Prob{\Yti \ne \Ytj|X=\xt} - \Prob{\Yti =\Yt, \Yti \ne \Ytj|X=\xt}\\
		&\qquad - \Prob{\Yti =\Yt, \Yti \ne \Ytj|X=\xt} \\
		&= \Prob{\Yti \neq \Ytj|X=\xt} -2\Prob{\Yti =\Yt, \Ytj \ne \Yt|X=\xt}. \\
		\implies \gamma_i(x_t^i) -\gamma_j(x_t^j) & = \pijt -2\Prob{\Yti =\Yt, \Ytj \ne \Yt|X=\xt}. \tag*{\qedhere}
	\end{align*}
\end{proof}

\SetBContx*
\begin{proof}
	Let $i_t^\star$ be an optimal arm for a context $x_t$. As $\pijt =\mathbb{P}\{\Yti \ne\Ytj |X=\xt \}$ and $i_t^\star$ is an optimal arm, we have $\forall j<i_t^\star:\, C_{i_t^\star} - C_j \le \mathbb{P}\{\Yt^{i_t^\star}\ne\Yt^j |X=\xt \} \implies C_{i_t^\star} - C_j \ngtr \mathbb{P}\{\Yt^{i_t^\star}\ne\Yt^j |X=\xt \} \implies \forall j<i_t^\star \notin \Bt$.  If any sub-optimal arm $h \in \Bt$ then $h > i_t^\star$ i.e.,
	\eqs{
		\Bt = \{i_t^\star, h_1, \ldots, h_n, K\},
	} 	
	where $i_t^\star < h_1 < \cdots < h_n < K$.	By construction of set $\Bt$, the minimum indexed arm in set $\Bt$ is only the optimal arm.
\end{proof}

\noindent
We need the following results to proof of \cref{thm:learnCWD}.
\begin{lem}
	\label{lem:xCostRange1}
	Let $i<j$ and $x_t \in \cX$ be any context. Assume 
	\begin{equation}
		\label{eqn:xCostRange1}
		C_j -C_i \notin \left (\gamma_i(x_t)-\gamma_j(x_t), \Prob{\Yti \ne \Ytj|X=\xt}\right].
	\end{equation}
	Then, $C_j-C_i >  \gamma_i(x_t)-\gamma_j(x_t) $ iff $C_j-C_i > \Prob{\Yti \ne \Ytj|X=\xt}$.
\end{lem}

\begin{proof}
	Assume that $C_j-C_i >  \gamma_i(x_t)-\gamma_j(x_t)$. As $C_j -C_i \notin \Big(\gamma_i(x_t)-\gamma_j(x_t),$ $\Prob{\Yti \ne \Ytj|X=\xt}\Big]$, we get $C_j-C_i > \Prob{\Yti \ne \Ytj|X=\xt}$.
	The proof of other direction follows by noting that  $ \Prob{\Yti \ne \Ytj|X=\xt}\geq  \gamma_i(x_t)-\gamma_j(x_t)$.
\end{proof}

\begin{lem}
	\label{lem:xCostRange2}
	Let $i>j$ and $x_t \in  \cX$ be any context. Assume 
	\begin{equation}
		\label{eqn:xCostRange2}
		C_i -C_j \notin \left (\gamma_j(x_t)-\gamma_i(x_t), \Prob{\Yti \ne \Ytj|X=\xt}\right ].
	\end{equation}
	Then, $C_i-C_j \leq \gamma_j(x_t)-\gamma_i(x_t) $ iff $C_j-C_i \leq \Prob{\Yti \ne \Ytj|X=\xt}$.
\end{lem}

\begin{proof}
	Let $C_i-C_j \leq  \gamma_j(x_t)-\gamma_i(x_t) $. As $\gamma_j(x_t)-\gamma_i(x_t) \leq \Prob{\Yti \ne \Ytj|X=\xt}$, we get $C_i-C_j \leq \Prob{\Yti \ne \Ytj|X=\xt}$.
	
	The condition $C_i-C_j \leq \Prob{\Yti \ne \Ytj|X=\xt}$ along with $C_i -C_j \notin \Big(\gamma_j(x_t)-\gamma_i(x_t),$ $\Prob{\Yti \ne \Ytj|X=\xt}\Big]$ implies the other direction, i.e., $C_i-C_j \leq  \gamma_j(x_t)-\gamma_i(x_t) $. 
\end{proof}

\begin{lem}
	\label{lem:xWD2} 
	Let  $\ist$ be an optimal arm for a context $x_t$. Any problem instance $P \in \USS$ is learnable if for every context in $P$ following holds:
	\begin{equation*}
		\forall j>\ist,\; C_j -C_\ist >\Prob{\Yt^\ist \ne \Ytj|X=\xt}.
	\end{equation*}
\end{lem}
The proof of \cref{lem:xWD2} follows from \cref{lem:xCostRange1} and \cref{lem:xCostRange2}. 
Now we give proof for Theorem \ref{thm:learnCWD}.

\learnCWD*
\begin{proof}
	Let $\ist$ be an optimal arm for a context $x_t$. It is enough to prove that any problem instance $P \in \USS$  is learnable if
	\begin{align*}
		\forall j>\ist,\; C_j -C_\ist >\Prob{\Yt^\ist \ne \Ytj|X=\xt}. \hspace{4mm}\text{(definition of $\CWD$ property)}
	\end{align*} 
	From \cref{lem:xCostRange1} and \cref{lem:xCostRange2}, if the optimal arm satisfies following conditions,
	\begin{align*}
		&\forall j>\ist, C_j -C_\ist \notin \left (\gamma_\ist(x_t)-\gamma_j(x_t), \Prob{\Yt^\ist \ne \Ytj|X=\xt}\right] \text{ and}\\
		&\forall j<\ist, C_{i^\star} -C_j \notin \left (\gamma_j(x_t)-\gamma_\ist(x_t), \Prob{\Yt^\ist \ne \Ytj|X=\xt}\right ],
	\end{align*}
	
	then, for $j >\ist, C_j -C_\ist>\gamma_\ist(x_t) - \gamma_j(x)$	iff 	$C_j -C_\ist >\Prob{\Yt^\ist \ne \Ytj|X=\xt}$ and for  $j <\ist, C_{i^\star}-C_j\leq  \gamma_{j}(x) -\gamma_\ist(x_t)$	iff 	$C_j -C_\ist \leq \Prob{\Yt^\ist \ne \Ytj|X=\xt}$. Hence we can use $\Prob{\Yti \ne \Ytj|X=\xt}$ as a proxy for $\gamma_\ist(x) -\gamma_j(x)$ to make decision about the optimal arm.
	Now notice that for $j <\ist$, $C_{i^\star}-C_j \leq \gamma_{j}(x) -\gamma_\ist(x_t)$. Hence, 
	\begin{align*}
		& \forall j<\ist, C_{i^\star} -C_j \notin \left (\gamma_j(x_t)-\gamma_\ist(x_t), \Prob{\Yt^\ist \ne \Ytj|X=\xt}\right] \text{ and} \\
		& \forall j>\ist, C_j -C_\ist \notin \left (\gamma_\ist(x_t)-\gamma_j(x_t), \Prob{\Yt^\ist \ne \Ytj|X=\xt}\right] \numberthis \label{equ:CostNotin}
	\end{align*}
	are sufficient for learnability. Note that \cref{equ:CostNotin} is equivalent to 
	\begin{equation}
		\label{equ:CWD_Thm1}
		\forall j>\ist,\; C_j -C_\ist >\Prob{\Yt^\ist \ne \Ytj|X=\xt}. 
	\end{equation}
	Note that if \cref{equ:CWD_Thm1} does not hold, then knowing $\Prob{\Yt^\ist \ne \Ytj|X=\xt}$ is not sufficient for finding the optimal arm.
\end{proof}

\subsection*{Regret  decomposition when contexts satisfy Weak Dominance property with some known probability}
Without knowing the disagreement probability, it is impossible to check whether a context satisfies $\WD$ property or not. Hence we consider a case where a context can satisfy $\WD$ property with some fixed probability. For such cases, we can decompose the regret into two parts: regret due to the contexts that satisfy $\WD$ property and regret due to the contexts that do not satisfy $\WD$ property. Note that the regret can be linear due to the contexts that do not satisfy the $\WD$ condition. 
Our next result gives the upper bound on the regret where the contexts satisfy $\WD$ property with a known fixed probability.
\begin{lem}
	Let $\rho$ be the probability of context that it does not satisfy the $\WD$ property and $R_{max}$ be the maximum regret incurred for any context. If $\Regret_T$ is the regret incurred when all contexts satisfy $\WD$ property then, the regret incurred when contexts satisfy $\WD$ with probability $(1-\rho)$ is given by
	\begin{equation*}
	\Regret_T^\prime \le (1-\rho)\Regret_T + \rho R_{max}T.
	\end{equation*}
\end{lem}
\begin{proof}
	Let $\rho$ be the probability of context that it does not satisfy the $\WD$ property and $r_t(I_t, i^\star_t)$ be the regret incurred for selecting sub-optimal arm $I_t$ for the context $x_t$. Then the regret can be decomposed into two parts as follows:
	\begin{align*}
		\Regret_T^\prime &= \EE{\sum_{t=1}^T \left[\one{x_t \mbox{ satisfies } \WD} r_t(I_t, i^\star_t) + \one{x_t \mbox{ does not satisfy } \WD} r_t(I_t, i^\star_t) \right]}\\
		&= \EE{\sum_{t=1}^T \one{x_t \mbox{ satisfies } \WD} r_t(I_t, i^\star_t)} + \EE{\sum_{t=1}^T \one{x_t \mbox{ does not satisfy } \WD} r_t(I_t, i^\star_t)}\\
		&= \sum_{t=1}^T \Prob{x_t \mbox{ satisfies } \WD} r_t(I_t, i^\star_t) + \sum_{t=1}^T\Prob{x_t \mbox{ does not satisfy } \WD} r_t(I_t, i^\star_t) \numberthis \label{eqn:probRegretUB}.
	\end{align*}
	
	First, we will bound the regret due to the contexts that do not satisfy $\WD$ property (second term of \cref{eqn:probRegretUB}). Note that the context that does not satisfy $\WD$ property, the learner can not make the correct decision hence always incurs regret. Since the maximum regret is upper bounded by $R_{max}$, we have	
	\begin{align*}
		\sum_{t=1}^T\Prob{x_t \mbox{ does not satisfy } \WD} r_t(I_t, i^\star_t) &\le \sum_{t=1}^T\Prob{x_t \mbox{ does not satisfy } \WD} R_{max}
		\intertext{Since $\rho$ is the probability of context that it does not satisfy the $\WD$ property, we get }
		&= \sum_{t=1}^T \rho R_{max}\\
		\implies \sum_{t=1}^T\Prob{x_t \mbox{ does not satisfy } \WD} r_t(I_t, i^\star_t) &\le \rho R_{max}T. \numberthis \label{eqn:probRegretNoCWD}
	\end{align*}
	
	Now we will bound the regret due to the contexts which satisfy $\WD$ property (first term in \cref{eqn:probRegretUB}). Since any context satisfies $\WD$ with $1-\rho$ probability, we have
	\begin{align}
	\label{eqn:probRegretCWD}
	\sum_{t=1}^T \Prob{x_t \mbox{ satisfies } \WD} r_t(I_t, i^\star_t) = \sum_{t=1}^T (1-\rho)  r_t(I_t, i^\star_t) = (1-\rho)  \sum_{t=1}^T r_t(I_t, i^\star_t).
	\end{align}
	
	By assuming that all contexts are satisfying $\WD$ property, we have regret $\Regret_T = \sum_{t=1}^T  r_t(I_t, i^\star_t)$. Using it with \cref{eqn:probRegretNoCWD} in \cref{eqn:probRegretCWD}, we get
	\begin{equation*}
		\Regret_T^\prime \le (1-\rho)\Regret_T + \rho R_{max}T. \tag*{\qedhere}
	\end{equation*}
\end{proof}

\subsection*{Missing proofs from \cref{sec:contxuss_algorithm}}
\minEigenLB*
\begin{proof}
	The result is adapted from \cite[Proposition 1]{ICML17_li2017provably}, which uses the standard random matrix theory result from \cite[Theorem 5.39]{Book_vershynin_2012}. We need to carefully construct the sample complexity bound for our case as the observations are only observed for a pair of arms.
\end{proof}

\noindent
The following result is needed to prove \cref{lem:theta_est_glm}.
\begin{restatable}{lem}{detVt}
	\label{lem:detVt} 
	Let $\oVt = \lambda \mathrm{I}_{d^\prime} + a\sum_{s\in S_{ij}^t}\Phi_{ij}(x_s)\Phi_{ij}(x_s)^\top$  for any $(i,j)$ pair of arms, and $n_{ij}^t = |S_{ij}^t|$. Then 
	\begin{equation*}
		det(\oVt) \le  \left(\lambda + an_{ij}^t/d^\prime\right)^{d^\prime}.
	\end{equation*}
\end{restatable}
\begin{proof} 
	The proof is adapted from Lemma 10 of \cite{NIPS11_abbasi2011improved}. By using inequality of arithmetic and geometric means, we have $det(\oVt) \le (trace(\oVt)/{d^\prime})^{d^\prime}$. As the trace of matrix is a linear mapping i.e. $trace(A+B) = trace(A) + trace(B)$, hence, we get
	\begin{align*}
		trace(\oVt) &= trace(\lambda I_{d^\prime}) + a\sum_{s\in S_{ij}^t}trace\left(\Phi_{ij}(x_s)\Phi_{ij}(x_s)^\top\right) \\
		&= \lambda{d^\prime} + a\sum_{s\in S_{ij}^t}\norm{\Phi_{ij}(x)}_2^2 \le \lambda{d^\prime} + an_{ij}^t. \hspace{2mm}\left(\text{as $\norm{\Phi_{ij}(x)}_2\le 1$ and $n_{ij}^t = |S_{ij}^t|$}\right)
	\end{align*}
	Using upper bound of $trace(\oVt)$ for bounding $det(\oVt)$, we get
	\begin{equation*}
		det(\oVt) \le (trace(\oVt)/{d^\prime})^{d^\prime} \le \left((\lambda{d^\prime} + an_{ij}^t)/{d^\prime}\right)^{d^\prime}  \le \left(\lambda + an_{ij}^t/{d^\prime}\right)^{d^\prime}. \tag*{\qedhere}
	\end{equation*}
\end{proof}

\thetaEstGLM*
\begin{proof}
	Let $\oVt = \lambda\mathrm{I}_{d^\prime} + \Vt$. If \cref{equ:glm_est} is used for estimation of unknown parameter $\Ts$ then by using Eq. (26) and Lemma 8 of \cite{ICML17_li2017provably} with $\lambda_{min}(V_{ij}^{m+1}) \ge 1$, we have
	\begin{align*}
		\norm{\Te - \Ts}_{\Vt} \le \frac{1}{\kappa}\norm{\sum_{s\in S_{ij}^t} \epsilon_s \Phi_{ij}(x_s)}_{(\Vt)^{-1}}
		\le \frac{(1-\lambda)^{\frac{-1}{2}}}{\kappa}\norm{\sum_{s\in S_{ij}^t} \epsilon_s \Phi_{ij}(x_s)}_{(\oVt)^{-1}}.
	\end{align*}
	Using upper bound of $\norm{\sum_{s\in S_{ij}^t}\epsilon_s\Phi_{ij}(x_s)}_{(\oVt)^{-1}}$ as given in Theorem 1 of \cite{NIPS11_abbasi2011improved} where $\epsilon_s$ is $\sigma-$subGaussian random variable, the following inequality holds with at least probability $1-2\delta/K^2$	
	\begin{align*}
		&\le   \frac{(1-\lambda)^{\frac{-1}{2}}}{\kappa}\sqrt{2\sigma^2\log\left(\frac{det(\oVt)^{1/2}det(\lambda\mathrm{I}_{d^\prime})^{-1/2}}{2\delta/K^2}\right)}\\ &=\frac{\sigma(1-\lambda)^{\frac{-1}{2}}}{\kappa}\sqrt{2\log\left(\frac{det(\oVt)}{det(\lambda\mathrm{I}_{d^\prime})}\right)^{\frac{1}{2}} + 2\log\left(\frac{K^2}{2\delta}\right)}.
	\end{align*}
	Upper bounding $det(\oVt)$ with $a = 1$, $\lambda=1/2$, and $n_{ij}^t \le t$ by using \cref{lem:detVt}, we get
	\begin{align*}
		\implies \norm{\Te - \Ts}_{\Vt} &\le\frac{2\sigma}{\kappa}\sqrt{\frac{d^\prime}{2}\log\left(1 + \frac{2n_{ij}^t}{d^\prime} \right) + \log\left(\frac{K^2}{2\delta}\right)} \\
		&\le \frac{2\sigma}{\kappa}\sqrt{\frac{d^\prime}{2}\log\left(1 + \frac{2t}{d^\prime} \right) + \log\left(\frac{K^2}{2\delta}\right)}.  \tag*{\qedhere} 
	\end{align*}
\end{proof}

\subOptimalLowerSelection*
\begin{proof}
	If sub-optimal arm $l<i^\star_t$ is preferred by \ref{alg:CUSS_WD} then using Eq. \eqref{def_contx_prefer_h}, we get
	\begin{align*}
		\one{l \succ_t i^\star_t, i^\star_t= i} &= \one{C_i - C_l > \tplit, I_t = l, i^\star_t= i} \le \one{C_i - C_l > \tplit}. \hspace{5mm}\text{(as $A\cap B \cap C \subseteq A$)}
	\end{align*}
	
	Using $C_i - C_l = p_{li}(x_t) - \xi_{li}(x_t)$ for  $l< i$, we have
	\begin{align*}
		\implies \one{l \succ_t i^\star_t, i^\star_t= i} &= \one{p_{li}(x_t) - \xi_{li}(x_t) > \tplit} = \one{p_{li}(x_t) - \tplit >  \xi_{li}(x_t)}. 
	\end{align*}
	
	Using definition of $p_{li}(x_t)$ and $\tplit$,
	\begin{align*}
		\implies \one{l \succ_t i^\star_t, i^\star_t = i} = \one{\mu(\Phi_{li}(x_t)^\top \Tsl) - \mu\left(\Phi_{li}(x_t)^\top \Tel + \alpha_{li}^t\norm{ \Phi_{li}(x_t)}_{\Vinvl}\right) > \xi_{li}(x_t)}. %\\
	\end{align*}
	
	Since $\mu(\cdot)$ is an increasing function and using $\alpha_{li}^t$ as defined in \cref{lem:theta_est_glm}, $\mu\left(\Phi_{li}(x_t)^\top \Tel + \alpha_{li}^t\norm{ \Phi_{li}(x_t)}_{\Vinvl}\right)$ is the upper bound on $\mu(\Phi_{li}(x_t)^\top \Tsl)$ for all $(l,i)$ pairs with probability at least $1-\delta/2$. We show it as follows:
	\begin{align*}
		\Phi_{li}(x_t)^\top \Tsl &= \Phi_{li}(x_t)^\top \Tel + \Phi_{li}(x_t)^\top (\Tsl - \Tel)\\
		&= \Phi_{li}(x_t)^\top \Tel + \norm{\Phi_{li}(x_t)}_{\Vinvl}\norm{\Tsl - \Tel}_{\Vtl}\\
		\implies \Phi_{li}(x_t)^\top \Tsl &\le \Phi_{li}(x_t)^\top \Tel + \alpha_{li}^t\norm{ \Phi_{li}(x_t)}_{\Vinvl}. &\hspace{-1.5cm}\left(\text{using }\norm{\Tsl - \Tel}_{\Vtl} \le \alpha_{li}^t\right) \\
		\intertext{Since $\mu(\cdot)$ is an increasing function,}
		\implies \mu(\Phi_{li}(x_t)^\top \Tsl) &\le \mu\left(\Phi_{li}(x_t)^\top \Tel + \alpha_{li}^t\norm{ \Phi_{li}(x_t)}_{\Vinvl}\right). 
	\end{align*}

	Hence, any sub-optimal arm smaller than the optimal arm is selected by \ref{alg:CUSS_WD} with probability at most $\delta/2$. It completes the proof of the lemma.
\end{proof}

\subOptimalUpperSelection*
\begin{proof}
	If sub-optimal arm $h>i^\star_t$ is preferred by \ref{alg:CUSS_WD} then using Eq. \eqref{def_contx_prefer_l}, we get	
	\begin{align*}
		\one{h \succ_t i, i^\star_t= i} &= \one{C_h - C_i < \tpiht, h \succ_t i^\star_t, i^\star_t= i}  \le \one{C_h - C_i < \tpiht}. \hspace{5mm}\text{(as $A\cap B \cap C \subseteq A$)}
	\end{align*}
	
	Using $C_h - C_i = p_{ih}(x_t) + \xi_{ih}(x_t)$ for $h > i$, we get
	\begin{align*}
		\implies \one{h \succ_t i, i^\star_t= i} &= \one{p_{ih}(x_t) + \xi_{ih}(x_t) < \tpiht} = \one{\tpiht - p_{ih}(x_t) >  \xi_{ih}(x_t)}. 
	\end{align*}
	
	Using definition of $p_{ih}(x_t)$ and $\tpiht$, 
	\begin{align*}
		\implies \one{h \succ_t i, i^\star_t= i}&= \one{\mu\left(\Phi_{ih}(x_t)^\top \Tsh + \alpha_{ih}^t\norm{ \Phi_{ih}(x_t)}_{\Vinvh}\right) - \mu(\Phi_{ih}(x_t)^\top \Teh)  > \xi_{ih}(x_t)}. 
	\end{align*}
	
	As $\mu$ is Lipschitz, $|\mu(z_1) - \mu(z_2)| \le k_\mu|z_1 - z_2|$ where $k_\mu$ is Lipschitz constant, we have
	\begin{align*}
		&\le \one{k_\mu|\Phi_{ih}(x_t)^\top \Tsh + \alpha_{ih}^t\norm{ \Phi_{ih}(x_t)}_{\Vinvh} - \Phi_{ih}(x_t)^\top \Teh| > \xi_{ih}(x_t)} \\
		&\le \one{k_\mu|\Phi_{ih}(x_t)^\top \Tsh - \Phi_{ih}(x_t)^\top \Teh| + k_\mu\alpha_{ih}^t\norm{ \Phi_{ih}(x_t)}_{\Vinvh}> \xi_{ih}(x_t)} \\
		&= \one{k_\mu|\Phi_{ih}(x_t)^\top (\Tsh - \Teh)| + k_\mu\alpha_{ih}^t\norm{ \Phi_{ih}(x_t)}_{\Vinvh} > \xi_{ih}(x_t)}.
	\end{align*}
	
	Using Cauchy-Schwartz inequality and $\norm{ x}_A^2 = x^\top Ax$, we get
	\begin{align*}
		&\le \one{k_\mu\norm{\Phi_{ih}}_{\Vinvh} \norm{\Tsh - \Teh}_{\Vth} + k_\mu\alpha_{ih}^t\norm{ \Phi_{ih}(x_t)}_{\Vinvh} > \xi_{ih}(x_t)}.
	\end{align*}
	
	As $\norm{\Tsh- \Teh}_{\Vth} \le \alpha_{ih}^t$, we get
	\begin{align*}
		&\le \one{k_\mu\alpha_{ih}^t\norm{ \Phi_{ih}(x_t)}_{\Vinvh}  + k_\mu\alpha_{ih}^t\norm{ \Phi_{ih}(x_t)}_{\Vinvh} > \xi_{ih}(x_t)}\\
		&= \one{2k_\mu\alpha_{ih}^t\norm{ \Phi_{ih}(x_t)}_{\Vinvh} > \xi_{ih}(x_t)}. 
	\end{align*}

	As $\norm{\Phi_{ih}(x_t)}_{\Vinvh} \le \norm{\Phi_{ih}(x_t)}_2/\sqrt{\lambda_{min}(\Vth)}$ where $\lambda_{min}(\Vth)$ is the smallest eigenvalue of matrix $\Vth$ and $\norm{\Phi_{ih}(x_t)}_2 \le 1$, we get
	\begin{align*}
		\implies \one{h \succ_t i, i^\star_t= i}&\le \one{2k_\mu\alpha_{ih}^t > \xi_{ih}(x_t)\sqrt{\lambda_{min}(\Vth)}}.  \label{equ:IndHighSelCond} \numberthis
	\end{align*}
	The event on LHS is subset of event of RHS in \cref{equ:IndHighSelCond}. By changing $i$ to $i^\star_t$ completes the proof of the lemma.
\end{proof}

\cumRegGLM*
\begin{proof}
	The regret for $T$ rounds in the Contextual USS problem is given by
	\begin{align*}
		\Regret_T &= \sum_{t=1}^T\left(C_{I_t} + \gamma_{I_t}(x_t) - (C_{i^\star_t} +\gamma_{i^\star_t}(x_t))\right).
	\end{align*}
	
	As $R_{max}$ denote the maximum regret incurred for any context, we get
	\begin{align}
		\Regret_T &\le R_{max}\sum_{t=1}^T \one{I_t \ne i^\star_t}. \label{equ:RegretInd}
	\end{align}
	
	As $\one{I_t \ne i^\star_t}$ has two random quantities $I_t$ and $i^\star_t$, we can re-write it as follows:
	\begin{align*}
		\one{I_t \ne i^\star_t} &= \sum_{l < i}\one{I_t = l, i^\star_t= i} + \sum_{h^\prime>i}\one{I_t = h^\prime, i^\star_t=i}. 
	\end{align*}
	
	Note that if \ref{alg:CUSS_WD} selects $l < i^\star_t$ then $l$ must be preferred over $i^\star_t$ whereas if $h^\prime > i^\star_t$ is selected then there exists an arm $h > i^\star_t$ which is preferred over $i^\star_t$. Hence, we have
	\begin{align} 
		\label{equ:IndWrongChoice}
		\one{I_t \ne i^\star_t}  &= \sum_{l < i}\one{l \succ_t i^\star_t, i^\star_t= i} + \sum_{h^\prime>i}\one{I_t = h^\prime, h \succ_t \ist, i^\star_t=i}\nonumber \le \sum_{l < i}\one{l \succ_t i^\star_t, i^\star_t= i} + \sum_{h>i}\one{h \succ_t \ist, i^\star_t=i}.
	\end{align}
	
	Using above bound in \cref{equ:RegretInd}, we get
	\begin{align*}
		\Regret_T &\le R_{max}\sum_{t=1}^T \left[\sum_{l < i}\one{l \succ_t i^\star_t, i^\star_t= i} + \sum_{h>i}\one{h\succ_t i^\star_t, i^\star_t=i}\right].
	\end{align*}

	From \cref{lem:subOptimalLowerSelection}, $\one{l \succ_t i^\star_t, i^\star_t= i} = 0$ for any $l<i$ with probability at least $1-\delta/2$, then the regret becomes
	\begin{align*}
		\Regret_T &\le R_{max}\sum_{t=1}^T\sum_{h>i}\one{h \succ_t i^\star_t, i^\star_t=i} = R_{max} \sum_{h>i}\sum_{t=1}^T \one{h \succ_t i^\star_t, i^\star_t=i} \le R_{max} \sum_{h=2}^K\sum_{t=1}^T \one{h \succ_t i^\star_t, i^\star_t<h}.
	\end{align*}
	
	Note that $\alpha_{ih}^t$ is slowly increasing value with $t$ that implies $\alpha_{ih}^t\le \alpha_{ih}^T$ for all $t \le T$.
	Using \cref{lem:min_eigen_lb} with $\Psi=  \left(\frac{2k_\mu \alpha_{ih}^T}{\xi_{ih}}\right)^2$, $\Sigma_{ih} = \EE{\Phi_{ih}(X_s)\Phi_{ih}(X_s)^T}$ where $s \in S_{ih}^t$, after
	\begin{equation*}
		n_{ih}^T \doteq \left(\frac{C_1\sqrt{d^\prime}+C_2\sqrt{\log (K^2/2\delta)}}{\lambda_{min}(\Sigma_{ih})}\right)^2+ \frac{2}{\lambda_{min}(\Sigma_{ih})}\left(\frac{2k_\mu\alpha_{ih}^T}{\xi_{ih}}\right)^2
	\end{equation*}
	observations  for arm pair $(i,h)$ the $\lambda_{min}(\Vth) \ge \left(\frac{2k_\mu\alpha_{ih}^T}{\xi_{ih}}\right)^2$ with probability at least $1-2\delta/K^2$. Therefore, after having $n_{ih}^T$ observations, the sub-optimal arm $h(>i)$ will not be preferred over optimal arm $i$ with probability at least $1-2\delta/K^2$. Therefore, with probability at least $1-2\delta/K^2$, following equations also hold
	\begin{align*}
		&\one{I_t = h, i^\star_t= i, |S_{ih}^t| \ge n_{ih}^T} = 0
		\implies \sum_{t=1}^{T}\one{I_t = h, i^\star_t= i, h>i} \le n_{ih}^T.
	\end{align*}

	Due to the problem structure, whenever an arm $h$ is selected, disagreement labels for all arm pair $(i,j)$ where $i<j\le h$ are observed. Therefore, with probability at least $1-2\delta/K$ (by union bound), the maximum number of times an arm $h$ is selected when the optimal arm's index is smaller than $h$ is $n_h^T$ such that 
	\begin{align*}
    	n_{h}^T &= \left(\frac{C_1\sqrt{d^\prime}+C_2\sqrt{\log (K^2/2\delta)}}{\lambda_{\Sigma}}\right)^2+ \frac{2}{\lambda_{\Sigma}}\left(\frac{2k_\mu\alpha_T}{\xi_{h}}\right)^2 \\
    	&= \left(\frac{C_1\sqrt{d^\prime}+C_2\sqrt{\log (K^2/2\delta)}}{\lambda_{\Sigma}}\right)^2+ \frac{8}{\lambda_{\Sigma}}\left(\frac{k_\mu\alpha_T}{\xi_{h}}\right)^2 %\label{equ:minSelectionH}
	\end{align*}
	where $\xi_h = \min\limits_{i<h,t \ge 1} \xi_{ih}(x_t)$, $\lambda_{\Sigma} = \min\limits_{i < j\le K}\lambda_{min}\left(\EE{\Phi_{ij}(X_s)\Phi_{ij}(X_s)^\top}\right)$ and $\alpha_T \ge \max\limits_{i<h}\alpha_{ih}^T$. By using union bound, we get following bound with probability at least $1-\delta/2K$
	\begin{align*}
		\sum_{t=1}^{T}\one{I_t = h, i^\star_t < h} \le n_{h}^T. \numberthis \label{equ:hIndBnd}
	\end{align*}
	
	From \cref{equ:hIndBnd}, using $\sum_{t=1}^T \one{I_t = h, i^\star_t<h} \le n_{h}^T$ and value of $n_h^T$, we get following upper bound on regret that holds with probability at least $1-\delta$ by union bound
	\begin{align*}
		\Regret_T \le R_{max}\sum_{h=2}^{K} n_{h}^T =  R_{max}\sum_{h=2}^{K} \left( \left(\frac{C_1\sqrt{d^\prime}+C_2\sqrt{\log (K^2/2\delta)}}{\lambda_{\Sigma}}\right)^2+ \frac{8}{\lambda_{\Sigma}}\left(\frac{k_\mu\alpha_T}{\xi_{h}}\right)^2\right).
	\end{align*}
	
	Using $\alpha_{T}= \frac{2\sigma}{\kappa}\sqrt{\frac{d^\prime}{2}\log\left(1 + 2T/{d^\prime} \right) + \log\left({K^2}/{2\delta}\right)}$ from Lemma \ref{lem:theta_est_glm} that ensures parameter $\Ts$ bounds for all pairs $(i,j)$ holds with probability at least $1 - \delta/2K$ (by union bound) for $T>m$ where $m=C\lambda_\Sigma^{-2}\left(d+\log(k^2/2\delta)\right) + 2\lambda_\Sigma^{-1}$ such that $\lambda_{min}(V_{ij}^{m+1}) \ge 1$ for all pair $(i,j)$, we have
	\begin{align*}
		\Regret_T &\le R_{max}\left(m + \sum_{h=2}^{K} n_{h}^T\right)\\
		\implies \Regret_T  &\le R_{max}\Bigg[m + \sum_{h=2}^{K} \Bigg(\hspace{-1mm} \Bigg(\frac{C_1\sqrt{d^\prime}+C_2\sqrt{\log \left(\frac{K^2}{2\delta}\right)}}{\lambda_{\Sigma}}\Bigg)^2  \\
		&\qquad + \frac{16}{\lambda_{\Sigma}}\left(\frac{k_\mu\sigma}{\xi_{h}\kappa}\right)^2  \left(\frac{d^\prime}{2}\log\left(1 + \frac{2T}{d^\prime} \right) + \log\left(\frac{K^2}{2\delta}\right)\right)\Bigg)\Bigg]. \tag*{\qedhere}
	\end{align*}
\end{proof}

%% file: chapter3/lambda_algorithm.tex
%!TEX root =  ../thesis.tex

\ref{alg:CUSS_WD} uses forced exploration by selecting arm $K$ until the correlation matrix $V_{ij}^t$ is not invertible for all $(i,j)$ pairs of arms. Further, the minimum eigenvalue of $V_{ij}^t$ for all $(i,j)$ pairs is needed to be larger than $1$ so that bound given in \cref{lem:theta_est_glm} holds. Alternatively, $V_{ij}^t$ can be initialized by adding a regularization term \citep{NIPS11_abbasi2011improved,ICML16_zhang2016online,NIPS17_jun2017scalable} to avoid forced exploration and then apply OFUL type analysis. We have given an algorithm named \ref{alg:CUSS_GLM2} which uses regularization term $\lambda\mathrm{I}_{d^\prime}$. However, its analysis still needed the minimum eigenvalue of the non-regularized part of the correlation matrix to become larger than some positive value (depends on $\lambda$ value), as shown in our next result.

\begin{algorithm}[!ht]
    \floatname{algorithm}{}
    \renewcommand{\thealgorithm}{USS-PD-$\lambda\mathrm{I}$}
	\caption{Algorithm for Contextual USS using Pairwise Disagreement with $\lambda\mathrm{I}$ Initialization}
	\label{alg:CUSS_GLM2}
	\begin{algorithmic}[1]
		\State \textbf{Input:} Tuning parameters: $\delta \in (0,1)$ and $\lambda>0$ 
		\State Select arm $K$ for first context $x_1$
		\State $\forall i< j \le K:$ set $\overline{V}_{ij}^{1} \leftarrow \lambda\mathrm{I}_{d^\prime} + \Phi_{ij}(x_{1}) {\Phi_{ij}(x_{1})}^\top$ and update $\hat\theta_{ij}^1$ by solving \cref{equ:glm_est} 
		\For{$t= 2, 3, \ldots$}
			\State Receive context $x_t$. Set $i=1$ and $I_t=0$
			\While{$I_t=0$}
				\State Play arm $i$ %and observe $\zti$
				\State $\forall j \in [i+1, K]:$ compute $\tpijt \leftarrow \mu\left(\Phi_{ij}(\xt)^\top \hat\theta_{ij}^{t-1} + \alpha_{ij}^{t-1}\norm{ \Phi_{ij}(\xt)}_{\left(\overline{V}_{ij}^{t-1}\right)^{-1}} \right)$
				\State If $\forall j \in [i+1, K]: C_j - C_i > \tpijt$ or $i=K$ then set $I_t = i$ else  set $i=i+1$
			\EndWhile
			\State Select arm $I_t$ and observe $Y_t^1, Y_t^2, \dots, Y_t^{I_t}$
			\State $\forall i< j \le I_t:$  update $\oVijt \leftarrow \overline{V}_{ij}^{t-1} + \Phi_{ij}(x_{t}) {\Phi_{ij}(x_{t})}^\top$ and $\Te$ by solving \cref{equ:glm_est} 
		\EndFor
	\end{algorithmic}
\end{algorithm}

\begin{restatable}{lem}{thetaEstGLM2}
	\label{lem:theta_est_glm2} 
	Let $\oVt = \lambda\mathrm{I}_{d^\prime} + \Vt$ for any $\lambda>0$ and $\norm{\theta_{ij}}_2 \le S$ for all $(i,j)$ pair. Then for any $t > \min\{s: \forall i<j \ni \lambda_{\min}(V_{ij}(s)) \ge 2\lambda\}$, the following event holds for \ref{alg:CUSS_GLM2} with probability at least $1-2\delta/K^2$,
	\begin{align*}
		&\norm{ \Te - \Ts }_{\oVt} \le \beta_{ij}^t,
	\end{align*}
	where $\beta_{ij}^t = \frac{2\sigma}{\kappa}\sqrt{\frac{d^\prime}{2}\log\left(1 + \frac{n_{ij}^t}{d^\prime\lambda} \right) + \log\left(\frac{K^2}{2\delta}\right)} + 2\lambda^{1/2}S $.
\end{restatable}
\begin{proof}
	By using $\norm{Z}_{A+B} \le \norm{Z}_{A} + \norm{Z}_{B}$ , we have
	\begin{align*}
		\norm{\Te - \Ts}_{\oVt} &\le \norm{\Te - \Ts}_{\Vt} + \norm{\Te - \Ts}_{\lambda\mathrm{I}_{d^\prime}}. & \hspace{-2cm}(\mbox{as } \oVt = \lambda\mathrm{I}_{d^\prime} + \Vt) \\
		\intertext{When \cref{equ:glm_est} is used for estimation of unknown parameter $\Ts$ then by using Eq. (26) and Eq. (27) of Lemma 8 of \cite{ICML17_li2017provably}, we have}
		& \le \frac{1}{\kappa}\norm{\sum_{s\in S_{ij}^t} \epsilon_s \Phi_{ij}(x_s)}_{\Vtinv} + 2\lambda^{1/2}S. &\hspace{-5cm} (\mbox{as } \norm{\theta_{ij}}_2 \le S) 
	\end{align*}
	
	Sherman Morrison formula gives $\norm{Z}_{\Vtinv} \le \left(1-\frac{\lambda}{\lambda_{\min}(\Vt)}\right)^{-\frac{1}{2}}\norm{Z}_{\oVtinv}$. Using it, we have
	\begin{align*}
		\norm{\Te - \Ts}_{\oVt} \le \frac{\left(1-\frac{\lambda} {\lambda_{\min}(\Vt)}\right)^{-\frac{1}{2}}}{\kappa}\norm{\sum_{s\in S_{ij}^t} \epsilon_s \Phi_{ij}(x_s)}_{\oVtinv} + 2\lambda^{1/2}S.
	\end{align*}
	
	Using upper bound on $\norm{\sum_{s\in S_{ij}^t}\epsilon_s\Phi_{ij}(x_s)}_{(\overline\Vt)^{-1}}$ as given in Theorem 1 of \cite{NIPS11_abbasi2011improved}, where $\epsilon_s$ is $\sigma-$subGaussian random variable and holds with probability at least $1-2\delta/K^2$, we get
	\begin{align*}	
		\norm{\Te - \Ts}_{\oVt} &\le \frac{\left(1-\frac{\lambda}{\lambda_{\min}(\Vt)}\right)^{-\frac{1}{2}}}{\kappa}\sqrt{2\sigma^2\log\left(\frac{det(\oVt)^{1/2}det(\lambda\mathrm{I}_{d^\prime})^{-1/2}}{2\delta/K^2}\right)}  + 2\lambda^{1/2}S \\
		&=\frac{\sigma\left(1-\frac{\lambda}{\lambda_{\min}(\Vt)}\right)^{-\frac{1}{2}}}{\kappa}\sqrt{2\log\left(\frac{det(\oVt)}{det(\lambda\mathrm{I}_{d^\prime})}\right)^{\frac{1}{2}} + 2\log\left(\frac{K^2}{2\delta}\right)}  + 2\lambda^{1/2}S. 
	\end{align*}
	
	By using \cref{lem:detVt} to upper bound $det(\oVt)$, where $t>s$ with $a = 1$, and $n_{ij}^t \le t$, we get
	\begin{align}
		\label{equ:minEigenValLambda}
		\norm{\Te - \Ts}_{\oVt} \le \frac{\sigma\left(1-\frac{\lambda}{\lambda_{\min}(\Vt)}\right)^{-\frac{1}{2}}}{\kappa} \sqrt{{d^\prime}\log\left(1 + \frac{n_{ij}^t}{d^\prime\lambda} \right) + 2\log\left(\frac{K^2}{2\delta}\right)} + 2\lambda^{1/2}S.
	\end{align}
	
	As $t>s$ such that $\lambda_{\min}(V_{ij}(s)) \ge 2\lambda$, we have
	\[\norm{\Te - \Ts}_{\oVt} \le \frac{2\sigma}{\kappa} \sqrt{\frac{d^\prime}{2}\log\left(1 + \frac{n_{ij}^t}{d^\prime\lambda} \right) + \log\left(\frac{K^2}{2\delta}\right)} + 2\lambda^{1/2}S  = \beta_{ij}^t. \tag*{\qedhere} \]
\end{proof}

Note that if $\lambda_{\min}(V_{ij}(s)) < \lambda$ then $\left(1-\frac{\lambda}{\lambda_{\min}(\Vt)}\right)^{-\frac{1}{2}}$ is not well defined and the bound given in \cref{lem:theta_est_glm2} does not hold. Therefore, $\lambda_{\min}(V_{ij}(s))$ need to be at least greater than $\lambda$. Let $m^\prime \doteq C\lambda_\Sigma^{-2}\left(d+\log(k^2/2\delta)\right) + 4\lambda_\Sigma^{-1}\lambda$ where $C>0$ is the universal constant. Recall $R_{max} \doteq \max_{i\in [K], x\in \mathcal{X}}$ $ \left[C_i +  \gamma_{i}(x) - \left( C_{i^\star} + \gamma_{i^\star}(x)\right)\right]$, where $i^\star$ is the optimal arm for a context $x$. Now we state the regret bounds for \ref{alg:CUSS_GLM2}.
\begin{restatable}{thm}{cumRegGLM2}
	\label{thm:cum_reg_glm2}
	Let $P \in \PCWD$, $\lambda>0$, $\delta \in (0,1)$, Assumption 1 holds, and $\xi_h = \min\limits_{t \ge 1} \xi_{i^\star_t h}(x_t)$. Then with probability at least $1-2\delta$, the regret of \ref{alg:CUSS_GLM2} for $T > m^\prime$ contexts is upper bounded as
	\begin{align*}
		\Regret_T &\le R_{max}\Bigg(m^\prime + \sum_{h=2}^{K} \Bigg(\hspace{-1mm} \Bigg(\frac{C_1\sqrt{d^\prime} + C_2\sqrt{\log \left(\frac{K^2}{2\delta}\right)}}{\lambda_{\Sigma}}\Bigg)^2  \hspace{-2mm} + \frac{32\lambda}{\lambda_{\Sigma}}\left(\frac{k_\mu\sigma}{\xi_{h}\kappa}\right)^2 \\
		&\qquad 
		\Bigg(\sqrt{\frac{d^\prime}{2}\log\left(1 + \frac{T}{d^\prime\lambda} \right) + \log\left(\frac{K^2}{2\delta}\right)} + 2\lambda^{1/2}S\Bigg)^2\Bigg)\Bigg).
	\end{align*} 
\end{restatable}
\begin{proof}
	The proof follows similar steps as \cref{thm:cum_reg_glm} by replacing $m$ by $m^\prime$ and $\alpha_{ij}^T$ by $\beta_{ij}^T$. Using $\beta_{ij}^T = \sqrt{\frac{d^\prime}{2}\log\left(1 + \frac{T}{d^\prime\lambda} \right) + \log\left(\frac{K^2}{2\delta}\right)} + 2\lambda^{1/2}S$ completes the proof. 
\end{proof}

%% file: chapter3/supp_experiments.tex
%!TEX root =  ../thesis.tex

Since the parameter of each arm (classifier) is known to us (but not to the algorithm), the optimal arm $i^\star_t$ can be computed for every context. Therefore, we can also calculate the fraction of contexts for which $\WD$ property holds to a given cost vector. To verify $\WD$ property for a given context $x_t$, we first compute disagreement probability for each $(i,j)$ pair of classifiers as\footnote{For computing disagreement probability, we assume that the feedback of any arm is independent of the feedback of other arms. Note that \ref{alg:CUSS_WD} does not need such an assumption.}
\[
\pijt = \mu(x_t^\top\theta_i)(1-\mu(x_t^\top\theta_j)) + \mu(x_t^\top\theta_j)(1-\mu(x_t^\top\theta_i)).
\]

When all $\pijt$ values and $i^\star_t$ are known, we can check whether a context $x_t$ satisfies $\WD$ property or not by using \cref{equ:ContxWDProp}. For all problem instances derived from the synthetic dataset, the cost vector and the fraction of contexts for which $\WD$ property holds are given in Table \ref{table:sysDatasetCases}. 

\begin{table}[H]
	\centering
	\setlength\tabcolsep{5pt}
	\setlength\extrarowheight{4pt}
	\begin{tabularx}{0.785\textwidth}{|p{2.5cm}|p{1cm}|p{1cm}|p{1cm}|p{1cm}|p{1cm}|p{2.5cm}|}
		\hline
		\textbf{PI/Classifiers} & Clf. 1        & Clf. 2        & Clf. 3  & Clf. 4 & Clf. 5     & $\WD$ fraction     \\ 
		\hline
		Costs for PI $1$ & 0.01 & 0.02  & 0.032  & 0.05 & 0.55 & 0.997\\ 
		\hline
		Costs for PI $2$ & 0.01 & 0.02  & 0.032  & 0.05 & 0.6 & 1.0\\ 
		\hline
		Costs for PI $3$ & 0.01 & 0.02  & 0.032  & 0.05 & 0.65 & 1.0\\
		\hline
		Costs for PI $4$ & 0.01 & 0.02  & 0.032  & 0.05 & 0.7 & 1.0\\
		\hline
	\end{tabularx}
	\vspace{2mm}
	\caption{Details of different problem instances (PIs) derived from synthetic datasets.}
	\label{table:sysDatasetCases}
\end{table}

\noindent
\textbf{Heart Disease dataset:} Each sample of the Heart Disease dataset has $12$ features. We split the features into three subsets and train a logistic classifier on each subset. We associate 1st classifier with the first $7$ features as input that include cholesterol readings, blood-sugar, and rest-ECG. The 2nd classifier, in addition to the $7$ features, utilizes the thalach, exang, and oldpeak features; and the 3rd classifier uses all the features. For performance evaluation, the different values of costs are used in three problem instances for both real datasets, as given in Table \ref{table:real_dataset}. The PIMA diabetes dataset has $768$ samples, whereas the Heart Disease dataset has only $297$ samples. As $5000$ contexts are used in our experiments, we select a sample in a round-robin fashion and give it as input to the algorithm.
\begin{table}[!ht]
	\centering
	\setlength\tabcolsep{5pt}
	\setlength\extrarowheight{4pt}
	\begin{tabularx}{0.99\textwidth}{|p{2.25cm}|p{1cm}|p{1cm}|p{1cm}|p{2.15cm}|p{1cm}|p{1cm}|p{1cm}|p{2.15cm}|}
		\hline
		\multirow{2}{*}{\parbox{2cm}{\bf Values/ \newline Classifiers}} &\multicolumn{4}{|c|}{\bf PIMA Indian Diabetes Dataset}&\multicolumn{4}{|c|}{\bf Heart Disease Dataset} \\ 
		\cline{2-9} 
		&Clf. 1 &Clf. 2&Clf. 3&WD Fraction &Clf. 1 &Clf. 2&Clf. 3&WD Fraction\\ 
		\hline
		Costs for PI $1$& 0.01 & 0.25 & 0.5 & 0.0692 &0.01 & 0.25 & 0.5 &0.1384\\ 
		\hline
		Costs for PI $2$& 0.01 & 0.3 & 0.5 & 0.1192 &0.01 & 0.3 & 0.5 &0.1454\\ 
		\hline
		Costs for PI $3$& 0.01 & 0.35 & 0.5 & 0.2204&0.01 & 0.35 & 0.5 &0.2426\\ 
		\hline
	\end{tabularx}
	\vspace{2mm}
	\caption{Details of different problem instances (PIs) derived from real datasets.}
	\label{table:real_dataset}
\end{table}

~\\ \noindent
\textbf{Experiments Results:} Through our experiments, we show that the stronger the $\CWD$ property (large value of $\xi$) for the problem instance, easier it is to identify the optimal arm and, hence, has lower regret, as shown in \cref{fig:xi_regret}. We also compare the performance of \ref{alg:CUSS_WD} with three baseline policies on problem instances derived from the Heart Disease dataset (same as the PIMA Indian Diabetes dataset). As expected, we observe that \ref{alg:CUSS_WD} outperforms the baseline policies, as shown in \cref{fig:heart}. Note that we used $\delta=0.05$ and $\sigma=0.1$ in all experiments.
\begin{figure}[!ht]
	\captionsetup[subfigure]{justification=centering}
	\begin{subfigure}[b]{0.495\linewidth}
		\includegraphics[width=\linewidth]{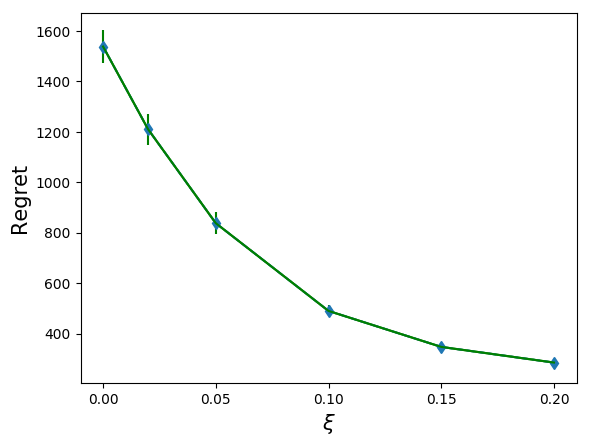}
		\caption{Regret v/s $\CWD$ property $(\xi)$.}
		\label{fig:xi_regret}
	\end{subfigure}
	\quad
	\begin{subfigure}[b]{0.495\linewidth}
		\includegraphics[width=\linewidth]{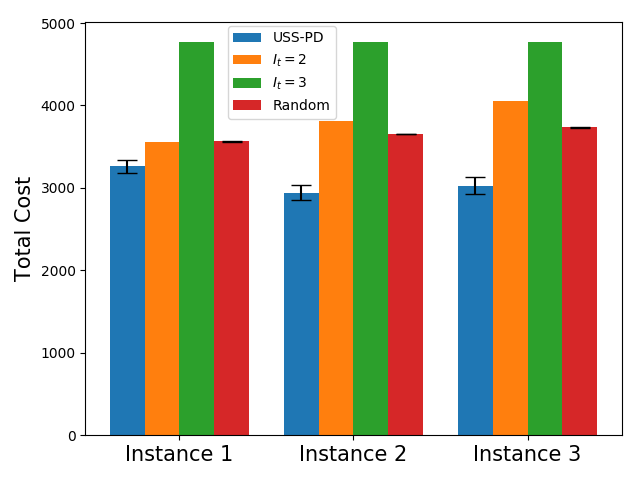}
		\caption{Total cost for PIMA Indian Diabetes dataset.}
		\label{fig:heart}
	\end{subfigure}
	\caption{\small Performance of \ref{alg:CUSS_WD}.}
	\label{fig:moreExp}
\end{figure}

\subsection*{Realizable Setting}

We consider the realizable case where all contexts satisfy \cref{equ:disProb} (by fixing $\theta_{ij}$ for each $(i,j)$ pair of arms) and $\WD$ property. Since $\WD$ holds, we can use \cref{lem:BContx} for finding the optimal arm. Note that the mean loss cannot be computed for this setting as we set parameters of disagreement probabilities instead of setting parameters for individual arms.  We use an upper bound on the regret to evaluate the performance of \ref{alg:CUSS_WD} on the Synthetic dataset, as shown in \cref{fig:realizable}. We repeat experiments $500$ times to get a tighter confidence interval. 
\begin{figure}[!ht]
	\captionsetup[subfigure]{justification=centering}
	\begin{subfigure}[b]{0.495\linewidth}
		\includegraphics[width=\linewidth]{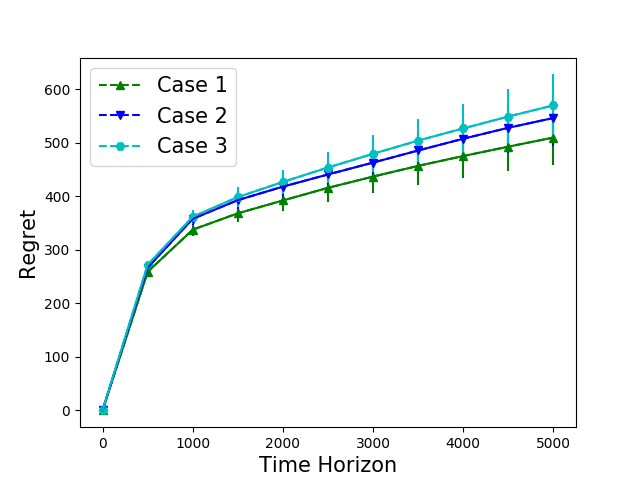}
		\caption{Synthetic dataset with $4$ classifiers where cost of using classifier $i$ in problem instance $j$ is $0.1 + (i-1)(0.09 + (j-1)0.01)$.}
		\label{fig:realizable4Cls}
	\end{subfigure}
	\quad
	\begin{subfigure}[b]{0.495\linewidth}
		\includegraphics[width=\linewidth]{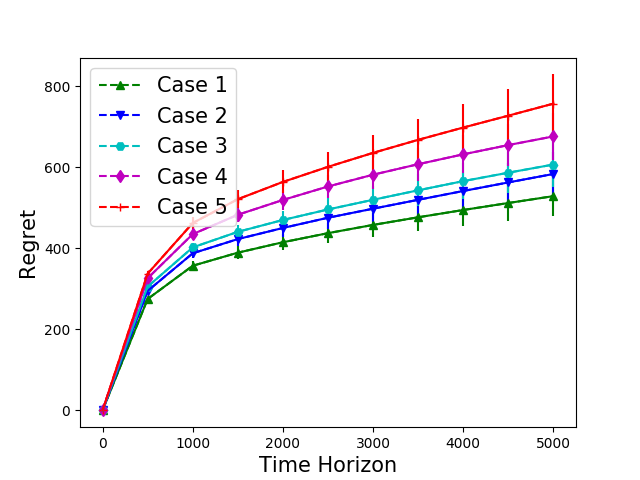}
		\caption{Synthetic dataset with $5$ classifiers where cost of using classifier $i$ for problem instance $j$ is $0.1 + (i-1)(0.06 + (j-1)0.01)$.}
		\label{fig:realizable5Cls}
	\end{subfigure}
	\caption{Performance of \ref{alg:CUSS_WD} for realizable setting where regret on y-axis is $\sum_{t=1}^T |C_{I_t} - C_{i^\star_t}| + \pist$, and it is an upper bound on the regret $\Regret_T$ defined in \cref{equ:cum_regret}. The value of $\xi$ largest for Case $1$, and it decreases for subsequent cases.}
	\label{fig:realizable}
\end{figure}

\paragraph{Regret used for Empirical Evaluation in Realizable Setting}~\\
Since the error-rate of arms is unknown, the regret defined in \cref{equ:cum_regret} can not be computed. Hence we define an alternative regret, which we call pseudo regret, as follows:
\begin{equation*}
	\Regret_T^s = \sum_{t=1}^T \left[C_{I_t} - C_{i^\star_t} + \pist \right].
\end{equation*}

\noindent
It is easy to verify that the actual regret $\Regret_T$ is upper bounded by above regret $\Regret_T^s$ as shown follows:
\begin{align*}
	\Regret_T &= \sum_{t=1}^T \left[C_{I_t} + \gamma_{I_t}(x_t) - \left(C_{i^\star_t} + \gamma_{i^\star_t}(x_t)\right) \right] = \sum_{t=1}^T \left[C_{I_t} - C_{i^\star_t} + \left(\gamma_{I_t}(x_t) -\gamma_{i^\star_t}(x_t)\right) \right]\\
	&\le \sum_{t=1}^T \left[C_{I_t} - C_{i^\star_t} + \pist \right] \hspace{5mm} \text{(Using \cref{lem:ContxErr_prob_contx})}\\
	\implies \Regret_T &\le \Regret_T^s.
\end{align*}

\subsection*{Contextual Strong  Dominance }
We next introduce contextual strong dominance property of the problem instance. 
\begin{defi}[Contextual Strong  Dominance $(\CSD)$ property] 
	\label{def:ContxSD} 
	A problem instance is said to satisfy $\CSD$ property if for all contexts following is true:
	\eqs{
		Y^i = Y \mbox{ for some } i \in[K] \implies  Y^j = Y, ~~\forall j \in [K]\setminus [i].
	}
	We represent the set of all instances satisfies $\CSD$ property by $\TCSD$.
\end{defi}
The $\CSD$ property implies that if the feedback of an arm is the same as the true reward of a given context then, the feedback of all the arms in the subsequent stages of the cascade is also the same as the true reward of a given context.

When any problem instance satisfies $\CSD$ property, the value of $\Prob{\Yti = \Yt, \Yti \ne \Ytj|X=x_t} = 0$ for $j>i$. Therefore, for any $(i,j)$ pair of arms and context $\xt$ the following is true:
\eqs{
	\forall j>i, \gamma_{i}(\xt) - \gamma_{j}(\xt) = \Prob{\Yti \ne \Ytj| X= \xt}.
} 
The above equation implies that $\CWD$ property holds trivially for the problem instances that satisfy $\CSD$ property as the difference of mean losses is the same as the probability of disagreement between two arms(fix arm $i = \ist$ for given context $\xt$).

\subsection*{Effect of adding more arms on Weak Dominance property}  
The performance of \ref{alg:CUSS_WD} can deteriorate as we increases as the number of arms. This is because the fraction of contexts that satisfy $\WD$ property can decrease with the increase in the number of arms. To see that, consider a contextual USS problem instance with three arms where arm $1$ has cost $0.1$, arm $2$ has cost $0.2$, and arm $3$ has cost $0.3$. Let there be two contexts $x_1$ and $x_2$ such that classifier $2$ is an optimal classifier for context $x_1$ and classifier $3$ for the context $x_2$, and both contexts satisfy $\WD$ property. When a new arm is added at the end of the classifiers cascade without changing the optimal arm for the contexts, let $p_{24}^{(1)}$ be the disagreement probability for classifier $2$ and $4$ for context $x_1$ and $p_{34}^{(2)}$ be the disagreement probability for classifier $3$ and $4$ for context $x_2$. It is easy to verify that if cost of using classifier $4$ is less than $\min\{0.2+p_{24}^{(1)}, 0.3+p_{34}^{(2)}\}$ then both contexts will not satisfy $\WD$ property.

%% file: chapter4/main.tex
%!TEX root =  ../thesis.tex

In this chapter, we consider the problem of sequentially allocating resources in a \emph{censored semi-bandits} setup, where the learner allocates resources at each step to the arms and observes loss. The loss depends on two hidden parameters, one specific to the arm but independent of the resource allocation, and the other depends on the allocated resource. More specifically, the loss equals zero for an arm if the resource allocated to it exceeds a constant (but unknown) arm dependent threshold. The goal is to learn a resource allocation that minimizes the expected loss. The problem is challenging because the loss distribution and threshold value of each arm are unknown. We study this setting by establishing its `equivalence' to Multiple-Play Multi-Armed Bandits (MP-MAB) and Combinatorial Semi-Bandits. Exploiting these equivalences, we derive optimal algorithms for our problem setting using known algorithms for MP-MAB and Combinatorial Semi-Bandits. The experiments on synthetically generated data validate the performance guarantees of the proposed algorithms.

\section{Censored Semi-Bandits (CSB)}
\label{sec:csb_introduction}
\input{chapter4/introduction}

\section{Problem Setting}		
\label{sec:csb_problemSetting}
\input{chapter4/problem_setting}

\section{Horizon-Dependent Algorithms with Known Lower Bound on Mean Losses}
\label{sec:csb_known}
\input{chapter4/known}

\nocite{NeurIPS19_verma2019censored}

\section{Anytime Algorithms}
\label{sec:csb_unknown}
\input{chapter4/unknown}
\nocite{Arxiv21_verma2021censored}

\section{Experiments}
\label{sec:csb_experiments}
\input{chapter4/experiment}

\section{CSB for Stochastic Network Utility Maximization}
\label{sec:csb_num}
\input{chapter4/num}
\nocite{INFOCOM20_verma2020stochastic}

\section{Appendix}
\label{sec:csb_appendix}
\input{chapter4/appendix}

%% file: chapter4/introduction.tex
%!TEX root =  ../thesis.tex

In the classical multi-armed bandit setup, the assumption is that the learner always observes a loss/reward sample as feedback by playing arms (or actions). In many applications, the learner first needs to assign the resources to the arms, and depending on the allocated resource, the loss may or may not be observed from the selected arms. When the loss is not observed, we say that `feedback is censored' and refer to the case as `censored feedback.' Sequential allocation problems with censored feedback have received significant interest in recent times as censoring occurs naturally in several applications. Some of the examples are:

\noindent
{\bf Example 1:} (Policing and poaching control) In opportunistic crime/poaching control, the goal is to minimize total crimes in some regions using available manpower. For this, the police may spread its manpower (resource allocation) across the regions (arms) for patrolling \citep{AOR14_adler2014location, NSE10_curtin2010determining, AAMAS18_gholami2018adversary, AAMAS16_nguyen2016capture, IJCAI17_rosenfeld2017security}. A thief/poacher intending to commit a crime may abstain from committing a crime if the patrol is heavy, otherwise, continue with his plan. Thus, the censoring of feedback occurs when the thief/poacher intending to commit a crime abstains due to fear of getting caught.

\noindent
{\bf Example 2:} (Auctions) In the auction of multiple items (arms), a bidder with fixed budge decides the amount to bid (resource) for each item \citep{MS19_balseiro2019learning, NIPS17_baltaoglu2017online, Allerton11_gummadi2011optimal, ICML14_mohri2014learning, COLT16_weed2016online}.  The bidder gets to see the item's actual worth (feedback) only if she wins; otherwise, they do not see it (censored feedback). Here, the winning of an item depends on the bidding amount.

\noindent
{\bf Example 3:} (Network Utility Maximization) Power is a scarce resource in wireless networks. In multi-channel communication, nodes need to split the power across the channels to maximize their sum-rate \citep{ETT1997_ChargingAndRateControl, INFOCOM20_verma2020stochastic}. Unless a node transmits with enough power level on a channel, its transmission always fails and succeeds with a certain probability when transmitted power is above a certain threshold. Thus, the recipient gets to observe the channel quality only when enough power is given to the nodes; otherwise, it is censored.

Censoring of feedback also occurs in the problem of supplier selection  \citep{NIPS16_abernethy2016threshold}, budget allocation  \citep{UAI12_amin2012budget,ALT18_dagan18a,UAI14_lattimore2014optimal,NIPS15_lattimore2015linear}, and several others. In all these applications, unless enough resource is applied to an arm, the feedback from arms gets censored. The challenge in these problems is how to learn the quality of all the arms by appropriately allocating the available resource and then optimally allocating resources to minimize the total loss incurred.

Classical approaches to this problem are to learn from historical data \citep{AOR14_adler2014location, NSE10_curtin2010determining, IJCAI17_rosenfeld2017security, AAMAS16_zhang2016using}.  Game-theoretic approaches have also been considered \citep{AAMAS18_gholami2018adversary, AAMAS16_nguyen2016capture, IJCAI18_sinha2018stackelberg}, where the user (buyer, criminal, etc.) knows the history of allocations and responds strategically. While the classical approach of learning from historical data fails to capture the problem's sequential nature, the game-theoretic approach is agnostic to the user (buyer, criminal, etc.) behavioral modeling. In this work, we balance these two approaches by proposing a simple yet novel threshold-based user behavioral model, which we term as \emph{Censored Semi-Bandits} (CSB).
Under the CSB model, the loss incurred from each arm follows a generative structure. The learner has access to a fixed amount of resources in each round which can be allocated to the arms. Each arm has an associated threshold that decides whether the learner observes reward from that arm: if the arm receives resources below a threshold, the learner observes a loss from that arm; otherwise, no loss value is observed. The threshold captures behaviors of the arms. For example, in the crime control problem, the threat perception of a thief/poacher being caught in an area in the presence of patrolling determines the threshold level in that area.

In the first variation of our proposed behavioral models, we assume the threshold (user behavioral) is uniform across arms (set of options). We establish that this setup (with known threshold) is `equivalent' to Multiple-Play Multi-Armed Bandits (MP-MAB), where a fixed number of arms is played in each round. We also study the more general variation, where the threshold is arm dependent. We establish that this setup (with known threshold) is equivalent to Combinatorial Semi-Bandits, where a subset of arms to be played is decided by solving a combinatorial $0$-$1$ knapsack problem. Formally, we tackle the sequential nature of the resource allocation problem by establishing its equivalence to the MP-MAB and Combinatorial Semi-Bandits framework. By exploiting this equivalence for our proposed threshold-based behavioral model, we develop novel resource allocation algorithms by adapting existing algorithms and providing optimal regret guarantees. More precisely, we make the following contributions in this chapter:
\begin{itemize} 
    \item In \cref{sec:csb_known}, we improve the state-of-the-art horizon dependent algorithms \citep{NeurIPS19_verma2019censored} for estimating the thresholds and mean losses with arms. The new algorithms are simpler and have better empirical performance as these algorithms collect the loss information during the threshold estimation. 

    \item We develop a novel sequential resource allocation algorithm to the CSB problem with multiple thresholds (the number of thresholds can be smaller than the number of arms). We prove that the regret bound of the algorithm is sub-linear and depends on the number of unique thresholds. We also show empirically that the proposed algorithms have better regret performance.
    
    \item The algorithms in \cref{sec:csb_known} are horizon $(T)$ dependent and requires the minimum mean loss $(\epsilon)$ and an accuracy tolerance $(\delta)$ that decides the stopping criteria for the threshold estimation method as input.  In \cref{sec:csb_unknown}, we develop anytime algorithms that do not need $T$, $\epsilon$, and $\delta$ as input. The anytime algorithms use a linear search based method to estimate the thresholds, which is different from the binary search based method used in the horizon dependent algorithms.
    
    \item We extend the CSB setup to reward maximization setting by discussing the stochastic Network Utility Maximization problem (NUM). In the reward setting, anytime algorithms developed for the loss setting cannot be applied directly. We give algorithms that work with the known value of the time horizon. The details are given in \cref{sec:csb_num}.
\end{itemize}

\subsection{Related Work}
The problem of resource allocation in many areas has received significant interest in recent times. Several directions have been considered in resource allocation problems to tackle crime \citep{NSE10_curtin2010determining, AAMAS18_gholami2018adversary, AAMAS16_nguyen2016capture}, some of which learn from historical data while others are game-theoretic. \cite{NSE10_curtin2010determining} employ a static maximum coverage strategy for spatial police allocation while  \cite{AAMAS18_gholami2018adversary} and \cite{AAMAS16_nguyen2016capture} study game-theoretic and adversarial perpetrator strategies. We, on the other hand, restrict ourselves to a stochastic setting. The work in  \citep{AOR14_adler2014location, IJCAI17_rosenfeld2017security} look at traffic police resource deployment and consider the optimization aspects of the problem using real-time traffic, etc., which differs from the main focus of our work. \cite{AAMAS15_zhang2015keeping} investigates dynamic resource allocation in the context of police patrolling and poaching for opportunistic criminals. Here, they attempt to learn a model of criminals using a dynamic Bayesian network. Our approach proposes simpler and realistic modeling of perpetrators, where we exploit the underlying structure effectively and efficiently.

We pose our problem in the exploration-exploitation paradigm, which involves solving the MP-MAB and combinatorial 0-1 knapsack problem. It is different from the bandits with Knapsacks setting studied in \cite{JACM18_badanidiyuru2018bandits}, where resources get consumed in every round. The work of \cite{NIPS16_abernethy2016threshold}, \cite{Arxiv20_bengs2020multi}, and \cite{ICML18_jain2018firing} are similar to us in the sense that they are also threshold-based settings. However, the thresholding we employ naturally fits our problem and significantly differs from theirs. Specifically, their thresholding is either on a sample generated from an underlying distribution \citep{NIPS16_abernethy2016threshold, ICML18_jain2018firing} or chosen by the learner \cite{Arxiv20_bengs2020multi} in each round. In contrast, we work in a Bernoulli setting where the thresholding is based on the allocation. Resource allocation with semi-bandits feedback  \citep{ALT18_dagan18a, ALT20_fontaine2020adaptive, UAI14_lattimore2014optimal, NIPS15_lattimore2015linear} is also a related but less general setup where the reward is based only on allocation and a hidden threshold. Our setting requires an additional unknown parameter for each arm, a `mean loss,' which also affects the reward. 
When the learner observes no loss in the CSB setup, it is difficult to say whether it is an actual loss or a censored loss due to enough resource allocation. This dilemma leads to the learner's inability to infer loss from observed feedback when enough resources are allocated to arms. The extreme forms of such problems are studied in \citep{AISTATS19_verma2019online, ACML20_verma2020thompson, NeurIPS20_verma2020online}, where the learner can not infer the loss/ reward from the observed feedback.

Resource allocation problems in the combinatorial setting have been explored in \citep{JCSS12_cesa2012combinatorial, NIPS16_chen2016combinatorial, ICML13_chen2013combinatorial, NIPS15_combes2015combinatorial, NeurIPS20_perrault2020statistical,NIPS14_rajkumar2014online, ICML18_wang2018thompson}. Even though these are not related to our setting directly, we derive explicit connections to the sub-problem of our algorithms to the setup of \cite{ICML15_komiyama2015optimal} and \cite{NeurIPS20_perrault2020statistical}.

%% file: chapter4/problem_setting.tex
%!TEX root =  ../thesis.tex

We consider a sequential learning problem where $K$ denotes the number of arms, and $Q$ denotes the amount of divisible resources. The loss at arm $i \in [K]$ where $[K] := \{1, 2, \ldots, K\}$, is Bernoulli distributed with mean $\mu_i \in [0, 1]$ and independent and identically distributed (IID), whose realization in the $t^{th}$ round is denoted by $X_{t, i}$. Each arm may be assigned a fraction of resources, which determines the feedback observed and the loss incurred from that arm. Formally, denoting the resources allocated to the arms by $\ba:=\{a_i: i\in [K]\ |\ a_i \in [0, Q]\}$, the loss incurred equals the realization of the arm $X_{t, i}$ if $a_i < \theta_i$, where $\theta_i \in [0, Q]$ is fixed but unknown threshold\footnote{One could consider a smooth function instead of a step function, but the analysis is more involved, and our results need not generalize straightforwardly.}. When $a_i \ge \theta_i$, which corresponds to the scenario when the allocated resources are more than the threshold, we do not observe $X_{t, i}$, and hence the loss equals $0$. \cref{fig:ThresholdFunction} depicts the relationship between allocated resources and mean loss of an arm. For each $i \in [K]$, $\theta_i$ denotes the threshold associated with arm $i$ and is such that a loss is incurred at arm $i$ only if $a_i< \theta_i$. An allocation vector $\ba$ is said to be feasible if $\sum_{i \in [K]} a_i \leq Q$ and set of all feasible allocations is denoted as $\A$. The goal is to find a feasible resource allocation that results in a maximum reduction in the total mean loss. 

\begin{figure}[!ht]
	\centering
	\includegraphics[width=0.4\linewidth]{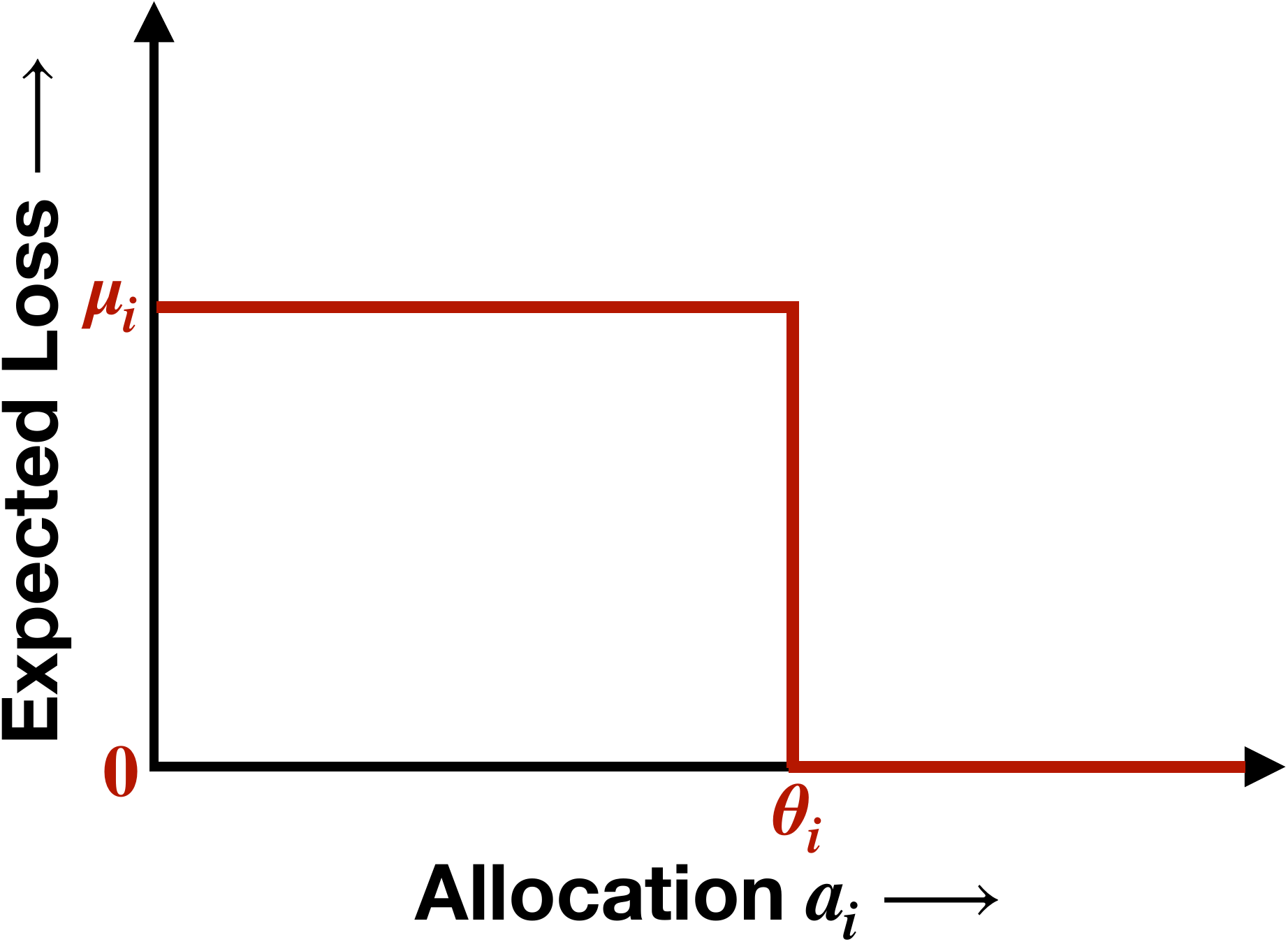}
	\caption{Relationship between allocated resources $(a_i)$ and mean loss $(\mu_i)$ of an arm $i$.}
	\label{fig:ThresholdFunction}
\end{figure}

In the CSB setup, the learner allocates resources to multiple arms. However, loss from the arms may not be observed depending on the amount of resources allocated to them. We thus have a version of the partial monitoring system \citep{MOR14_bartok2014partial,ICML12_bartok2012partial,MOR06_cesa2006regret} with semi-bandit feedback. The vectors $\btheta =\{\theta_i\}_{ i\in [K]}$ and $\bmu =\{\mu_j\}_{i \in [K]}$ are unknown and identify an instance of CSB problem, which we denote henceforth using $P=(\bmu, \btheta, Q) \in [0,1]^{K}\times \R_+^K \times \R_+$. The collection of all CSB instances is denoted as $\PCSB$. For simplicity of discussion, we assume that means are ordered as $\mu_1 \ge  \mu_2 \ge \ldots \ge \mu_K$ and for any integer $M$, refer to the first $M$ arms in the order as the top-$M$ arms. Of course, the algorithm is not aware of this order. For instance $P \in \PCSB$ with known $\bmu$, $\btheta$, and $Q$, the optimal allocation can be computed by solving the following $0$-$1$ knapsack problem:
\begin{equation*}
	\ba^\star  \in \argmin _{\ba \in \A} \sum_{i=1}^K \mu_i \one{a_i< \theta_i}.
\end{equation*}
Here, $\A = \{\ba \in [0, Q]^K | \sum_{i} a_i \le Q\}$ denotes the set of all feasible resource allocations.  Since $\bmu$ and $\btheta$ are unknown, we estimate them in an online fashion using the observations made in each round.
The interaction between the environment and a learner is given in \ref{alg:protocol}.
\begin{algorithm}[!ht]
	\renewcommand{\thealgorithm}{Algorithm 1}
	\floatname{algorithm}{}
	\caption{CSB Problem with instance $(\bmu, \btheta, Q)$}
	\label{alg:protocol}
	In round $t$: 
	\begin{enumerate}
		\setlength{\itemsep}{0pt}
		\item \textbf{Environment} generates a vector $\boldsymbol{X_t} = (X_{t,1}, X_{t,2},\ldots, X_{t,K}) \in \{0,1\}^K$, where $\EE{X_{t,i}}=\mu_i$ and the sequence $(X_{t,i})_{t\geq 1}$ is IID for all $i\in [K]$   
		\item \textbf{Learner} picks an resource allocation vector $\ba_t \in \A$
		\item \textbf{Feedback and Loss:} The learner observes a random feedback $\boldsymbol{Y_t}=\{Y_{t,i}: i\in [K]\}$, where $Y_{t,i}=X_{t,i}\one{a_{t,i}<\theta_i}$ and incurs loss $\sum_{i \in [K]}Y_{t,i}$
	\end{enumerate}
\end{algorithm}

We aim to design optimal strategies that accumulate minimum mean loss and measure its performance by comparing its mean cumulative loss with that of an Oracle that makes the optimal resource allocation in each round. Specifically, we define regret for $T$ rounds as
\begin{equation*}
\EE{\Regret_T} =  \sum_{t=1}^T\sum_{i=1}^K \mu_i \one{a_{t,i}< \theta_i} -   \sum_{t=1}^T\sum_{i=1}^K \mu_i \one{a^\star_i< \theta_i}.%\sum_{t=1}^T\sum_{i=1}^K\mu_i \left(\one{a_{t,i}< \theta_i} -  \one{a^\star_i< \theta_i}\right).
\end{equation*}

Note that minimizing the mean cumulative loss of a policy is the same as minimizing the policy's regret. Our goal is to learn a policy that gives sub-linear expected regret, i.e., $\EE{\Regret_T}/T \rightarrow 0$ as $T \rightarrow \infty$. It implies that a leaner collects almost as much reward in the long run as an oracle who knows the best action from the beginning.

\subsection{Allocation Equivalent}
Next, we define when a pair of threshold vectors for the given loss vector and resources to be `equivalent.'
\begin{definition}[Allocation Equivalent]
	For fixed loss vector $\bmu$ and resources $Q$, two threshold vectors $\btheta$ and $\hat{\btheta}$ are {\em allocation equivalent} if the following holds:
	\begin{equation*}
	\min_{\ba \in \A} \sum_{i=1}^K\mu_i \one{a_i\ge \theta_i}  = 
	\min_{\ba \in \A} \sum_{i=1}^K\mu_i \one{a_i\ge \hat{\theta}_i}.
	\end{equation*}
\end{definition}
In simple words, we say that two threshold vectors $\btheta$ and $\hat{\btheta}$ are allocation equivalent if the minimum mean loss in instances $(\bmu,\btheta, Q)$ and $(\bmu, \hat{\btheta}, Q)$ are the same for fixed loss vector $\bmu$ and resource $Q$. Such equivalence allows us to estimate the threshold vector within some tolerance.

For ease of exposition and to bring out the algorithmic ideas clearly, in \cref{sec:csb_known}, we start with a setting where we assume that time horizon ($T$) is known and mean rewards are larger than some known $\epsilon>0$, i.e., $\mu_i\geq\epsilon$ for all $i\in[K]$. This setup will aid in connecting our problem with the Multi-player bandits. In \cref{sec:csb_unknown}, we relax these assumptions and develop anytime algorithms that do not need to know $\epsilon$. The algorithms in \cref{sec:csb_known}, are based on binary search methods, while that in \cref{sec:csb_unknown}, are based on linear search methods.

%% file: chapter4/known.tex
%!TEX root =  ../thesis.tex

This section introduces the algorithms for solving the CSB problem, where the time horizon ($T$) and the lower bound on the mean losses ($\epsilon$) are known. With this information, we can estimate the allocation equivalent using a binary search based method. Once allocation equivalent is known, the mean losses are estimated, and accordingly, resources are allocated among the arms. We first study a simple case where all arms have the same threshold and then study the more general case where all arms may not have the same threshold.

\subsection{Arms with Same Threshold}
\input{chapter4/known_same}

\subsection{Arms with Multiple Threshold}
\input{chapter4/known_multiple}

%% file: chapter4/known_same.tex
%!TEX root =  ../thesis.tex

We first focus on the simple case, where the threshold of all arms are the same, i.e., $\theta_i=\theta_s$ for all $i \in [K]$ to bring out the main ideas of the algorithm we develop. With abuse of notation, we continue to denote an instance of CSB with the same threshold as $(\bmu, \theta_s, Q)$, where $\theta_s \in (0, Q]$. Note that the threshold is the same, but the mean losses can be different across the arms.  Though $\theta_s$ can take any value in the interval $(0, Q]$, a threshold equivalent to $\theta_s$ can be confined to a finite set. The following lemma shows that a threshold equivalent lies in a set consisting of the $K$ elements. 
\begin{restatable}{lemma}{ThetaSet}
	\label{lem:thetaSet}
	Let $\theta_s \in (0,Q]$, $M=\min\{\lfloor Q/\theta_s \rfloor,$ $K\}$ and $\hat{\theta}_s=Q/M$. 
	Then $\theta_s$ and $\hat{\theta}_s$ are threshold equivalent. Further, $\hat{\theta}_s \in \Theta$ where $\Theta = \{ Q/K, Q/(K-1), \cdots, Q\}$. 
\end{restatable}

Let $M= \min\{\lfloor Q/\theta_s\rfloor, K\}$. When arms are sorted in the decreasing order of mean losses, we refer to the first $M$ arms as the \emph{top-}$M$ arms and the remaining arms as \emph{bottom-}$(K-M)$ arms. The optimal allocation with the same threshold $\theta_s$ is to allocate $\theta_s$ amount of resource to each of the \emph{top-}$M$ arms and allocate the remaining resources to the other arms. The detailed proof of \cref{lem:thetaSet} and all other missing proofs appear in \cref{sec:csb_appendix}.

\cref{lem:thetaSet} shows that the candidates for the threshold equivalent $\hat{\theta}_s$ for any instance $(\bmu,\theta_s,Q)$ are finite. Once the threshold equivalent is known, the problem reduces to identifying the \emph{top-}$M$ arms and assigning resource $\hat{\theta}_s$ to each one of them to minimize the total mean loss. The latter part is equivalent to solving a Multiple-Play Multi-Armed Bandits problem, as discussed next.

After knowing the allocation equivalent, a learner's optimal policy is to allocate $\hat\theta_s$ fraction of resource among $M$ arms having the highest mean loss. As initially, mean losses are not known, empirical estimates of the losses can be used. When resource $\hat{\theta}_s$ is allocated to $M$ arm having the highest empirical losses, no loss is observed from them, but a loss of each of the remaining $K-M$ arms is observed (semi-bandits). In bandits literature, such problems where one can sample rewards (losses) from a subset of arms is known as the Stochastic Multiple-Play Multi-Armed Bandits (MP-MAB) problem. Thus once the learner identifies a threshold equivalent of $\theta$, the CSB problem is equivalent to solving an MP-MAB problem. We adapt the MP-TS algorithm \citep{ICML15_komiyama2015optimal} to our problem as it is shown to achieve optimal regret bound for Bernoulli distributions.

\subsubsection{Equivalence to Multiple-Play Multi-Armed Bandits}
The learner can play a subset of arms in each round known as superarm \citep{TAC1987_MultiPlayBandits_Anatharam} in the stochastic Multiple-Play Multi-Armed Bandits (MP-MAB) \cite{ICML15_komiyama2015optimal}. The size of each superarm is fixed (and known). The mean loss of a superarm is the sum of the means of its constituting arms. The learner plays a superarm in each round and then observes the loss from each arm played (semi-bandit feedback). The learner's goal is to play a superarm that has the smallest mean loss. A policy in MP-MAB selects a superarm in each round based on past information. The policy's performance is measured in terms of regret, defined as the difference between cumulative loss incurred by policy and that incurred by playing an optimal superarm in each round. Let $(\bmu, m) \in [0,1]^K \times \N_+$ denote an instance of MP-MAB where $\bmu$ denotes the mean loss vector, and $m \le K$ denotes the size of each superarm. Let $\PCSB_s \subset \PCSB$ denote the set of CSB instances with the same threshold for all arms. For any  $(\bmu,\theta_s, Q) \in \PCSB_s$ with $K$ arms and known threshold $\theta_s$, let $(\bmu, m)$ be an instance of MP-MAB with $K$ arms and each arm has the same Bernoulli distribution as the corresponding arm in the CSB instance with $m=K-M$, where $M=\min\{\floor{Q/\theta_s}, K\}$ as earlier. Let $\PMP$ denote the set of resulting MP-MAB problems and $f: \PCSB_s \rightarrow \PMP$ denote the above transformation.

Let $\pi$ be a policy on $\PMP$. We can use the policy $\pi$ for any $(\bmu,\theta_s, Q) \in \PCSB_s$ with known $\theta_s$ to select which set of arms to allocate resources. It is done as follows: In round $t$, let the information $(C_1, Y_1, C_2,Y_2, \ldots, C_{t-1}, Y_{t-1})$ collected from a CSB instance, where $C_r$ is the set of $K-M$ arms where no resource is allocated  in round $r$ and $Y_r$ is the samples observed from these arms. This information is given to policy $\pi$, which returns a set $C_t$ with $K-M$ elements in round $t$. Then all arms other than arms in $C_t$ are given resource $\theta_s$. Let this policy on $(\bmu,\theta_s, Q) \in \PCSB_s$ be denoted as $\pi^\prime$. Similarly, let $\beta^\prime$ be a policy on $\PCSB_s$ that can be adapted to yield a policy for $\PMP$ as follows: In round $t$, let the information $(M_1, Y_1, M_2,Y_2, \ldots, M_{t-1}, Y_{t-1})$ collected from an MP-MAB instance, where $M_r$ is the superarm played in round $r$ and $Y_r$ is the associated loss observed from each arms in $M_r$,  is given to the policy $\beta^\prime$ which returns a set $M_t$ of $K-M$ arms where no resources has to be applied. The superarm corresponding to $M_t$ is then played. Let this policy on $\PMP$ be denoted as $\beta$. Note that when $\theta_s$ is known, the mapping is invertible. Our next result gives regret equivalence between the MP-MAB problem and CSB problem with a known same threshold. 
\begin{restatable}{proposition}{RegretEquiST}
	\label{prop:RegretEquiST}
	Let  $f: \PCSB_s \rightarrow \PMP$ and $P=(\bmu,\theta_s, Q) \in \PCSB_s$ with known $\theta_s$. Then the regret of  policy $\pi^\prime$ on $P$ is same as the regret of policy $\pi$ on $f(P)$. Similarly, let $P^\prime=(\bmu,m) \in \PMP$, then the regret of a policy $\beta$ on $P^\prime$ is same as the regret of policy $\beta^\prime$ on $f^{-1}(P^\prime)$. Thus the set $\PCSB$ with a known $\theta_s$ is 'regret equivalent' to $\PMP$, i.e., $\Regret(\PCSB_s)=\Regret(\PMP)$. 
\end{restatable}

The above proposition suggests that any algorithm which works well for the MP-MAB problem also works well for the CSB problem once the threshold is known. Hence one can use MP-MAB algorithms like MP-TS \citep{ICML15_komiyama2015optimal} and ESCB \citep{NIPS15_combes2015combinatorial} after knowing the threshold equivalent of $\theta_s$. MP-TS uses Thompson Sampling, whereas ESCB uses UCB (Upper Confidence Bound) and KL-UCB type indices. One can use any one of these algorithms. But we adapt MP-TS to our setting as it gives better empirical performance and is shown to achieve optimal regret bound for Bernoulli distributed rewards (losses). We next discuss the lower bound for CSB instances with the same threshold.

\paragraph{Lower bound.} As a consequence of the above equivalence and one-to-one correspondence, a lower bound on MP-MAB is also a lower bound on the CSB instance with the same threshold. Therefore, the following lower bound given for any strongly consistent algorithm \cite[Theorem 3.1]{TAC1987_MultiPlayBandits_Anatharam} is also a lower bound on the CSB problem with the same threshold:
\begin{equation}
\label{eqn:LowerBound}
\lim_{T\rightarrow \infty} {\frac{\mathbb{E}[\Regret_T]}{\log T} \ge \sum_{i=1}^{M}  \frac{\mu_{i}-\mu_{M+1}}{d(\mu_{M+1}, \mu_{i})} },
\end{equation}
where $d(p,q)$ is the KL divergence between two Bernoulli distributions with parameter $p$ and $q$. Also note that we are in loss setting.

\subsubsection{Algorithm: \ref{alg:CSB-ST}}
We develop an algorithm named \ref{alg:CSB-ST} for solving the Censored Semi-Bandits problem having the same threshold for all arms. It exploits the result in \cref{lem:thetaSet}, to learn an allocation equivalent of threshold and regret equivalence established in \cref{prop:RegretEquiST} to minimize the regret using an MP-MAB algorithm. \ref{alg:CSB-ST} works as follows: It takes $\delta$ and $\epsilon$ as input, where $\delta$ is the confidence on the correctness of estimated allocation equivalent and $\epsilon$ is such that $\mu_K \ge \epsilon > 0$. The value of $\delta$ can be a function of horizon ($T$), e.g., $\delta=1/T$. We set the prior distribution for each arm's mean loss as the Beta distribution $\beta(1, 1)$. For each arm $i \in [K], ~S_i$ represents the number of rounds when the loss is $1$, and $F_i$ represents the number of rounds when the loss is $0$ whenever the arm $i$ receives resource above its threshold.

We initialize $\Theta=\{Q/K, Q/(K-1), \dots, Q\}$ as given in Lemma \ref{lem:thetaSet}. The elements of $\Theta$ are in increasing order, and each of them is a candidate for allocation equivalent of  $\theta_s$. We use the set $\Theta$ to find the threshold estimate $\hat{\theta}_s$, which is threshold equivalent to the underlying threshold $\theta_s$ with high probability (at least $1-\delta$) by doing a binary search over it. The search starts by taking $\hat{\theta}_s$ to be the middle element in $\Theta$. The variables $l,~u,$ and $j$ are maintained to keep track of the estimation of allocation equivalent. The variable $l$ represents the lowest index of the possible candidate for allocation equivalent, $u$ represents the largest index of the possible candidate for allocation equivalent, and $j$ represents the element of the set $\Theta$, which will be used as a threshold in the next round. Let $S_i(t)$ and $F_i(t)$ denote the values of $S_i$ and $F_i$ in the starting of the round $t$. In round $t$, a sample $\hat\mu_{t,i}$ is drawn from $\beta(S_i(t), F_i(t))$ for each arm $i \in [K]$, which is independent of other arms. The values of $\hat\mu_{t,i}$ are ranked in the decreasing order, and the \emph{top-}$(Q/\hat\theta_s)$ (denoted as set $A_t$) arms are allocated $\hat\theta_s$ amount of resource, and their losses are observed.

Before knowing allocation equivalent, if a loss is observed at any of the arms in the set $A_t$, it implies that $\hat{\theta}_s$ is an underestimate of allocation equivalent. Hence $\hat{\theta}_s$ and all the candidates smaller than the value of $\hat{\theta}_s$ in $\Theta$ are eliminated, and the binary search is repeated in the remaining half of the elements again by starting with the middle element. The loss and no-loss counts are also updated as $S_i = S_i + X_{t,i}, F_i = F_i + 1-X_{t,i}+Z_i$ for all arms. The variable $Z_i, ~\forall i\in [K]$ keeps track of how many times no loss is observed for arm $i$ before a loss is observed when the arm $i$ has allocated $\hat{\theta}_s$ amount of resources. The variable $Z_i, i \in [K]$ is maintained for each arm because the arms in $A_t$ may not be the same in each round. It allows us to distinguish the zeros observed when the arm receives over and under resource allocation. Once a loss is observed for any arm in set $A_t$, the variable $Z_i$ is reset to zero for all arms.

If no loss is observed for all arms in the set $A_t$, $Z_i$ is incremented by $1$ for each arm $i \in A_t$ and variable $C$ is incremented by $1$. The variable $C$ keeps track of the number of consecutive rounds for which no loss is observed on all the arms that are allocated $\hat{\theta}_s$ amount of resource. It changes to $0$ either after observing a loss or if no loss is observed for consecutive $W_\delta$ rounds, where the value of $W_\delta$ ensures $\hat{\theta}_s$ is an allocation equivalent with the probability of at least $1-\delta$. If $C$ equals $W_\delta$, then with high probability, $\hat{\theta}_i$ is possibly an overestimate of allocation equivalent. Accordingly, all the candidates larger than the current value of $\hat{\theta}_s$ in $\Theta$ are eliminated, and the binary search is repeated, starting with the middle element in the remaining half. Note that the current value of $\hat{\theta}_s$ is not eliminated because it is possible that $\hat{\theta}_s$ may be only upper bound for threshold. The value of $C$ as well as $Z_i,\; \forall i \in [K]$ are reset to $0$. Resetting $Z_i$ values to zero once the number of zeros observed reaches $W_\delta$ ensures that they do not add to $F_i$ values when the resources are over-allocated. After this, the loss and no-loss counts are updated as $S_i = S_i + X_{t,i}, F_i = F_i + 1-X_{t,i}$ for each arm $i \in [K] \setminus A_t$.

\begin{algorithm}[!ht] 
	\renewcommand{\thealgorithm}{\bf CSB-SK} 
	\floatname{algorithm}{}
	\caption{Algorithm for CSB problem having Same Threshold with Known Horizon and $\epsilon$}
	\label{alg:CSB-ST}
	\begin{algorithmic}[1]
		\State \textbf{Input:} $\delta, \epsilon$
		\State Set $W_\delta = {\log(\log_2(K)/\delta)}/({\log(1/(1-\epsilon))})$ and $\forall i \in [K]: S_i=1, F_i=1, Z_i=0$ 
		\State Initialize $\Theta$ as given in Lemma \ref{lem:thetaSet}, $C=0, l=1, u = K, j = \floor{(l+u)/2}$ 
		\For{$t=1,2,\ldots,$}
			\State Set $\hat{\theta}_s = \Theta[j]$ and $\forall i \in [K]: \hat{\mu}_{t,i} \leftarrow \beta(S_i, F_i)$
			\State $A_t \leftarrow$ set of \emph{top-}$({Q}/{\hat{\theta}_s})$ arms with the largest values of $\hat{\mu}_{t,i}$
			\State $\forall i \in A_t:$ allocate $\hat{\theta}_s$ resource and observe $X_{t,i}$
			\If{$j \ne u$} 
				\If{$X_{t,a} = 1$ for any $a \in A_t$} 
					\State Set $l=j+1, ~j = \floor{(l+u)/2}, C=0$
					\State $\forall i \in [K]$: set $S_i = S_i+X_{t,i}, F_i = F_i+1-X_{t,i} + Z_i, Z_i =0$
				\Else
					\State Set $C = C + 1$ and $\forall i \in A_t: Z_i = Z_i + 1$
					\State If $C = W_\delta$ then set $u=j, j = \floor{(l+u)/2}$, $C=0, \forall i \in [K]: Z_i =0$
					\State $\forall i \in [K]\setminus A_t: S_i = S_i + X_{t,i}, F_i = F_i + 1 - X_{t,i}$
				\EndIf
				\Else
					\State $\forall i \in [K]\setminus A_t: S_i = S_i + X_{t,i}, F_i = F_i + 1 - X_{t,i}$
			\EndIf
		\EndFor
	\end{algorithmic}
\end{algorithm}

Since $\Theta$ has $K$ elements, the search for an allocation equivalent of $\theta_s$ terminates in a finite number of rounds with high probability. Once this happens, the algorithm allocates resources among \emph{top-}$(Q/\hat\theta_s)$ arms (from Lemma \ref{lem:thetaSet}) in the subsequent rounds and observes losses from remaining arms, i.e., the losses are observed for $K-M$ arms (multiple-play) in each round, where $M = Q/\hat\theta_s$. Observe that the \emph{top-}$(Q/\hat\theta_s)$ arms correspond to top arms with the highest estimated means, which are generated from an associated beta distribution. Hence after finding the allocation equivalent of $\theta_s$, our algorithm is the same as MP-TS. We leverage this observation to adapt the regret bounds of MP-TS to our loss setting.

Once $\hat\theta_s$ is known, the mean losses vector $\bmu$ needs to be estimated. The resources can be allocated such that no losses are observed for maximum $M$ arms. As our goal is to minimize the mean loss, we have to select $M$ arms with the highest mean loss and then allocate $\hat\theta_s$ to each of them. It is equivalent to find $K-M$ arms with the least mean loss, then allocate no resources to these arms and observe their losses. These losses are then used for updating the empirical estimate of the mean loss of arms.

\subsubsection{Analysis of \ref{alg:CSB-ST}}
\label{sssec:sameThetaRegretBounds}
Note that when $\hat{\theta}_s$ is an underestimate, and no loss is observed for consecutive $W_\delta$ rounds, then $\hat{\theta}_s$ will be reduced, which leads to a wrong estimate of $\hat{\theta}_s$. To avoid this, we set the value of $W_\delta$ such that the probability of happening of such an event is upper bounded by $\delta$. The next lemma gives a bound on the number of rounds needed to find threshold equivalent for threshold $\theta_s$ with high probability.

\begin{restatable}{lemma}{sameThresholdEstRounds}
	\label{lem:sameThresholdEstRounds}
	Let $(\bmu, \theta_s, Q)$ be an CSB instance with same threshold, where $\mu_1 \geq  \epsilon>0$. Then with probability at least $1-\delta$, the number of rounds needed by \ref{alg:CSB-ST} to find the  threshold equivalent of  $\theta_s$ is upper bounded by 
	\begin{equation*}
	T_{\theta_{s}^k}\le \frac{\log(\log_2(K)/\delta)}{\log\left({1}/{(1-\epsilon)}\right)}\log_2(K).
	\end{equation*}
\end{restatable}

For instance $(\bmu,\theta, Q)$ and any feasible allocation $\ba \in \A$, we define $\nabla_{\ba} = \sum_{i=1}^K\mu_i\big(\one{a_i < \theta_i} - \one{a_i^\star < \theta_i}\big)$, $\nabla_{\max} = \max\limits_{\ba \in \A } \nabla_{\ba}$, and $\nabla_{\min} = \min\limits_{\ba \in \A } \nabla_{\ba}$. We are now ready to state the regret bound.

\begin{restatable}{theorem}{regretSameThreshold}
	\label{thm:regretSameThreshold}
	Let $\mu_K\geq \epsilon>0$, $W_\delta = {\log(\log_2(K)/\delta)}/{\log(1/(1-\epsilon))}$, $\mu_{M} > \mu_{M+1},$ and $T>T_{\theta_{s}^k}$. Set $\delta=T^{-(\log T)^{-\alpha}}$ in \ref{alg:CSB-ST} such that $\alpha >0$. Then the expected regret of \ref{alg:CSB-ST} is upper bounded by
	\begin{equation*}
		\EE{\Regret_T} \le W_\delta\log_2{(K)}\nabla_{\max}   + O\left((\log T)^{{2}/{3}}\right) + \sum_{i \in [M]} \frac{(\mu_i-\mu_{M+1} )\log {T}}{d( \mu_{M+1},\mu_i)}.
	\end{equation*}
\end{restatable}

The first term in the regret bound of Theorem \ref{thm:regretSameThreshold} corresponds to the regret due to the estimation of allocation equivalent, and the remaining regret corresponds to the expected regret incurred after knowing the allocation equivalent.
Observe that the assumption $\mu_K\ge\epsilon>0$ is only required to guarantee that the estimation of  allocation equivalent terminates in a finite number of rounds. This assumption is not needed to get the bound on expected regret after knowing allocation equivalent. The assumption $\mu_{M} > \mu_{M+1}$ ensures that Kullback-Leibler divergence in the regret bound is well defined. This assumption is also equivalent to assuming that the set of \emph{top-}$M$ arms is unique.

\begin{corollary}
	\label{cor:OptimalBoundST}
	The regret of \ref{alg:CSB-ST} is asymptotically optimal.
\end{corollary}

Setting $\delta=T^{-(\log T)^{-\alpha}}$ in \ref{alg:CSB-ST} for any  $\alpha>0$ leads to $W_\delta = O\left((\log T)^{1-\alpha}\right)$. Now the proof of Corollary \ref{cor:OptimalBoundST} follows by comparing the expected regret bound with the lower bound given in \cref{eqn:LowerBound}.

%% file: chapter4/known_multiple.tex
%!TEX root =  ../thesis.tex

We now consider a more general case, where the threshold may not be the same for all arms. We assume that the number of different thresholds are $n$. If $n=K$ then all thresholds are different. The first difficulty with this setup is finding an optimal allocation that needs not be just allocating resource to top $M$ arms. To see this, consider a  problem instance $(\bmu, \btheta, C)$ with $\bmu = (0.9,0.6,0.4)$, $\btheta = (0.6, 0.55, 0.45)$, and $Q=1$. The optimal allocation is $\ba^\star = (0, 0.55, 0.45)$ with no resource allocated to the top arm. Our next result gives the optimal allocation for an instance in $\PCSB$. Let $KP(\bmu,\btheta, Q)$ denote a $0$-$1$ knapsack problem with capacity $Q$ and $K$ items where item $i$ has weight $\theta_i$ and value $\mu_i$. 

\begin{restatable}{proposition}{diffThetaOptiSoln}
	\label{prop:diffThetaOptiSoln}
	Let $P=(\bmu,\btheta,Q) \in \PCSB$. Then the optimal allocation for $P$ is a solution of $KP(\bmu,\btheta, Q)$.
\end{restatable}

Observe that assigning $\theta_i$ resource to arm $i$ decreases the total mean loss by an amount $\mu_i$. As the goal is to allocate resources such that the total mean loss is minimized, i.e., $\min_{\ba \in \A}$  $\sum_{i\in[K]}\mu_i\one{a_i < \theta_i}$. It is equivalent to solving a 0-1 knapsack with capacity $Q$ where item $i$ has weight $\theta_i$ and value $\mu_i$. The second difficulty of having different thresholds is that the estimation of each arm's threshold is needed to be done separately. Unfortunately, we do not have a result equivalent of Lemma \ref{lem:thetaSet} so that the search space can be restricted to a finite set. We need to search over the entire $(0, Q]$ interval for each arm.

For an instance $P:=(\bmu,\btheta, Q)$, recall that $\ba^\star=(a_1^\star, \ldots, a_K^\star)$ denotes the optimal allocation. Let $r = Q - \sum_{i: a_i^\star \ge \theta_i}\theta_i$, where $r$ is the residual resources after the optimal allocation. Define $\gamma:=r/K$. Any instance with $\gamma = 0$ becomes a `hopeless' problem instance as the only vector that is the allocation equivalent of $\btheta$ is $\btheta$ itself, i.e., $a_i^\star=\theta_i, \;\forall i \in [K]$, which needs $\theta_i$ values to be estimated accurately to achieve optimal allocation. However, for $\gamma>0$, one can find the allocation equivalent with small errors in $\theta_i$ values; hence it can be estimated in a finite time as shown next result.
\begin{restatable}{lemma}{diffTheteEst}
	\label{lem:diffTheteEst}
	Let $\gamma = r/K$ and  $\forall i \in  [K]: \hat\theta_i \in [\theta_i, \ceil{\theta_i/\gamma} \gamma]$. Then $\hat{\btheta}$ is allocation equivalent of $\btheta$. % \theta_i+
\end{restatable}
The proof follows by an application of Theorem 3.2 in \cite{DO13_hifi2013sensitivity}, which gives conditions for two weight vectors $\btheta_1$ and $\btheta_2$ to have the same solution in $KP(\bmu,\btheta_1,Q)$ and $KP(\bmu,\btheta_2, Q)$ for fixed $\bmu$ and $Q$. 
The next definition describes when we can say that two thresholds are different.
\begin{definition}
	We say that two thresholds $\theta_i$ and $\theta_j$ are {\em different} if 
	$\ceil{{\theta_i}/{\gamma}} \ne \ceil{{\theta_j}/{\gamma}}$.
\end{definition}    

\cref{lem:diffTheteEst} and the above definition implies that two thresholds are different if they have different thresholds in the allocation equivalent vector $\hat{\btheta}$.

Once we estimate the allocation equivalent with accuracy such that the estimated $\hat{\btheta}$ is an allocation equivalent of $\btheta$, the problem is equivalent to solving the $KP(\bmu,\hat{\btheta}, Q)$ provided we learn $\bmu$. The learning $\bmu$ is equivalent to solving a Combinatorial Semi-Bandits \citep{ICML13_chen2013combinatorial,NIPS15_combes2015combinatorial,NeurIPS20_perrault2020statistical,ICML18_wang2018thompson} problem. Combinatorial Semi-Bandits is a generalization of MP-MAB, where one needs to identify a superarm (a subset of arms from a collection of subsets) such that the sum of reward/loss of the arms in the selected superarm is the highest/ lowest. The selected superarm's size in each round may not be the same in the Combinatorial Semi-Bandits problem. We could use an algorithm that works well for the Combinatorial Semi-Bandits, like SDCB \citep{NIPS16_chen2016combinatorial}, CTS \citep{ICML18_wang2018thompson}, and CTS-BETA \citep{NeurIPS20_perrault2020statistical} for solving the CSB problem with the known threshold vector. CTS and CTS-BETA use Thompson Sampling, whereas SDCB uses the UCB type index. 
Our following result gives regret equivalence between the Combinatorial Semi-Bandits and CSB problem with multiple thresholds.
\begin{restatable}{proposition}{MultiThetaEquivalence}
	\label{prop:MultiThetaEquivalence}
	The CSB problem with the known threshold vector $\btheta$ is regret equivalent to a Combinatorial Semi-Bandits where Oracle uses $KP(\bmu,\btheta, Q)$ to identify the optimal superarm.
\end{restatable}

\subsubsection{Algorithm: \ref{alg:CSB-MT} }
We develop an algorithm named \ref{alg:CSB-MT} for solving the Censored Semi-Bandits problem with multiple thresholds. It exploits the result of \cref{lem:diffTheteEst} and the regret equivalence established in \cref{prop:MultiThetaEquivalence} to learn a good estimate of the threshold for each arm and minimizes the regret using the existing algorithm for Combinatorial Semi-Bandits. 
\ref{alg:CSB-MT} works as follows: It takes $n, \delta, \epsilon$ and $\gamma$ as inputs, where $n$ be the number of different thresholds\footnote{If the number of thresholds is unknown then the value of $n$ is set to $K$ in \ref{alg:CSB-MT}. It is equivalent to assuming that all thresholds are different.}, $\delta$ is the confidence on the correctness of estimated allocation equivalent, $\epsilon$ is such that $\mu_K \ge \epsilon > 0$, and $\gamma$ is the K$^{th}$ fraction of the leftover resources after having an optimal allocation of resources. We initialize each arm's prior distribution as the Beta distribution $\beta(1, 1)$. For each arm $i \in [K], S_i$ represents the number of rounds when the loss is $1$, and $F_i$ represents the number of rounds when the loss is $0$ whenever the arm $i$ receives resource above its threshold. The variable $Z_i$ keeps the count of consecutive $0$ for the arm $i$ when allocated the required resource. $Z_i$ changes to $0$ either after observing a loss or if no loss is observed for consecutively $W_\delta$ rounds where the value of $W_\delta$ ensures $\hat{\btheta}$ is an allocation equivalent with the probability of at least $1-\delta$.

The algorithm needs to find a threshold vector that is allocation equivalent of $\btheta$ with high probability. It is achieved by ensuring that $\hat\theta_i \in [\theta_i, \ceil{\theta_i/\gamma}\gamma]$ for each $i \in [K]$ (\cref{lem:diffTheteEst}). The algorithm maintains the variables $\theta_{l,i}, \theta_{u,i}$, $\theta_{g,i}$, and $\hat{\theta}_{i}$ for the estimation of allocation equivalent, where $\hat{\theta}_{i}$ is the estimated value of ${\theta}_{i}$; $\theta_{u,i}$ and $\theta_{l,i}$ is the upper and lower bound of the search region for the threshold of arm $i$ respectively; and $\theta_{g,i}$ indicates whether the current estimate of the threshold lies in the interval $[\theta_i, \ceil{\theta_i/\gamma}\gamma]$ for arm $i$.  
The algorithm also keeps track of set $\Theta_n$ and variable $n_i$, where $\Theta_n$ is the set of estimated thresholds and $n_i$ is the index of arm whose threshold will be searched in the set $\Theta_n$. The set $\Theta_n$ is initialized as empty set whereas the value of $n_i$ is set to $1$ if $n<K$ otherwise $0$. The value of $n_i=0$ ensures that when all thresholds are different, then the threshold is estimated separately for each arm.

Let $S_i(t)$ and $F_i(t)$ denote the value of $S_i$ and $F_i$ at the start of round $t$. In round $t$, for each $i \in [K]$ an independent sample for estimated loss $(\hat\mu_{t,i})$ is drawn from $\beta(S_i(t), F_i(t))$. If there exists any arm whose threshold is not good, then the allocation equivalent needs to be estimated. We say that the threshold estimate of arm $i$ is good by checking the condition $\theta_{u,i}  - \theta_{l,i} \le \gamma$. If the condition satisfies, then the estimated threshold of the arm is within the desired tolerance, and it is indicated by setting $\hat\theta_{g, i}=1$; otherwise, it remains $0$. 

\begin{algorithm}[!ht]
	\small 
	\renewcommand{\thealgorithm}{\bf CSB-MK}
	\floatname{algorithm}{}
	\caption{Algorithm for CSB problem having Multiple Threshold with Known Horizon and $\epsilon$}
	\label{alg:CSB-MT}
	\begin{algorithmic}[1]
		\State \textbf{Input:} $n, \delta, \epsilon, \gamma$
		\State Initialize: $\forall i \in [K]: S_i = 1, F_i = 1, Z_i =0, \theta_{l,i} = 0, \theta_{u,i} = Q, \theta_{g,i} = 0,  \hat\theta_{i} = Q/2$ 
		\State Set $\Theta_n = \emptyset,W_\delta = \log (K\log_2(\lceil 1 + Q/\gamma\rceil)/\delta)/\log(1/(1-\epsilon)),$ if $n<K$ then $n_i =1$ else $n_i=0$ 
		\For{$t=1,2, \ldots,$}
		\State $\forall i \in [K]: \hat{\mu}_{t,i} \leftarrow \text{Beta}(S_i, F_i)$
		\If{$\theta_{g,j} = 0$ for any $j \in [K]$}
		\If{$n < K$}
		\While{$\theta_{g,n_ i} = 1$}
			\State  Add $\hat\theta_{n_i}$ to $\Theta_n$ and set $n_i = n_i + 1$.  Sort $\Theta_n$ in increasing order
		\EndWhile
		\If{there exists no $j \in [|\Theta_n|]$ such that $\theta_{l,n_i} < \Theta_n[j] \le \theta_{u,n_i}$ or $\Theta_n = \emptyset$}
		\State  Set $\hat\theta_{n_i}=(\theta_{l,i}+\theta_{u,i})/2$ 
		\Else
		\State Set $l = \min\{k: \Theta_n[k] > \theta_{l,n_i}\}, u = \max\{k: \Theta_n[k] \le \theta_{u,n_i}\},$ and $j = \floor{(l+u)/2}$
		\State If $\Theta_n[j]=\theta_{u,n_i}$ then set $\hat\theta_{n_i} = \Theta_n[j]-\gamma$ else  $\hat\theta_{n_i} = \Theta_n[j]$
		\EndIf
		\EndIf
		
		\State $\forall i \in [K]\setminus \{n_i\}$: update $\hat\theta_{i}$ using Eq. \eqref{equ:updateTheta}. Allocate $\hat\theta_{i}$ resource to arm $i$ and observe $X_{t,i}$
		\For{$i = \{1,2,\ldots, K\}$}
		\If{$\theta_{g,i} = 0$ and $\hat\theta_{i}> \theta_{l,i}$}
		\State If $X_{t,i}=1$ then set $\theta_{l,i} = \hat\theta_{i}, S_i = S_i + 1, F_i = F_i + Z_i, Z_i =0$ else  $Z_i = Z_i + 1$
		\State If {$Z_i= W_\delta$} then set $\theta_{u,i} = \hat\theta_{i}, Z_i=0 $
		\State If $\theta_{u,i} - \theta_{l,i} \le \gamma$ then set $\theta_{g,i}=1$ and $\hat\theta_i = \theta_{u,i}$
		\ElsIf{$\hat\theta_{i} \le \theta_{l,i}$ or $\big\{ \theta_{g,i} = 1$ and $\hat\theta_{i}<\hat\theta_{u,i} \big\}$}
		\State Set $S_i = S_i+X_{t,i}$ and $F_i = F_i+1-X_{t,i}$ 
		\EndIf	
		\EndFor	
		
		\Else
		\State $A_t \leftarrow$ Oracle$\big( KP(\hat\bmu_{t}, \hat\btheta, C)\big)$ and $\forall i \in A_t:$ allocate $\hat\theta_{i}$ resource
		\State $\forall i \in [K]\setminus A_t:$ observe $X_{t,i}$, update $S_i = S_i+X_{t,i}$ and $F_i = F_i+1-X_{t,i}$
		\EndIf
		\EndFor
	\end{algorithmic}
\end{algorithm}

The threshold is estimated for each arm for finding a threshold equivalent vector. For this, the set $\Theta_n$ is updated by having all the estimated threshold from the arm having $\hat\theta_{g, i}=1$. The elements of the set $\Theta_n$ are sorted in increasing order, and the value of $n_i$ is incremented accordingly. By algorithm design, all the arms whose indices are smaller than the value of $n_i$ are having a good estimate of the threshold. For the arm whose index matches with the value of $n_i$, its threshold is first searched in the set $\Theta_n$ by doing a binary search over elements of the set $\Theta_n$.  If there is no element of the set $\Theta_n$ lies in between the values of lower and upper bound (element can be same as the value of upper bound) of the arm's threshold, then it implies that the threshold of the arm is not in the set $\Theta_n$. Hence, the threshold for arm is estimated using binary search in the interval $(\theta_{l,i}, \theta_{u,i}]$ by setting its value to $(\theta_{l,i}+\theta_{u,i})/2$ in the subsequent rounds.pose there exists an element of the set $\Theta_n$ in between the values of the lower and upper bound of the arm's threshold. In that case, the binary search is used to search the threshold in set $\Theta_n$ by finding the index of smallest $(l)$ and largest element $(u)$ in set $\Theta_n$ whose value is just larger than the lower bound and smaller than or equal to upper bound of the arm's threshold respectively. The element with index $\floor{(l+u)/2}$ is selected as threshold estimate. If the value of the selected threshold matches with the value of the upper bound of the arm's threshold, then it is decreased by $\gamma$ amount to ensure the estimate is indeed the good threshold value for the arm.

For all arms except the arm having index $n_i$, the resource allocation is updated after computing the following events:
\begin{align*}
B_i = \left\{ \frac{\theta_{l,i} + \theta_{u,i}}{2} \le Q - \sum_{\forall j < i:\theta_{g,j}=0} \left(\frac{\theta_{l,j} + \theta_{u,j}}{2} \right) \right\} \mbox{ and} \\
H_i = \left\{\theta_{u,i} \le Q - \sum_{\substack{j \in [K]:\theta_{g,j}=0 \\ \hat\theta_{j} \ne 0}} \hat\theta_{j} -  \sum_{\substack{k \in [K],\theta_{g,k}=1\\ \hat\mu_{t,k}/\theta_{u,k} > \hat\mu_{t,i}/\theta_{u,i} }}  \theta_{u,k} \right\}.
\end{align*}
The event $B_i$ is defined for all arm having a bad threshold estimate, i.e., $\theta_{g, i}=0$ and indicates whether the arm can get desired resources or not. The event $H_i$ is defined for all arms having good threshold estimates, i.e., $\theta_{g, i}=1$ and indicates if the arm can get the required resources or not. By construction, the event $B_i$ does not happen for arms having good threshold estimates, and the event $H_i$ does not happen for arms having a bad threshold estimate. The resources are first allocated among arms having bad threshold estimates to find the allocation equivalent as soon as possible. The leftover resource is allocated to arms with good threshold estimates to decrease the total loss. Among the arms having bad thresholds, the arm with the smallest index gets resources first, followed by the next smallest index. Whereas in the arms having good thresholds, the arms having the highest empirical loss to resource ratio, i.e., $\hat\mu_j/\hat\theta_{i}$ gets resource first, followed by second highest.  The $\hat\theta_{i}$ for arm $i$ is updated as follows:
\begin{align}
\label{equ:updateTheta}
\hat\theta_{i} = 
\begin{cases}
\hat\theta_{i}  							  &\mbox{if $H_i(t)$ happens,} \\
\frac{\theta_{l,i} + \theta_{u,i}}{2}  &\mbox{if $B_i$ happens,} \\
0 												   & \mbox{Otherwise.}
\end{cases}
\end{align}

In round $t$, $\hat\theta_{i}$ amount of resources is allocated to arm $i \in [K]$ and then loss $X_{t,i}$ is observed. If a loss is observed from the arm $i$ that is having a bad threshold estimate ($\theta_{g, i}=0$) and $\hat\theta_i > \theta_{l, i}$, then it implies that $\hat\theta_i$ is an underestimate of $\theta_i$ and the lower end of search region (lower bound of threshold) is increased to $\hat\theta_i$, i.e., $\theta_{l,i}=\hat\theta_i$. 
The success and failure counts are also updated as $S_i = S_i + 1, F_i = F_i + Z_i$, and $Z_i$ is reset to $0$. If no loss is observed, then $Z_i$ is incremented by $1$. If no loss is observed after allocating $\hat\theta_i$ resources for successive $W_\delta$ rounds for arm $i$ with a bad threshold estimate, then it implies that $\hat\theta_i$ is overestimated. 
So, the upper bound of threshold is set to $\hat\theta_{i}$, i.e, $\theta_{u,i}=\hat{\theta}_{t,i}$ and $Z_i$ is reset to $0$.  After updating the lower or upper bound, the condition $\theta_{u, i}  - \theta_{l, i}\le \gamma$ is checked for knowing the goodness of the estimated threshold. If the condition holds, then the arm's threshold estimate is within desired tolerance, which is indicated by setting $\theta_{g, i}$ to 1 and $\hat\theta_i=\theta_{u, i}$ for the subsequent rounds. 
For arms either having resources less than lower bound of threshold ($\hat\theta_{i} \le \theta_{l, i}$) or having good threshold estimate with $\hat\theta_i < \theta_{u, i}$, their success and failure counts are updated as $S_i = S_i + X_{t,i}, F_i = F_i + 1 - X_{t,i}$.

Once we have good threshold estimates for all arms, we could adapt to any algorithm that works well for Combinatorial Semi-Bandits. We adapt the CTS-BETA \citep{NeurIPS20_perrault2020statistical} to our setting due to its better empirical performance. Oracle uses $KL(\hat\bmu_t, \hat\btheta, C)$ to identify the arms in the round $t$ where the learner has to allocate the required resource (denoted as set $A_t$). Each arm $i \in A_t$ has allocated $\hat\theta_i$ amount of resources. A loss $X_{t, i}$ is observed from each arm $i \in [K]\setminus A_t$ and then $S_i = S_i + X_{t,i}, F_i = F_i + 1 - X_{t,i}$ are updated.

\subsubsection{Analysis of \ref{alg:CSB-MT}}
\label{sssec:differentThetaRegretBounds}
The value of $W_\delta$ in \ref{alg:CSB-MT} is set such that the probability of estimated threshold does not lie in $[\theta_i, \ceil{\theta_i/\gamma}\gamma]$ for all arms is upper bounded by $\delta$.  The following lemma gives the upper bound on the number of rounds required to find the allocation equivalent for threshold vector $\btheta$ with a probability of at least $1-\delta$.

\begin{restatable}{lemma}{MultiTheta}
	\label{lem:MultiTheta}
	Let $n$ be the number of different thresholds, $A_{\theta_n}$ be the set of first $n$ arms having different thresholds, and  $(\bmu,\btheta, Q)$ be an instance of CSB such that $\gamma>0$ and $\mu_1 \geq  \epsilon>0$. Then with probability at least $1-\delta$, the number of rounds needed by threshold estimation phase of \ref{alg:CSB-MT} to find the allocation equivalent for threshold vector $\btheta$ is upper bounded by 
	\begin{equation*}
	T_{\theta_n} \le  \frac{\log( K \log_2(\ceil{1 +{Q}/{\gamma}})/\delta)} {\log(1/(1-\epsilon))} \left[\sum_{i \in A_{\Theta_n}} {\log_2 (\ceil{1 +  {Q}/{\gamma}})} + K{\log_2 (n +  1)}\right].
	\end{equation*}
\end{restatable}

Let $\nabla_{\max}$ and $\nabla_{\min}$ be defined as in Section \ref{sssec:sameThetaRegretBounds}. We redefine $W_\delta = \log (K\log_2(\lceil 1 + Q/\gamma\rceil)/\delta)/\log(1/(1-\epsilon))$. Let $\nabla_{i, \min}$ be the minimum regret for superarms containing arm $i$ and $K^\prime$ be the maximum number of arms in any feasible resource allocation. We are now ready to state the regret bound of \ref{alg:CSB-MT}.
\begin{restatable}{theorem}{regretDiffThreshold}
	\label{thm:regretDiffThreshold}
	Let $(\bmu,\btheta, Q)\in \PCSB$ such that $\gamma>0$, $\mu_K\geq \epsilon$, and $T>T_{\theta_n}$. Set $\delta=$ $T^{-(\log T)^{-\alpha}}$ in \ref{alg:CSB-MT} such that $\alpha >0$. Then the expected regret of \ref{alg:CSB-MT} is upper bounded as 
	\begin{align*}
		&\EE{\Regret_T} \le W_\delta\left[\sum_{i \in A_{\Theta_n}} {\log_2 (\ceil{1 +  {Q}/{\gamma}})} + K{\log_2 (n +  1)}\right] \nabla_{\max} + O\left(\sum_{i \in [K]}\frac{\log^2(K^\prime)\log T}{\nabla_{i, \min}} \right).
	\end{align*}
\end{restatable}

The first term of expected regret is due to the estimation of allocation equivalent. As it takes $T_{\theta_n}$ rounds to complete, the maximum regret due to the estimation of allocation equivalent is bounded by $T_{\theta_n}\nabla_{\max}$, where $\nabla_{\max}$ is the maximum regret that can be incurred in any round. The remaining terms correspond to the regret after knowing the allocation equivalent. The expected regret of \ref{alg:CSB-MT} is $O(K\log^2(K^\prime)\log T/\nabla_{\min})$, where $\nabla_{\min}$ is the minimum gap between the mean loss of optimal allocation and any non-optimal allocation. Since the regret scales as $\Omega(K\log T/ \nabla_{\min})$ for the combinatorial semi-bandits \citep{NeurIPS20_perrault2020statistical}, the regret of \ref{alg:CSB-MT} matches to the lower bound up to a logarithmic term.

%% file: chapter4/unknown.tex
%!TEX root =  ../thesis.tex

In this section, we propose algorithms for the CSB problem that do not need to know the time horizon and minimum mean loss. As in the previous section, we deal with cases of the same and different thresholds separately.

\subsection{Arms with Same Threshold}
\input{chapter4/unknown_same}

\subsection{Arms with Different Threshold}
\input{chapter4/unknown_different}

%% file: chapter4/unknown_same.tex
%!TEX root =  ../thesis.tex

First, we develop a Thompson-sampling based algorithm named \ref{alg:CSB-JS} for the CSB problem where all arms have the same threshold $\theta_s$. \ref{alg:CSB-JS} starts with equally distributing the resources among all the $K$ arms and continues to do the same in the following rounds until no loss is observed on any of the arms. Once the loss is observed from any of the arms, then it equally distributes the resources among top $K-1$ arms having the largest estimates of mean losses. The process is repeated till no loss is observed from arms that have been allocated resources. Along the way, the algorithms identify the allocation equivalent of $\theta_s$ and also learns the optimal allocation of resources. 

The pseudo-code of the algorithm is given in \ref{alg:CSB-JS}. It works as follows: For each $i\in [K]$, the variables $S_i$ and $F_i$ are used to keep track of the number of rounds in which the loss is observed or not observed, respectively. No loss is only observed from arm $i$ when it receives at least $\theta_s$ amount of resource. The prior loss distribution of each arm is set as the Beta distribution $\beta(1, 1)$ by initializing $S_i=1$ and $F_i=1$. For each arm $i \in [K]$, let $S_i(t)$ and $F_i(t)$ denote the values of $S_i$ and $F_i$ at the starting of round $t$. In every round $t$, a sample $\hat\mu_i$ is drawn for each arm $i \in [K]$ from $\beta(S_i(t), F_i(t))$ independent of everything else.  Then the top-$L$ arms having the largest empirical mean loss (denoted as set $A_t$) is selected to distribute the resources equally. The value of $L$ is initialized by $K$.

\begin{algorithm}[!ht] 
	\renewcommand{\thealgorithm}{\bf CSB-SU} 
	\floatname{algorithm}{}
	\caption{Algorithm for CSB with Same threshold with Unknown parameters}
	\label{alg:CSB-JS}
	\begin{algorithmic}[1]
		\State Set $L=K, S_i = 1, F_i = 1, Z_i=0 ~~\forall i \in [K]$ 
		\For{$t=1, 2, \ldots$}
		\State $\forall i \in [K]: \hat{\mu}_i(t) \leftarrow \beta(S_i, F_i)$
		\State $A_t \leftarrow$ set of $L$ arms with the largest values of $\hat{\mu}_{t,i}$
		\State $\forall i \in A_t$: allocate ${Q}/{L}$ resources and observe $X_{t,i}$
		\If{$X_{t,j}=1$ for any $j \in A_t$}
		\State  Set $L = L-1$. $\forall i \in A_t:$ update $S_i = S_i+$ $X_{t,i}, F_i = F_i+1-X_{t,i} + Z_i$. $\forall j \in [K]:$ set $Z_j=0$
		\Else
		\State $\forall i \in A_t:$ update $Z_i= Z_i+1$
		\EndIf
		\State $\forall i \in [K]\setminus A_t:$  update $S_i = S_i+X_{t,i}$, $F_i = F_i+1-X_{t,i}$
		\EndFor
	\end{algorithmic}
\end{algorithm}

If a loss is observed on any arms in the set $A_t$, then it implies that the current value of $\hat{\theta}_s$ is an underestimate of $\theta_s$. Hence $L$ is decreased by $1$ and then the success and failure counts are also updated as $S_i = S_i + X_{t,i}, F_i = F_i + 1-X_{t,i}+Z_i$ for each arm $i \in A_t$, and for the all $j \in [K]$, $Z_j$ is reset to $0$. The variables $(Z_i:\forall i\in [K])$ keep track of how many times no loss is observed for arms in the set $A_t$ before a loss is observed for any of arm in set $A_t$. Its value is reset to zero for all arms once a loss is observed for any arm in $A_t$. The variable $Z_i, i\in [k]$ is useful to distinguish between the loss due to randomness when resources are under-allocated and no loss due to over-allocation of resources. If no loss is observed for all arms in the set $A_t$, then $Z_i$ is incremented by $1$ for each arm $i \in A_t$. The values of $S_i$ and $F_i$ are updated for each arm where no resources are allocated.

Since there are only $K$ possible candidates for allocation equivalent, the allocation equivalent for $\theta_s$ is found in the finite number of rounds. Once allocation equivalent is known, the algorithm allocates resources equally among top-$M$ arms (\cref{lem:thetaSet}) in the subsequent rounds and observes loss samples for the remaining $K-M$ arms. The selected arms correspond to top-$M$ arms with the highest estimated mean losses. Hence after an allocation equivalent of $\theta_s, $ is reached, in each round, samples from the $K-M$ arms are observed, which corresponds to selecting the $K-M$ arms with the smallest means. \ref{alg:CSB-JS} is the same as MP-TS that plays $K-M$ arms in each round and aims to minimize the sum of mean losses incurred from $K-M$ arms. We exploit this observation to adapt the regret bounds of MP-TS.

\subsubsection{Analysis of \ref{alg:CSB-JS}}
\label{ssec:analysisCSB_JS}
Let $T_{\theta_s}$ denote number of rounds required to find an allocation equivalent of $\theta_s$. %For instance $(\bmu,\theta_s, Q)$ and any feasible allocation $\ba\in \A$, we define $\Delta_a = \sum_{i=1}^K\mu_i\big(\one{a_i < \theta_i} - \one{a_i^\star < \theta_i}\big)$ and $\Delta_m = \max_{\ba \in \A } \Delta_a$. 
The first result gives the upper bounds on expected value of $T_{\theta_s}$.
\begin{restatable}{lemma}{sameThetaEstRounds}
	\label{lem:sameThetaEstRounds}
	Let $M$ be the number of arms in the optimal allocation. For CSB problem instance $(\bmu,\theta_s, Q)$, the expected number of rounds needed by \ref{alg:CSB-JS} to find an allocation equivalent for threshold $\theta_s$ is upper bounded as
	\begin{equation*}
	\EE{T_{\theta_s}} \le \sum_{L=M + 1}^{K} \frac{1}{1 - \Pi_{i \in [K]/[K-L]}(1-\mu_i)}.
	\end{equation*}	
\end{restatable}

\noindent
Let $\nabla_{\max}$ be defined as in Section \ref{sssec:sameThetaRegretBounds}. We are now ready the state the regret bounds.
\begin{restatable}{theorem}{regretJointSameThreshold}
	\label{thm:regretJointSameThreshold}
	Let $(\bmu,\theta_s, Q)$ be the CSB problem instance with same threshold, $\mu_{M} > \mu_{M+1}$ and $T>T_{\theta_s}$. Then the expected regret of \textnormal{\ref{alg:CSB-JS}} is upper bound as
	\begin{align*}
	\EE{\Regret_T} &\le \sum_{L=M + 1}^{K} \frac{\Delta_{\max}}{1 - \Pi_{i \in [K]/[K-L]}(1-\mu_i)} + O\left(\sum_{i =1}^{M} \frac{(\mu_i - \mu_{M+1})\log {T}}{d(\mu_{M+1},\mu_i)}\right).
	\end{align*}
\end{restatable}

 The proof of \cref{lem:sameThetaEstRounds} follows by deriving the number of rounds required to observe a sample of `$1$' from a set of independent Bernoulli random variables. Whereas for \cref{thm:regretJointSameThreshold}, the first term in the regret bound corresponds to the expected regret incurred due to the estimation of allocation equivalent. The second term in the regret bound corresponds to the expected regret due to the MB-MAB based regret minimization algorithm MP-TS \citep{ICML15_komiyama2015optimal}. The assumption $\mu_{M} > \mu_{M+1}$ ensures that Kullback-Leibler divergence in the regret bound is well defined. 

\begin{corollary}
	\label{cor:OptimalBoundJointST}
	The regret of \textnormal{\ref{alg:CSB-JS}} is asymptotically optimal.
\end{corollary}
The proof follows by comparing the asymptotic regret bound of \ref{alg:CSB-JS} with the lower bound of regret given in \cref{eqn:LowerBound}.

%% file: chapter4/unknown_different.tex
%!TEX root =  ../thesis.tex

In this section, we develop an algorithm named \ref{alg:CSB-JD} for the CSB problem where the thresholds may not be the same. It exploits \cref{lem:diffTheteEst} to find allocation equivalent. \ref{alg:CSB-JD} works as follows: It takes $\gamma$ as input. We initialize each arm's prior distribution as the Beta distribution $\beta(1, 1)$. For each arm $i \in  [K]$, algorithm maintains a variable $L_i$ and set $Z_i$. The variable $L_i$ is the lower bound of the threshold for arm $i$, set $Z_i$ keeps count of the number of time no loss is observed from the arm $i$ for different resource allocations, and $Z_i[\hat\theta_i]$ represents the count of no losses for resource allocation $\hat\theta_i$ to arm $i$. The value of $L_i$ is initially set to $0$, and set $Z_i$ is initialized as an empty set. The set $Z_i[\cdot]$ plays a similar role as to the variable $Z_i$ in \ref{alg:CSB-JD}; however, it needs to store the counts for different resource allocations.

\begin{algorithm}[!ht] 
	\renewcommand{\thealgorithm}{\bf CSB-DU}
	\floatname{algorithm}{}
	\caption{Algorithm for CSB with Different threshold with Unknown parameters}
	\label{alg:CSB-JD}
	\begin{algorithmic}[1]
		\State \textbf{Input:} $\gamma$
		\State $\forall i \in [K]:$ set $S_i = 1, F_i = 1, L_i=0$, and $Z_i = \phi$
		\For{$t=1, 2, \ldots$}	
		\State $\forall i \in [K]: \hat\mu_{i} \leftarrow \beta(S_i, F_i)$.
		\If{$Q - \sum_{i \in [K]} (L_i + \gamma) \ge 0$}
		\State Compute $\hat\btheta$ using \eqref{equ:updateAllocation} and set $A_t \leftarrow [K]$
		\Else
		\State $A_{t} \leftarrow KP( \hat{\bmu}, \hat\btheta, Q)$ where $\hat\theta_{i}=L_i + \gamma$
		\EndIf
		
		\For{$i \in A_t$}
		\State Assign $\hat\theta_{i}$ resources to arm $i$ and observe $X_{t,i}$
		\If {$X_{t,i} = 1$} 
		\State If $L_i < \hat\theta_{i}$ then change $L_i = \hat\theta_{i}$ 
		\State Update $S_i \leftarrow S_i + 1$, $F_i \leftarrow F_i + \sum_{\hat\theta_i \le L_i} Z_i[\hat\theta_i]$, and $\forall \hat\theta_i \le L_i:$ set $Z_i[\hat\theta_i] = 0$
		\Else
		\State If {$\hat\theta_{i}$ is not in $Z_i$} then add $Z_i[\hat\theta_{i}] = 1$ to $Z_i$ otherwise update $Z_i[\hat\theta_{i}] = Z_i[\hat\theta_{i}] + 1$
		\EndIf
		\EndFor
		\State $\forall i \in [K]\setminus A_t:$ observe $X_{t,i}$ and update $S_i = S_i + X_{t,i}, F_i = F_i + 1 - X_{t,i}$
		\EndFor
	\end{algorithmic}
\end{algorithm}

Let $S_i(t)$ and $F_i(t)$ denote the value of $S_i$ and $F_i$ at beginning of the round $t$. In round $t$, for each $i \in [K]$, an independent sample $(\hat\mu_{t,i})$ is drawn from $\beta(S_i(t), F_i(t))$. Initially, the value of the lower bound of the threshold for each arm is set to $0$. At the start, the resources are equally distributed among the arms. In the subsequent rounds, it is incremented by an amount of $\gamma$ for arms on which a loss is observed while uniformly distributing leftover resources among other arms as follows:
\begin{align}
\label{equ:updateAllocation}
\theta_{i} = \begin{cases}
L_i + \gamma & \text{if } L_i \ne 0 \\
Q_l/L_0       & \text{otherwise}
\end{cases}
\end{align}
where $Q_l := Q - \sum_{i \in [K]: L_i \ne 0} (L_i + \gamma)$ are the leftover resources and $L_0 := |\{i: L_i =0\}|$ is the number of arms whose lower bound of threshold is still $0$. Allocating resources equally among arms leads to a better initial lower bound on thresholds.  This process is continued until all the arms can get the required resources.

If resources are not enough for all arms, then the set of arms is selected by solving $KP(\hat{\bmu}, \hat{\btheta}, Q)$ problem (denoted as set $A_t$). Each arm $i \in A_t$ has given resource $\theta_{i}=L_i + \gamma$ and a sample $X_{t,i}$ is observed. If a loss is observed, then it implies that the arm is under-allocated. Accordingly, the lower bound of the threshold for that arm is updated. The success and failure counts are also updated as $S_i = S_i + 1, F_i = F_i + \sum_{\hat\theta_i \le L_i} Z_i[\hat\theta_i]$, and values of set $Z_i[\hat\theta_i]$ with $\hat\theta_i \le L_i$ are changed to $0$. If no loss is observed for arms having required resources and $Z_i[\hat\theta_i]$ is not in set $Z_i$, then add $Z_i[\hat\theta_i]$ to set $Z_i$ with value $1$; otherwise, increment $Z_i[\hat\theta_i]$ by $1$. The success and failure counts are also updated for each arm $i \in [K]\setminus A_t$ as $S_i = S_i + X_{t,i}$ and $F_i = F_i + 1 - X_{t,i}$.

\subsubsection{Analysis of \ref{alg:CSB-JD}} 
Let $T_{\theta_d}$ denote the number of rounds required to find an allocation equivalent of $\btheta$. Our following result gives an upper bound on the expected value of $T_{\theta_d}$.
\begin{restatable}{lemma}{diffThetaEstRounds}
	\label{lem:diffThetaEstRounds}
	For CSB problem instance $(\bmu,\btheta, Q)$ with $\gamma > 0$, the expected number of rounds needed by \ref{alg:CSB-JD} to find an allocation equivalent vector for $\btheta$ is upper bounded as
	\begin{equation*}
	\EE{T_{\theta_d}} \le \sum_{i \in [K]: \mu_i \ne 0} \floor{\frac{\theta_i}{\gamma}} \left( \frac{1}{\mu_i} \right).
	\end{equation*}
\end{restatable}
The lower bound of the threshold for an arm having zero mean loss remains $0$. Therefore, when resources are not enough, \ref{alg:CSB-JD} allocate only $\gamma$ amount of resources to such arms. Let $\nabla_{\max}$, $\nabla_{i,\min}$, and $K^\prime$ be the same as in Section \ref{sssec:differentThetaRegretBounds}. We are now ready to state the regret bound. 
\begin{restatable}{theorem}{regretJointDiffThreshold}
	\label{thm:regretJointDiffThreshold}
	Let $(\bmu,\btheta, Q)$ be the CSB problem instance with $\gamma > 0$  and $T>T_{\theta_d}$. Then the expected regret of \textnormal{\ref{alg:CSB-JD}} is upper bound as
	\begin{align*}
		\EE{\Regret_T} \le &\sum_{i \in [K]: \mu_i \ne 0} \floor{\frac{\theta_i}{\gamma}} \left( \frac{1}{\mu_i} \right)\Delta_{\max} + O\left(\sum_{i \in [K]}\frac{\log^2(K^\prime)\log T}{\nabla_{i,\min}} \right).
	\end{align*}
\end{restatable}
The first term of expected regret is the regret incurred due to the estimation of allocation equivalent. The expected number of rounds needed to find the allocation equivalent is given by \cref{lem:diffThetaEstRounds}. The second term corresponds to the expected regret due to the combinatorial semi-bandits algorithm CTS-BETA \citep{NeurIPS20_perrault2020statistical}.

\paragraph{Anytime Algorithm for CSB problem with Multiple Thresholds.}
Anytime algorithms maintain the lower bound of thresholds and linearly increase resources after observing a loss for current resource allocation. These algorithms do not tell us whether the estimated threshold is good or bad. Hence it is not possible to maintain a set of good thresholds as done in \ref{alg:CSB-MT}. Suppose it is possible to maintain a list of possible candidates for thresholds. In that case, the resources are allocated accordingly among arms, which may not have a good threshold estimate. Since the anytime algorithms do not handle the over-estimation problem, there is no way to reduce the over-allocated resources to an arm as done by \ref{alg:CSB-MT} which waits for a certain number of rounds before reducing resources. We observe that anytime algorithms are not possible when considering the CSB problems in the reward setting. More discussion about this can be found in \cref{sec:csb_num}, where we discuss the application of CSB setup for the Network Utility Maximization.

\paragraph{Anytime Algorithms versus Horizon-Dependent Algorithms.}
The significant difference between horizon dependent algorithms and anytime algorithms is the way they estimate the threshold vector. After knowing the allocation equivalent for the threshold vector, the algorithms work similarly. The horizon dependent algorithms use binary search and wait for a fixed number of rounds with over-allocated resources. In contrast, resources are increased linearly after observing a loss by anytime algorithms. Anytime algorithms perform poorly for CSB problems with different thresholds as the search space for allocation equivalent can be very large, but perform better for CSB problems with the same threshold, where the search space is small. \cref{table:regretComparision} summarizes the number of rounds taken for threshold estimation by proposed algorithms.

\begin{table}[!ht]
	\centering
	\setlength\tabcolsep{1pt}
	\setlength\extrarowheight{5pt}
	\begin{tabular}{ |c|c|c|} 
		\hline
		\diagbox[width=0.23\textwidth]{\bf Cases}{\bf $\boldsymbol{{\text{Rounds}}}$}
		&\thead{With Known Parameters $(\epsilon, \delta) $ \\ (Rounds with High Probability) }&\thead{With Unknown Parameters \\ (Expected Rounds)} \\ 
		\hline
		Same Threshold & $\frac{\log(\log_2(K)/\delta)}{\log\left({1}/{(1-\epsilon)}\right)} \log_2(K)$ & $\sum\limits_{L=M + 1}^{K} \frac{1}{1 - \Pi_{i \in [K]/[K-L]}(1-\mu_i)}$ \\ 
		\hline
		Different Threshold & $KW_\delta {\log_2\left(\ceil{1 +  {Q}/{\gamma}}\right)}$ & $\sum\limits_{i \in [K]:\mu_i \ne 0}\floor{\frac{\theta_i}{\gamma}} \left( \frac{1}{\mu_i} \right)$ \\ 
		\hline
		$1<n <K$ & $W_\delta \left(\sum\limits_{i \in A_{\Theta_n}} {\log_2 \ceil{1 +  \frac{Q}{\gamma}}} + K{\log_2 (n +  1)}\right)$ & -- \\ 
		\hline
	\end{tabular}
	\caption{Comparing upper bounds on the (expected) number of rounds needed to find allocation equivalent for the proposed algorithms. The value of $W_\delta$ used in the table is ${\log( K \log_2(\ceil{1 +{Q}/{\gamma}})/\delta)}/ {\log(1/(1-\epsilon))}$}
	\label{table:regretComparision}
\end{table}

%% file: chapter4/experiment.tex
%!TEX root =  ../thesis.tex

We empirically evaluate the performance of proposed algorithms on four synthetically generated instances. In instances I and II, the threshold is the same for all arms. In contrast, the thresholds vary across arms in Instance III and IV. The details are as follows:

\noindent
\textbf{Identical Threshold:} Both instance I and II have $K = 50, Q=15$ and $\theta_s=0.5$.  The mean loss of arm $i\in [K]$ is $x - (i-1)/100$. We set $x=0.5$ for instance I and $x=0.7$ for instance II. 

\noindent
\textbf{Different Thresholds:} Both Instance III and IV has $K = 10, Q=3$, and $\gamma=10^{-2}$. For Instance III, the mean loss vector is $\bmu = [0.9, 0.8, 0.42, 0.6, 0.5, 0.2, 0.1, 0.3, 0.7, 0.98]$ and the corresponding threshold vector is $\btheta=[0.65, 0.55, 0.3, 0.46, 0.37, 0.2, 0.07, 0.25, 0.3, 0.8]$. Whereas, Instance IV has the mean loss vector  $\bmu = [0.9, 0.8, 0.42, 0.6, 0.5, 0.2, 0.1, 0.3, 0.7, 0.98]$ and the corresponding threshold vector $\btheta=[0.55, 0.55, 0.3, 0.55, 0.55, 0.55, 0.3, 0.3, 0.3, 0.55]$.

The losses of the arm $i \in [K]$ are Bernoulli distributed with mean $\mu_i$. We repeated the experiment 100 times and plotted the regret with a 95\% confidence interval (the vertical line on each curve shows the confidence interval). 

\subsection{Performance of  Algorithms}
In our first set of experiments, we empirically evaluate the performance of horizon dependent algorithms. First, we vary the amount of resource $Q$ for Instance II and observe the regret of \ref{alg:CSB-ST} as given in \cref{fig:QSameTheta}. We observe that when resources are small, the learner can allocate resources to a few arms but observes loss from more arms. On the other hand, when resources are more, the learner allocates resources to more arms but observes loss from fewer arms. Thus as resources increase, we move from semi-bandit feedback to bandit feedback. Therefore, regret increases with an increase in the amount of resources. Next, we only vary $\theta_s$ in Instance II, and the regret of \ref{alg:CSB-ST} for different value of same threshold $\theta_s$ is shown in \cref{fig:TSameTheta}. Similar trends are observed as the decrease in threshold leads to an increase in the number of arms that can be allocated resources and vice-versa. Therefore the amount of feedback decreases as the threshold decreases and leads to more regret. The empirical results also validate sub-linear regret bounds for the proposed algorithm. 

\begin{figure}[!ht]
	\centering
	\scriptsize
	\begin{subfigure}[b]{0.40\textwidth}	\includegraphics[width=\linewidth]{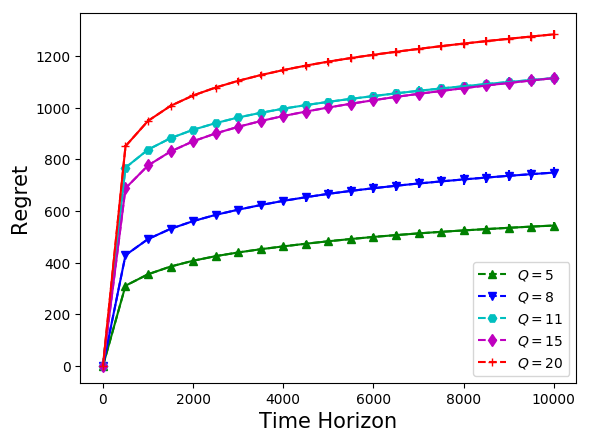}
		\caption{Varying resources in Instance II.}
		\label{fig:QSameTheta}
	\end{subfigure}\qquad
	\begin{subfigure}[b]{0.40\textwidth}
		\includegraphics[width=\linewidth]{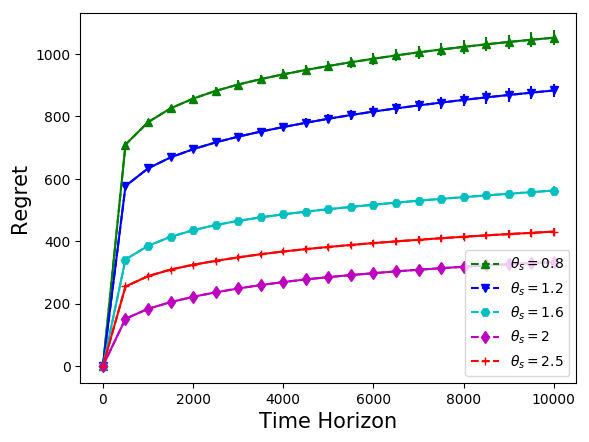}
		\caption{Varying value of same threshold.}
		\label{fig:TSameTheta}
	\end{subfigure}
	\caption{Regret of \ref{alg:CSB-ST} versus time horizon for the CSB problem with same threshold (Instance II).}
	\label{fig:knwon}
\end{figure}

Since the regret depends on the optimal allocation and the amount of resources (threshold), the regret can vary with different resources (threshold) for the same optimal resource allocation. We can observe this behavior of regret in \cref{fig:QSameTheta} and \cref{fig:TSameTheta}. Note that horizon dependent algorithms need to know the lower bound on $\mu_i$ value and find the allocation equivalent with the probability of at least $1- \delta$. We set the lower bound on mean loss as $\epsilon=0.1$ and confidence parameter $\delta=1/T$ in the experiment that involves horizon dependent algorithms.

In our next experiments, we change the available amount of resources in Instance III and IV. The regret of \ref{alg:CSB-MT} for the different amount of resources versus time horizon plots are shown in \cref{fig:multiTheta}. As expected, a similar behavior like \ref{alg:CSB-ST} is observed.

\begin{figure}[!ht]
	\centering
	\scriptsize
	\begin{subfigure}[b]{0.40\textwidth}	\includegraphics[width=\linewidth]{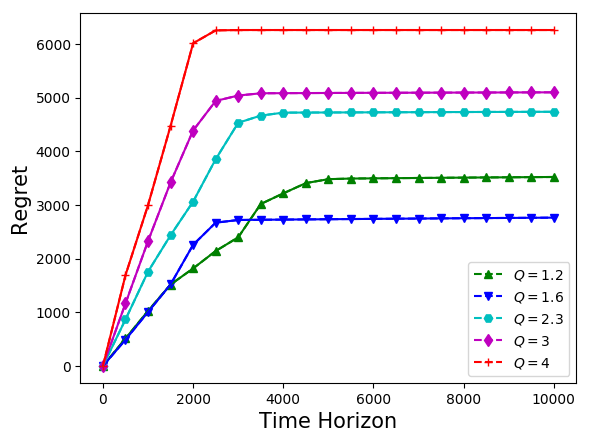}
		\caption{Varying resources in Instance III.}
		\label{fig:multiTheta1}
	\end{subfigure}\qquad
	\begin{subfigure}[b]{0.40\textwidth}
		\includegraphics[width=\linewidth]{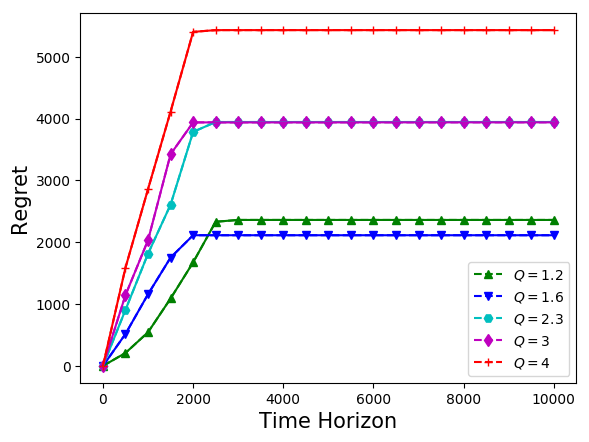}
		\caption{Varying resources in Instance IV.}
		\label{fig:multiTheta2}
	\end{subfigure}
	\caption{Regret of \ref{alg:CSB-MT} versus time horizon for the CSB problems with multiple thresholds.}
	\label{fig:multiTheta}
\end{figure}

We also run a similar set of experiments for anytime algorithms. The regret of anytime algorithms versus time horizon plots are shown in \cref{fig:SameTheta} and \cref{fig:differemtTheta}. As expected, we observe the same behavior as horizon dependent algorithms.
\begin{figure}[!ht]
	\centering
	\begin{subfigure}[b]{0.40\textwidth}
		\includegraphics[width=\linewidth]{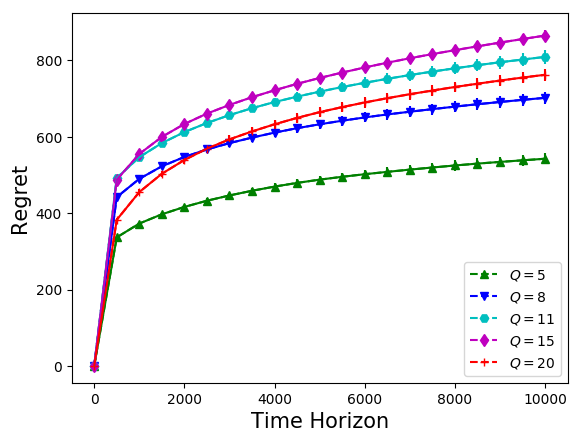}
		\caption{Varying resources in Instance II.}
		\label{fig:sameVaryQ}
	\end{subfigure}\qquad
	\begin{subfigure}[b]{0.40\textwidth}
		\includegraphics[width=\linewidth]{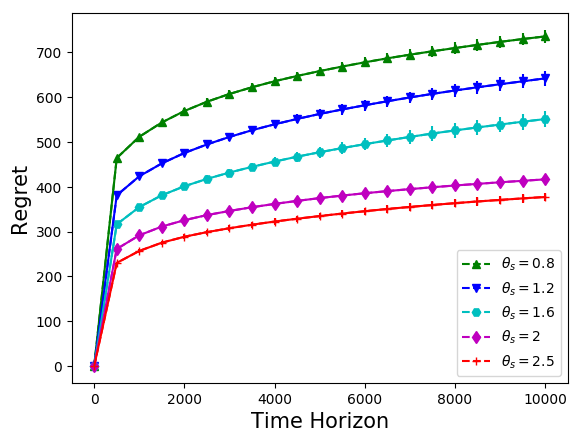}
		\caption{Varying value of same threshold.}
		\label{fig:sameVaryT}
	\end{subfigure}
	\caption{Regret of \ref{alg:CSB-JS} versus time horizon for the CSB problems with same thresholds.}
	\label{fig:SameTheta}
\end{figure}

\begin{figure}[!ht]
	\centering
	\scriptsize
	\begin{subfigure}[b]{0.40\textwidth}	\includegraphics[width=\linewidth]{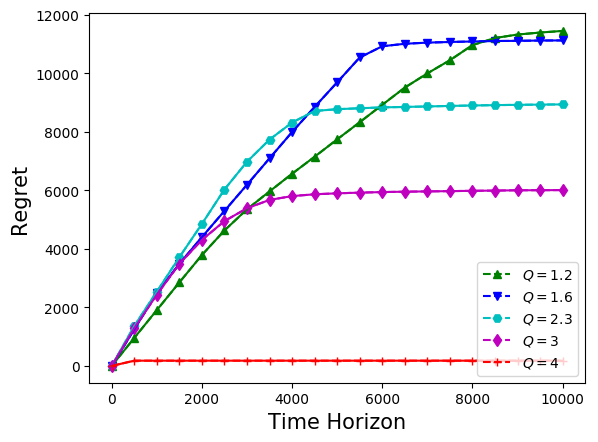}
		\caption{Varying resources in Instance III.}
		\label{fig:differemtTheta1}
	\end{subfigure}\qquad
	\begin{subfigure}[b]{0.40\textwidth}
		\includegraphics[width=\linewidth]{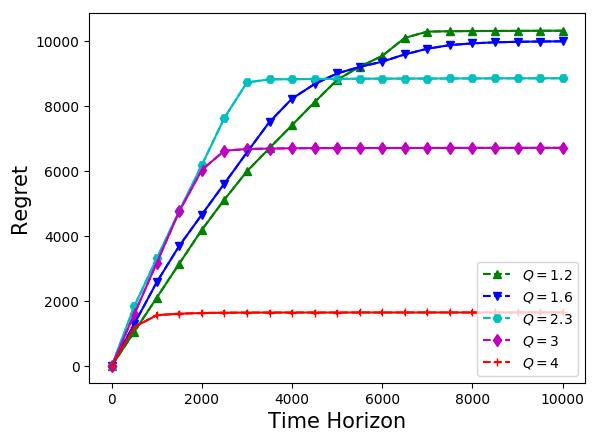}
		\caption{Varying resources in Instance IV.}
		\label{fig:differemtTheta2}
	\end{subfigure}
	\caption{Regret of \ref{alg:CSB-JD} versus time horizon for the CSB problems with different thresholds.}
	\label{fig:differemtTheta}
\end{figure}

\subsection{Comparison between Algorithms}
We compare \ref{alg:CSB-JS}, \ref{alg:CSB-ST}, and state-of-the-art CSB-ST algorithm \citep{NeurIPS19_verma2019censored} for the CSB problems with the same threshold. Our algorithms outperforms CSB-ST for instance I and II as shown in \cref{fig:sameComp1} and \cref{fig:sameComp2}, respectively. Even though \ref{alg:CSB-ST} and CSB-ST use binary search for threshold estimation as compared to linear search in \ref{alg:CSB-JS}, there waiting delay with the overestimate of threshold leads to more rounds spend for the threshold estimation in considered CSB problems as compare to \ref{alg:CSB-JS}. Therefore, \ref{alg:CSB-JS} has the smallest regret than the other two algorithms.
\begin{figure}[!ht]
	\centering
	\begin{subfigure}[b]{0.40\textwidth}
		\includegraphics[width=\linewidth]{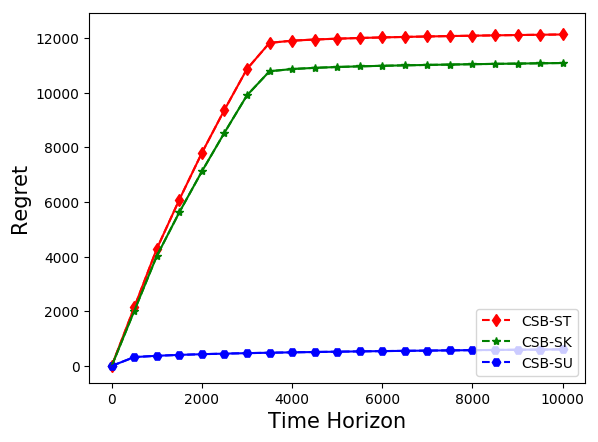}
		\caption{Comparison for Instance I}
		\label{fig:sameComp1}
	\end{subfigure}\qquad
	\begin{subfigure}[b]{0.40\textwidth}
		\includegraphics[width=\linewidth]{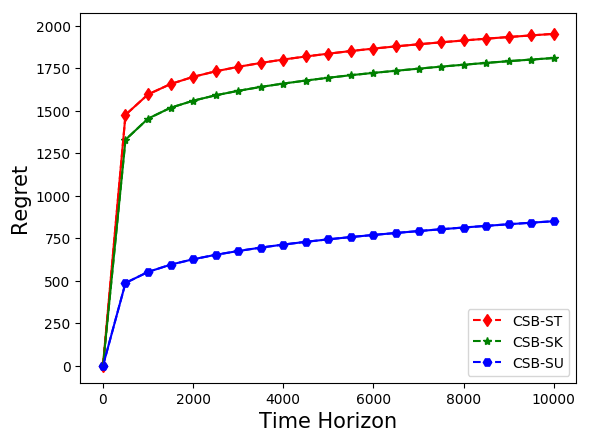}
		\caption{Comparison for Instance II}
		\label{fig:sameComp2}	
	\end{subfigure}
	\captionsetup{justification=centering}
	\caption{Comparing regret of \ref{alg:CSB-JS}, \ref{alg:CSB-ST}, and CSB-ST \citep{NeurIPS19_verma2019censored}.}
	\label{fig:CompSameTheta}
\end{figure}

We compare \ref{alg:CSB-MT}, CSB-DK (\ref{alg:CSB-MT} with $n=K$), \ref{alg:CSB-JD}, and state-of-the-art CSB-DT algorithm \citep{NeurIPS19_verma2019censored} for the CSB problems with different thresholds. \ref{alg:CSB-MT} and CSB-DT also uses binary search to estimate the threshold for each arm. These algorithms use the same threshold estimate for the fixed number of rounds, which depends upon the value of $\epsilon$ and $\delta$. The smaller the value of $\epsilon$, the more these algorithms wait for observing a loss and incur more regret as well. On the other hand, \ref{alg:CSB-JD} uses a linear search to estimate the threshold and does not need to know $\epsilon$ and $\delta$. As expected \ref{alg:CSB-MT} and CSB-DK outperform CSB-DT as shown in \cref{fig:diffComp2}. Whereas the performance of \ref{alg:CSB-MT} matches with CSB-DK in \cref{fig:diffComp1} as only two arms have the same threshold in Instance III. Since \ref{alg:CSB-JD} uses a linear search for threshold estimation, it needs more rounds to estimate allocation equivalent when the threshold has a larger search region than the algorithms that use binary search. Therefore, \ref{alg:CSB-JD} incurs more regret.

\begin{figure}[!ht]
	\centering
	\begin{subfigure}[b]{0.40\textwidth}
		\includegraphics[width=\linewidth]{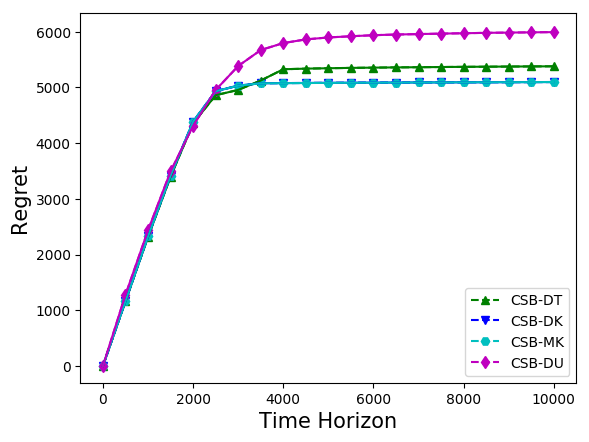}
		\caption{Comparison for Instance III}
		\label{fig:diffComp1}
	\end{subfigure}\qquad
	\begin{subfigure}[b]{0.40\textwidth}
		\includegraphics[width=\linewidth]{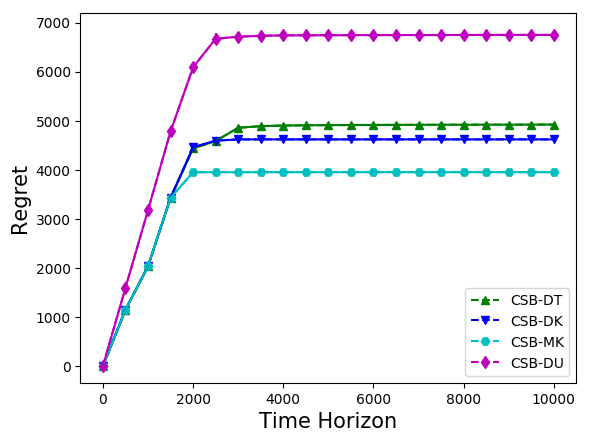}
		\caption{Comparison for Instance IV}
		\label{fig:diffComp2}	
	\end{subfigure}
	\captionsetup{justification=centering}
	\caption{\centering Comparing regret of \ref{alg:CSB-MT}, CSB-DK (assuming all thresholds are different which is equivalent to \ref{alg:CSB-MT} with $n=K$), \ref{alg:CSB-JD} and CSB-DT \citep{NeurIPS19_verma2019censored}.}
	\label{fig:diffTheta}
	
\end{figure}

\subsection*{Computation complexity of 0-1 Knapsack with fractional weight and value}
Even though $KP(\bmu,\btheta, Q)$ is an NP-Hard problem; it can be solved by a pseudo-polynomial time algorithm\footnote{The running time of pseudo-polynomial time algorithm is a polynomial in the numeric value of the input whereas the running time of polynomial-time algorithms is polynomial of the length of the input.} using dynamic programming with the time complexity of O$(KQ)$. But such an algorithm for $KP(\bmu,\btheta, Q)$ works when the value and weight of items are integers. In the case of $\mu_i$ and $\theta_i$ are fractions, they need to be converted in integers with the desired accuracy by multiplying by large value $S$. The time complexity of solving $KP(S\bmu, S\btheta, SQ)$ is O$(KSQ)$ as a new capacity of Knapsack is $SQ$. Therefore, the time complexity of solving $KP(S\bmu, S\btheta, SQ)$ in each of the $T$ rounds is O$(TKSQ)$. 
Since solving the $0$-$1$ Knapsack problem is computationally expensive, we can solve it after $N$ rounds as the empirical mean losses do not change drastically in consecutive rounds in practice (except initial rounds). We have used $S = 10^4$ and $N = 20$ in our experiments involving the different thresholds.

%% file: chapter4/num.tex
%!TEX root =  ../thesis.tex

In this section, we study the application of CSB for the Stochastic {\em Network Utility Maximization} problem. Network Utility Maximization (NUM) is an approach for resource allocation among multiple agents such that the total utility of all the agents (network utility) is maximized. In its simplest form, NUM solves the following optimization problem:
\begin{align*}
	\underset{\ba}{\text{maximize}} & \sum_{i=1}^{K}U_i(a_i)\\
	&\mbox{subject to}  \hspace{4mm}\sum_{i=1}^K a_i \leq Q
\end{align*}
where $U_i(\cdot)$ denotes the utility of agent $i$, variable $\ba=(a_1,a_2,\ldots, a_K) \in \R_+^K$ denote the allocated resource vector, and $Q \in \R_+$ is  amount of resource available. Utilities define the agents' satisfaction level, which depends on the amount of resources they are allocated. A resource could be bandwidth, power, or rates they receive. Since the seminal work of \cite{ETT1997_ChargingAndRateControl}, there has been a tremendous amount of work on NUM and its extensions. NUM is used to model various resource allocation problems and improve network protocols based on its analysis. We refer the readers to \cite{JSAC2006_TutorialOnNUM_PalomarChinag} and \cite{TAC2007_AlternateDistributed_PalomarChinag} for an informative tutorial and survey on this subject.

The nature of utility functions is vital in the analysis of the NUM problem and assumed to be known or can be constructed based on the agent behavior model and operator cost model. However, agent behavior models are often difficult to quantify. Therefore, we consider the NUM problem where the utilities of the agent are unknown and stochastic. The earlier NUM problems considered deterministic settings. Significant progress has been made to extend the NUM setup to consider the stochastic nature of the network and agent behavior \citep{ETT2008_StochasticNUM_YiChiang}. For both the static and stochastic networks, the works in the literature often assume that the utility functions are smooth concave functions and apply Karush-Kuhn-Tucker conditions to find the optimal allocation. However, if the utility functions are unknown, these methods are useful only once the utilities are learned. Many of the NUM variants with full knowledge of utilities aim to find an optimal policy that meets several constraints like stability, fairness, and resource \citep{INFOCOM2010_DelayBasedNUM_Neely,WiOpt2017_DRUM_EryilmazKoprulu,WiOpt2018_NUMHetrogeneous_SinhaModiano}. In this work, we only focus on resource constraint due to limited divisible resource (bandwidth, power, rate). \cref{fig:SNUM} depicts the Stochastic Network Utility Maximization problem.

\begin{figure}[!ht]
	\centering 
	\includegraphics[width=0.9\linewidth]{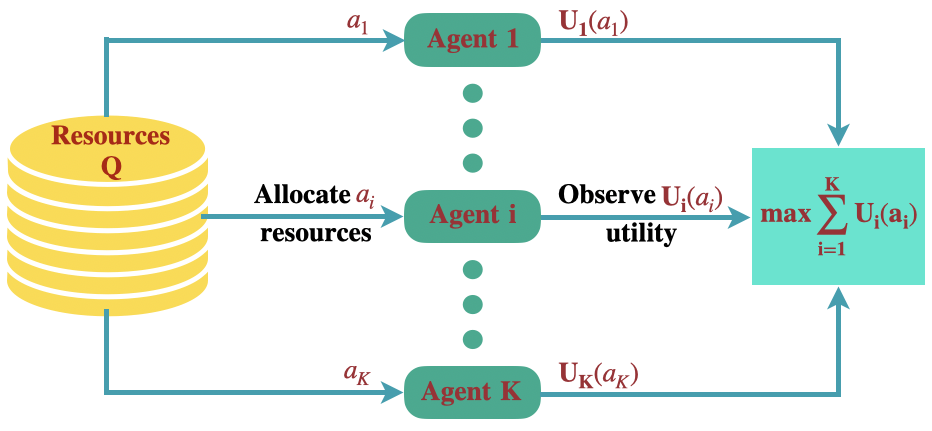} 
	\caption{Stochastic Network Utility Maximization problem where the resources are allocated among multiple agents such that the total average utility of all the agents (network utility) is maximized.}
	\label{fig:SNUM}
\end{figure}

Since learning an arbitrary utility function is not always feasible,  we assume the utilities belong to a class of `threshold' type functions. Specifically, we assume that each agent's utility is stochastic with some positive mean only when it is allocated a certain minimum resource.  We refer to the minimum resource required by an agent as its `threshold' and the mean utility it receives when it is allocated resource above the threshold as its `mean reward.' Thus the expected utility of each agent is defined by two parameters -- a threshold and a mean reward.  Such threshold type utilities correspond to hard resource requirements. For example, an agent can transmit and obtain a positive rate  (reward) only if its power or bandwidth allocation is above a certain amount. In each round, the operator allocates a resource to each agent and observes the utilities the agent obtains. The goal of the operator is to allocate resources such that the expected network utility is maximized. We pose this problem as a Censored Semi-Bandits problem in the reward maximization setting, where the operator corresponds to a learner, agents to arms, and utilities to rewards. The learner's goal is to learn a policy that minimizes the difference between the best achievable expected network utility with full knowledge of the agent utilities and that obtained by the learner under the same resource constraint with the estimated utilities of agents.

\subsection{CSB setup for Reward Maximization}
The CSB problems that are considered in \cref{sec:csb_problemSetting} works only in the loss setting. Now we extend the CSB setup to reward maximization setting, where the optimal allocation can be computed as follows:
\begin{equation*}
	\label{equ:networkUtility}
	\ba^\star  \in \argmax _{\ba \in \A} \sum_{i=1}^K\mu_i \one{a_i\ge \theta_i}.
\end{equation*}

The interaction between a learner and the environment that governs rewards for the arms is as follows: In the round $t$, the environment generates a reward vector $(X_{t,1}, X_{t,2},\ldots, X_{t,K}) \in \{0,1\}^K$, where $X_{t,i}$ denotes the true reward for arm $i$ in round $t$.  The sequence $(X_{t,i})_{t\geq 1}$ is generated IID with the common mean $\EE{X_{t,i}}=\mu_i$ for each $i \in [K]$. The learner selects a feasible allocation $\ba_t=\{a_{t,i}: i\in [K]\}$ and observes reward vector $X_t^\prime=\{X^\prime_{t,i}: i\in [K]\}$, where $X_{t,i}^\prime=X_{t,i}\one{a_{t,i}\ge \theta_i}$ and collects reward $r_t(\ba_t)=\sum_{i \in [K]}X_{t,i}^\prime$. A policy of the learner is to select a feasible allocation in each round based on the observed reward such that the cumulative reward is maximized. The performance of a policy that makes allocation $\{\ba_t\}_{t\geq1}$ in round $t$ is measured in terms of expected regret for $T$ rounds given by

\begin{equation*}
	\mathbb{E}[\Regret_T] = T\sum_{i=1}^K\mu_i \one{x^\star_i\ge \theta_i} - \EE{ \sum_{t=1}^T\sum_{i=1}^K Y_{t,i} \one{x_{t,i}\ge \theta_i}}.
\end{equation*}

A good policy must have sub-linear regret, i.e., $\EE{\Regret_T}/T \rightarrow 0$ as $T \rightarrow \infty$. Next, we define the notion of treating a pair of thresholds for the given reward vector and resource to be `equivalent.'
\begin{definition}[Allocation Equivalent]
	For fixed reward vector $\bmu$ and amount of resource $Q$, two threshold vectors $\btheta$ and $\hat{\btheta}$ are {\em allocation equivalent} if the following holds:
	\begin{equation*}
		\max_{\ba \in \A} \sum_{i=1}^K\mu_i \one{a_i\ge \theta_i}  = \max_{\ba \in \A} \sum_{i=1}^K\mu_i \one{a_i\ge \hat{\theta}_i}.
	\end{equation*}
\end{definition}

\subsection{Algorithms for Network Utility Maximization}
We first focus on the special case of the network utility maximization problem where $\theta_i=\theta_s$ for all $i \in [K]$. We develop an algorithm named Network Utility Maximization with the Same Threshold (\ref{alg:ONUM-ST}). This algorithm is adapted from \ref{alg:CSB-ST} for the reward maximization setup. There are two major differences: 1) the feedback (reward) is only observed when the resource allocation is more than a certain threshold, and 2) a sample for the mean reward estimate is drawn from the beta distribution. Similarly, we develop an algorithm named Network Utility Maximization with the Multiple Threshold (\ref{alg:ONUM-DT}), which is adapted from \ref{alg:CSB-MT} to the reward maximization setup.

\begin{algorithm}[!ht] 
	\renewcommand{\thealgorithm}{\bf NUM-SK}
	\floatname{algorithm}{}
	\caption{Algorithm for NUM problem having Same Threshold with Known Horizon and $\epsilon$}
	\label{alg:ONUM-ST}
	\begin{algorithmic}[1]
		\State \textbf{Input:} $\delta, \epsilon$
		\State Set $W_\delta = {\log(\log_2(K)/\delta)}/({\log(1/(1-\epsilon))})$ and $\forall i \in [K]: S_i=1, F_i=1, Z_i=0$ 
		\State Initialize $\Theta$ as given in Lemma \ref{lem:thetaSet}, $C=0, l=1, u = K, j = \floor{(l+u)/2}$ 
		\For{$t=1,2,\ldots,$}
		\State Set $\hat{\theta}_s = \Theta[j]$ and $\forall i \in [K]: \hat{\mu}_{t,i} \leftarrow \beta(S_i, F_i)$
		\State $A_t \leftarrow$ set of top-$({Q}/{\hat{\theta}_s})$ arms with the largest values of $\hat{\mu}_{t,i}$ 
		\State $\forall i \in A_t:$ allocate $\hat{\theta}_s$ resource and observe $X_{t,i}$
		\If{$j \ne u$} 
		\If{$X_{t,a} = 1$ for any $a \in A_t$} 
		\State Set $u=j, ~j = \floor{(l+u)/2}, C=0$
		\State $\forall i \in A_t$: set $S_i = S_i+X_{t,i}, F_i = F_i+1-X_{t,i} + Z_i, Z_i =0$
		\State $\forall i \in [K]\setminus A_t: F_i = F_i+ Z_i, Z_i =0$
		\Else
		\State Set $C = C + 1$ and $\forall i \in A_t: Z_i = Z_i + 1$
		\State If $C = W_\delta$ then set $l=j+1, j = \floor{(l+u)/2}$, $C=0, \forall i \in [K]: Z_i =0$
		\EndIf
		\Else
		\State $\forall i \in A_t: S_i = S_i + X_{t,i}, F_i = F_i + 1 - X_{t,i}$
		\EndIf
		\EndFor
	\end{algorithmic}
\end{algorithm}

Next, we will give the regret upper bounds for \ref{alg:ONUM-ST} and \ref{alg:ONUM-DT}. For simplicity of discussion, we assume that arms are indexed according to their decreasing mean rewards, i.e., $\mu_1 \geq \mu_2, \ldots, \geq \mu_K$, but the algorithms are not aware of this ordering. We refer to the first $M$ arms as \emph{top-}$M$ arms. For a instance $(\bmu,\btheta, C)$ and any feasible allocation $\ba \in \A$, we define the sub-optimality gap as $\Delta_a = \sum_{i=1}^K\mu_i\big(\one{a_i^\star \ge \theta_i} - \one{a_i \ge \theta_i}\big)$. The maximum and minimum regret incurred in a round is $\nabla_{\max} = \max\limits_{\ba \in \A } \nabla_{\ba}$ and $\nabla_{\min} = \min\limits_{\ba \in \A } \nabla_{\ba}$, respectively. Now we will give regret bound of \ref{alg:ONUM-ST}.

\begin{algorithm}[!ht]
	\small 
	\renewcommand{\thealgorithm}{\bf NUM-MK} 
	\floatname{algorithm}{}
	\caption{Algorithm for NUM problem having Multiple Threshold with Known Horizon and $\epsilon$}
	\label{alg:ONUM-DT}
	\begin{algorithmic}[1]
		\State \textbf{Input:} $n, \delta, \epsilon, \gamma$
		\State Initialize: $\forall i \in [K]: S_i = 1, F_i = 1, Z_i =0, \theta_{l,i} = 0, \theta_{u,i} = Q, \theta_{g,i} = 0,  \hat\theta_{i} = Q/2,$ 
		\State Set $\Theta_n = \emptyset,W_\delta = \log (K\log_2(\lceil 1 + Q/\gamma\rceil)/\delta)/\log(1/(1-\epsilon)),$ if $n<K$ then $n_i =1$ else $n_i=0$ 
		\For{$t=1,2, \ldots,$}
		\State $\forall i \in [K]: \hat{\mu}_{t,i} \leftarrow \text{Beta}(S_i, F_i)$
		\If{$\theta_{g,j} = 0$ for any $j \in [K]$}
		\If{$n < K$}
		\While{$\theta_{g,n_ i} = 1$} 
		\State Add $\hat\theta_{n_i}$ to $\Theta_n$ and set $n_i = n_i + 1$.  Sort $\Theta_n$ in increasing order
		\EndWhile
		\If{there exists no $j \in [|\Theta_n|]$ such that $\theta_{l,n_i} < \Theta_n[j] \le \theta_{u,n_i}$ or $\Theta_n = \emptyset$}
		\State  Set $\hat\theta_{n_i}=(\theta_{l,i}+\theta_{u,i})/2$ 
		\Else
		\State Set $l = \min\{k: \Theta_n[k] > \theta_{l,n_i}\}, u = \max\{k: \Theta_n[k] \le \theta_{u,n_i}\},$ and $j = \floor{(l+u)/2}$
		\State If $\Theta_n[j]=\theta_{u,n_i}$ then set $\hat\theta_{n_i} = \Theta_n[j]-\gamma$ else  $\hat\theta_{n_i} = \Theta_n[j]$
		\EndIf
		\EndIf
		
		\State $\forall i \in [K]\setminus \{n_i\}$: update $\hat\theta_{i}$ using Eq. \eqref{equ:updateTheta}. Allocate $\hat\theta_{i}$ resource to arm $i$ and observe $X_{t,i}$
		\For{$i = \{1,2,\ldots, K\}$}
		\If{$\theta_{g,i} = 0$ and $\hat\theta_{i}> \theta_{l,i}$}
		\State If $X_{t,i}=1$ then set $\theta_{u,i} = \hat\theta_{i}, S_i = S_i + 1, F_i = F_i + Z_i, Z_i =0$ else  $Z_i = Z_i + 1$
		\State If {$Z_i= W_\delta$} then set $\theta_{l,i} = \hat\theta_{i}, Z_i=0 $
		\State If $\theta_{u,i} - \theta_{l,i} \le \gamma$ then set $\theta_{g,i}=1$ and $\hat\theta_i = \theta_{u,i}$
		\ElsIf{$\hat\theta_{i} \ge \theta_{u,i}$ or $\big\{ \theta_{g,i} = 1$ and $\hat\theta_{i}\ge \hat\theta_{u,i} \big\}$}
		\State Set $S_i = S_i+X_{t,i}$ and $F_i = F_i+1-X_{t,i}$ 
		\EndIf	
		\EndFor	
		
		\Else
		\State $A_t \leftarrow$ Oracle$\big( KP(\hat\bmu_{t}, \hat\btheta, C)\big)$
		\State $\forall i \in A_t:$ allocate $\hat\theta_{i}$ resource and observe $X_{t,i}$, update $S_i = S_i+X_{t,i}$ and $F_i = F_i+1-X_{t,i}$
		\EndIf
		\EndFor
	\end{algorithmic}
\end{algorithm}

\begin{theorem}
	\label{thm:regretONUM_ST}
	Let $\mu_K\geq \epsilon>0$, $\mu_{M} > \mu_{M+1}$, $W_\delta = {\log(\log_2(K)/\delta)}/{\log(1/(1-\epsilon))}$, and $T>W_\delta\log_2{(K)}$. Then with probability at least $1-\delta$, the expected regret of \ref{alg:ONUM-ST} is upper bounded as
	\begin{align*}
		\EE{\Regret_T} &\le  W_\delta\log_2{(K)}\Delta_{\max} + O\left((\log T)^{{2}/{3}}\right) + \mbox{$\sum_{i \in [K]\setminus [M]}$} \frac{(\mu_M-\mu_i)\log {T}}{d( \mu_i,\mu_M)}.
	\end{align*}
\end{theorem}
\noindent
As we are in the reward setting, the regret bound of MP-TS can be used as it is. The remaining proof follows similar steps as the proof of \cref{thm:regretSameThreshold}.

Let $\nabla_{i, \min}$ be the minimum regret for superarms containing arm $i$ and $K^\prime$ be the maximum number of arms in any feasible resource allocation. We redefine $W_\delta = \log (K\log_2(\lceil 1 + Q/\gamma\rceil)/\delta)/\log(1/(1-\epsilon))$. Now we are ready to state the regret bound of \ref{alg:ONUM-DT}.
\begin{theorem}
	\label{thm:regretUNUM_DT}
	Let $\gamma >0$, $\mu_K\geq \epsilon>0$, and $T>KW_\delta\log_2 \left(\lceil 1+Q/\gamma\rceil \right)$. Then with probability at least $1-\delta$, the expected regret of \ref{alg:ONUM-DT} is upper bounded by 
	\begin{align*}
		\EE{\Regret_T} &\le {KW_\delta\log_2 \left(\lceil 1+Q/\gamma\rceil \right)\Delta_{\max}}  +  O\left(\sum_{i \in [K]}\frac{\log^2(K^\prime)\log T}{\nabla_{i, \min}} \right). 
	\end{align*}
\end{theorem}
\noindent
Since we are in the reward setting, the regret bound of combinatorial bandits algorithm CTS-BETA can be used as it is. The remaining proof follows similar steps as the proof of \cref{thm:regretDiffThreshold}. 

\paragraph{Anytime Algorithms for reward maximization setting.}
For simplicity, consider the reward setting with a single threshold. When the allocated resource exceeds the arm's threshold, the learner may continue to observe sample values of $0$ due to the stochastic nature of reward generation. Thus, the learner needs to observe enough samples to be confident that the resource allocated is above the threshold. To decide how much is enough, the learner needs to know $T$ so that exploration and exploitation can be well balanced. However, note that this issue does not arise in the loss setting; if the learner continues to observe a sample $0$, there is no need to increase the allocation further, and the learner can continue the same resource allocation on the arm. The same argument applied if the learner has to start by allocating the higher amount of resources and keep decreasing it until it goes below the threshold.

\subsection{Experiments}
We evaluate the performance of \ref{alg:ONUM-ST} and \ref{alg:ONUM-DT} empirically on three synthetically generated instances. In instance $1$, the threshold is the same for all arms, whereas, in instances $2$ and $3$, thresholds vary across arms. We ran the algorithm for $T=10000$ rounds in all the simulations. All the experiments are repeated $100$ times, and the regret curves are shown with a $95\%$ confidence interval. The vertical line on each curve shows the confidence interval. The following empirical results validate sub-linear bounds for our algorithms. The details about the problem instances are as follows:

\noindent
\textbf{Instance $1$ (Identical Threshold):} It has $K = 50, Q=20,$ $\theta_s=0.7, \delta=0.1$ and $\epsilon=0.1$.  The mean reward of arm $i\in [K]$ is $0.25 + (i-1)/100$. 

\noindent
\textbf{Instance $2$ (Different Thresholds):} It has $K = 5,Q=2,$ $\delta=0.1,$ $\epsilon=0.1$ and $\gamma=10^{-3}$. The mean reward vector is $\bmu = [0.9,$ $0.89,0.87,0.6,0.3]$ and the corresponding  threshold vector is $\btheta=[0.7,0.7,0.7,0.6,0.35]$. 

\noindent
\textbf{Instance $3$ (Different Thresholds):} It has $K=10, Q=3,$ $\delta=0.1$, $\epsilon=0.1$ and $\gamma=10^{-3}$. The mean reward vector is $\bmu \hspace{-0.2mm}=\hspace{-0.2mm} [0.9, 0.8, 0.42, 0.6, 0.5, 0.2, 0.11, 0.7, 0.3, 0.98]$ and the corresponding threshold vector is $\btheta = [0.6, 0.55, 0.3,$ $ 0.46, 0.34, 0.2, 0.07, 0.3, 0.25, 0.8]$. 

We considered two different reward distributions of arms: 1) Bernoulli, where the rewards of arm $i$ are Bernoulli distributed with parameter $\mu_i$, and 2) Uniform, where the rewards of arm $i$ is uniformly distributed in the interval $[\mu_i - 0.1, \mu_i+0.1]$. For any continuous reward distribution with support in $(0,1]$, the value of $W_\delta$ is set to $1$ because the reward is observed with probability $1$ when the allocated resource is above its threshold on any arm. For the Bernoulli distribution, the value of $W_\delta$ is $38$ for instance $1$, $62$ for instance $2$ and $69$ for instance $3$. Hence, we observe less regret for uniformly distributed rewards than Bernoulli distributed rewards. This difference is more significant when the arms have different thresholds.

\paragraph{Experiments with the same threshold:} We perform two different experiments on problem instance $1$ using \ref{alg:ONUM-ST}. First, we varied the amount of resources $Q$ while keeping other parameters unchanged. With more resource, the learner can allocate resource to more arms. Hence learner can observe rewards from more arms in each round, which leads to faster learning and low cumulative regret, as shown in Fig. \eqref{fig:UniSameThetaC} for the uniformly distributed rewards. For the uniform distribution we use binarization trick \citep{COLT12_agrawal2012analysis} to apply \ref{alg:ONUM-ST}: when a real-valued reward $X_{t,i}\in (0,1]$ is observed, the algorithm is updated with a fake binary reward that is drawn from Bernoulli distribution with parameter $X_{t,i}$, i.e., $X_{t,i}^f \sim Ber(X_{t,i}) \in \{0,1\}$. The different amount of resource has different optimal allocation and sub-optimality gap. Hence with large $W_\delta$ value for Bernoulli distributed rewards, we may not observe similar behavior (less regret with more resource)  as shown in Fig. \eqref{fig:BerSameThetaC}. 

\begin{figure}[!ht]
	\centering
	\begin{subfigure}[b]{0.40\linewidth}
		\includegraphics[width=\linewidth]{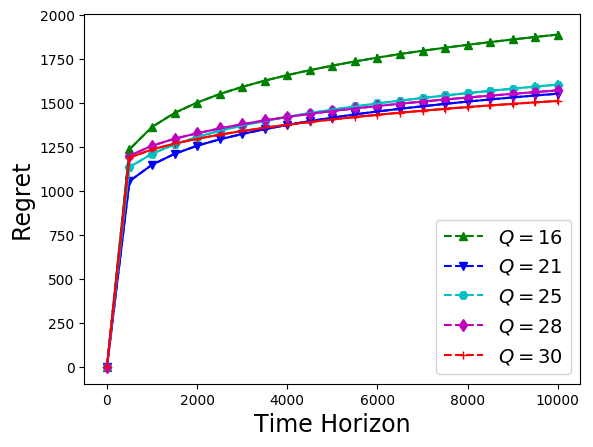}
		\caption{\footnotesize Bernoulli Distributed Reward}
		\label{fig:BerSameThetaC}
	\end{subfigure}\qquad
	\begin{subfigure}[b]{0.40\linewidth}
		\includegraphics[width=\linewidth]{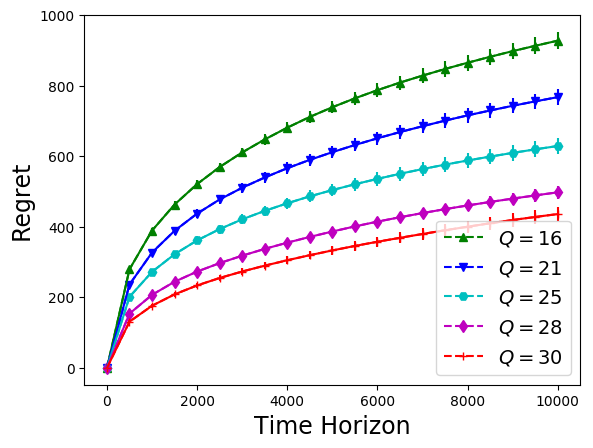}
		\caption{\footnotesize Uniform Distributed Reward}
		\label{fig:UniSameThetaC}	
	\end{subfigure}
	\caption{Comparing regret of \ref{alg:ONUM-ST} v/s amount of resource for problem instance $1$ with Bernoulli and uniformly distributed reward.}
	\label{fig:SameThetaC}
\end{figure}

Second, we varied the threshold $\theta_s$ while keeping other parameters unchanged. As a smaller threshold allows the allocation of resource to more arms, we observe that a smaller threshold leads to faster learning due to more feedback. These trends are shown in Fig. \eqref{fig:BerSameThetaT} and \eqref{fig:UniSameThetaT} for Bernoulli and uniformly distributed rewards, respectively.
\begin{figure}[!ht]
	\centering
	\begin{subfigure}[b]{0.40\linewidth}
		\includegraphics[width=\linewidth]{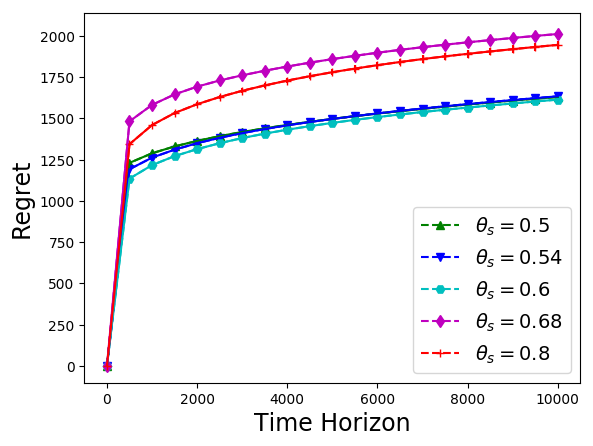}
		\caption{\footnotesize Bernoulli Distributed Reward}
		\label{fig:BerSameThetaT}
	\end{subfigure}\qquad
	\begin{subfigure}[b]{0.40\linewidth}
		\includegraphics[width=\linewidth]{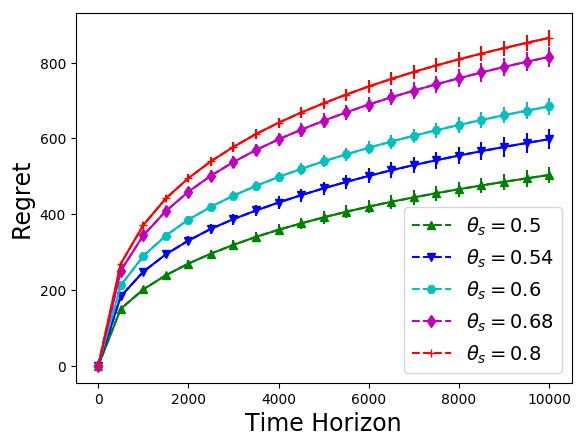}
		\caption{\footnotesize Uniform Distributed Reward}
		\label{fig:UniSameThetaT}
	\end{subfigure}
	%	\captionsetup{justification=centering}
	\caption{Comparing regret of \ref{alg:ONUM-ST} v/s different values of same threshold for problem instance $1$ with Bernoulli and uniformly distributed reward.}
	\label{fig:SameThetaQ}
\end{figure}

\paragraph{Experiments with different thresholds:} 
We evaluate the performance of \ref{alg:ONUM-DT} on problem instances $2$ and $3$. We varied the amount of resources $Q$ while keeping other parameters unchanged. As the thresholds are different across arms, an increase in the resource may lead to a selection of a different set of arms leading to different sub-optimality gaps. Hence, it does not show the same behavior (less regret with more resource) as observed for the same threshold. 
\begin{figure}[!ht]
	\centering
	\begin{subfigure}[b]{0.40\linewidth}
		\includegraphics[width=\linewidth]{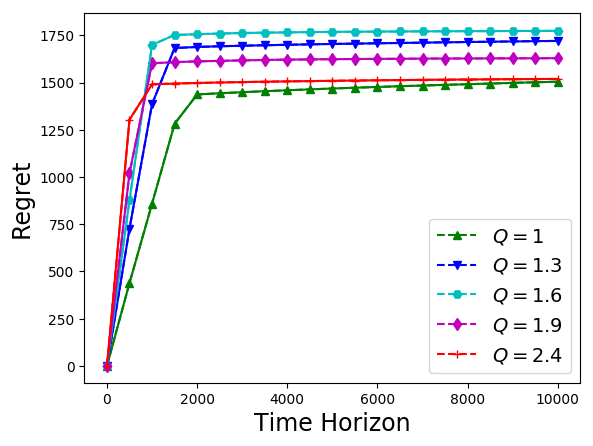}
		\caption{\footnotesize Bernoulli Distributed Reward}
		\label{fig:BerDiffTheta2}
	\end{subfigure}\qquad
	\begin{subfigure}[b]{0.40\linewidth}
		\includegraphics[width=\linewidth]{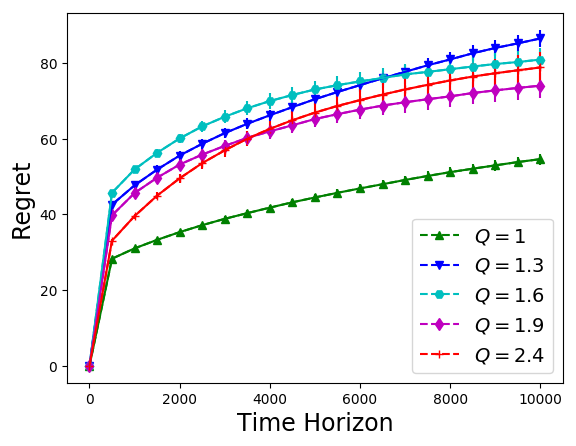}
		\caption{\footnotesize Uniform Distributed Reward}
		\label{fig:UniDiffTheta2}
	\end{subfigure}
	\caption{Comparing regret of \ref{alg:ONUM-DT} v/s amount of resource for problem instance $2$ with Bernoulli and uniformly distributed reward.}
	\label{fig:DiffTheta1}
\end{figure}

\begin{figure}[!ht]
	\centering
	\begin{subfigure}[b]{0.40\linewidth}
		\includegraphics[width=\linewidth]{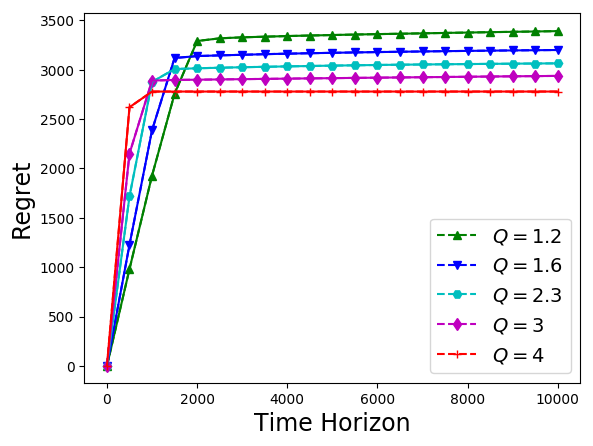}
		\caption{\footnotesize Bernoulli Distributed Reward}
		\label{fig:BerDiffTheta3}
	\end{subfigure}	\qquad
	\begin{subfigure}[b]{0.40\linewidth}
		\includegraphics[width=\linewidth]{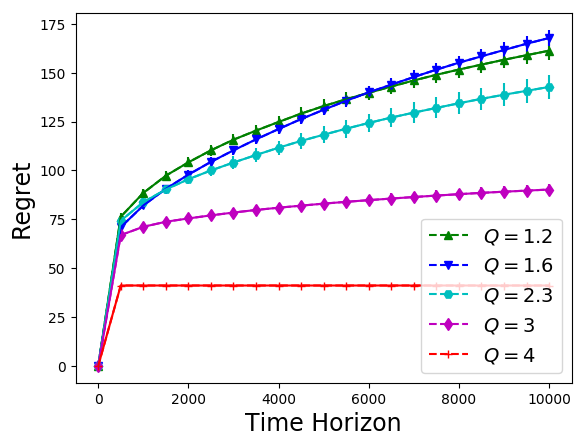}
		\caption{\footnotesize Uniform Distributed Reward}
		\label{fig:UniDiffTheta3}
	\end{subfigure}
	\caption{Comparing regret of \ref{alg:ONUM-DT} v/s amount of resource for problem instance $3$ with Bernoulli and uniformly distributed reward.}
	\label{fig:DiffTheta2}
\end{figure}

But we observe that the allocation equivalent is learned faster as the reward of more arms can be observed simultaneously with more resource. These observations are shown in Figs. \eqref{fig:BerDiffTheta2} and \eqref{fig:UniDiffTheta2} generated on instance $2$ for Bernoulli and uniformly distributed rewards on instance $2$, and same is repeated in Figs. \eqref{fig:BerDiffTheta3} and \eqref{fig:UniDiffTheta3} on instance $3$.

We run experiment $500$ times for uniformly distributed rewards (Figs. \eqref{fig:UniDiffTheta2} and \eqref{fig:UniDiffTheta3}) on instance $2$ and $3$ as confidence intervals overlapped for $100$ runs.

%% file: chapter4/appendix.tex
%!TEX root =  ../thesis.tex

\subsection*{Missing proofs from Section \ref{sec:csb_known}}

\ThetaSet*
\proof{Proof.}
	The case $\floor{Q/\theta_s} \geq K$ is trivial. We consider the case $\floor{Q/\theta_s} < K$. By definition $M=\min\{\floor{Q/\theta_s},K\}$. We have $M \le Q/\theta_s$ and $\theta_s \le Q/M \doteq \hat{\theta}_s$. Hence $\hat{\theta}_s \ge \theta_s$. Therefore, $\hat{\theta}_s$ fraction of resource allocation for an arm has same reduction in the mean loss as $\theta_s$. Further, in both the instances $(\bmu, \theta_s, Q)$ and $(\bmu, \hat{\theta}_s, Q)$ the optimal allocations incur no loss from the {\emph top-}$M$ arms and the same amount of loss from the {\emph bottom-}$(K-M)$ arms. Hence the mean loss reduction for both the instances is same. This argument completes the proof of first part. As $M \in \{1, \ldots, K\}$ and $\hat{\theta}_s \le Q$, the possible value of $\hat{\theta}_s$ is only one of the elements in the set $\Theta = \{  Q/K, Q/(K-1), \cdots, Q\}$. 
\endproof

\sameThresholdEstRounds*
\proof{Proof.}
	When $\hat{\theta}_s < \theta_s$, it is possible that no loss is observed for $W_\delta$ consecutive rounds that leads to incorrect estimation of $\theta_s$. We want to set $W_\delta$ in such a way that the probability of occurring such event is upper bounded by $\delta$. This probability is bounded as follows:
	\begin{align*}
		&\Prob{\text{No loss is observed on $Q/\hat{\theta}_s$ arms for $W_\delta$ consecutive rounds at $\hat{\theta}_s<\theta_s$ (underestimate)}} \\
		&\qquad \le \prod_{i > K- Q/\hat{\theta}_s}^{K} (1 - \mu_i)^{W_\delta} \hspace{5mm}\text{\big(as $(1 - \mu_i)$ is the probability of not observing loss for arm $i$\big)}\\
		&\qquad \le \prod_{i > K-Q/\hat{\theta}_s}^{K} (1 - \epsilon)^{W_\delta} \hspace{8mm} \mbox{\big(as $\epsilon \ge \mu_i, ~\forall i \in [K]$\big)} \\
		&\qquad = (1 - \epsilon)^{\frac{QW_\delta}{\hat{\theta}_s}} \\	&\qquad  \le (1 - \epsilon)^{W_\delta}. \hspace{5mm}\text{\big(as $\hat{\theta}_s \le Q$\big)}%\label{equ:roundMistakeProb}
	\end{align*}
	Since we are using binary search and the set $\Theta$ has $K$ elements, the algorithm goes through at most $\log_2(K)$ underestimates of $\theta_s$. Let $I$ denote the set of indices of these underestimates in $\Theta$
	\begin{align*}
		&\Prob{\text{No loss is observed for consecutive $W_\delta$ rounds at any  underestimate of $\theta_s$}} \\
		&\qquad \le \sum_{i \in I}\Prob{\text{No loss is observed for consecutive $W_\delta$ rounds at the underestimate $\Theta(i)$}} \\
		&\qquad \le (1 - \epsilon)^{W_\delta}\log_2(K).
	\end{align*}
	As we are interesting in bounding the probability of making mistake by $\delta$, we get,
	\begin{align*}
	&(1 - \epsilon)^{W_\delta}\log_2(K) \le \delta \\
	\implies &(1 - \epsilon)^{W_\delta} \le \delta/\log_2(K).
	\end{align*}
	Taking log on both side of above equation, we get
	\begin{align*}
		&W_\delta\log(1 - \epsilon) \le \log(\delta/\log_2(K)) \\
		\implies &W_\delta\log\left(\frac{1}{1 - \epsilon}\right) \ge \log(\log_2(K)/\delta)\\ 
		\implies &W_\delta \ge \frac{\log(\log_2(K)/\delta)}{\log\left(\frac{1}{1 - \epsilon}\right)}.
	\end{align*}
	We set
	\begin{equation}
		W_\delta=\frac{\log(\log_2(K)/\delta)}{\log\left(\frac{1}{1 - \epsilon}\right) }.
	\end{equation}
	Hence, the minimum rounds needed to find a threshold that is an allocation equivalent with probability of at least $1-\delta$ is $W_\delta\log_2(K)$.
\endproof

\RegretEquiST*
\proof{Proof.}
	This result is an extension of  Proposition 1 in \cite{NeurIPS19_verma2019censored} to the case where $\theta_s \le Q$ instead of $\theta_s \le 1$.
	Let $\pi^\prime$ be a policy on $P:=(\bmu,\theta_s,Q) \in \PCSB_s$. The regret of policy $\pi^\prime$ on $P$ is given by
	\[\Regret_T(\pi^\prime, P)=\sum_{t=1}^T \left(\sum_{i=1}^K \mu_i\one{a_{t,i} < \theta_s} -\sum_{i=1}^K \mu_i\one{a^\star_i < \theta_s}\right), \]
	where $\ba^\star$ is the optimal allocation for $P$. Consider $f(P)=(\bmu,m) \in \PMP$ where $\bmu$ is same as in $P$ and $m=K-M$, where $M=\min\{\lfloor Q/\theta_s\rfloor,K\}$. The regret of policy $\pi$ on $f(P)$ is given by 
	\[\Regret_T(\pi, f(P))=\sum_{t=1}^T\left(\sum_{i \in M_t} \mu_i - \sum_{i>K-M}^{K}\mu_i\right),\]
	where $M_t$ is the superarm played in round $t$. Recall the ordering $\mu_1 \ge \mu_2 \ge \ldots \ge \mu_K$. It is clear that $\sum_{i=1}^K \mu_i\one{a^\star_i< \theta_s}=\sum_{i>M}^K\mu_i$. 
	Let $C_t$ be the set of arms where no resources are allocated by policy $\pi^\prime$ in round $t$. Since, loss only incurred from arms in the set $C_t$, we have $\sum_{i=1}^K \mu_i\one{a_{t,i} < \theta_s}=\sum_{i \in C_t}\mu_i$. By definition, the policy $\pi$ selects superarm $M_t=C_t$ in round $t$, i.e., set of arms returned by policy $\pi^\prime$ for which no resourced are applied. Hence $\sum_{i=1}^K \mu_i\one{a_{t,i}< \theta_s}=\sum_{i \in M_t} \mu_i $. This establishes the regret of policy $\pi^\prime$ on $P$ is same as regret of policy $\pi$ on $f(P)$ and hence, we get $\Regret(\PMP)\leq \Regret (\PCSB_s)$.
	Similarly, we can also establish the other direction of the proposition and get $\Regret(\PCSB_s)\leq \Regret (\PMP)$. Thus we conclude that $\Regret(\PCSB_s)= \Regret (\PMP)$.
\endproof

\noindent
We need the following results to prove the \cref{thm:regretSameThreshold}.
\begin{theorem}
	\label{thm:MPRegret}
	Let $\hat\theta_s$ be allocation equivalent of $\theta_s$ for instance $(\bmu,\theta_s,Q)$. Then, the expected  regret of \ref{alg:CSB-ST} for $T$ rounds after knowing the allocation equivalent is upper bounded as     
	\begin{align}
	\EE{\Regret_T} \le O\left((\log T)^{{2}/{3}}\right) + \sum_{i \in [M]} \frac{(\mu_i-\mu_{M+1} )\log {T}}{d( \mu_{M+1},\mu_i)}.
	\end{align}
\end{theorem}
\proof{Proof.}
	As $\hat\theta_s$ is the allocation equivalent of $\theta_s$, the instances $(\bmu,\theta_s,Q)$ and $(\bmu,\hat{\theta}_s,Q)$ have the same minimum loss. After knowing the allocation equivalent, the CSB problem with the same threshold is equivalent to solving an MP-MAB instance (\cref{prop:RegretEquiST}). Hence, we can directly apply Theorem 1 of \cite{ICML15_komiyama2015optimal} to obtain the regret bounds by setting $k=K-M$ and noting that we are in the loss setting and incur regret only when an arm $i \in [M]$ is in selected superarm.
\endproof

\begin{theorem}
	\label{thm:regretHighConf}
	With probability at least $1-\delta$, the expected cumulative regret of \ref{alg:CSB-ST} is upper bounded as
	\begin{equation*}
		\EE{\Regret_T} \le W_\delta\log_2{(K)} \nabla_{\max}  + O\left((\log T)^{{2}/{3}}\right) + \sum_{i \in [M]} \frac{(\mu_i-\mu_{M+1} )\log {T}}{d( \mu_{M+1},\mu_i)}.
	\end{equation*}
\end{theorem}

\proof{Proof.}
	The regret of \ref{alg:CSB-ST} can be divided into two parts: regret before knowing allocation equivalent and after knowing it. The threshold estimation completes in at most $W_\delta\log_2{(K)}$ rounds and returns a threshold which is an allocation equivalent with the probability of at least $1-\delta$. The maximum regret incurred for estimating allocation equivalent is $W_\delta\log_2{(K)} \nabla_{\max}$. The regret incurred after knowing allocation equivalent is given by \cref{thm:MPRegret}. Thus the expected regret of \ref{alg:CSB-ST} is the sum of regret incurred in the two parts and holds with the probability of at least $1-\delta$.  
\endproof

\noindent
We are now ready to give the proof of \cref{thm:regretSameThreshold}.
\regretSameThreshold*
\proof{Proof.}
	The regret bound follows from Theorem \ref{thm:regretHighConf} by setting $\delta=T^{-(\log T)^{-\alpha}}$ and unconditioning the expected regret incurred after knowing the allocation equivalent in \ref{alg:CSB-ST} .
\endproof

\diffThetaOptiSoln*
\proof{Proof.}
	Assigning $\theta_i$ fraction of resources to an arm $i$ reduces the total mean loss by amount $\mu_i$. Our goal is to allocate resources such that total mean loss is minimized, i.e., $\min\limits_{\ba \in \A}\sum_{i\in[K]}\mu_i\one{a_i < \theta_i}$. observe that the maximization version of same optimization problem is $\max\limits_{\ba \in \A}\sum_{i\in[K]}\mu_i\one{a_i \ge \theta_i}$ which is exactly same as solving a 0-1 knapsack with capacity $Q$ where item $i$ has value $\mu_i$ and weight $\theta_i$. 
\endproof

\diffTheteEst*
\proof{Proof.}
	Let $L^\star = \left\{i: a_i^\star < \theta_i\right\}$ and $r = Q - \sum_{i: a_i^\star \ge \theta_i}\theta_i$. If resource $r$ is allocated to any arm $i \in L^\star$, minimum value of mean loss will not change as $r < \min_{i \in L^\star} \theta_i$. If we can allocate $\gamma=r/K$ fraction of $r$ to each arm $i \in K$, the minimum mean loss still remains the same. If estimated threshold of every arm $i \in K$ lies in $[\theta_i, \ceil{\theta_i/\gamma}\gamma]$ then using Theorem 3.2 of \cite{DO13_hifi2013sensitivity}, $KP(\bmu,\btheta, Q)$ and $KP(\bmu, \hat\btheta, Q)$ has the same optimal solution because of having the same mean loss for both the problem instances.
\endproof

\MultiTheta*
\proof{Proof.}
	For any arm $i \in [K]$, we want $\hat{\theta}_i \in [\theta_i, \ceil{\theta_i/\gamma}\gamma]$. As $\theta_i \in (0,Q]$, we can divide interval $[0,Q]$ into a discrete set $\Theta \doteq \left\{0, \gamma, 2\gamma, \ldots, Q\right\}$ and note that $|\Theta| = \ceil{1+ {Q}/{\gamma}}$. As search space is reduced by half in each change of $\hat\theta_i$, the maximum change in $\hat\theta_i$ is upper bounded by $\log_2|\Theta|$ to make sure that $\hat{\theta}_i \in [\theta_i, \ceil{\theta_i/\gamma}\gamma]$. When $\hat{\theta}_i$ is underestimated and no loss is observed for consecutive $W_\delta$ rounds, a mistake happens by assuming that current allocation is an overestimate. We set $W_\delta$ such that the probability of estimating wrong $\hat{\theta}_i$ is small and bounded as follows:
	\begin{align*}
		&\Prob{\text{No loss is observed for consecutive $W_\delta$ rounds when $\hat{\theta}_i$ is underestimated}} \\
		&\qquad = (1 - \mu_i)^{W_\delta} \hspace{5mm}\text{\big(as $(1 - \mu_i)$ is the probability of not observing loss at arm $i$\big)}\\
		&\qquad \le (1 - \epsilon)^{W_\delta}. \hspace{6.8mm} \text{\big(since $\forall i \in [K]: \mu_i > \epsilon$\big)}
	\end{align*}
	Since we are doing binary search, the algorithm goes through at most $\log_2(|\Theta|)$ underestimates of $\theta_i$. Let $I$ denote the set of indices of these underestimates in $\Theta$
	\begin{align*}
		&\Prob{\text{No loss is observed for consecutive $W_\delta$ rounds when  $\hat{\theta}_i$ is underestimated}} \\
		&\qquad \le \sum_{i\in I}\Prob{\text{No loss is observed for consecutive $W_\delta$ rounds when  $\hat{\theta}_i$ is underestimated}} \\
		&\qquad \le (1 - \epsilon)^{W_\delta}\log_2(|\Theta|).
	\end{align*}
	Next, we will bound the probability of making mistake for any of the arm. That is given by
	\begin{align*}
	&\Prob{\exists i  \in [K], \hat\theta_i \in \Theta: \text{No loss is observed for consecutive $W_\delta$ rounds when  $\hat{\theta}_i$ is underestimated}} \\
	&~~ \le \sum_{i=1}^{K}\Prob{\exists \hat\theta_i \in \Theta: \text{No loss is observed for consecutive $W_\delta$ rounds when  $\hat{\theta}_i$ is underestimated}}  \\
	&~~ \le K(1 - \epsilon)^{W_\delta}\log_2(|\Theta|).
	\end{align*}
	
	As we are interested in bounding the above probability of making a mistake by $\delta$ for all arms, we have the following expression,
	\begin{align*}
		&K (1 - \epsilon)^{W_\delta}\log_2(|\Theta|) \le \delta\\
	 	\implies &(1 - \epsilon)^{W_\delta} \le \delta/K \log_2(|\Theta|).
	\end{align*}
	Taking log both side, we get
	\begin{align*}
		&W_\delta\log(1 - \epsilon) \le \log(\delta/K \log_2(|\Theta|))\\
	 	\implies & W_\delta\log\left({1}/{(1 - \epsilon)}\right) \ge \log(K \log_2(|\Theta|)/\delta)\\
		\implies &W_\delta \ge \frac{\log(K \log_2(|\Theta|)/\delta)}{\log\left({1}/{(1 - \epsilon)}\right) }.
	\end{align*}
	As $|\Theta| = \ceil{1+ Q/\gamma}$, we set 
	\begin{equation}
		W_\delta = \frac{\log(K \log_2 \ceil{1+ Q/\gamma}/\delta)}{\log\left({1}/{(1 - \epsilon)}\right) }.
	\end{equation}
	
	Therefore, the minimum number of rounds needed to find a threshold $\hat{\theta}_i$ for an arm $i$, which is an element of allocation equivalent vector with the probability of at least $1-\delta/K$ is upper bounded by $W_\delta\log_2 \ceil{1+ Q/\gamma}$. 
	
	Since $n$ is the number of different thresholds, there are $n$ different groups of arms where group $G_i$ consists of arms having the same estimated threshold $\hat{\theta}_i$ in the estimated allocation equivalent vector. We divided the number of rounds to know allocation equivalent into two parts. The first deals with the maximum number of expected rounds needed to find $n$ thresholds. In comparison, the second part deals with finding the good thresholds for remaining arms using known thresholds.
	
	Let consider the worst case where only one threshold is estimated at a time. Then the maximum expected rounds needed to estimate threshold associated with $G_i$ is $W_\delta \log_2\ceil{1 + Q/\gamma}$. Using this fact with definition of $A_{\theta_n}$, the maximum expected rounds needed to estimate $n$ thresholds is $W_\delta\sum_{i \in A_{\theta_n}} \log_2\ceil{1 + Q/\gamma}$. Once all $n$ thresholds are known then the threshold for any arm $k$ in remaining arms with non-zero mean loss need to search over $n$ possible values of thresholds and hence the expected number of rounds needed to its estimate is $W_\delta \log_2 (n+1)$. Therefore, the maximum number of rounds needed to estimate threshold for all arms in $A_{\theta_n}^c$ is $W_\delta\sum_{k \in A_{\theta_n}^c: \mu_k \ne 0} \log_2 (n+1)$ which further upper bounded by $W_\delta K\log_2 (n+1)$. With this argument, the proof is complete.
\endproof

\subsection*{Equivalence of CSB with different thresholds and Combinatorial Semi-Bandit }
In stochastic Combinatorial Semi-Bandits (CoSB), a learner can play a subset of $K$ arms in each round, also known as superarm, and observes the loss from each arm played \citep{ICML13_chen2013combinatorial,NIPS16_chen2016combinatorial,ICML18_wang2018thompson}. The size of a superarm can vary, and the mean loss of a superarm only depends on the mean of its constituent arms. The goal is to select a superarm that has the smallest loss. A policy in CoSB selects a superarm in each round based on past information. The performance of a policy is measured in terms of regret, defined as the difference between cumulative loss incurred by the policy and that incurred by playing an optimal superarm in each round. Let $(\bmu, \mathcal{I}) \in [0,1]^K \times 2^{[K]}$ denote an instance of CoSB, where $\bmu$ denote the mean loss vector and $\mathcal{I}$ denotes the set of superarms. Let $\PCSB_d \subset \PCSB$ denote the set of CSB instances with different thresholds. For any  $(\bmu,\btheta,Q) \in \PCSB_d$ with $K$ arms and known threshold $\btheta$, let $(\bmu, \mathcal{I})$ be an instance of CoSB with $K$ arms and each arm has the same Bernoulli distribution as the corresponding arm in the CSB instance. Let $\PCoSB$ denote set of resulting CoSB problems and $g: \PCSB_d \rightarrow \PCoSB$ denote the above transformation.

Let $\pi$ be a policy on $\PCoSB$. The policy $\pi$ can also be adapted for any $(\bmu,\btheta,Q) \in \PCSB_d$ with known $\btheta$ to decide which set of arms are allocated resource as follows: In round $t$, let information $(C_1, Y_1, C_2,Y_2, \ldots, C_{t-1}, Y_{t-1})$ collected from a CSB instance, where $C_s$ is the set of arms where no resource is applied and $Y_s$ is the samples observed from these arms, is given to $\pi$ which returns a set $C_t$. Then all arms other than arms in $C_t$ are given resource equal to their  estimated good threshold. Let this policy on $(\bmu,\btheta,Q) \in \PCSB_d$ is denoted as $\pi^\prime$. Similarly, a policy $\beta^\prime$ on $\PCSB_d$ can be adopted to yield a policy for $\PCoSB$ as follows: In round $t$, the information $(M_1, Y_1, M_2,Y_2, \ldots, M_{t-1}, M_{t-1})$, where $M_s$ is the superarm played in round $s$ and $Y_s$ is the associated loss observed from each arms in $M_s$, collected on an CoSB instance is given to the policy $\beta^\prime$. Then the policy $\beta^\prime$ returns a set $M_t$ where no resources has allocated. The superarm corresponding to $M_t$ is then played. Let this policy on $\PCoSB$ be denoted by $\beta$. Note that when $\btheta$ is known, the mapping is invertible. 
Our next result gives regret equivalence between the CoSB problem and the CSB problem with the known thresholds.

\MultiThetaEquivalence*
\vspace{-2.5mm}
\proof{Proof.}
	Let $\pi^\prime$ be a policy on $P:=(\bmu,\btheta,Q) \in \PCSB_d$. The regret of policy $\pi^\prime$ on $P$ is given by
	\[\Regret_T(\pi^\prime,P)=\sum_{t=1}^T \left(\sum_{i=1}^K \mu_i\one{a_{t,i}< \theta_i} -\sum_{i=1}^K \mu_i\one{a^\star_i < \theta_i}\right), \]
	where $\ba^\star$ is the optimal allocation for $P$. Consider $g(P)=(\bmu,\mathcal{I}) \in \PCoSB$ where $g: \PCSB_d \rightarrow \PCoSB$ and $\bmu$ is the same as in $P$ and $\mathcal{I}$ contains all superarms (set of arms) for which resource allocation is feasible. The regret of policy $\pi$ on $g(P)$ is given by 
	\vspace{-1.25mm}
	\[\Regret_T(\pi,g(P))=\sum_{t=1}^T\big(l(M_t,\bmu) - l(M^\star,\bmu)\big),\]
	where $M_t$ is the superarm played in round $t$,  $M^\star$ is optimal superarm, and $l$ returns mean loss for given superarm.
	The outcome of $l(M,\bmu)$ only depends on mean loss of constituents arms of the superarm $M$. In our setting, $l(M,\bmu) = \sum_{i \in M}\mu_i$ where $M= \left\{i: a_i <\theta_i \right\}$ for allocation $\ba \in \A$. It is clear that $\sum_{i=1}^K \mu_i\one{a^\star_i < \theta_i}=l(M^\star, \bmu)$. Let $C_t$ be the set of arms where no resource is allocated by $\pi^\prime$ in round $t$. Since, loss is only incurred for arms in the set $C_t$, we have $\sum_{i=1}^K \mu_i\one{a_{t,i}< \theta_s}=\sum_{i \in C_t}\mu_i$. By definition the policy $\pi$ selects superarm $M_t=C_t$ in round $t$, i.e., set of arms returned by $\pi^\prime$ for which no resourced are applied. Hence $\sum_{i=1}^K \mu_i\one{a_{t,i}< \theta_s}=\sum_{i \in M_t} \mu_i $. This establishes the regret of $\pi^\prime$ on $P$ is same as regret of $\pi$ on $g(P)$ and hence, $\Regret(\PCoSB)\leq \Regret (\PCSB_d)$.
	Similarly, we can establish the other direction of the proposition and get $\Regret(\PCSB_d)\leq \Regret (\PCoSB)$. Thus we conclude $\Regret(\PCSB_d)= \Regret (\PCoSB)$.
\endproof

Let $\nabla_{\max}$, $\nabla_{\min}$, $\nabla_{i,\min}$, and $K^\prime$ be the same as in Section \ref{sssec:differentThetaRegretBounds}. Let $k_\star$ be the minimum number of arms in the optimal superarm and $\cM$ be the set of all feasible superarms. As it is not possible to sample $\hat\mu_i(t)$ to be precisely the true value $\mu_i$ using Beta distribution, we need to consider the $\eta$-neighborhood of $\mu_i$, and such $\eta$ term is common in the analysis of most Thompson Sampling algorithms (see \cite{NeurIPS20_perrault2020statistical,ICML18_wang2018thompson} for more details). We need the following results to prove \cref{thm:regretDiffThreshold}. 
\begin{theorem}
	\label{thm:CTSRegret}
	Let $\hat\btheta$ be allocation equivalent of $\btheta$ for instance $(\bmu,\btheta,Q)$. Then, the expected regret of \ref{alg:CSB-MT} in $T$ rounds after knowing the allocation equivalent is upper bounded by
	$$16\log_2^2(16K^\prime)\sum_{i\in [K]} \frac{ \log \left(2^{K^\prime} |\cM|T\right)}{\nabla_{i,\min}} + \nabla_{\max}(K+1) + \frac{4K (K^\prime)^2\nabla_{\max}}{\left(\nabla_{\min}-2(k_\star^2+1)\eta\right)^2}  + \nabla_{\max} \frac{C}{\eta^2}\left(\frac{C^\prime}{\eta^4}\right)^{k_\star},$$
	 where $C$, $C^\prime$ are two universal constants, and $\eta \in (0, 1)$ is such that $\nabla_{\min}-2(k_\star^2+1)\eta > 0$. Further, the expected regret of \ref{alg:CSB-MT} in $T$ rounds  is also upper bounded by $O\left(\sum\limits_{i \in [K]}\frac{\log^2(K^\prime)\log T}{\nabla_{i,\min}} \right)$.
\end{theorem}
\proof{Proof.}
	Once the allocation equivalent of $\btheta$ is known, the CSB problem is equivalent to a Combinatorial Semi-Bandit problem (from \cref{prop:MultiThetaEquivalence}). Now the proof of \cref{thm:CTSRegret} follows by verifying Assumptions $1-3$ in \cite{NeurIPS20_perrault2020statistical} for the Combinatorial Semi-Bandit problem and applying their regret bound. Assumption $1$ states that the agent has access to an oracle that can compute the optimal superarm. Whereas, Assumption $3$ states that the losses of arms are bounded and mutually independent. It is clear that both of these assumptions hold for our setting. We next proceed to verify Assumption $2$. For fix allocation $\ba\in \A$, the mean loss incurred from loss vector $\bmu$ is given by $l(M,\bmu)=\sum_{i \in M}\bmu_i$ where $M=\left\{i:a_i < \hat\theta_i\right\}$. For any two loss vectors $\bmu$ and $\bmu^\prime$, we have
	\begin{align*}
	l(M, \bmu)-l(M, \bmu^\prime)&=\sum_{i \in M}(\mu_i - \mu_i^\prime) \\
	&= \sum_{i=1}^K  \one{a_i< \hat\theta_i}\left (\mu_i -\mu_i^\prime \right) \hspace{5mm} \text{$\Bigg($as $\sum_{i \in M}\mu_i= \sum_{i=1}^K \mu_i \one{a_i< \hat\theta_i}\Bigg)$}\\
	&\leq    \sum_{i=1}^K  \left (\mu_i -\mu_i^\prime \right)\\
	&\leq    \sum_{i=1}^K   |\mu_i -\mu_i^\prime | \\
	&= B\parallel \bmu- \bmu^\prime \parallel_1
	\end{align*}
	where $B=1$. After knowing the allocation equivalent, the allocation to each arm remains the same in each round ($\hat{\theta}_i$ is given to each arm $i \in [K]\setminus A_t$). Thus we are solving a Combinatorial Semi-Bandit with parameter $B=1$. By using Theorem $1$ in \cite{NeurIPS20_perrault2020statistical}, we get the desired bounds.
\endproof

\begin{theorem}
	\label{thm:regretDiffThresholdHighConf}
	With probability at least $1-\delta$, the expected cumulative regret of \ref{alg:CSB-MT} is upper bounded as 
	\begin{align*}
		\EE{\Regret_T} \le \frac{\log( K \log_2\left(\ceil{1 +\frac{Q}{\gamma}}\right)/\delta)} {\log(1/(1-\epsilon))} &\left[\sum_{i \in A_{\Theta_n}} {\log_2 \left(\ceil{1 + \frac{Q}{\gamma}}\right)} + K{\log_2 (n +  1)}\right] \nabla_{\max} +\\
		&\qquad O\left(\sum\limits_{i \in [K]}\frac{\log^2(K^\prime)\log T}{\nabla_{i,\min}} \right).
	\end{align*}
\end{theorem}
\proof{Proof.}
	The first term of expected regret is due to the estimation of allocation equivalent. It takes $T_{\theta_n}$ rounds to complete, and $\nabla_{\max}$ is the maximum regret that can be incurred in any round. Then the maximum regret due to threshold estimation is bounded by $T_{\theta_n}\nabla_{\max}$ (replace $T_{\theta_n}$ by its value). The remaining term in regret corresponds to the expected regret incurred after knowing the allocation equivalent, that is upper bounded by \cref{thm:CTSRegret}.  
\endproof

\noindent
Let $W_\delta$  be the same as in Section \ref{sssec:differentThetaRegretBounds}. We are now ready to give the proof of \cref{thm:regretDiffThreshold}.
\regretDiffThreshold*
\proof{Proof.}
	The regret bound follows from Theorem \ref{thm:regretDiffThresholdHighConf} by setting $\delta=T^{-(\log T)^{-\alpha}}T$ and unconditioning the expected regret incurred after knowing the allocation equivalent in \ref{alg:CSB-MT}.
\endproof

\subsection*{Missing proofs from  \cref{sec:csb_unknown} }

\sameThetaEstRounds*
\proof{Proof.}
    Let $X_1, X_2, \ldots, X_P$ be the independent Bernoulli random variables where $X_i$ has mean $\mu_i$.  The samples from all random variables are observed at the same time. Let $R_W$ is a random variable that counts the number of rounds needed to observe a sample of `$1$' for any of $\{X_i\}_{i \in [P]}$. First, we compute $\Prob{R_W = w}$, i.e.,
    \begin{equation*}
        \Prob{R_W = w} = \Pi_{i \in [P]}(1-\mu_i)^{w-1} \left(1-\Pi_{i \in [P]}(1-\mu_i)\right).
    \end{equation*}
    The previous results follows from the fact that there a sample of `$1$' is not observed for any of $\{X_i\}_{i \in [P]}$ in the first $w-1$ rounds and a sample of `$1$' is observed for at least one of the random variable in the $w^{th}$ round. The expectation of $R_W$ is given as follows:
    \begin{align*}
        \EE{R_W} &= \sum_{w=1}^{\infty} w \Prob{R_W = w}  \\
        &= \sum_{w=1}^{\infty} w \Pi_{i \in [P]}(1-\mu_i)^{w-1} \left(1-\Pi_{i \in [P]}(1-\mu_i)\right)\\
        &= \left(1-\Pi_{i \in [P]}(1-\mu_i)\right)\sum_{w=1}^{\infty} w \Pi_{i \in [P]}(1-\mu_i)^{w-1}. 
        \intertext{Let $\bar{p} = 1-\Pi_{i \in [P]}(1-\mu_i)$, we have}
        \implies \EE{R_W}  &= \bar{p} \sum_{w=1}^{\infty} w (1 - \bar{p})^{w-1} \\
        &= \bar{p}\left[ \frac{d}{d\bar{p}}\sum_{w=1}^{\infty} -(1-\bar{p})^{w}\right] \\ 
        &= \bar{p}\left[ \frac{d}{d\bar{p}}\left(\frac{-1}{\bar{p}} \right) \right] \\
        &= \bar{p}\left(\frac{1}{\bar{p}^2} \right) = \frac{1}{\bar{p}} \\
        &= \frac{1}{1- \Pi_{i \in [P]}(1-\mu_i)}.
    \end{align*}
    
    \ref{alg:CSB-JS} starts equal resources to all $L=K$ arm. When a loss is observed for any of the arms, it implies that current resource allocation is a underestimate of threshold and then resources are equally allocated among $L=K-1$ arms. Let $T_\theta(L)$ denote the number of the rounds needed to observe a loss when $L$ arms are allocated resources. By taking top $L$ arms in each round, the upper bound on expected value of $T_\theta(L)$ is given as:
    \begin{equation*}
        \EE{T_\theta(L)} \leq  \frac{1}{1 - \Pi_{i \in [K]/[K-L]}(1-\mu_i)}.
    \end{equation*}
    
    Note that $M$ is the number of arms in the optimal allocation. Consider all wrong values of $L \in \{M+1, M+2, \ldots, K-1, K \}$, the upper bound on expected number of rounds needed to reach to correct allocation, i.e., $(Q/M)$ is given as follows:
    \begin{align*}
        \EE{T_{\theta_s}} &= \sum_{L=M+1}^{K} \EE{T_\theta(L)} \\
        \implies \EE{T_{\theta_s}} &\le \sum_{L=M+1}^{K} \frac{1}{1 - \Pi_{i \in [K]/[K-L]}(1-\mu_i)}. %\tag*{\qedhere}
    \end{align*}
\endproof

\noindent
Now, we need the following results to prove the regret bound of \ref{alg:CSB-JS}.
\regretJointSameThreshold*
\proof{Proof.}
    The regret of \ref{alg:CSB-JS} can be divided into two parts: regret before knowing allocation equivalent and after knowing it. The first part of regret bound is the expected regret incurred while estimating allocation equivalent, which is $\EE{T_{\theta_s}}\nabla_{\max}$. The second part of regret is due to the MP-MAB algorithm (MP-TS) and is given by \cref{thm:MPRegret}.
\endproof

\diffThetaEstRounds*
\proof{Proof.}
    The expected number of rounds needed to observe a loss from an under-allocated arm with non-zero mean loss are $1/ \mu_i$ (by Geometric distribution). When a loss is observed for an arm, \ref{alg:CSB-JD} increments resources by $\gamma$ amount for that arm. In worse case, $\theta_i$ or more resources are allocated only after $\floor{\theta_i/\gamma}$ number of increments in resource allocation for the arm $i$. Therefore, the expected number of rounds needed to estimate $\hat{\theta}_i \in [\theta_i, \ceil{\theta_i/\gamma}\gamma]$ are $\floor{\theta_i/\gamma}(1/\mu_i)$.
    
    Let consider the worst case where only one threshold is estimated at a time. Then the maximum expected rounds needed to estimate all thresholds are $\sum_{i \in [K]:\mu_i \ne 0} \floor{{\theta_i}/{\gamma}} \left( {1}/{\mu_i} \right)$. With this argument, the proof is complete.
\endproof

\paragraph{Remark:} \ref{alg:CSB-JD} can estimate thresholds of multiple arms by starting with the same allocation of resources to all arms. Hence the number of rounds needed for finding allocation equivalent might be very small in practice than given in \cref{lem:diffThetaEstRounds} where the worst case is considered.

\regretJointDiffThreshold*
\proof{Proof.}
    Similar to \ref{alg:CSB-JS}, the regret of \ref{alg:CSB-JD} can also be divided into two parts: regret before knowing allocation equivalent and after knowing it. We get the first part of expected regret by using the upper bound on the expected number of rounds needed to find allocation equivalent from \cref{lem:diffThetaEstRounds} and the fact that $\Delta_{\max}$ is the maximum regret that can be incurred in any round. Once an allocation equivalent threshold is found, by exploiting equivalence with combinatorial semi-bandit, the second part of the expected regret is due to using a combinatorial semi-bandit algorithm (CTS-BETA) and is given by \cref{thm:CTSRegret}.
\endproof

%% file: chapter5/main.tex
%!TEX root =  ../thesis.tex

We consider a communication network consisting of multiple users and multiple channels where each user wishes to use one of the available channels. We propose distributed strategies that minimize the sample complexity and regret of acquiring the best subset of channels. Agents cannot directly communicate with each other, and no central coordination is possible. Each agent can transmit on one channel at a time. If multiple agents transmit on the same channel simultaneously, a collision occurs. When a collision occurs, no agent gets any information about the channel gain or how many other agents are transmitting on that channel.  If no collision occurs, the agent observes a reward (or gain) sample drawn from an underlying distribution associated with the channel.  We modeled this problem as a Multi-Player Multi-Armed bandits problem. One important property of our algorithms that distinguishes it from the prior work is that it requires no information about the difference of the means of the channel gains. Our approach results in fewer collisions with improved regret performance compared to the state-of-the-art algorithms.  We validate our theoretical guarantees with experiments.

\section{ Multi-Player Multi-armed Bandits}
\label{sec:mpb_introduction}
\input{chapter5/introduction}

\section{ Problem Setup}
\label{sec:mpb_problem_setup}
\input{chapter5/problem_setup}
\nocite{WiOPT19_verma2019distributed}

\section{Algorithm with No Information Exchange}
\label{sec:mpb_dlc}
\input{chapter5/dlc}

\section{Algorithm with Limited Information Exchange}
\label{sec:mpb_dlc_lix}
\input{chapter5/dlc_lix}

%\section{Algorithm with Periodic Information Exchange}
%\label{sec:mpb_dlc_pix}
%\input{chapter5/dlc_pix}

\section{Experiment}
\label{sec:mpb_experiment}
\input{chapter5/experiment}

\section{Appendix}
\label{sec:mpb_appendix}
\input{chapter5/appendices}

%% file: chapter5/introduction.tex
%!TEX root =  ../thesis.tex

Consider a communication network consisting of $N$ users and $K$ channels where each user wishes to use one of the $K$ channels. A user can transmit on any of the channels. To keep the model more general, users cannot communicate directly between them and any information about the other users can be obtained only by collisions. In particular, if user $n$ transmits on channel $k$, it gets a sample of the gain of channel $k$ if no other user transmis on channel $k$, otherwise it  only gets to know that at least one other user transmitted on channel $k$ at the same time. We assume that the expected gain of each of the $K$ channels is distinct. Even though users cannot communicate and have to make decisions in a distributed fashion, we consider a non-strategic setting where users have a common goal to maximize the total gain in the network, i.e., to achieve an optimal allocation, which is realized when all the users occupy non-overlapping channels in the top $N$ channels. Here the top $N$ channels refer to the set of $N$ channels with highest expected gain. As the goal is to achieve optimal network allocation,  we assume that each user is satisfied if she gets one of the top $N$ channels among the $K$, when no other better channel is free. Our aim in this work is to find a  strategy that minimizes the time by which the users reach an optimal allocation while keeping the number of collisions low.

This problem is motivated by ad hoc cognitive radio networks (CRN), where multiple users try to access the same set of channels \citep{ad_hoc_1,ad_hoc_2}, without any direct communication between them. In this setting, there is no central controller or a common control channel that can be used to resolve contentions and all channel selection decisions have to be done in a decentralized fashion \citep{ad_hoc_1,ad_hoc_2}. Moreover, this problem is well suited for upcoming $5G$ standards like device-to-device communication mode in cellular networks where multiple mobile nodes try to form groups to reduce signaling load on the central base station. Such models are also being envisioned for futuristic ultra-dense networks deployed to offer high peak rates \citep{NR_1}.

In ad hoc CRNs, users may not know the quality of channels, and the number of other users present in the network. The goal is still to maximize the number of successful transmissions (or sum rate/ throughput) in the network. To achieve this, the users not only have to learn the channel qualities but also have to learn to co-ordinate by selecting non-overlapping channels. This distributed learning problem is a multi-player version of the multi-armed bandit problem with $K$ arms, with a total of $N$ players. Individual players do not know the number of other players and cannot communicate with others about their arm selection strategy. A player can select at most one arm at a time. If multiple players select a common arm, a collision occurs and none of them receive any reward.

The performance of the arm selection policy is measured as regret, i.e., the difference of the total expected reward obtained by the policy and the total expected reward gained by playing an optimal strategy in each round by all the players. The total number of collisions is the sum of collisions experienced by all the players. From the applications point of view, e.g., in a CRN, where users are mostly battery operated, a higher number of collisions will result in reduced operational life.  Hence, in addition to minimizing the regret, the algorithms should work with as fewer collisions as possible.

The problem of multi-player multi-armed bandits with the unknown number of users is studied in \cite{ICML16_MultiplayerBandits_RosenkiShamir} where Musical chair (MC) is developed to minimize regret and its performance is shown to be superior compared to other algorithms (for the unknown number of players) \citep{JSAC11_DistributedLearning_Anadakumar, KDD14_ConcurrentBandits_AvnerMannor}. In the MC algorithm,  each player selects the arms uniformly at random for a fixed number of rounds and then estimate the top $N$ arms and the number of users based on the reward samples and collisions observed. The number of rounds required to get good estimates is set based on the separation between the mean reward of the arms and is assumed to be lower bounded by a known positive number. However, this assumption is restrictive as it is not easy to identify the lower bound in real-life applications. In this work, we propose a learning algorithm which achieves the optimal allocation quickly with fewer collision without requiring the knowledge of the gap between the mean reward of the arms. Our contributions of this chapter can be summarized as follows:
\begin{itemize}
	\item We design a communication protocol for exchanging information among players through collisions that do not need any direct communication medium between players.
	\item We develop an algorithm named Distributed Learning and Coordination (DLC) that achieves optimal allocation in the constant number of the rounds with high probability. This bound also translates to a constant regret with high probability. The DLC algorithm does not need to know the minimum gap between the arms. 
	\item We give a variant of DLC that improves utilization of channels through efficient signaling, resulting in performance improvement in terms of regret. 
	\item We validate our claims by numerical experiments and demonstrate performance gains over the state-of-the-art.
\end{itemize}
	
The DLC algorithm combines the ideas of pure exploration methods with efficient signaling schemes to achieve optimal allocation with high probability. In DLC, each player first learns the number of players in the network by colliding with others in a specific pattern. One of the players then identifies the top $N$ arms via pure explorations and signals all the players to occupy one of top $N$ arms without any overlap.

\subsection{Related Work: Multi-player Multi-armed Bandits}
Most work on stochastic bandits with multiple-players require some negotiation or pre-agreement phase to avoid collisions between the players. The $\mbox{dUCB}_4$ algorithm in  \cite{TIT14_DecentralizedLearning_KalthilNayyarJain} achieves this using Bertsekas' auction mechanism for players to negotiate unique arm. The time divisions fair sharing (TDFS) algorithm in \cite{TSP10_DistributedLearning_LiuZhao} requires players to agree on a time division of slots before the game. Such negotiations are hard to realize in a completely distributed (ad hoc) setup \citep{ad_hoc_1,ad_hoc_2}. The $\rho^{\mbox{{\tiny RAND}}}$ algorithm in \cite{JSAC11_DistributedLearning_Anadakumar} is communication free and completely decentralized but requires knowledge of the number of users. The modified $\rho^{\mbox{{\tiny EST}}}$ algorithm overcomes this issue, but its guarantees are asymptotic and do not hold for the finite time like ours.  Performance improvements of  $\rho^{\mbox{{\tiny RAND}}}$ are studied in \cite{ALT2018_MultiplayerBandits_BessonKauf}. Other variants of stochastic multi-player bandits also consider fixed number of players and address the issue of fairness \citep{TSP14_DistributedStochastic_GaiKrishnamachari} and asymmetric arm characterizations across the players \citep{INFOCOM16_MultiUserLax_AvnerMannor}. Another set of works in \cite{TWC16_DistributedStochastic_ZandiDongGrami}, \cite{TMC11_CongnitiveMediumAccess_LaiGamalJiangPoor} considers selfish behavior of players and analyze their equilibrium behavior.

The works most similar to ours are \cite{KDD14_ConcurrentBandits_AvnerMannor} and \cite{ICML16_MultiplayerBandits_RosenkiShamir} which consider a communication-free setting with the unknown number of players that can vary during the game. The MEGA algorithm in \cite{KDD14_ConcurrentBandits_AvnerMannor} uses the classical $\epsilon$- greedy MAB algorithm and ALOHA based collision avoidance mechanism. Though collision frequency reduces in MEGA as the game proceeds it may not go to zero as shown in \cite{ICML16_MultiplayerBandits_RosenkiShamir}. To overcome this,  \cite{ICML16_MultiplayerBandits_RosenkiShamir} proposed a musical chairs (MC) algorithm that incurs collisions only in the initial phase and guarantees collision free sampling subsequently. Though MC performs better than MEGA, its performance in the initial rounds is poor -- MC uses collision information to estimate the number of players and forces a large number of collisions to get a good estimate. As discussed earlier, this may not be suitable for applications like CRN. 

SIC-MMAB and ESER algorithm in \cite{NeurIPS19_boursier2019sic} and \cite{Arxiv2019_HeterogeneousNetworks}, respectively, propose signaling mechanisms that use a suitable pattern of collisions among the players for exchanging information. Although some of the ideas of the DLC and SIC-MMAB are similar, DLC uses different learning and signaling scheme.  Specifically, the exploration phase in SIC-MMAB and ESER is based on sequential hopping whereas in DLC it is based on pure-exploration.

%% file: chapter5/problem_setup.tex
%!TEX root =  ../thesis.tex

The standard stochastic $K$-armed bandit problem consists of a single player with $K>1$ arms. The multi-player $K$-armed bandit is similar but consists of multiple players. Let $N \leq K$ denote the total number of players. Playing arm $k \in [K]$ where $[K]:=\{1,2,\ldots, K\}$, gives reward drawn independently from a distribution with support $[0, \; 1]$.  The reward distributions are stationary, homogeneous, and independent across the players, and $\mu_k$ denotes the mean of arm $k$. The players are not aware of how many other players are present, and there exist no control channels over which they can communicate with each other. When a player samples an arm, the reward is obtained if she alone happens to play that arm. Otherwise, all the players choosing that arm will get zero rewards. We refer to the latter case as `collision.' For any distributed policy, let $I_{n,t}$ and $\eta_{n,t}$ indicate the arm played by player $n \in [N]$ and her collision indicator in the round $t$, respectively. $\eta_{n,t} =1$ if together with player $n$ at least one other player plays the same arm as $n$ at time $t$ and is set equal to zero otherwise.

To achieve an optimal allocation in a distributed setting, each player needs to learn the means of the arms to arbitrary precision which require them to play a large number of times making the goal infeasible for all practical purposes.  Thus, we focus on achieving an approximate optimal allocation of arms with high probability that is defined as below. Let $\pi:= {\pi_t: t\geq 1}$ denote a policy, where $\pi_t: [N]\rightarrow [K]$ denotes the allocation at time $t$, i.e., $I_{n,t}=\pi_t(n)$. 
\begin{defi}[Def. 1 of \cite{Arxiv2019_HeterogeneousNetworks}]
	%Let K be the number of arms 
%	with mean reward $\mu_1>\mu_2> \ldots >\mu_K$ 
	%and $N (<K)$ be number of players. %Then 
	For a given tolerance $\epsilon \ge 0$ and confidence $\delta \in (0,1/2]$, an allocation by a policy $\pi$ is said to be $(\epsilon,\delta)-$optimal if there exists $T:=T(\pi)<\infty$ such that        
	\begin{equation}
		\Pr\left\{\sum_{n\in [N^\star]}\mu_n - \sum_{n\in[N]}\mu_{\pi_t(n)} \le \epsilon\right\}\ge 1-\delta. \quad \forall \;t \geq T
	\end{equation} 
	where $N^\star$ denotes the best subset of  $N$ arms with highest mean rewards and $T(\pi)$ denotes the sample complexity of the policy $\pi$.. 
	%and $\pi(n)$ is the index of arm assigned to player $n$ after the termination of the algorithm $\pi$.
\end{defi}
\noindent
This definition can be viewed as a generalization of  \underline{p}robably-\underline{a}pproximately-\underline{c}orrect (PAC) performance guarantee in the pure exploration multi-armed bandits problems with single player \citep{ICML12_PAC_Subset_Selection, COLT13_Bandit_Subset_Selection} to the multi-player case. 
We define the regret of policy $\pi$ over period $T$ as 
\begin{equation}
	\Regret_T(\pi)=\sum_{t=1}^T \sum_{n \in [N^*]} \mu_{n} -\sum_{t=1}^T\sum_{n \in [N]}\mu_{\pi_t(n)}(1-\eta_{n,t}).
\end{equation}  
%where $\pi(n,t)$ is the index of arm assigned by algorithm $\pi$  to player $n$ in the round $t$.
%\begin{equation}
%\eta_{j,t}=
%\begin{cases}
%1 \quad \mbox{ if  } \exists j^\prime \neq j \mbox{ such that } I_{j,t}=I_{j^\prime,t} \\
%0 \quad \mbox{ otherwise}
%\end{cases}
%\end{equation} 
The total number of collisions incurred by a policy $\pi$ over period $T$ is defined as
\begin{equation}
C(\pi)=\sum_{t=1}^T \sum_{n \in [N]}  \eta_{n,t}.
\end{equation}
Note that a collision on the arm is counted multiple times because all the players involved in a collision have to re-sample in another slot incurring extra transmission cost. % from practical perspective. 
%\added{Without loss of generality, we assume arms are indexed in descending order by their mean rewards i.e., $\mu_1>\mu_2> \ldots >\mu_K$.}

Our goal is to develop distributed algorithms that give $(\epsilon,\delta)-$optimal allocation quickly while keeping the regret $\Regret_T$ and the total number of collisions $\C$ as low as possible. Specifically, we develop algorithms whose regret is constant with high probability, i.e., the algorithm incurs regret only for finitely many rounds. 

Similar to the prior work \citep{ICML16_MultiplayerBandits_RosenkiShamir, ALT2018_MultiplayerBandits_BessonKauf,Arxiv18_hanawal2018multi, NeurIPS19_boursier2019sic,Arxiv2019_HeterogeneousNetworks}, we assume players are synchronized and know arms' indices before entering into the network. Further, all the players are assumed to see the same gains on all the channels, which is often the case in dense networks because of the close proximity of users. 
% In the next section we give regret guarantees knowing the value of $T$ and later in Section \ref{sec:AlgoSensing} we relax this condition.

%% file: chapter5/dlc.tex
%!TEX root =  ../thesis.tex

In this section, we propose an algorithm named, Distributed Learning and Coordination (DLC), in which all the players settle on top $N$ arms quickly. The players cannot communicate explicitly but can exchange information with each other by colliding in a particular fashion. We refer to such deliberate collisions in the network for information exchange as {\em signals}. All signals are counted as collisions and add to the regret. 

\subsection{{\normalfont{DLC}}  Algorithm}
\ref{DLC} algorithm consists of mainly $4$ phases, namely 1) Orthogonalization 2) Player Indexing 3) Adaptive Learning and 4) Communication. These phases run sequentially one after another. 
\begin{algorithm}[!ht]
	\renewcommand{\thealgorithm}{DLC}
	\floatname{algorithm}{}
	\caption{\textbf{ D}istributed \textbf{L}earning and \textbf{C}oordination} 
	\label{DLC}
	\begin{algorithmic}[1]
		\State	Input: {$K,\; \epsilon\geq 0,\; \delta \in (0,1/2]$} 

		\State Select the arms uniformly at random for $T_{RP}$ rounds and find reserved arm $i$
		\State Play arm $i$ for $2i$ rounds. After that sequentially hop for $K - i$ rounds and then play arm $i$ for  $K - i - 1$ rounds
		\State Find number of players ($N$) and their reserved arm. Designate as Leader if has smallest reserved arm's index
		\If {Player is Leader}
		\State Find top $N$ arms using {\em AdaptiveExplore} sub-routine
		\State Assign top $N$ arms among players without overlap
		\Else
		\State Transmit on arm $i$ periodically. When a collision is observed on arm $i$, play it for next $\ceil{\log_2(K) }$ rounds and then play received arm in the subsequent rounds. 
		\EndIf
	\end{algorithmic}
\end{algorithm}

\subsubsection{Orthogonalization}
The first phase of \ref{DLC} finds orthogonal arm allocations through random hopping in which each player selects an arm uniformly at random in each round until she observes a collision-free transmission on the selected arm. Once it happens, she continues to play that arm till the end of the phase. The length of {\em Orthogonalization} phase $(T_{RP})$ is set such that all the players are on different arms by the end of the phase with probability at least $1-\delta$. We refer to the last arm played by a player in this phase as her {\em reserved arm}.

\subsubsection{Player Indexing}
This phase is similar to the initialization phase in SIC-MMAB algorithm \citep{NeurIPS19_boursier2019sic}. In this phase, each player estimates how many other players are there in the network and finds their reserved arms using a specific pattern of collisions. A player with reserved arm $i$ plays arm $i$ for $2i$ rounds. After this, she starts playing arm with higher index in each round (sequential hopping)  for next $K - i$ rounds and then plays her reserved arm for next $K - i - 1$ rounds. This process makes sure that players begin sequential hopping in a delayed fashion and the player with reserved arm $i$ will collide with the player with the reserved arm $j (>i)$ at time $T_{RP} + i + j$. At the end of round $T_{RP} + 2K - 1$, each player knows the number of players $(N)$ in the network and their respective reserved arms. 

The player with smallest reserved arm's index will be designated as {\em Leader}. The Leader knows the reserved arm of other players and she assigns a ranking to the players based on the index of their reserved arms as follows: The Leader takes herself rank $0$. The player with the second smallest index of the reserved arm is assigned rank $1$. The third smallest index of the reserved arm has rank $2$ and so on, i.e.,  the player on $j^{th}$ smallest index of the reserved arm is given rank $j-1$.

\subsubsection{Adaptive Learning}
In the next phase, the Leader uses the {\em AdaptiveExplore} sub-routine to find the top $N$ arms while the non-Leaders check for a signal from the Leader by transmitting on their reserved arms periodically. Periodic transmissions of all players in the non-Leaders set are designed such that only one player among the non-Leaders transmits at a time. It helps the Leader to inform the non-Leaders which arms they should occupy when she has completed the task of identifying the top $N$ arms. We first describe a method for the Leader to learn the top $N$ arms and then describe a method how she can inform the Non-Leaders which arms to occupy. The Leader can use any pure exploration algorithm \citep{COLT13_Bandit_Subset_Selection,ICML12_PAC_Subset_Selection, COLT2013_LiLUCB_JamiesonMalloyNowakBubeck} to find the top $N$-arms.  We will adapt the KL-LUCB algorithm given in \cite{COLT13_Bandit_Subset_Selection} to our scenario as it has the best-known performance guarantees. 

\begin{algorithm}[!ht]
	\caption*{\textbf{Sub-routine:} {\em AdaptiveExplore}}
	\label{alg:kl-lucb}
	\begin{algorithmic}[1]
		\State Input : $K,\; N \le K,\; \epsilon\geq 0,\; \delta \in (0,1/2], t=1$
		\State Initialize: $\alpha = 1.1,\; k_1 = 505.5,\; B(1)=\infty$ 
		\State Play each arm once. Compute $U_a(1), L_a(1) \; \forall  \; a\in[K]$
		\While {$B(t)>\epsilon/N$}
		\State Compute $u_t$ and $l_t$ using \eqref{equ:arm_selection} and play arm $u_t$ and $l_t$ in round-robin fashion. If no non-Leader is going to play selected arm then play it else play other arm
		\State Update $J(t)$, $U_a(t)$ and $L_a(t)\; \forall a \in [K]$ using \eqref{equ:KLBounds}
		\State  $B(t) \leftarrow U_{u_t}-L_{l_t}, t \leftarrow t+2$
		\EndWhile
		\State $T_{AE} \leftarrow  t$
		\State Return $J(T_{AE}),\; T_{AE}$ 
	\end{algorithmic}
\end{algorithm}

AdaptiveExplore takes $(K,N,\epsilon, \delta)$ as input and maintains a set $J(t)$ of the $N$ arms with highest empirical mean rewards in each round $t$. $U_a(t)$ and $L_a(t)$ represent the upper and lower confidence bounds on a mean reward $\mu_a$ of arm $a$. These bound are computed using empirical mean reward $\hat \mu_a(t)$ of arm $a$ at time $t$ as follows:
\begin{align}
	U_a(t) :=& \max\{q\in[\hat\mu_a(t),1]: N_a(t)d(\hat\mu_a(t), q) \le \beta(t,\delta)\} \nonumber\\
	L_a(t) :=& \max\{q\in[0,\hat\mu_a(t)]: N_a(t)d(\hat\mu_a(t), q) \le \beta(t,\delta)\} \nonumber\\
	\quad \mbox{where } &\beta(t,\delta) = \log\left(\frac{k_1Ks^\alpha}{\delta}\right) + \log \log\left(\frac{k_1Ks^\alpha}{\delta}\right), \nonumber\\
	\alpha > 1,\; &k_1 > 2e + 1+ \frac{e}{\alpha-1} +  \frac{e+1}{(\alpha-1)^2}, s=\ceil{t/2} ,\label{equ:KLBounds}
\end{align}
and $d(p,q)$ is the Kullback-Leibler divergence (differential entropy) between two Bernoulli distributions with parameter $p$ and $q$, and $\delta$ is the fixed confidence.

\noindent
In round $t$, AdaptiveExplore selects two arms $u_t$ and $l_t$ such that 
\begin{align}
	\label{equ:arm_selection}
	u_t = \arg\max_{j\notin J(t)} U_j(t) \mbox{ and } l_t = \arg\min_{j\in J(t)} L_j(t).
\end{align}
AdaptiveExplore finds $(\epsilon,\delta)$-optimal allocation if it terminates after: 
\begin{align}
	\label{equ:stoppingCriteria}
	\setlength{\fboxsep}{5pt}
	\fbox{$U_{u_t}(t) - L_{l_t}(t) < \epsilon/N$}
\end{align}
where $\epsilon \ge 0$ is fixed tolerance. Its proof follows from Theorem 1 of  \cite{ICML12_PAC_Subset_Selection} by setting tolerance equal to $\epsilon/N$. 

AdaptiveExplore selects two arms in each alternative rounds, and they are played over two slots in a round-robin fashion until stopping criteria in \cref{equ:stoppingCriteria} is not met. In each round, before playing an arm, the Leader checks that none of the non-Leaders will play that arm in that round. If any non-Leader is going to play the selected arm (due to periodic signals), Leader plays the other arm to avoid any collision. The periodic signaling by the non-leaders is designed such that they play only one arm in any round.

When AdaptiveExplore terminates, it returns the top $N$ arms and number of rounds $(T_{AE})$ it takes to terminate.  Next, Leader does a one-to-one mapping of top $N$ arms to the $N$ players. This mapping can be done randomly or by using an injective function.  In our implementation, we assign the top arm to the player with the best rank and the next best arm to the player with next best rank and so on. The Leader takes the $N^{th}$ best arm. After mapping arms to players, Leader informs other players which arm they should occupy using collisions based communication protocol. 

Note that mapping of top $N$ arms to players can be easily adapted to be `fair' allocation amongst players so that all the players have same average reward asymptotically.  For example, after the Leader communicates the allocations, they can sequentially hop on the $N$ arms, which ensures no collision and optimal allocation in every round.

\subsubsection{Communication}
While the Leader explores the arms to find the top $N$ arms, others play their reserved arms periodically. Each non-Leader plays her reserved arm based on her rank. A player with better rank plays earlier than the higher ranked players. Specifically, a non-Leader ranked $r$ plays her reserved arm for the $m^{th}$ time in the round $P(r, m)$ given by 
\begin{align*}
	&P(r, m) = mT_P+ (r - 1) B_K\\
	&\quad\mbox{where } B_K = (\ceil{\log_2(K)} + 1) \mbox{  and  }T_P \ge (N-1)B_K.
\end{align*}
Once AdaptiveExplore sub-routine is initiated by the Leader, the player with rank 1 plays her reserved arm once after $T_P$ rounds, and then the player with next rank (i.e., 2) plays her reserved arm. The value of $T_P$ is chosen such a way that all players can play their reserved arm and receive arm information from Leader within it, therefore, $T_P\ge (N-1)(\ceil{\log_2K} + 1)$.  This signaling strategy enables the Leader to transmit arm information to a player immediately after she finishes it for the previous player and continues until every player gets their mapped arm from the Leader. 

After termination of AdaptiveExplore sub-routine, i.e., at time $T_{AE}$, Leader computes the next time when a player with rank $r$ will play her reserved arm, denoted $P_L(r)$, as $P_L(r)=P(r, \tau)$ if $P(r, \tau) > T_{AE}$, otherwise $P_L(r)=P(r, \tau + 1)$ where $\tau = \floor{T_{AE}/T_P}$. When the Leader wants to send arm information to the player with rank $r$, she will play the reserved arm of that player in round $P_L(r)$. The occurrence of the first collision notifies the player that she is going to get information of the arm to occupy from the Leader and starts playing her reserved arm for $\ceil{\log_2K}$ rounds.  As the index of each arm is uniquely represented by the binary number of length $\ceil{\log_2K}$,  these many rounds are sufficient to convey the arm index. The Leader informs a non-Leader  the index of the arm to occupy by colliding with her according to the binary sequence of the arm's index over $\ceil{\log_2K}$ rounds. Leader conveys bit `$1$' by causing a collision on her reserved arm, while a bit `$0$' is conveyed if no collision occurs. With such a collision pattern, all players get to know which arms they should occupy in the subsequent rounds.

\subsection{Analysis of \normalfont{DLC}}
We prove sample complexity, regret and collision bounds of \ref{DLC} algorithm using the following lemmas which give expected number of rounds for players to 1) orthogonalize 2) learn the number of players in the system with their ranks 3) find top $N$ arms and 4) arm assignment through signaling.

\begin{thm}
	\label{thm:dlc_results} 
	Let arms be ordered in descending order of their mean rewards and set $c = (\mu_{N} + \mu_{N+1})/2$ where $c \notin \{0,1\}$. For any $\epsilon \ge 0$ and $\delta \in (0,1/2]$ the sample complexity of {\normalfont{\ref{DLC}}} is bounded with probability at least $1-2\delta$ as
	\begin{align*}
	T(DLC) \le &   \ceil{\frac{\log(\delta/K)}{\log\left(1-\frac{1}{4K}\right)}} + 2K-1 + C_0(\alpha)H_{\epsilon,c} \\
	& \log\left(\frac{k_1K(H_{\epsilon,c})^\alpha}{\delta}\right) + N(\ceil{log_2(K) + 1}
	\end{align*}
	where 	$H_{\epsilon,c} := \sum_{a\in [K]} \frac{1}{\max\{d(\mu_a,c), \epsilon^2/2\}}$, $k_1 > 2e + 1+ \frac{e}{\alpha-1} +  \frac{e+1}{(\alpha-1)^2}$ with $\alpha>1$ and $C_0(\alpha)$ is a problem independent constant.
	Further, the total collisions in {\normalfont{\ref{DLC}}} is bounded as	
	\begin{align*}
	C(DLC) \hspace{-0.5mm}\le\hspace{-0.5mm} N\hspace{-0.65mm} \ceil{\hspace{-0.75mm} \frac{\log(\delta/K)}{\log\left(1-\frac{1}{4K}\right)}\hspace{-0.75mm}} \hspace{-0.5mm}+\hspace{-0.5mm} (\hspace{-0.35mm} N \hspace{-0.5mm} -\hspace{-0.5mm} 1\hspace{-0.35mm} )\hspace{-0.75mm} \left(\hspace{-0.65mm} \frac{N}{2} \hspace{-0.5mm}+\hspace{-0.5mm} \ceil{\log_2K} \hspace{-0.5mm}+\hspace{-0.65mm} 1\hspace{-0.75mm}\right)\hspace{-0.75mm}.
	\end{align*}
\end{thm}
The proof follows by bounding the number of rounds required to complete each phase and the number of collisions incurred in these phases. These bounds are given by the following lemmas whose proofs are deferred to \cref{sec:mpb_appendix}.
\begin{lem}
	\label{lma:OrthoArms}
	Let $\delta$ be same as in Theorem \ref{thm:dlc_results}. All players will orthogonalize with probability at least $1-\delta$ in Orthogonalization phase within $T_{RP}$ rounds where
	\begin{align*}
	T_{RP}:= \ceil{\frac{\log(\delta/K)}{\log\left(1-\frac{1}{4K}\right)}}.
	\end{align*}
	The number of collisions in the Orthogonalization phase is at most $C_{RP} \le NT_{RP}$.
\end{lem}  

\begin{lem}
	\label{lma:IndexPlayer}
	In the Player's Indexing phase, each player learns about number of players in the network, their reserved arm, and rank  in $T_{IP}$ rounds where  $T_{IP} := 2K-1.$ During this phase, the number of collisions is $C_{IP} := N(N-1)/2.$
\end{lem}

\begin{lem}
	\label{lma:samples_AE}
	Assume conditions  in Theorem \ref{thm:dlc_results}  hold. Then with probability at least $1-\delta$,  AdaptiveExplore terminates with $(\epsilon,\delta)$-optimal allocation after $T_{AE}$ rounds where
	\begin{align*}
	T_{AE} &\le C_0(\alpha)H_{\epsilon,c}\log\left(\frac{k_1K(H_{\epsilon,c})^\alpha}{\delta}\right)
	\end{align*}
\end{lem}
where $\alpha>1$$,  C_0(\alpha)$ is a problem independent constant satisfying $C_0(\alpha) \ge \alpha\log(C_0(\alpha)) + 1 + \alpha/e,k_1 > 2e + 1+ \frac{e}{\alpha-1} +  \frac{e+1}{(\alpha-1)^2}$, and $H_{\epsilon,c} := \sum_{a\in [K]} \frac{1}{\max\{d(\mu_a,c), \epsilon^2/2\}}$.
\begin{proof}
	KL-LUCB algorithm samples two arms in each round whereas AdaptiveExplore sub-routine samples one arm. Using $t = 2s$ in the Theorem 3 of \cite{COLT13_Bandit_Subset_Selection}, we achieve stated bound.
\end{proof}

%When a non-Leader is in {\em SignalLeader} phase, she plays her {\em reserved arm} after $T_P$ rounds. Therefore, the Leader needs to wait for a maximum of $T_P$  before she can inform the non-Leader to indicate which arm to occupy in the subsequent rounds.

%Let us define a function $f_o: \{N_1, N_2\} \rightarrow \mathbb{I}$ for $N_1 < N_2$ that returns total number of bits which are one in the binary representation of integer from $N_1$ to $N_2$. For example $f_o(0,2) = 2$ and $f_o(2, 7) = 11$.
\begin{lem}
	\label{lma:signaling}
	Let $T_P = (N-1)(\ceil{\log_2K} + 1)$. Then the number of rounds to terminate the Communication phase ($T_{CP}$) and the number of collisions incurred ($C_{CP}$) are bounded as
	\begin{align*}
	T_{CP} &\le N(\ceil{log_2(K)} + 1),\\
	C_{CP} &\le (N - 1)(\ceil{\log_2K} + 1). 
	\end{align*} 
\end{lem}
We now return to the proof of Theorem \ref{thm:dlc_results}. $T(DLC)$ is the total rounds from the beginning of Orthogonalization phase to end of Communication phase, i.e.,
\begin{align}
\label{equ:sample_comp}
T(DLC) = T_{RP} + T_{IP} + T_{AE} + T_{CP}.
\end{align}	
Substituting the bounds on each of these terms  from Lemmas \ref{lma:OrthoArms}, \ref{lma:IndexPlayer}, \ref{lma:samples_AE} and \ref{lma:signaling} in \eqref{equ:sample_comp} we get the bound.

Similarly, the total number of collisions is the sum of that occurred during Orthogonalization phase, Player Indexing phase, and Communication phase. Note that there is no collision in the Adaptive Learning phase. Therefore, the total collisions of \ref{DLC} are given by
\begin{align}
\label{equ:totalCollisions}
C(DLC) = C_{RP} + C_{PI} + C_{CP}.
\end{align}
Substitute the number of collisions in different phases from Lemmas \ref{lma:OrthoArms}, \ref{lma:IndexPlayer}, and \ref{lma:signaling} in \eqref{equ:totalCollisions} we get the stated bound. We next bound the regret of DLC.
\begin{thm}
	\label{cor:regretDLC}
	Let the conditions in Theorem \ref{thm:dlc_results} hold. Then with probability at least $1-2\delta$, the regret of {\normalfont{\ref{DLC}}} is bounded by
	\begin{align*}
	\Regret (DLC) \le &NT(DLC).
	\end{align*}
\end{thm}
\begin{proof}
	As $\mu_i \in [0, 1]$ for arm $ i \in [K]$, the maximum regret each player can have for any round is 1. The maximum regret for any round is bounded by $N$. Hence the regret of {\normalfont{\ref{DLC}}} is upper bounded by $NT(DLC)$.
\end{proof}
\begin{cor}
	Let the conditions in Theorem \ref{thm:dlc_results} hold. The regret of $DLC$ is bound with probability $1-2\delta$ is bounded as
	\begin{equation*}
	\	Regret (DLC) = \mathcal{O}(NK \log (K/\delta)/\Delta^2 )
	\end{equation*}
	where $\Delta=\mu_N - \mu_{N+1}$ is the smallest sub-optimality gap.
\end{cor}

Note that the minimizing of the sample complexity leads to lower regret for \ref{DLC}. Further, the regret of \ref{DLC} is linearly depended on the number of players and number of arms.
%(ignoring $\log$ term).}
%The proofs of  Lemmas \ref{lma:OrthoArms}, \ref{lma:IndexPlayer}, and \ref{lma:signaling} are deferred to \S\ref{sec:appendix}.
%

\subsection{Regret of \normalfont{\ref{DLC}} vs MC}
In the \ref{DLC}, the number of rounds to complete all phases except Adaptive Learning phase are independent of the gap between the expected reward of the $N^{th}$ best arm and the $(N+1)$ best arm. AdaptiveExplore sub-routine of \ref{DLC} is adapted from KL-LUCB algorithm that has the best-known performance guarantees for identifying the best $N$ arms.

The regret of state-of-the-art MC algorithm \citep[Theorem 1]{ICML16_MultiplayerBandits_RosenkiShamir} is bounded by $O\left({N^2K\log(K^2/2\delta)}/{\epsilon^2_n}\right)$, with probability at least $1-2\delta$ where $\epsilon_n$ is the lower bound on the gap between the mean reward of the $N^{th}$ best arm and the $N + 1$ best arm, i.e.,  $\epsilon_n < \Delta$. The value of $\epsilon_n$ is assumed to be known. The regret bound of \ref{DLC} improves the bound of MC by a factor $N$ and it is agnostic to problem-specific information like lower bound $(\epsilon_n)$. Its regret bound \ref{DLC} depends on the mean reward of the arms of the given problem instance. Thus performance gain of the DLC is compared to MC is significant.  As we will see later in the experimental section, the performance of DLC is an order of magnitude better than the MC algorithm. This improved performance gain is mainly due to the long random hopping phase of MC where users are made to collide for the large number of rounds which is used to estimate the number of users in the network.

%% file: chapter5/dlc_lix.tex
%!TEX root =  ../thesis.tex

In \ref{DLC} algorithm, when the Leader explores the arms, all other players transmit only periodically and hence do not get any reward for other rounds until they are assigned an arm by the Leader. However, it was necessary for players other than the Leader to be idle so that the Leader can explore the arms freely in the Adaptive Learning phase of \ref{DLC}, leading to under-utilization of arms and poor performance in term of regret. We next discuss a modification of \ref{DLC} where the non-Leaders do not need to be idle while the Leader explores top $N$ arms, which allows efficient utilization of arms resulting in the improved regret performance.
%DLC algorithm has more regret than the algorithm which allows players to samples their reserved arm when Leader is learning Top-$N$ arm.  

\subsection{{\normalfont{DLC}} with Limited Information Exchange {\normalfont{(DLC-LIX)}}}
In this section, we give a variant of \ref{DLC} called  DLC with Limited Information Exchange (\ref{DLC-LIX}) which allows the non-Leader to play their reserved arms while the Leader explores the unreserved arms and share with her the information they have gathered about their reserved arms, thus the Leader may not be needed to explore much the reserved arms. The \ref{DLC-LIX} consists of $3$ phases: 1) Orthogonalization 2) Player Indexing and 3) Adaptive Learning with Communication (ALC).

The first two phases of \ref{DLC-LIX} are the same as in \ref{DLC}.  Once Leader is selected, other players continue to play their reserved arms until the Leader tells her to change the arm by sending signals. The Leader partitions the arms into two sets before entering to Adaptive Learning with Communication phase, namely {\em Reserved} and {\em Unreserved} arms.  The set of the reserved arm of all the players except the Leader forms the Reserved set and the other set of unoccupied (empty) arms together with the reserved arm of the Leader forms the Unreserved set. 

Now the Leader explores only the Unreserved set to find the best arm using the AdaptiveExplore sub-routine with $K-N+1$ arms. Once AdaptiveExplore sub-routine terminates, Leader has the best arm among Unreserved set of arms. The Leader then collides on the arm that is the reserved arm for the player with rank $r$.  When the player observes a collision, it acts as a signal to transfer collected reward information of her reserved arm to Leader. The rewards are real-valued in $[0,1]$ and communicating them using with arbitrary precision requires binary codes of large lengths, we thus allow the players to exchange average reward within a known fixed error. Specifically, we set the error to be  $\epsilon_p$, i.e., $(R_{T_{c,r}} - \hat{R}_{T_{c,r}})/T_{c,r} \le \epsilon_p$
%\begin{align*}
%	R_{T_{c,r}} - \hat{R}_{T_{c,r}} \le \epsilon_p
%\end{align*}
where $T_{c,r}$ is the time taken by a player with rank $r$ to observe a collision after the end of the Player Indexing phase, $R_{c}$ is the total reward collected by the player before observing collision and $\hat{R}_{c}$ is the total reward received by the Leader.  The received reward $\hat{R}_{c}$ is always less than (or equal to) the collected reward $R_{c}$ as some fractional part may get truncated due to the finite fixed precision. The number of rounds needed to transfer average reward information to Leader within $\epsilon_p$ error is $ \ceil{\log_2(1/\epsilon_p)}$\footnote{Note that for binary rewards $\{0,1\}$ case, any player can transfer exact collected reward information to the Leader in the $\ceil{\log_2T_{c,r}}$ rounds where the value of $\epsilon_p$ for player with rank $r$ is $\epsilon_p(r) := 1/T_{c,r}$.}. In our experiments, we have set $\epsilon_p = \epsilon/2$. 

\begin{algorithm}[!ht]
	\renewcommand{\thealgorithm}{DLC-LIX}
	\floatname{algorithm}{}    
	\caption{\textbf{ DLC} with \textbf{L}imited \textbf{I}nformation E\textbf{x}change}
	\label{DLC-LIX}
	\begin{algorithmic}[1]
		\State	Input: {$K,\; \epsilon\geq 0,\; \delta \in (0,1/2]$} %\; T_{RH},
		\State Select the arms uniformly at random for $T_{RP}$ rounds and settle on arm $i$
		\State Play arm $i$ for $2i$ rounds. After that sequentially hop for $K - i$ rounds and then play arm $i$ for  $K - i - 1$ rounds
		\State Find number of players ($N$) and their reserved arms. Designate as Leader if has smallest reserved arm's index%Player with the smallest index of the reserved arm is the Leader
		\If {Player is Leader}% {($n==1$)}
		\For{$r = 1,2, \ldots, N-1$}
		\State Find the best arm using {\em AdaptiveExplore} sub-routine in the Unreserved set of arms
		\State Signal player with rank $r$. Collect rewards information of her reserved arm and  assign the current best arm to her
		\State Updates the Reserved and Unreserved set of arms
		\EndFor
		\State Find the best arm using {\em AdaptiveExplore} sub-routine in the Unreserved set of arms 
		\State Play the best arm in subsequent rounds
		\Else
		\State Play arm $i$. When a collision is observed, send collected rewards information to Leader and after that received a arm to play in next $\ceil{\log_2(K) }$ rounds
		\State Play the received arm in subsequent rounds
		\EndIf
	\end{algorithmic}
	%	\end{spacing}
\end{algorithm}

After receiving the mean reward from the player with rank $r$, Leader recomputed the best arm. If the best arm is same as the reserved arm of player, then Leader will not play it for next $\ceil{\log_2K}$ rounds. Otherwise, she collides with the player for next $\ceil{\log_2K}$ rounds as per the binary coding of arm index the player should take. In the former case, the player will continue to play her reserved arm in subsequent rounds. In the latter case, the player moves to the new arm assigned to her. The Leader then removes the just assigned arm from the Unreserved set and moves it to the Reserved set. The arm that is left unoccupied by this shifting is moved to the Unreserved set. 
%play reserved arm of player with rank $r+1$ and then follows the same process.
%If the best arm is same as the reserved arm of player, then Leader will not play it for next $\ceil{\log_2K}$ rounds that implies the reserved arm of player with rank $r$ is one of the top $N$ arms. The player will continue to play her reserved arm for subsequent rounds while Leader will play reserved arm of player with rank $r+1$ and the same process follows.
%In case of the reserved arm is not the best arm, Leader transfers best arm information by encoding its index into $\ceil{\log_2K}$ bits with suitable collisions to the player in next $\ceil{\log_2K}$ rounds. 
%After transferring of last bit, Leader adds the best arm to Reserved set and reserved arm of the player with rank $r$ into Unreserved set. 
The Leader then recomputes the best arm using AdaptiveExplore sub-routine and repeats the process with the next ranked player. This process stops when Leader has allocated arms to all non-Leaders from top $N$ arms. Finally, the Leader finds the best arm among Unreserved set for herself by using AdaptiveExplore sub-routine. At the end of this process, each player has been assigned an arm from top $N$ arms and remaining $K - N$ arms are left in the Unreserved set.

%For keeping track of occupied and unoccupied arms, Leader creates two sets after RandomHop phase: The {\em Deputies} is the set of the reserved arm of all players except Leader and the other is {\em SArms}; the set of unoccupied (empty) arms and the reserved arm of Leader. 
%
%Now, the Leader explores only the SArms set to find the best arm  {\em (BestArm)} using the AdaptiveExplore phase with $N=1$ and $K^\prime = K-N+1$. Once this phase terminates, Leader has BestArm among SArms and then finds out Best-Deputy who has the lowest index for the reserved arm in Deputies set and removes the reserved arm of Best-Deputy from Deputies set. After this, the Leader plays the reserved arm of Best-Deputy and Best-Deputy observes a collision which acts as a signal for her that she is going to get one of the Top-$N$ arms in next $\ceil{\log_2K}$ rounds. Then, Leader removes BestArm from SArms set and changes her reserved arm to that of Best-Deputy. Since, Leader does not have good confidence on the new reserved arm, she will play it more often than other arms in SArms set. This process stops when Deputies set has no more arm left.  Finally, Leader finds the best arm among SArms for herself by going to AdaptiveExplore phase one more time. By the end of this process, each player has been assigned an arm from Top-$N$ arms and SArms is left with remaining $K-N$ arms.

\subsection{Analysis of \normalfont{DLC-LIX}}
%!TEX root =  main.tex
We need a variant of Lemma \ref{lma:samples_AE} to prove sample complexity, regret and collision bounds of the \ref{DLC-LIX}. 

\begin{lem}
	\label{lem:dlcIXLearning}
	%	Let arms be in descending order by their mean rewards i.e., $\mu_a > \mu_{a+1} \ge 0$ for any $a \in [K-1]$, $m = \arg\min\limits_{a\in [K]}(\mu_a - \mu_{a+1})$, $c = (\mu_{m} + \mu_{m+1})/2$, tolerance $\epsilon \ge 0$, $K\ge N$, and $s=\ceil{t/2}$. 
	%	Let $\beta(t,\delta) = \log\left(\frac{k_1Ks^\alpha}{\delta}\right) + \log \log\left(\frac{k_1Ks^\alpha}{\delta}\right)$ with $k_1 > 2e + 1 +\frac{e}{\alpha - 1} + \frac{e+1}{(\alpha-1)^2}$ with $\alpha > 1$.  
	Assume technical conditions stated in Theorem \ref{thm:dlc_results} hold. Let $\epsilon_p = \epsilon/2$, $c_a:=(\mu_{a} + \mu_{a+1})/2$, and $m = \arg\min\limits_{a\in [K-1]}\big[\min\{d(\mu_a,c_a), $ $d(\mu_{a+1},c_a)\}\big]$.
	For any $\epsilon \ge 0$ and $\delta \in (0,1/2]$, the ALC phase of \ref{DLC-LIX} terminates with $(\epsilon,\delta)$-optimal allocation after $T_{ALC}$ rounds with probability at least $1-\delta$ where 
	%	For any $\epsilon \ge 0$ and $\delta \in (0,1/2]$, the number of rounds needed to terminate the ALC phase of \ref{DLC-LIX} with probability at least $1-\delta$ is bounded as
	% the number of rounds needed to terminate the Adaptive Learning with Communication phase of {\normalfont{\ref{DLC-LIX}}} with $(\epsilon, \delta)-$optimal allocation is bounded by $T_{ALC}$ where
	\begin{align*}
	&T_{ALC} \le C_0(\alpha)H_{\epsilon,c_m}\log\left(\frac{k_1K(H_{\epsilon,c_m})^\alpha}{\delta}\right)\\
	&\quad +  (N-1)\left(\ceil{\log_2\left(1/\epsilon\right)} +\ceil{\log_2(K)}\right)
	%		&\quad\qquad \mbox{where } H_{\epsilon,c_m} := \frac{K}{\max\{d(\mu_m,c_m), \epsilon^2/2\}},\\
	%		&\qquad\qquad C_0(\alpha) \ge \alpha\log(C_0(\alpha)) + 1 + \frac{\alpha}{e}. % \mbox{, and}
	\end{align*}
	where $\alpha>1$$,  C_0(\alpha)$ is a problem independent constant satisfying $C_0(\alpha) \ge \alpha\log(C_0(\alpha)) + 1 + \alpha/e,k_1 > 2e + 1+ \frac{e}{\alpha-1} +  \frac{e+1}{(\alpha-1)^2}$, and
	$H_{\epsilon,c_m} := \frac{K}{\max\{d(\mu_m,c_m), \epsilon^2/2\}}$. % \mbox{, and}
	%	 $T_{c,r}$ is the time taken for observing the first collision by a player with rank $r$ after starting of Adaptive Learning with Communication phase.
\end{lem}

\begin{proof}
	In case of identifying top $N$ arms by AdaptiveExplore sub-routine, the problem dependent variable $H_{\epsilon,c}$ is defined as
	\begin{align}
	H_{\epsilon,c} &= \sum\limits_{a\in [K]} \frac{1}{\max\{d(\mu_a,c), \epsilon^2/2\}}. \nonumber
	\intertext{Since $\forall a \in [K]$, $d(\mu_{m},c_m) \le d(\mu_a,c)$,} 
	H_{\epsilon,c}&\le \frac{K}{\max\{d(\mu_m,c_m), \epsilon^2/2\}} = H_{\epsilon,c_m}. \label{equ:UBsamples}
	\end{align}
	Adaptive Learning with Communication phase of {\normalfont{\ref{DLC-LIX}}} identifies top $N$ arms one after other from a subset of arms (Unreserved set). It is similar to solving best arm identification problem by $N$ number of times. 	
	By allowing information exchange between Leader and non-Leaders, this problem becomes equivalent to solve top $N$ arms identification problem except the problem dependent variable $H_{\epsilon,c}$ can change for the given subset of arms. Therefore, we consider worst problem dependent variable $H_{\epsilon,c_m}$ that depends on the smallest gap between the mean reward of any two arms instead of the $N^{th}$ best arm and $N+1$ best arm.
	
	The maximum rounds needed for transferring collected mean reward of arms by non-Leaders to Leader are $\sum_{r=1}^{N-1}\ceil{\log_2\left({2}/{\epsilon}\right)}$.
	%	\begin{align}
	%	\label{equ:infoExRounds}
	%		\sum_{r=1}^{N-1}\ceil{\log_2\left(\frac{2}{\epsilon}\right)}.
	%	\end{align}
	Further, Leader also needs $\ceil{\log_2(K)}$ rounds to assign a arm to a non-Leader. Combine these facts with  \eqref{equ:UBsamples}, we get above stated bound.
\end{proof}

\begin{thm}
	\label{thm:dlcIX_results} 
	Assume technical conditions stated in Lemma \ref{lem:dlcIXLearning} hold. For any $\epsilon \ge 0$ and $\delta \in (0,1/2]$, the sample complexity of {\normalfont{\ref{DLC-LIX}}}  with probability at least $1-2\delta$ is bounded as
	\begin{align*}
	&T(\mbox{DLC-LIX}) \le    \ceil{\frac{\log(\delta/K)}{\log\left(1-\frac{1}{4K}\right)}} + 2K-1  \\ 
	&\qquad + C_0(\alpha)H_{\epsilon,c_m}\log\left(\frac{k_1K(H_{\epsilon,c_m})^\alpha}{\delta}\right) \\
	&\qquad +  (N-1)\left(\ceil{\log_2\left(2/\epsilon\right)} +\ceil{\log_2(K)}\right).
	%	&\quad\qquad\mbox{where } H_{\epsilon,c} := \sum\limits_{a\in [K]} \frac{1}{\max\{d(\mu_a,c), \epsilon^2/2\}} \\
	%	&\qquad\qquad\mbox{and } C_0(\alpha) \ge \alpha\log(C_0(\alpha)) + 1 + \frac{\alpha}{e}.
	\end{align*}
	Further with probability at least $1-2\delta$, the regret of {\normalfont{\ref{DLC-LIX}}} is bounded as 
	\begin{align*}
	\Regret(\mbox{DLC-LIX}) \le &NT(\mbox{DLC-LIX}).
	\end{align*}
\end{thm}
\begin{proof}
	%	Let $T_{DLC-LIX}$ denote the number of rounds Algorithm {\normalfont{\ref{DLC-LIX}}} takes for assigning $(\epsilon,\delta, N)$-optimal allocation to the $N$ players. $T_{DLC-LIX}$ is the total rounds taken by Leader are
	The total number of rounds Algorithm {\normalfont{\ref{DLC-LIX}}} takes for assigning $(\epsilon,\delta)$-optimal allocation to the $N$ players is 
	\begin{align}
	\label{equ:sampleCompIX}
	T(\mbox{\it DLC-LIX}) = T_{RP} + T_{IP} + T_{ALC}.
	\end{align}
	Substitute the number of rounds needed for different phases from Lemmas \ref{lma:OrthoArms}, \ref{lma:IndexPlayer} and \ref{lem:dlcIXLearning} in \eqref{equ:sampleCompIX} to get above stated bound of $T(\mbox{\it DLC-LIX})$. As the maximum regret for any round is $N$, the regret of the {\normalfont{\ref{DLC-LIX}}} is upper bounded by $NT(\mbox{\it DLC-LIX})$.
\end{proof}

\begin{thm}
	Let conditions of Lemma \ref{lem:dlcIXLearning} hold.
	Then with probability at least $1-2\delta$, the maximum number of collisions that can occur in {\normalfont{\ref{DLC-LIX}}} is $C(\mbox{DLC-LIX})$ where
	\begin{align*}
	& C(\mbox{DLC-LIX}) \le N\ceil{\frac{\log(\delta/K)}{\log\left(1-\frac{1}{4K}\right)}} + \frac{N(N-1)}{2} \\
	&\quad +  (N-1)\left( \ceil{\log_2\left(2/\epsilon\right)} +\ceil{\log_2(K)}\right).
	\end{align*}
\end{thm}

\begin{proof}
	The total collisions of \ref{DLC-LIX} is 
	% 	the sum of the number of collisions occur during Orthogonalization phase, Player's Indexing phase, and Adaptive Learning Communication phases which 
	is given by
	\begin{align}
	\label{equ:totalCollisionsIx}
	C(\mbox{\it DLC-LIX}) = C_{RP} + C_{PI} + C_{ALC} 
	\end{align}
	The maximum number of collisions observed in Adaptive Learning with Communication phase of \ref{DLC-LIX} is upper bounded by $C_{ALC}$ where
	\begin{align*}
	C_{ALC} \le (N-1)\left(\ceil{\log_2\left(2/\epsilon\right)} +\ceil{\log_2(K)}\right)
	\end{align*}
	Add the number of collisions occurs in other phases from Lemmas \ref{lma:OrthoArms} and \ref{lma:IndexPlayer} with $C_{ALC}$ in \eqref{equ:totalCollisionsIx} to get stated bound.
\end{proof}

\begin{thm}
	Let $T_P$ be the rounds for which non-Leader waits to play her reserved arm in Adaptive Learning phase of \ref{DLC} and $R_L$ is the set of reserved arms of non-Leaders after the end of the Player Indexing phase. If $T_{L}$ is the rounds spent for exploration of top $N$ arms by both \ref{DLC} and \ref{DLC-LIX} then
	\begin{align*}
	\Regret(DLC)_{T_L} - \Regret(\mbox{DLC-LIX})_{T_L} \ge \left(T_{L} - \ceil{\frac{T_{L}}{T_P}}\right)\sum_{n \in R_L} \mu_n. 
	\end{align*}
\end{thm}
\begin{proof}
	It follows from the fact that \ref{DLC} needs non-Leaders play periodically in the learning phase so that the Leader can explore the arms freely whereas \ref{DLC-LIX} allows non-Leaders continue to play their reserved arms in the learning phase.
\end{proof}

%% file: chapter5/experiment.tex
%!TEX root =  ../thesis.tex

We implement \ref{DLC} and \ref{DLC-LIX} and compare their empirical performances with state-of-the-art MC algorithm \citep{ICML16_MultiplayerBandits_RosenkiShamir} and SIC-MMAB algorithm \citep{NeurIPS19_boursier2019sic}. MC algorithm runs for a fixed number of rounds$(T_0)$ and uses collisions information to find top $N$ arms whereas SIC-MMAB has problem dependent termination rule and allows communication among players as in \ref{DLC} and \ref{DLC-LIX}.  We repeated the experiment $20$ times and plotted regret with $95\%$ confidence interval (the vertical line on each curve shows the confidence interval).

We compare regret for MC, \ref{DLC} and \ref{DLC-LIX} using the same set of parameters used to evaluate performance of MC in \cite{ICML16_MultiplayerBandits_RosenkiShamir} -- $K = 10, N = 6, T = 10000,$ with mean rewards of the arms varying between $0.95$ and $0.5$ with separation of $0.05$ between two consecutive arms. For this problem instance, MC uses $T_0=3000$ rounds for the learning phase, which translates the lower bound on the gap between the means as $\epsilon^\prime = \sqrt{{16K\ln\left(4K^2/\delta\right)}/{3000}}$. We set $\epsilon (\approx \epsilon^\prime/2) = 0.3$ and $\delta=0.05$ in \ref{DLC} and \ref{DLC-LIX} for fair comparison. As expected, our algorithms have lower regret than MC, as shown In \cref{fig:mc}.

We compares regret of SIC-MMAB with \ref{DLC} and \ref{DLC-LIX} using the parameters -- $K = 9, N = 6, T = 300000, \delta=0.05$ as in \cite{NeurIPS19_boursier2019sic} but we used a larger gap between the mean reward of arms for faster convergence. The arms' mean reward varies between $0.95$ and $0.55$, with the same gap of $0.05$ between two consecutive arms.  We set $\epsilon = 0.05$ for \ref{DLC} and \ref{DLC-LIX}. After termination, the regret of \ref{DLC} and \ref{DLC-LIX} perform significantly better than SIC-MMAB, as shown in \cref{fig:sic-mmab}. \ref{DLC} does not allow non-Leader to play arm in every round during the learning phase, leading to more regret in initial rounds. In contrast, \ref{DLC-LIX} allows non-Leader to play an arm during the learning phase. This approach leads to lower regret but more sample needed for learning, as shown in \cref{fig:colSam}.

\begin{figure}[!ht]
	\centering
	\begin{minipage}[b]{00.4\textwidth}
		\captionsetup{justification=centering}
		\includegraphics[width=\linewidth]{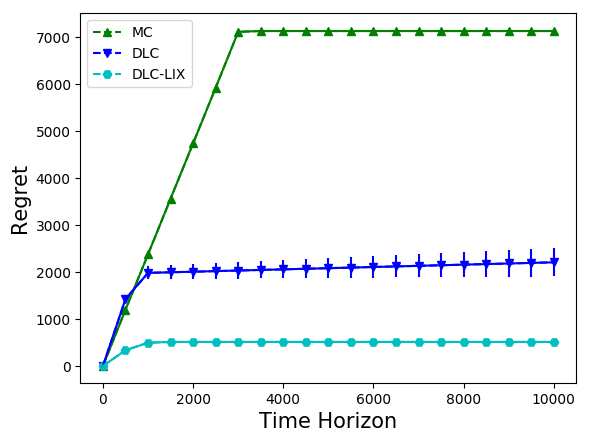}
		\caption{Regret comparison with MC Algorithm}
		\label{fig:mc}
	\end{minipage}\qquad
	\begin{minipage}[b]{00.4\textwidth}
		\captionsetup{justification=centering}
		\includegraphics[width=\linewidth]{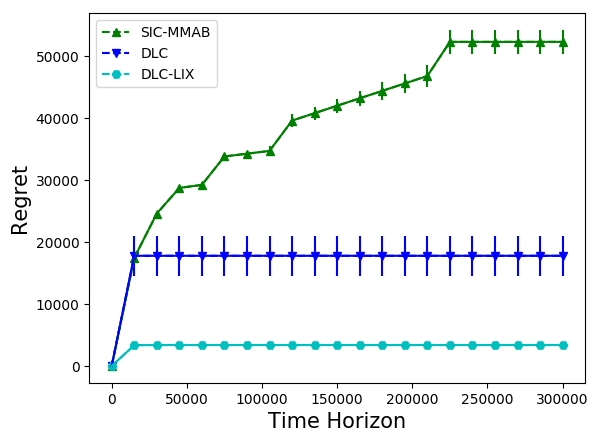}
		\caption{Regret comparison with SIC-MMAB}
		\label{fig:sic-mmab}
	\end{minipage}
\end{figure}

The \cref{fig:arms} and \cref{fig:players} depicts comparison of regret for different number of players and number of arms for a problem instance where $K = 10$ (when varying number of players), $N = 5$ (when varying number of arms), $T = 300000, \delta=0.05, \epsilon=0.05$ and with highest mean reward set to $0.95$ with others decreasing uniformly with gap of $0.05$.
\begin{figure}[!ht]
	\centering
	\begin{minipage}[b]{00.4\textwidth}
		\captionsetup{justification=centering}
		\includegraphics[width=\linewidth]{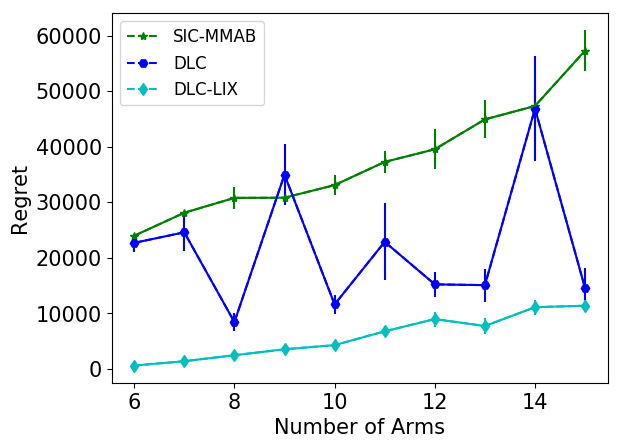}
		\caption{Regret versus different number of arms}
		\label{fig:arms}
	\end{minipage}\qquad
	\begin{minipage}[b]{00.4\textwidth}
		\captionsetup{justification=centering}
		\includegraphics[width=\linewidth]{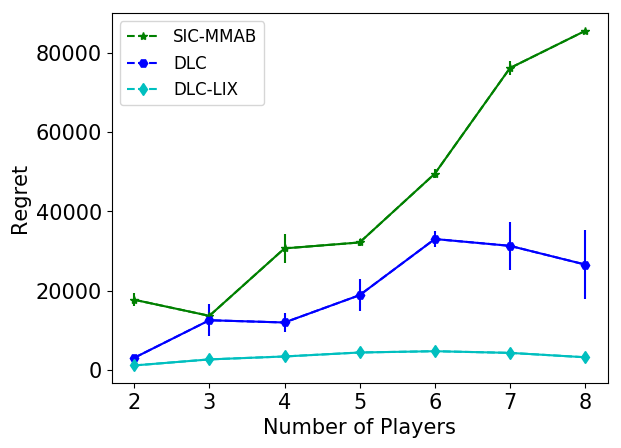}
		\caption{Regret versus different number of players}
		\label{fig:players}
	\end{minipage}
\end{figure}

The regret increases with an increase in the number of arms since each arm has to be sufficiently sampled to get better estimates of mean reward. As allocations are guaranteed to be only $(\epsilon,\delta)-$optimal, sometimes \ref{DLC} does not terminate with top $N$ arms due to the bad estimate of the mean reward in the initial rounds. However, it does not happen in \ref{DLC-LIX} as non-Leaders keep playing their reserved arms during the exploration phase and share reward information with the Leader, which helps the Leader to have good estimates for the arms she did not explore. When \ref{DLC} does not terminate with top $N$ arms, it incurs constant regret for subsequent rounds that is the reason for random behavior of  \ref{DLC}. \cref{fig:arms} shows that with increasing $K$, the improvement in the performance of \ref{DLC-LIX} over SIC-MMAB is significant.

As the number of players increases, our algorithms perform better than SIC-MMAB as information exchange between the Leader and non-Leader happens once. In contrast, it happens multiple times in SIC-MMAB. SIC-MMAB runs in multiple phases, where exploration and exploitation occur in each phase. At the end of exploration in each phase, players communicate the reward information of the sampled arms with other players adding a significant amount of communication overhead. 

This overhead increases with an increase in the number of players and arms, resulting in more regret, more rounds needed for termination, and more collisions in SIC-MMAB, as shown in \cref{fig:players} and \cref{fig:colSam}. The number of rounds needed for the termination of \ref{DLC} and \ref{DLC-LIX} depends upon the gap between the mean reward of the arm and the average of the mean reward of $N^{th}$ best arm and $N+1$ best arm (as shown in \cref{lma:samples_AE}). Therefore, initially, regret increases with an increase in the number of players, but decreases after the number of players become more than $K/2$, as shown in \cref{fig:players}. 

\begin{figure}[!ht]
	\centering
	\begin{minipage}[b]{0.4\textwidth}
		\captionsetup{justification=centering}
		\includegraphics[width=\linewidth]{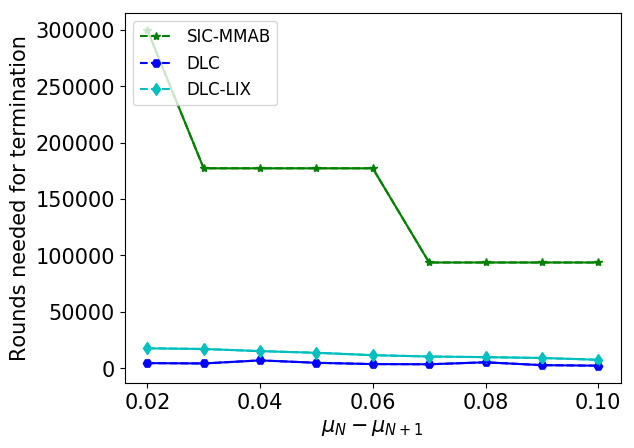}
		%		\caption{Varying gap between arms}
		%		\label{fig:deltaCollision}
	\end{minipage}\qquad
	\begin{minipage}[b]{0.4\textwidth}
		\captionsetup{justification=centering}
		\includegraphics[width=\linewidth]{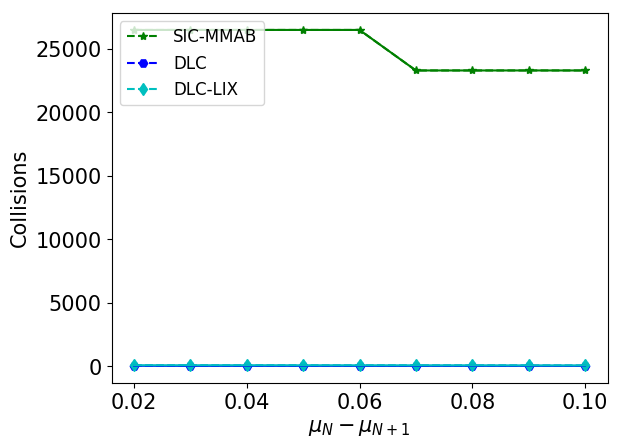}
		%		\caption{Varying number of players}
		%		\label{fig:deltaSamples}
	\end{minipage}
	\caption{Collisions and number of rounds needed for termination versus gap between consecutive arms.}
	\label{fig:colSam}
\end{figure}

\ref{DLC} and \ref{DLC-LIX} also have fewer collisions and need a lesser number of rounds for finding and allocating the top $N$ arms as compare to SIC-MMAB, which is shown in \cref{fig:colSam}. The collisions in MC are much higher than \ref{DLC} and \ref{DLC-LIX}.  Hence, the comparison of collisions in MC with \ref{DLC} and \ref{DLC-LIX} is not presented.

%% file: chapter5/appendices.tex
%!TEX root =  ../thesis.tex

\noindent
{\bf Proof of Lemma \ref{lma:OrthoArms}:} The proof of the first part of this lemma is adapted from \cite[Lemma 1]{WIOPT2013_TrekkingBased_KumarYadavDarak}. Let $p_c$ denote the collision probability of a players when all the players are randomly selecting an arm to play from $[K]$ in each round $t$. Probability that each a player  will observe a collision-free play on an arm after  $T_{01}$ rounds is given by:

\[\sum_{t=1}^{T_{01}}p_c^{t-1}(1-p_c).\]
Setting this value to be at least larger than $1- \frac{\delta}{N}$ for each player we get
\begin{align}
	\sum_{t=1}^{T_{01}}& p_c^{t-1}(1-p_c) \geq 1- \frac{\delta}{N} \nonumber\\
	&\iff 1-{p_c}^{T_{01}} \geq 1- \frac{\delta}{N} \nonumber \\
	&\iff T_{01} \log{p_c} \leq \log\bigg({\frac{\delta}{N}}\bigg) \nonumber \\
	&\iff T_{01} \geq \frac{\log\big({\frac{\delta}{N}}\big)}{\log{p_c}} \label{eqn:TRHBound}.
\end{align}
We next give a uniform upper bound on $p_c$. Note that in any round some players may be selecting arms sequentially (call them SH players) while others uniformly at random (call them RH players). Fix a round $t$ and let $N_r \geq 1$ denote the number of players selecting arms uniformly at random. Let $p_{cr}$ denote the probability of observing collision by a RH player. Then
\begin{eqnarray*}
	\lefteqn{1-p_c= \Pr\{\mbox{no collision from RH players}\}} \\
	&+& \Pr\{\mbox{no collision from SH players}\} \\
	&\geq & \sum_{j=1}^{N_r}\frac{(1-p_{cr})}{K}\geq \frac{(1-p_{cr})}{K} \\
	&=& \frac{(1-1/K)^{N_r-1}}{K}\\
	&\geq& \frac{(1-1/K)^{N-1}}{K}\geq 1/4K (\mbox{ for all } K>1).
\end{eqnarray*}  

Substituting the bound on $p_c$ in (\ref{eqn:TRHBound}) and using union bound we see that within $T_{rh}(\delta)$ rounds all the players will orthogonalize with probability at least $1-\delta$. 

As at least two players need to play the same arm for any collision to happen, the maximum number of collisions are bounded by $N$ for any round. Therefore, total collisions in {\em Orthogonalization} phase are trivially bounded by $NT_{RH}$.

~\\
{\bf Proof of Lemma \ref{lma:IndexPlayer}:} As the player with channel index one starts indexing other players, it needs to check the $K-1$ channel for other players and $N-1$ collisions happen during this process.  The player with arm's index $j$ starts sequential hopping after $2j$ rounds and needs to check only $K-j$ arms. She observes $N-j$ collisions in the process. The player with arm index $K-1$ starts sequential hopping after $2(K-1)$ round and need to check only the arm with index $K$. Therefore, the total number of rounds $(T_{IP})$ required to complete the {\em Player Indexing} phase is $T_{IP} = 2K - 1$. The total number of collisions $(C_{IP})$ in the {\em Player Indexing} phase is bounded by $C_{IP} = \sum_{j=1}^{N} (N-j) = N(N-1)/2.$

~\\
{\bf Proof of Lemma \ref{lma:signaling}:}
	Let $T_P = (N-1)(\ceil{\log_2K} + 1)$ be the fixed number of rounds after which each player will play his reserved arm. In the worst case, if AdaptiveExplore sub-routine terminates just after any player played his reserved arm, then Leader has to play $T_P$ more rounds in which he transfer arm information to other players and then takes $\ceil{\log_2K} + 1$ rounds to transfer arm information to that player. Therefore, $T_{CP} \le N(\ceil{\log_2K} + 1)$.
	
	Since Leader has to play reserved arm in the specific round before transferring the arm information to any player, this will lead to $N-1$ collisions. The arm information is transferred in its binary form which requires only $\ceil{\log_2K}$ rounds. Therefore, the maximum number of collisions needed  for transferring the arm information to other players is $(N-1)\ceil{\log_2K}$. Therefore, the total number of collisions ($C_{CP}$) can occur in the {\em Communication} phase is bounded by $C_{CP} \le (N-1)(\ceil{\log_2K}+1).$

%% file: chapter6/conclusion.tex
%!TEX root =  ../thesis.tex

In this chapter, we summarize the various contributions of this thesis, followed by some interesting directions for future research.

\section{Summary}
In this thesis, we consider sequential decision problems with weak feedback, where the loss incurred by selecting an action may not be inferred from observed feedback. A major part of the thesis (\cref{cha:uss} and \cref{cha:contextual_uss}) focuses on the unsupervised sequential selection problem, where the loss incurred for selecting an action can not be inferred from observed feedback. In \cref{cha:csb}, we introduced the Censored Semi-Bandits problem, where the loss incurred for selecting an action can be observed under certain conditions. Finally, we study the channel selection in the communication networks via distributed learning in Multi-Player Multi-Armed Bandits in \cref{cha:mp-mab}, where the reward for an action is only observed in the round when no other player selects that action to play in that round. 
The major contributions of this thesis are summarized as follows:

\begin{enumerate}
	\item {\em Unsupervised Sequential Selection.~~} We study the unsupervised sequential selection (USS) problem in \cref{cha:uss}, where both mean loss and cost of using arms are important. It is a variant of the stochastic partial monitoring problem, where the losses are not observed. However, one can still compare the feedback of two arms to see if they agree or disagree. We estimate the disagreement probability between each pair of the arms and develop the upper confidence bound and Thompson Sampling based algorithms (\ref{alg:USS_WD} and \ref{alg:TS_USS}) that give sub-linear regret under the `Weak Dominance $(\WD)$ property.' We show that both algorithms enjoy regret of order $\tilde{O}(T^{2/3})$ (hiding logarithmic terms) for the USS problem that satisfies $(\WD)$ property.  We also show that when the problem instance satisfies the more stringent Strong Dominance property, the regret bound improves to $\tilde{O}(T^{1/2})$. We demonstrate our algorithms' performance on the problem instances derived from synthetic and two real datasets and empirically show that any USS problem instance satisfying $\WD$ property has sub-linear regret.

	\item {\em Contextual Unsupervised Sequential Selection.~~} We extend the study of the unsupervised sequential selection setup to the case in \cref{cha:contextual_uss}, where the contextual information about the task is available. The problem is a partial monitoring stochastic contextual bandit, where the loss of an arm can not be inferred from the observed feedback. Since the disagreement between arms can be observed, we model the disagreement probability between each pair of the arms as linearly parameterized and develop an algorithm named \ref{alg:CUSS_WD} that achieves $O(\log T)$ regret with high probability. The analysis of \ref{alg:CUSS_WD} is based on recent results on the sample normality confidence bound for the maximum likelihood estimator of generalized linear models.

	\item {\em Censored Semi-Bandits.~~} We introduce a novel framework for resource allocation problems using a variant of semi-bandits and name it censored semi-bandits (CSB) in \cref{cha:csb}. In the CSB setup, the loss observed from an arm depends on the amount of resource allocated, and hence, it can be censored. We propose a threshold-based model where loss from an arm is observed when the allocated resource is below a threshold. The goal is to assign a given resource to arms such that total expected loss is minimized. We consider two variants of the problem, depending on whether or not the thresholds are the same across the arms. For the variant where thresholds are the same across the arms, we establish that it is equivalent to the Multiple-Play Multi-Armed Bandit problem. For the second variant where the threshold can depend on the arm, we establish that it is equivalent to a more general Combinatorial Semi-Bandit problem. Exploiting these equivalences, we develop algorithms that enjoy optimal performance guarantees.  We also showcase the application of CSB setup to stochastic network utilization maximization by extending the CSB setup to reward maximization setting.

	\item {\em Distributed Learning in Multi-Player Bandits.~~} Motivated by channel selection problem in communication networks, we consider a multi-player bandits problem in \cref{cha:mp-mab}, where multiple players access the same set of arms in the absence of any central coordination. The players aim to maximize total reward in a distributed fashion. The mean rewards of arms and the number of players present in the network are unknown. We set up the problem as a stochastic multi-player multi-armed bandits and develop completely distributed algorithms that achieve constant regret with a high probability. We also empirically show that our algorithms have constant regret.
\end{enumerate}

\section{Future Research Directions}
There are various open questions arising out of this thesis. A few of them are as follows:
\begin{enumerate}
	\item {\em Unsupervised Sequential Selection.~~} We ignore the inherent side observations due to the arms' cascade structure. By using these side observations, we can have better regret bounds. Another interesting future direction is to develop algorithms that relax the cascade structure assumption and selects the best subset of arms. The problem-dependent lower bound of regret for the USS problem is still an open question.

	\item {\em Contextual Unsupervised Sequential Selection.~~} We exploit the contextual information for a better selection of arm but ignore the inherent side observations due to the arms' cascade structure. By using these side observations, we can tighten the regret bounds.  In the contextual USS setup, the optimal action for contexts that do not satisfy $\WD$ cannot be learned. However, these contexts are still useful to improve estimates of the parameters and it will be interesting to characterize how much they will affect the regret. It will also be interesting to develop algorithms that decide whether it needs to go further down in the cascade when more information about context is revealed along the cascade. Many open questions like regret lower bounds for the contextual USS problem, contexts with private information, and the problems with real-valued feedback can be explored in the future.

	\item {\em Censored Semi-Bandits.~~} The CSB setting considered so far does not use any similarity metric between the arms. For example, in the case of the police patrol allocation, the nearby nodes may have similar parameters, and we may be able to make use of the spatial coherence. It will be interesting to integrate existing work on contextual bandits with the graph structure and the CSB setup. As we only consider one type of resource, it will be interesting to consider multi-type resources into the CSB setup. Another extension of the CSB setup is to relax the assumption that the lower bound on leftover resources after the optimal allocation is known. One can also extend CSB setup from threshold type loss to continuous type loss functions and from the stochastic environment to the adversarial environment. Settings such as these are left for future work.

	\item {\em Distributed Learning in Multi-Player Bandits.~~} We consider that the quality of the arms is the same across all the players (symmetric), and the number of users was fixed throughout (static). It will be interesting to study the setting, where the quality of arms could potentially differ across players (asymmetric) and allow the number of users to vary with time (dynamic networks). There are many interesting extensions of the proposed setup, like anytime algorithm for regret minimization, problems without collision sensing, and learning with asynchronous players that can be explored in the future.
	
\end{enumerate}